\documentclass[10pt]{amsart}

\calclayout

\usepackage[utf8]{inputenc}
\usepackage{amsthm,amsmath,amsfonts,amssymb}

\usepackage[all]{nowidow}

\usepackage[numbers]{natbib}

\usepackage[colorlinks,citecolor=blue,urlcolor=blue,allcolors=blue]{hyperref}
\usepackage{graphicx}
\usepackage[british]{babel}
\usepackage{mathtools}
\mathtoolsset{showonlyrefs} 
\usepackage{booktabs,longtable}
\usepackage{multirow}
\usepackage{mathrsfs}
\usepackage[super]{nth}
\usepackage[plain,noend,ruled]{algorithm2e}
\usepackage[noend]{algpseudocode}
\usepackage{todonotes}
\usepackage{thmtools}
\usepackage{thm-restate}
\usepackage{caption}
\theoremstyle{plain}
\newtheorem{proposition}{Proposition}
\newtheorem{theorem}{Theorem}
\newtheorem{lemma}{Lemma}
\newtheorem{corollary}{Corollary}

\newtheorem{appendixlemma}{Lemma}[section]
\newtheorem{appendixthm}{Theorem}[section]

\theoremstyle{remark}
\newtheorem{definition}{Definition}
\newtheorem{appendixdef}{Definition}[section]
\newtheorem{notation}{Notation}[section]
\newtheorem{assumption}{Assumption}
\newtheorem{example}{Example}[section]
\newtheorem{examples}{Examples}[section]

\newtheorem*{motivation*}{Motivation}
\newtheorem{remark}{Remark}[section]

\usepackage{bm}
\usepackage{bbm}
\usepackage{enumerate}
\usepackage{mathtools}
\usepackage{soul}
\usepackage{todonotes}
\newtheorem*{theorem*}{Theorem}
\newtheorem*{lemma*}{Lemma}
\newtheorem*{proposition*}{Proposition}

\theoremstyle{plain} 
\newcommand{\thistheoremname}{}
\newtheorem{genericthm}[definition]{\thistheoremname}

\usepackage{mathrsfs} 

\DeclareSymbolFont{extraup}{U}{zavm}{m}{n}
\DeclareMathSymbol{\vardiamond}{\mathalpha}{extraup}{87}

\usepackage{accents}

\makeatletter
\providecommand*{\sha}{%
  \mathbin{\mathpalette\sha@{}}%
}
\newcommand*{\sha@}[2]{%
  \sbox0{$#1\vcenter{}$}%
  \kern .15\ht0 
  \rlap{\vrule height .25\ht0 depth 0pt width 2.5\ht0}%
  \raise.1\ht0\hbox to 2.5\ht0{%
    \vrule height 1.75\ht0 depth -.1\ht0 width .17\ht0 %
    \hfill
    \vrule height 1.75\ht0 depth -.1\ht0 width .17\ht0 %
    \hfill
    \vrule height 1.75\ht0 depth -.1\ht0 width .17\ht0 %
  }%
  \kern .15\ht0 
}
\makeatother

\usepackage{stmaryrd}

\usepackage[makeroom]{cancel} 

\makeatletter
\DeclareRobustCommand{\cev}[1]{%
  {\mathpalette\do@cev{#1}}%
}
\newcommand{\do@cev}[2]{%
  \vbox{\offinterlineskip
    \sbox\z@{$\m@th#1 x$}%
    \ialign{##\cr
      \hidewidth\reflectbox{$\m@th#1\vec{}\mkern4mu$}\hidewidth\cr
      \noalign{\kern-\ht\z@}
      $\m@th#1#2$\cr
    }%
  }%
}
\makeatother

\makeatother

\makeatletter
\newcommand{\oset}[3][-1.75ex]{%
  \mathrel{\mathop{#3}\limits^{
    \vbox to#1{\kern-2\ex@
    \hbox{$\scriptstyle#2$}\vss}}}}
\makeatother
\makeatletter
\newcommand{\osettext}[3][-1.15ex]{%
  \mathrel{\mathop{#3}\limits^{
    \vbox to#1{\kern-2\ex@
    \hbox{$\scriptstyle#2$}\vss}}}}
\makeatother

\newcommand{\vertiii}[1]{{\left\vert\kern-0.25ex\left\vert\kern-0.25ex\left\vert #1 
    \right\vert\kern-0.25ex\right\vert\kern-0.25ex\right\vert}}
\newcommand{\bvertiii}[1]{{\big\vert\kern-0.25ex\big\vert\kern-0.25ex\big\vert #1 
    \big\vert\kern-0.25ex\big\vert\kern-0.25ex\big\vert}}

\usepackage{tikz}
\usepackage[super]{nth}

\newcommand{\coloneqq}{\mathrel{\mathop:}=}
\newcommand{\eqqcolon}{\mathrel{=\hspace{-2.75pt}\mathop:}}

\newcommand{\supp}{\mathrm{supp}}

\newcommand\restr[2]{{
  \left.\kern-\nulldelimiterspace 
  #1 
  \vphantom{\big|} 
  \right|_{#2} 
  }}

\usepackage{enumitem}
\usepackage{accents}
\usepackage{float}
\usepackage{subfig}
\usepackage{xfrac}

\newcommand{\R}{\mathbb{R}}
\newcommand{\I}{\mathbb{I}}
\newcommand{\N}{\mathbb{N}}
\newcommand{\Z}{\mathbb{Z}}

\newcommand{\p}{\mathfrak{p}}

\newcommand{\DP}{\mathrm{DP}_{\!d}}

\newcommand{\W}{\mathcal{W}}
\newcommand{\E}{\mathbb{E}}
\newcommand{\sig}{\mathfrak{sig}}
\newcommand{\tepsilon}{{\tilde{\varepsilon}}}

\makeatletter
\providecommand*{\shuffle}{%
  \mathbin{\mathpalette\shuffle@{}}%
}
\newcommand*{\shuffle@}[2]{%
  \sbox0{$#1\vcenter{}$}%
  \kern .15\ht0 
  \rlap{\vrule height .25\ht0 depth 0pt width 2.5\ht0}%
  \raise.1\ht0\hbox to 2.5\ht0{%
    \vrule height 1.75\ht0 depth -.1\ht0 width .17\ht0 %
    \hfill
    \vrule height 1.75\ht0 depth -.1\ht0 width .17\ht0 %
    \hfill
    \vrule height 1.75\ht0 depth -.1\ht0 width .17\ht0 %
  }%
  \kern .15\ht0 
}
\makeatother

\newcommand{\bdiam}{\tikz[baseline={([yshift=-.9ex]current bounding box.center)}]{\node[fill=black,rotate=45,inner sep=.1ex, text height=0.75ex, text width=0.75ex] {};%
\node[ font=\color{white}] (wi) {};}}

\makeatletter
\newcommand{\leqnomode}{\tagsleft@true\let\veqno\@@leqno}
\newcommand{\reqnomode}{\tagsleft@false\let\veqno\@@eqno}
\makeatother

\usepackage{thmtools, thm-restate}
\usepackage[multiple]{footmisc}

\usepackage{amsaddr}

\begin{document}

\title[Nonlinear ICA for Time-Dependent Signals]{Nonlinear Independent Component Analysis\\For Discrete-Time and Continuous-Time Signals}


\author{\vspace{-2em}\footnotesize Alexander Schell and Harald Oberhauser}
\address[A1]{Mathematical Institute\\University of Oxford} 

\begin{abstract}
We study the classical problem of recovering a multidimensional source signal from observations of nonlinear mixtures of this signal. We show that this recovery is possible (up to a permutation and monotone scaling of the source's original component signals) if the mixture is due to a sufficiently differentiable and invertible but otherwise arbitrarily nonlinear function and the component signals of the source are statistically independent with `non-degenerate' second-order statistics. The latter assumption requires the source signal to meet one of three regularity conditions which essentially ensure that the source is sufficiently far away from the non-recoverable extremes of being deterministic or constant in time. These assumptions, which cover many popular time series models and stochastic processes, allow us to reformulate the initial problem of nonlinear blind source separation as a simple-to-state problem of optimisation-based function approximation. We propose to solve this approximation problem by minimizing a novel type of objective function that efficiently quantifies the mutual statistical dependence between multiple stochastic processes via cumulant-like statistics. This yields a scalable and direct new method for nonlinear Independent Component Analysis with widely applicable theoretical guarantees and for which our experiments indicate good performance. 
\end{abstract}

\thanks{\href{https://mathscinet.ams.org/mathscinet/msc/msc2020.html}{\color{black}{\textit{MSC2020 subject classification:}}} Primary 62H25, 62M99; secondary 62H05, 60L10, 62M45, 62R10.\newline 
\indent\indent\textit{Keywords and Phrases:} Blind Source Separation, Independent Component Analysis, inverse problem, \indent statistical independence, latent variable model, functional data analysis, nonlinear BSS, nonlinear ICA}
\email[A1]{\href{mailto:alexander.schell@maths.ox.ac.uk}{\texttt{alexander.schell@maths.ox.ac.uk}} \normalfont{and} \href{mailto:harald.oberhauser@maths.ox.ac.uk}{\texttt{harald.oberhauser@maths.ox.ac.uk}}} 

\vspace*{-2em}
\maketitle

\vspace{-1em}
\section{Introduction}
\noindent 
A common problem in science and engineering is that an observed quantity, $X$, is determined by an unobserved source, $S$, which one is interested in.
Denoting by $f$ the deterministic relationship between $X$ and $S$, one thus arrives at the equation    
\begin{equation}\label{ICA-Relation}
X = f(S)
\end{equation}
where $X$ is known but both the relation $f$ and the source $S$ are unknown.\\[-0.5em]

\noindent
The premise that the data $X$ is determined by its source $S$ reflects in the assumption that $f$ is a deterministic function, while the premise that $S$ can be completely inferred from $X$ --- i.e.\ that no information be lost in the process of going from $S$ to $X$ --- is reflected in the assumption that the function $f$ is one-to-one; for simplicity, it is typically also assumed that $f$ is onto. Any function $f$ of this kind will be referred to as a mixing transformation.\\[-0.5em] 

\noindent
The central challenge, known as the problem of \emph{Blind Source Separation (BSS)}, then becomes to infer --- or `identify' --- the hidden source $S$ from the given data $X$:
\begin{equation}\label{intext:ProblemOfBSS}
\begin{gathered}
\text{\emph{Under which assumptions is it possible to recover the source data $S$ in \eqref{ICA-Relation} if only}}\\[-0.35em] \text{\emph{its mixture $X$ is observed? To what extent can such a recovery be achieved}}\\[-0.35em] \text{\emph{and how can it be performed in practice?}}
\end{gathered}
\end{equation}
It is clear that without additional assumptions, the above problem of inference \eqref{intext:ProblemOfBSS} is severely underdetermined: If $X$ and equation \eqref{ICA-Relation} is the only information available but both $f$ and $S$ are unknown, then we may generally find infinitely many possible `explanations' $(\tilde{S},\tilde{f})$ for $X$ which all satisfy \eqref{ICA-Relation} but are not otherwise meaningfully related to the true explanation $(S,f)$ underlying the data. In many cases, however, this `indeterminacy of $S$ given $X$ with $f$ unknown' can be controlled by imposing certain statistical conditions on the source $S$.\\[-0.5em] 

The following simple example illustrates this situation. 

\begin{example}\label{example:intro}
Suppose that you are on a video-call and want to follow the simultaneous speeches of two speakers $S^1$ and $S^2$, modelled as real-valued time series each. As the propagation of sound adheres to the superposition principle, the acoustic signals $X^1$ and $X^2$ that reach your left and right ear, respectively, may be modelled as linear mixtures $X^i = a_{i1}S^1 + a_{i2}S^2$ of the individual speech signals $S^1$ and $S^2$. Denoting $X\equiv(X^1, X^2)^\intercal$ and $S\equiv(S^1, S^2)^\intercal$ and $A\equiv(a_{ij})\in\mathbb{R}^{2\times 2}$, the relation between the audio data $X$ and its underlying sources $S$ can hence be expressed by the model equation $X=A\cdot S$, which for $A$ invertible is a special case of \eqref{ICA-Relation} for the linear map $f\coloneqq A$. The above problem \eqref{intext:ProblemOfBSS} then becomes to recover the constituent speeches $S^1$ and $S^2$ from their observed mixtures $X^1, X^2$ alone, given that the relationship between $X$ and $S$ is linear. Now without further assumptions, the true explanation $(S,A)$ of the data $X$ cannot be distinguished from any of its `alternative explanations' $\{(\tilde{S}, \tilde{A})\equiv(B\cdot S, AB^{-1})\mid B\in\mathbb{R}^{2\times 2} \text{ invertible}\}$. But if the speech signals $S^1$ and $S^2$ were assumed to be uncorrelated, say, then the above family of best-approximations of $(S,A)$ reduced to $\{(\tilde{S}, \tilde{A})\equiv(B\Lambda\cdot S, A\Lambda^{-1}B^\intercal)\mid \Lambda\in\R^{2\times 2} \text{ (invertible) diagonal}, \ B\in\mathbb{R}^{2\times 2} \text{ orthogonal}\}$;\footnote{\ Indeed: The assumption of uncorrelatedness complements the original model equation \eqref{ICA-Relation} by the additional (statistical) source condition $\mathrm{Cov}(\tilde{S},\tilde{S}) = \mathrm{Cov}(S,S) = \mathrm{I}_2$, which implies that $B^\intercal B = \mathrm{Cov}(\tilde{S},\tilde{S}) = \mathrm{I}_2$ (where the components of $\tilde{S}$ are assumed to be scaled to unit variance).} hence if they are uncorrelated,  $S^1$ and $S^2$ may be recovered from $X$ uniquely up to scale and a rotation.{\hfill \bdiam}
\end{example}

\noindent
This simple observation can be significantly improved by way of the classical Darmois-Skitovich theorem \citep{reiersol1950,DAR,SKI} which implies that for $f$ linear, the original source $S$ may be identified from $X$ even up to scaling and a permutation of its components if $S$ is modelled as a random vector whose coordinates $S^i$ are not only uncorrelated but statistically independent. This mathematical insight, elaborated in P.\ Comon's seminal framework \cite{COM}, quickly became the theoretical foundation of \emph{Independent Component Analysis (ICA)}, a popular statistical method that has since seen far-reaching theoretical investigations and extensions, e.g.\ \citep{BJO,SAM}, and has been successfully implemented in numerous widely-applied algorithms, e.g.\ \citep{BCM,CD2,HAT,HYF}; see for instance \cite{ERK,HRS,MNT} as well as the monographs \citep{HBS,HKO} for an overview.\\[-0.5em] 

\noindent
Comon's contribution is arguably the most conceptionally influential answer to the above inference task \eqref{intext:ProblemOfBSS} to date that was both practically relevant and mathematically rigorous. 
However, Comon's approach applies to linear relationships \eqref{ICA-Relation} between $X$ and $S$ only, because among nonlinear mixing functions on $\R^d$ there are many `non-trivial' transformations that preserve the mutual statistical independence of their input vectors \cite{HYP}. This is a substantial limitation not only from a theoretical perspective but also in applications, where real-world data is often assumed to depend nonlinearly on certain nonredundant (independent) explanatory source signals and the instantaneous invertible nonlinear model \eqref{ICA-Relation} is deemed an adequate description of this dependence. See for instance \citep{ardizzone2018applications, cranmer2020applications, he2018applications, ican2017applications, khoshnevis2019applications, ding2019miningmachine,noe2020applications} and the references therein for a few according example applications of nonlinear BSS ranging from the analysis of star clusters in interstellar gas clouds and biomedical tissue monitoring during surgery over electroencephalography and molecular simulation to statistical process monitoring, vibration analysis and stock market prediction.\\[-0.5em]    

\noindent
Overcoming the traditional confinement to linearity has thus been a long-standing scientific endeavour, and the past twenty-six years have seen various attempts of establishing alternative identifiability approaches to recover multivariate data from their nonlinear transformations. Prominent ideas in this direction include the optimisation of mutual information over outputs of (adversarial) neural networks, e.g. \citep{ALM, BRB, HJE, khemakhem20iVAE, TWZ}, or the idea of `linearising' the generative relation \eqref{ICA-Relation} by mapping the observable $X$ into a high-dimensional feature space where it is then subjected to a linear ICA-algorithm \cite{HAR}.\\[-0.5em]

\noindent
More recently, the works of Hyvärinen et al.\ \citep{TCL,HYM,AUX} achieved significant progress regarding the recovery of nonlinearly mixed sources with temporal structure (e.g.\ time series, instead of random vectors in $\R^d$) by first augmenting the observed mixture of these sources with an auxiliary variable such as time \citep{TCL} or its history \citep{HYM}, and then training logistic regression to discriminate (`contrast') between the thus-augmented observable and some additional `variation' of the data. This variation is obtained by augmenting the observable with a randomized auxiliary variable of the same type as before, thus linking the asymptotical recovery of the source $S = f^{-1}(X)$ to a trainable optimisation problem, namely the convergence of a universal function approximator (e.g.\ a neural network) learning a classification task. These identifiability results were extended and embedded into the context of variational autoencoders in \cite{khemakhem20iVAE}.\\[-0.5em]

\noindent
Motivated by the classical ICA framework of Comon \cite{COM} and the recent contrastive learning breakthrough \citep{HYM}, we revisit the inference problem \eqref{intext:ProblemOfBSS} for stochastic processes\footnote{\ Throughout, ``stochastic process'' means ``continuous-time stochastic process'' unless mentioned otherwise.} $X=(X_t)$ and $S=(S_t)$ with recent tools from stochastic analysis.
We believe the following to be our main contributions to the existing literature:\\[-1em] 

\begin{description}[leftmargin=1.5em]
\item[\textbf{Identifiability for Stochastic Processes}] We provide general identifiability results that generalise Comon's classical independence-based identifiability criterion from linear mixtures of random vectors to nonlinear mixtures of discrete- and continuous-time stochastic processes; cf.~Theorems~\ref{thm:Comon},\,\ref{thm:NICA_stat},\,\ref{cor:NICA_MainCor}. On a theoretical level, working with infinite-dimensional (i.e.~path-valued) random variables poses new challenges that we address by using rough path theory. From an applied perspective, many models are naturally formulated in continuous time rather than in discrete time (e.g.~in biology, physics, medicine or finance), which our approach accounts for by naturally covering both discrete-time and continuous-time models alike, including Stochastic Differential Equations (SDEs) in particular. 
\item[\textbf{Blind Source Separation via Signature Cumulants}] Our identifiability theory allows us to reformulate the problem of nonlinear blind source separation as an easy-to-state optimisation problem which involves the minimisation of statistical dependence between multiple stochastic processes, see Theorem~\ref{thm:optimisation}. Unlike for vector-valued data, statistical dependence between stochastic processes can manifest itself inter-temporally, in the sense that different coordinates of the processes may exhibit statistical dependencies both instantaneously and over different points in time. We propose to quantify such complex dependency relations by using so-called signature cumulants \citep{bonnier2019signature} as objective functions. These signature cumulants can be seen as generalising the concept of cumulants from vector-valued data to path-valued data. Analogous to classical cumulants, signature cumulants then provide a graded, parsimonious, and efficiently computable quantification of the degree of statistical (in)dependence between stochastic processes. Joined with our optimisation approach, this combines to a widely applicable new and robust statistical method for the nonlinear blind source separation of time-dependent signals, see Theorem~\ref{thm:optimisation} and Section \ref{sect:algorithm}.   

\item[\textbf{Consistency With Respect to Time Discretization and Sample Size}]When applying our methodology in practice, the following issues arise: Firstly, although the underlying stochastic model is often formulated in continuous time, in practice one usually has access to time-discretized samples only, often taken over non-equally spaced time grids. Secondly, oftentimes only a single (time-discretized) sample path of the process is available rather than many independent realisations, for example in the classical cocktail party problem.  
We address both of these issues and show that our method is statistically consistent even if only a single, time-discretized and finite sample of the observable is given, see Theorem~\ref{thm:consistency}.
This is also the setting in which our experiments are carried out in Section~\ref{sect:experiments}. 
\end{description} 

\begin{figure}[htpb] 
\vspace{-0.5em}
\centering
\hspace*{-1.25em}
\includegraphics[trim={0em 0 0em 0},clip,scale=0.35]{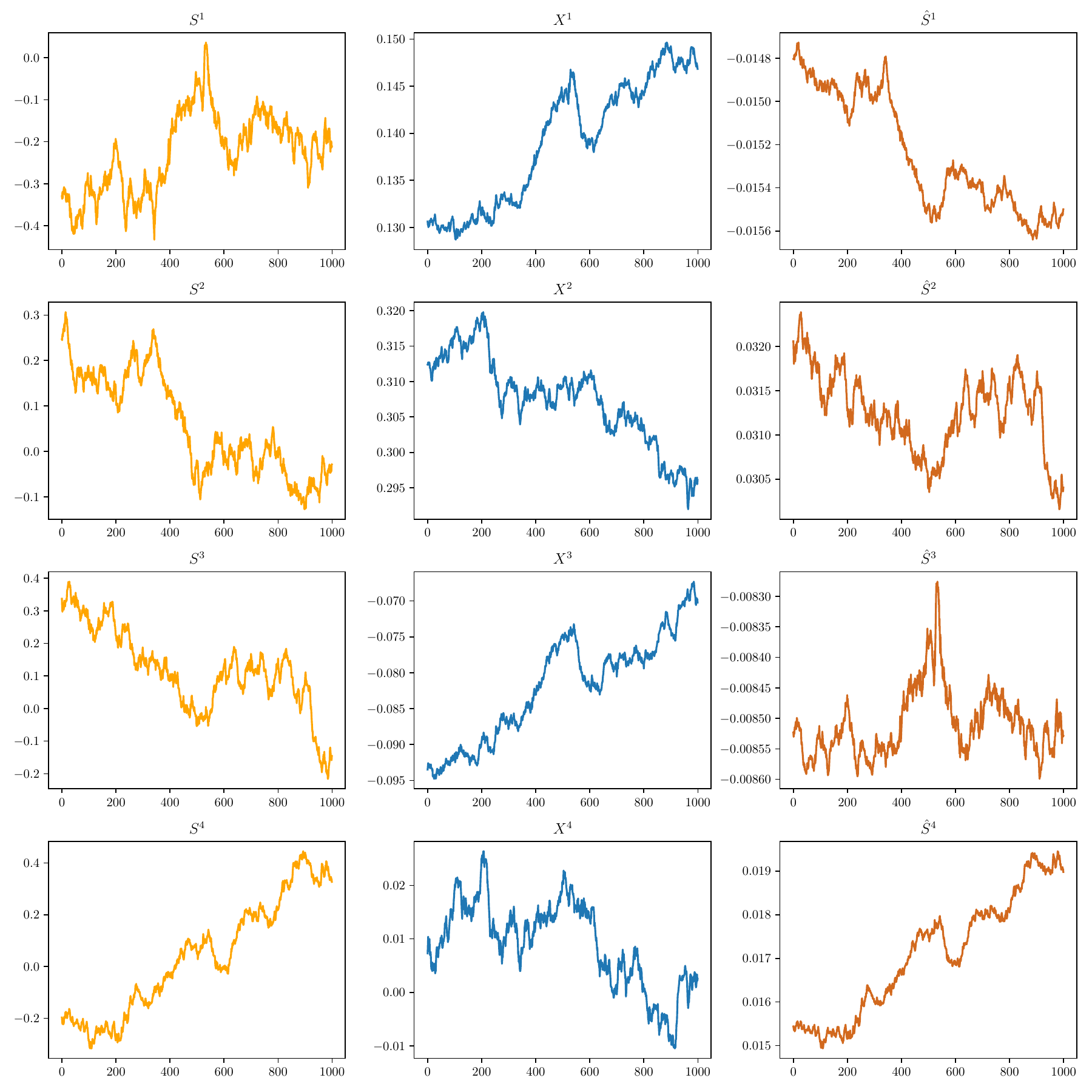}
\caption{\emph{A source $S$ with four components $S^1, \ldots, S^4$ (orange) is mapped under some nonlinear transformation $f$, resulting in the observed mixture $X$ (blue). We present a new method to recover the original source $S$ from its mixture $X$ up to a minimal deviation: Given $X$, this approach returns an estimate $\hat{S}=(\hat{S}^1,\cdots,\hat{S}^4)$ (brown) that approximates $S$ up to the original order of its channels and a monotone scaling. In this example, $\hat{S}^1$ estimates the original component $S^2$, and $\hat{S}^2$ estimates $S^3$, and $\hat{S}^3$ estimates $S^1$, and $\hat{S}^4$ estimates $S^4$.\newline \phantom{.}\hfill (See Example \ref{example:summary} for details.)\vspace{-1em}}}\label{fig:example} 
\end{figure}

\noindent
This article is structured as follows. We precede our statistical analysis with an informal yet concise summary of this paper's main contributions (Section \ref{sect:overview_of_contribution}). The formal exposition of our approach towards the recovery of nonlinearly mixed independent sources begins thereafter by recalling the main results of \cite{COM} as conceptional points of reference (Section \ref{sect:comon}). The core of our identifiability theory is developed in the subsequent two sections: advocating for the incorporation of time as an integral dimension of our source model (Section \ref{sect:intro_BSS}), we show how sources admitting a non-degenerate `temporal structure' harbour sufficient mathematical richness to encode any nonlinear action performed upon them as a sort of `intrinsic statistical fingerprint', based on which the constituent relation \eqref{ICA-Relation} may then be inverted up to a minimal deviation by maximizing an independence criterion (Section \ref{sect:core_theory}). Our approach covers sources of various types of statistical regularity, including popular time series models, various Gaussian processes and Geometric Brownian Motion (Section \ref{sect:example_sources}). The practical applicability of our ICA-method is enabled by a novel independence criterion for time-dependent data (Section \ref{sect:independence_criterion}) that leads to a practical and statistically consistent separation algorithm (Section \ref{sect:consistency}) that we demonstrate in a series of numerical experiments (Section \ref{sect:experiments}). The paper ends with a brief conclusion and an outlook on future directions (Section \ref{sect:conclusion}). Most proofs are given in the appendix along with some technical auxiliaries and further remarks, including an explication of how, as promised in the title, all results and methods in this paper are directly applicable to the separation of discrete-time signals as well (Section \ref{appendix:NICA_discrete}). 

\section{Summary of Contribution}\label{sect:overview_of_contribution}
\noindent
Motivated by recent breakthroughs of Hyvärinen and Morioka \citep{TCL,HYM}, we propose a new approach to the problem of nonlinear blind source separation \eqref{intext:ProblemOfBSS} for multidimensional time-dependent signals that leverages modern tools from stochastic analysis: For an unknown discrete- or continuous-time signal $S=(S_t)$ in $\R^d$ and an unknown function $f : \R^d\rightarrow \R^d$, a new statistical method to recover $S$ from its transformation $X = f(S) \equiv (f(S_t))$ via `signature cumulants' is presented.\\[-0.5em]  

\noindent
In essence, we provide a new algorithm\footnote{\ That is, an explicitly computable map -- or estimator, in the statistical sense -- that takes in [a realisation of] the mixture $X$ and returns an `optimal' approximation of [the corresponding realisation of] $S$ as an output.} that performs the inversion, or `retransformation',
\begin{equation}\label{summary_of_cont:eq1}
X \, \longmapsto \, S
\end{equation}     
of the generative relation \eqref{ICA-Relation} in the case that $f$ and $S$ are not explicitly known and $f$ is sufficiently differentiable and (by necessity) invertible\footnote{\ Invertible at least on the smallest subset of $\R^d$ which is actually reached by $S$, but see Def.\ \ref{def:spatial_support} and \eqref{intext:mixingtrafo_as_homoeomorphism}.} but otherwise arbitrarily nonlinear.\\[-0.5em]                 

\noindent
Finding ways to achieve this `blind inversion' \eqref{summary_of_cont:eq1} has been of long-standing scientific interest, and efforts in this direction gave rise to an established area of specialised statistical research that has been very active for nearly three decades now. Apart from only a small number of exceptions, however, related works were predominantly confined to the very limiting assumption that the hidden relation $f$ be a linear map on $\R^d$ -- the few existing approaches towards the blind inversion of nonlinear causal relations were either heuristic or required $f$ to belong to very narrowly defined function classes only, and it was not until the recent breakthroughs of Hyvärinen et al.\ that the first mathematically justified ideas for the blind inversion of general nonlinear relationships between $X$ and $S$ have emerged. Our work is a contribution to the dawning research on nonlinear blind inversion.

\subsection{Identifiability (Theorems \ref{thm:NICA_stat}, \ref{cor:NICA_MainCor})} To achieve a meaningful recovery \eqref{summary_of_cont:eq1} of the source $S$ from $X$, we need to compensate for the blindness regarding $f$ and $S$ by imposing some additional assumptions on the latter. The most established such assumption, and arguably the most relevant one in practice, is that the component signals of $S$ are statistically independent; we adopt this assumption throughout.\\[-0.5em]     

\noindent
Many of the conceptional issues that arise in the nonlinear blind reconstruction of an indepen- dent-component source $S$ from $X$ can then be anticipated from the classical i.e.\ linear case $f\in\R^{d\times d}$ already. Similar to the classical case, cf.\ Theorem \ref{thm:Comon},
\begin{itemize} 
\item the blindness\footnote{\ That is, the fact that the inverse problem \eqref{summary_of_cont:eq1} is inherently underdetermined since the constituents $f$ and $S$ of the RHS in \eqref{ICA-Relation} are both unknown.} underlying \eqref{summary_of_cont:eq1} makes an exact recovery of $S$ impossible, but statistical prior information on the source allows to identify $S$ from $X$ up to a minimal ambiguity, namely up to a permutation and monotone scaling of the source's original component signals;\\[-0.75em]
\item these minimally ambiguous (in the above sense) estimates $\hat{S}$ of the original source $S$ preserve the initial condition of intercomponental independence (IC), but under some natural assumptions on $S$ the converse is also true: those retransformations of $X$ which are IC must be minimally ambiguous to $S$.          
\end{itemize} 
These insights into the blind inversion \eqref{summary_of_cont:eq1}, which are rigorously discussed in Section \ref{sect:core_theory}, are the mathematical heart of our approach. Especially the equivalence stated in the last point, which is made precise in Theorems \ref{thm:NICA_stat} and  \ref{cor:NICA_MainCor}, is a central new finding:\\[-0.5em] 

\noindent
Under some mild statistical conditions on the source $S$, we can show that the assumed IC property of the source is strong enough to trivialise\footnote{\ Here, `trivialise' means reduce to the composition of a permutation and a componentwise monotone scaling.} the action of any spatial diffeomorphism which preserves this property; in other words: their property of having minimal intercomponental statistical dependence distinguishes the minimally ambiguous estimates $\hat{S}$ of $S$ from any other invertible nonlinear transformations of $X$.\\[-0.5em]  

\noindent
This makes `minimisation of intercomponent-dependence' an illuminating optimisation principle for the initially blind search for $S$, which immediately translates into the following strategy for the desired inversion \eqref{summary_of_cont:eq1}:      
\begin{equation}\label{summary_of_cont:eq2}
\text{as an estimate $\hat{S}$ for $S$, \quad choose \quad $\hat{S}= \theta_\star(X)$, \ $\theta_\star$ invertible, \  s.t.\ \ $\theta_\star(X)$ is IC\,;}  
\end{equation}
i.e., the right retransformations of $X$ are those that minimise intercomponental dependence.\\[-0.5em] 

\noindent  
As mentioned, the sources $S$ for which this strategy works are those that `carry their IC property well enough' for this property to characterise them, up to minimal ambiguity, among their (invertible) nonlinear transformations. But not every source is of this kind, as becomes particularly clear from considering two `unrecoverable' statistical extreme cases: If the source $S$ is deterministic,\footnote{\ That is, if $S$ attains exactly one sample path with probability one.} then the IC property is void and a meaningful blind inversion \eqref{summary_of_cont:eq1} of the source's mixtures is generally impossible. If $S$ is constant in time, i.e.\ $S = (Z)_{t\in[0,1]}$ for some random vector $Z$ in $\R^d$, then the IC property on $S$ cannot manifest cross-componentally over different time-points and is then generally too weak to support the strategy \eqref{summary_of_cont:eq2} for nonlinear mixtures, see \cite{HYP} and Example \ref{example:ComonAndNonlinearTrafos}.\\[-0.5em]      

\noindent
These unidentifiable source types can be seen as degenerate extremes that are naturally interpolated by the mathematical model class of continuous-time stochastic processes, and said interpolation can be controlled at the level of the second-order finite-dimensional distributions (fdds) of such processes, see Section \ref{sect:intro_BSS}. In fact, we can formulate three regularity assumptions on the family of fdds of a source $S$ which enable the IC-based identifiability \eqref{summary_of_cont:eq2} of the source by ensuring that it is sufficiently far away from the above degeneracies (Section \ref{sect:core_theory}). More specifically, our non-degeneracy assumptions on the source require that sufficiently many of its fdds admit a probability density which is sufficiently complex in that it satisfies one of the following conditions:      
\begin{enumerate}[label=(\alph*)]
\item\label{summary_of_cont:cond1} the density avoids local factorisations and is not of a certain `pathological' Gaussian-like shape, as is made precise in Definition \ref{def:psG_nonsep_stochproc};
\item\label{summary_of_cont:cond2} the density has locally non-vanishing mixed log-derivatives that lie outside certain nullsets, as specified in Definition \ref{def:log-regular}.  
\end{enumerate} 
While the non-factorizability and non-vanishing-log-derivative conditions ensure that the source is `stable enough' to make its IC property unfold\footnote{\ Instead of holding it merely within its fixed-time marginals, as in the generally unidentifiable case of IC random vectors in $\R^d$.} into its component signals in such a way that the (`residual') action inflicted upon $S$ by the composition of the mixing transformation $f$ with an IC-enforcing retransformation [as in \eqref{summary_of_cont:eq2}] does not collapse when considered jointly at different points in time, the exclusion of Gaussian-like shapes or algebraically degenerate density configurations ensures that this residual action on $S$ is `expressive' enough (as per implying a non-degenerate eigenspectrum of a Jacobian). All of this is made precise in Section \ref{sect:overview} and the proofs of Theorems \ref{thm:NICA_stat} and \ref{cor:NICA_MainCor}.\\[-0.5em]  

\noindent
The source conditions \ref{summary_of_cont:cond1} and \ref{summary_of_cont:cond2} again generalise classical theory in a natural way (cf.\ Section \ref{sect:comon} and the remarks on p.\ \pageref{rem:lebesgue_boundary} and Remark \ref{rem:Gaussian_processes}), and in Section \ref{sect:example_sources} we illustrate their broad applicability by compiling a set of widely used signal classes to which these conditions apply.\\[-0.5em]

\noindent
Thus far, our work has established the dependence-minimising approach \eqref{summary_of_cont:eq2} as a successful mathematical strategy to achieve the nonlinear blind source separation task \eqref{summary_of_cont:eq1}, see Theorem \ref{thm:NICA_stat} and Theorem \ref{cor:NICA_MainCor}: We identified natural probabilistic conditions \ref{summary_of_cont:cond1} and \ref{summary_of_cont:cond2} on the source which guarantee that its IC property manifests strongly enough to characterise that source among any invertible (re)transformations of $X$ up to some inevitable ambiguity\footnote{\ That is, as we recall, up to a permutation and monotone scaling of the source's original component signals.}.

\subsection{Blind Inversion via Optimisation (Theorem \ref{thm:optimisation})}In the second part of the paper, we propose a way to turn this theoretical strategy into a ready-to-use statistical method that can be easily implemented in practice. What we need to do for this is provide the observer of the mixture $X$ with three things, namely
\begin{itemize}
\item[--] a set $\Theta$ of invertible candidate demixing transformations on $\R^d$ which is `large enough' to include approximations of the original inverse $f^{-1}$ up to permutation and scale, and for consistency is endowed with a suitable approximation topology;\footnote{\ See the hypothesis on $\Theta$ that is formulated in Theorem \ref{thm:optimisation} for the first, and Assumption \ref{sect:capping:assumptions} (on p.\ \pageref{sect:capping:assumptions}) for the latter assumption.}    
\item[--] a `pair of goggles' $\phi$ that allows the observer to gauge the degree of intercomponental statistical dependence of any given (re)transformation of $X$: the weaker the statistical dependence between the component signals of a process $Y$, the smaller shall be $\phi(Y)\in\R_{\geq 0}$; the desired inversion \eqref{summary_of_cont:eq1} is then performed [via \eqref{summary_of_cont:eq2}] by choosing those transformations $\theta(X)$, $\theta\in\Theta$, of $X$ for which the value $\phi(\theta(X))$ is minimal;
\item[--] an automatable optimisation procedure that combines $\Theta$ and $\phi$ and returns
\begin{equation}\label{summary_of_cont:eq3}
\theta_\star\,\in\,\underset{\theta\in \Theta}{\operatorname{arg\,min}}\,\phi\big(\theta(X)\big) \quad\text{ and then }\quad \hat{S}=\theta_\star(X)
\end{equation} 
as the desired [minimally ambiguous] estimate of $S$, in accordance with \eqref{summary_of_cont:eq2}.  
\end{itemize}
The above is formalised in Theorem \ref{thm:optimisation}. A natural choice in practice is to implement $\Theta$ as the realisation space of an invertible artificial neural network (NN) with $d$ input nodes, cf.\ Remark \ref{rem:thm_optimisation} \ref{rem:thm_optimisation:it0.2} and Section \ref{sect:experimentsII}. Adding $\phi$ as a loss function to the NN, the optimisation \eqref{summary_of_cont:eq3} can then be performed efficiently via backpropagation; for details see Sections \ref{sect:independence_criterion}, \ref{sect:consistency} and \ref{sect:experiments}.\\[-0.5em]

\noindent
Intuitively speaking, in the course of the optimisation \eqref{summary_of_cont:eq3} the observer gradually performs the desired inversion \eqref{summary_of_cont:eq1} directly by comparing different transformations of the data and choosing as most akin to the true inverse those that minimize the $\phi$-quantified statistical dependence of $X$. For nonlinear $f$ the theoretical justification of this procedure is new, while the underlying idea of source separation via quantified dependence minimisation is a well-established concept for the recovery of linearly mixed random vectors in $\R^d$, see e.g.\ Corollary \ref{cor:Comon}.\\[-0.5em] 

\noindent
Inspired by another classical concept, cf.\ \eqref{intext:Comon:CF} on page \pageref{intext:Comon:CF}, in Section \ref{sect:independence_criterion} we propose as dependence quantification $\phi$ a `cross-cumulant'-based energy functional of the form 
\begin{equation}\label{summary_of_cont:eq4}
\phi(Y)= \sum_{m=2}^\infty\sum_{\bm{q}_m}\bar{\kappa}_{\bm{q}_m}(Y)^2 
\end{equation} 
where $\bar{\kappa}_{\bm{q}}(Y)$ denotes `the (standardised) signature cumulant at index $\bm{q}$' of a stochastic process $Y$ in $\R^d$, see Definition \ref{def:sigcumulant} and Notation \ref{notation:index_sum}, and the inner sums run over all `cross-shuffles' of word-length $m$ (see \eqref{sect:capping:normalised_series} on page \pageref{sect:capping:normalised_series}). The entirety of all signature cumulants $(\bar{\kappa}_{\bm{q}}(Y))$, which can be thought of as a carefully chosen `coordinate vector' for the distribution of the multidimensional stochastic process $Y$, provides a hierarchical and parsimous description of the statistical dependence relations within $Y$, which may occur simultaneously between coordinates and over different points in time. The functional \eqref{summary_of_cont:eq4} summarises the aspects of this description that are most central for us, namely `how much' of this dependence there is between the multiple component signals of $Y$. Since the above $\phi$ vanishes over exactly those processes that are IC (Proposition \ref{prop:sig_cums}), the functional \eqref{summary_of_cont:eq4} is well suited to operationalise the inversion strategy \eqref{summary_of_cont:eq2} via the optimisation scheme \eqref{summary_of_cont:eq3}, as described in Theorem \ref{thm:optimisation}; further aspects are discussed in Sections \ref{sect:consistency} and \ref{sect:experiments}. 

\subsection{Consistency (Theorem \ref{thm:consistency})}Up to this point, we discussed the method \eqref{summary_of_cont:eq3} in a setting where the whole distribution of $X$ is assumed to be known. This idealisation is of course difficult to uphold in practice, where mixtures are typically not available as continuous-time stochastic processes and only discrete-time sample trajectories of $X$, i.e.\ finite sequences of data points in $\R^d$, are observed.\\[-0.5em] 

\noindent
The statistical guarantees of Theorem \ref{thm:consistency} ensure that our method remains applicable under these practical constraints. More specifically, a statistical consistency analysis of the procedure \eqref{summary_of_cont:eq3} requires to simultaneously deal with     
\begin{itemize} 
\item[--]time-discretization: if $S$, and hence $X$, are continuous-time signals, then `full' sample observations of the underlying model $X$ (i.e.\ continuous paths in $\R^d$) are not available in real-world applications, where only discrete-time data can be used;\footnote{\ In spirit, this is similar to the well-developed statistical question of parameter estimation for
stochastic differential equations where also only time-discretized sample trajectories are observed.} 
\item[--]finite samples: typically, one of two situations arise in applications. One is that $n$ presumably independent [discrete-time] sample trajectories of the observable are recorded, e.g.\ medical recordings of $n$ patients. The other situation is that only one [discrete-time] sample trajectory of $X$ is given and ergodicity or mixing assumptions are invoked to make inference about the underlying distribution; for example, this situation is common in finance and economics. 
\end{itemize} 
We show that under general conditions, which for example are satisfied by many classical SDE models, our method \eqref{summary_of_cont:eq3} is (strongly) consistent in a sense that addresses both of these points: As the grid of observational time-points gets finer and the length of the observed time series increases, our method \eqref{summary_of_cont:eq3} produces a signal $\hat{S}$ that gets closer to the unobserved source signal $S$, even when the model for $S$ is formulated in continuous time; see Theorem \ref{thm:consistency} for the precise statement. Additionally, Theorem \ref{thm:consistency} shows that our method is robust under approximations of the contrast function $\phi$ (for computational purposes, the series \eqref{summary_of_cont:eq4} of signature cumulants needs to be truncated in practice). The key ingredients to establish this result are tools from stochastic analysis, natural assumptions on the topology of function approximators (e.g., deep neural networks), and statistical approaches to the optimality of extremum estimators. Practitioners may find the displayed algorithm in Section \ref{sect:algorithm} a useful summary.\\[-0.5em]    

\noindent
Our exposition is complemented by a number of numerical examples (Section \ref{sect:experiments}) which further illustrate the practical utility of our method by applying it to a series of nonlinear blind inversion problems \eqref{summary_of_cont:eq1} with multidimensional source signals in discrete and continuous time.\\[-0.5em]    

As a concrete illustration of our blind inversion method \eqref{summary_of_cont:eq3} and its underlying procedure, let us draw on one of these examples here (see Section \ref{sect:experimentsII} for details).   
\begin{example}\label{example:summary}  
Imagine a context where you are interested in a set of `hidden' quantities $S^1, \ldots, S^d$ that are related to some observable data $X^1, \ldots, X^d$ by some unknown invertible function $f:\R^d\rightarrow\R^d$. Assume further that these quantities are time-dependent, so that $ S^i = (S^i_t)$ is a real-valued time series (in discrete or continuous time) and $(X^1_t,\cdots, X^d_t) = f(S^1_t,\cdots, S^d_t)$, and that you may regard $S^1, \ldots, S^d$ as mutually statistically independent. For example, suppose that $d=4$ and your application context is the vibration analysis of wind turbines for fault detection: the quantities of interest $S^i$ could then, e.g., be vibration responses excited by cracked gears or other engine faults in the turbine, which are mixed together during their transmission by an unknown, generally nonlinear \cite{ding2019miningmachine} mixing process $f$ determined by the gearbox configuration; the resulting mixtures $X^i$ are observable vibrations recorded by multi-channel sensors at the outside of the gearbox.\footnote{\ This particular application context is motivated by and adapted from \citep{ding2019miningmachine,liyanwangpeng2016} and the references therein.} Statistically, your recorded data is a time-discretised sample $\mathfrak{x}\equiv(\mathfrak{x}_j^1,\cdots,\mathfrak{x}^4_j)^\intercal\in\R^{4\times\N}$ drawn from the stochastic process $X=(X^1,\cdots, X^4)$, and might locally look like the blue signals shown in the middle column of Figure \ref{fig:example}. In your search for the hidden vibrations $\mathfrak{s}$ that `caused' your data, with $\mathfrak{s}\equiv(\mathfrak{s}^1_j,\cdots,\mathfrak{s}^4_j)^\intercal\in\R^{4\times\N}$ seen as a discretised sample of $S=(S^1,\cdots,S^4)$, your ignorance with regards to the actual relation $f$ between $\mathfrak{s}$ and $\mathfrak{x}$ requires you to perform a blind inversion \eqref{summary_of_cont:eq1}. You know that the closest (``minimally ambiguous'') estimate $\hat{\mathfrak{s}}_\star$ of $\mathfrak{s}$ that you could then obtain is one that coincides with the original $\mathfrak{s}$ up to some permutation and scale, that is where, for $\tau$ some permutation of $\{1,\ldots,4\}$ and $\alpha_i:\R\rightarrow\R$ strictly monotone, 
\begin{equation}\label{summary_of_cont:eq5}
\hat{\mathfrak{s}}_\star = \Big(\alpha_1\!\big(\mathfrak{s}_j^{\tau(1)}\big), \,\cdots, \alpha_4\!\big(\mathfrak{s}_j^{\tau(4)}\big)\Big). 
\end{equation}     
Based on the (very likely correct) assumption that the model $S=(S^1,\cdots, S^4)$ of $\mathfrak{s}$ satisfies one of the non-degeneracy assumptions \ref{summary_of_cont:cond1} or \ref{summary_of_cont:cond2}, your hope is to arrive at \eqref{summary_of_cont:eq5} by subjecting your data to the proposed inversion scheme \eqref{summary_of_cont:eq3}. For this you need to specify a suitable set of retransformations $\Theta$ [of $\mathfrak{x}$], e.g.\ a neural network, and cap the series \eqref{summary_of_cont:eq4} at some finite order $m_0$.\footnote{\ For this example, detailed in Section \ref{sect:experimentsII}, we used the network $\Theta\coloneqq\Theta_2$ specified in \eqref{sect:experimentsII:eq1} and capped the contrast \eqref{summary_of_cont:eq4} at $m_0=5$ (in general, $m_0$ may be chosen larger the more complicated the nonlinearity $f$ is assumed to be). For more information, including on the mixing $f$ and the optimisation \eqref{summary_of_cont:eq3}, see also Appx.\ \ref{appendix:sect:numerics_ann} and \cite{githubSigNICA}.} Next you compute from $\mathfrak{x}$ and each $\theta\in\Theta$ a consistent estimate $\hat{\phi}(\theta)$ of the (capped) contrasts $\phi(\theta(X))$, which may be done as summarized in Section \ref{sect:algorithm}. An empirical approximation of \eqref{summary_of_cont:eq5} is then obtained by finding a minimizer $\hat{\theta}_\star\in\Theta$ of $\hat{\phi}$ and setting $\hat{\mathfrak{s}}\coloneqq\hat{\theta}_\star(\mathfrak{x})$. The convergence, in the limit of infinite data, $\hat{\mathfrak{s}}\rightarrow\hat{\mathfrak{s}}_\star$ is ensured by the (strong) consistency guarantees of Theorem \ref{thm:consistency}. When applied to the sensory observations $\mathfrak{x}$ [Fig.\ \ref{fig:example}, blue column] of our example this method returns a source estimate $\hat{\mathfrak{s}}\in\R^{4\times\N}$ [Fig.\ \ref{fig:example}, brown], and a comparison with the true vibration signals $\mathfrak{s}^1,\ldots,\mathfrak{s}^4$ [Fig.\ \ref{fig:example}, orange] shows that, to a good approximation, the estimate $\hat{\mathfrak{s}}$ coincides with $\mathfrak{s}$ up to permutation and scale, as desired. 
{\hfill \bdiam}     
\end{example}  

\noindent
We emphasize that the above methodology in its entirety, including any of our definitions or theorems, applies to both continuous-time and discrete-time signals alike\footnote{\ For the case of discrete-time signals, everything basically applies as in the continuous case up to very minor modifications necessitated by the change from (path-)connected to discrete realisations of the underlying signals.}. The latter type includes signals that are ``genuinely discrete'', i.e.\ generated from a discrete-time process, and signals that are of continuous origin but ``discretely observed'', i.e.\ obtained from sampling a continuous-time process at a discrete set of time points. These cases are treated in detail in Sections \ref{appendix:NICA_discrete} and \ref{sect:consistency}, which are referenced accordingly throughout the text.\footnote{\ For overview: Section \ref{rem:discrete} explicates our identifiability theory (Sects.\ \ref{sect:intro_BSS} to \ref{sect:independence_criterion}) for the exact inversion of genuinely discrete mixtures, while the (asymptotic) recovery of signals from samples of their discretely observed nonlinear mixtures is developed as part of Section \ref{sect:consistency} (Theorem \ref{thm:consistency} in particular) and in Section \ref{rem:discrete:consistency}.}\\[-0.5em] 

In total, the contents of this paper combine to a general and flexible new statistical method for the nonlinear blind source separation of multidimensional time-dependent signals.\\[-0.5em] 

\subsection{Notation}
Below is some of the notation that we use throughout.
\begin{scriptsize}
\begin{center}\label{tab:notation}
\renewcommand{\arraystretch}{1.5}
  \begin{longtable}{cp{10cm}r}
\toprule\vspace{0.5em}
Symbol & Meaning & Page \\
\toprule 
$[k]$ & $\coloneqq\{1, \ldots, k\}$, \ and \ $[k]_0\coloneqq[k]\cup\{0\}$ \ ($k\in\mathbb{N}$). & \pageref{tab:notation}\\

$S_d$ & $\coloneqq \{\tau : [d]\rightarrow[d]\mid \text{$\tau$ is bijective}\}$; the group of all permutations of $[d]$.  & \pageref{tab:notation}\\

$\Delta_d$ & $\coloneqq \{\Lambda=(\lambda_i\cdot\delta_{ij}) \in \operatorname{GL}_d \ | \  \lambda_1, \ldots, \lambda_d \in\mathbb{R}\setminus\{0\}\}$; the group of (real) invertible diagonal $d\times d$ matrices. &  \pageref{thm:Comon}\\

$\mathrm{P}_d$ & $\coloneqq \{(\delta_{\sigma(i),j})_{i,j\in[d]} \in \operatorname{GL}_d \ | \ \sigma\in S_d \}$; the $d\times d$ permutation matrices.  & \pageref{thm:Comon}\\

$\mathfrak{M}\cdot Y$ & $\coloneqq\{M(Y)\equiv(M(Y_t)) \mid M\in\mathfrak{M}\}$; the image of a process $Y=(Y_t)$ under a set $\mathfrak{M}$ of transformations from $\R^d$ to $\R^d$(resp.\ from the spatial support \eqref{def:spatial_support:eq1} of $Y$ to $\R^d$). & \pageref{thm:Comon:eq1}\\ 

$\underset{a\in A}{\operatorname{arg\,min}}\ \phi(a)$ & $\coloneqq \phi^{-1}\!\big(\min_{a\in A}\phi(a)\big)$; the set of all points in $A$ at which the function $\phi : A \rightarrow \mathbb{R}$ attains its global minimum. & \pageref{cor:Comon}\\ 

$\operatorname{GL}_d$& $\coloneqq\{A \in \mathbb{R}^{d\times d} \ | \
                       \operatorname{det}(A)\neq 0\};$ the general linear group of degree $d$ over $\mathbb{R}$. & \pageref{cor:Comon}\\
                       
$\mathrm{M}_d$ & $\coloneqq \{M \in \operatorname{GL}_d \ | \ M=D\cdot P \ \text{ for } \ D \in\operatorname{GL}_d \text{ diagonal }\text{ and } P \in \mathrm{P}_d\}$; the group of (real) monomial matrices of degree $d$. & \pageref{cor:Comon}\\ 

$\mathcal{C}_d$ & $\equiv C(\I; \mathbb{R}^d) \coloneqq \{ x : \I \rightarrow \R^d \mid \I\ni  t \mapsto x(t)\eqqcolon x_t \text{ continuous} \}$; the space of continuous paths from $\I$ (compact interval) into $\R^d$; we use $\I=[0,1]$ unless mentioned otherwise. & \pageref{subsect:stochprocs_interpolate}\\ 

\multirow{3}{*}{$\pi_J^I$} & the canonical projection from $E_I\coloneqq\{(u_i)_{i\in I}\mid u_i\in E_i \text{ for all } i\in I\}$ onto $E_J$ ($(E_i \mid i\in I)$ some indexed family of sets, $J\subseteq I$); that is $\pi^I_J((u_i)_{i\in I}) = (u_i)_{i\in J}$ where the tuple-indexation follows the order of $I$ and $J$, resp. The superscript $I$ will be omitted if the domain of $\pi^I_J$ is clear. E.g.: $\pi_{\{1,3,5\}}(x_1,x_2,\cdots,x_6) = (x_1,x_3,x_5)$, and $\pi_i\coloneqq\pi_{\{i\}}$.& \multirow{3}{*}{\pageref{def:stochastic_process:eq1}}\\ 

$\mathrm{IC}$ & A stochastic process in $\R^d$ is called \emph{IC} if its component signals are mutually independent. & \pageref{def:stochastic_process}\\ 

$\Delta_2(\mathbb{I})$ & $\coloneqq\{(s,t)\in\mathbb{I}^{\times 2}\mid s < t\}$; the (relatively) open 2-simplex on $\I\times\I$. & \pageref{2ndfdds}\\

wrt./\,s.t./\,wlog & `with respect to'/\,`such that'/\,`without loss of generality' & \pageref{def:spatial_support:eq1}\\

$\mathrm{int}(A)$ &the topological interior of a set $A\subseteq\R^d$ (wrt.\ the Euclidean topology).& \pageref{def:density_reg}\\

$J_\varphi$ & $\coloneqq \big(\frac{\partial}{\partial x_j}\varphi_i\big)_{ij}$; the Jacobian of $\varphi\equiv(\varphi_1,\cdots,\varphi_d)^\intercal\in C^1(G;\R^d)$.& \pageref{rem:TrafoProjProbDens:eq0A}\\

$\restr{\varphi}{\tilde{A}}$&the restriction of a map $\varphi : A\rightarrow B$ to a subdomain $\tilde{A}\subseteq A$.& \pageref{def:PseudoGaussian:eq2}\\ 

$C^{k,k}(D)$ & $\coloneqq\big\{h : D \rightarrow \R^d \ \big| \ h\in\mathrm{Diff}^k(G) \text{ for some open } G\supseteq D\big\}$, $D\subseteq\R^d$; the set of all $C^k$-invertible transformations on $D$. & \pageref{intext:BSS_relation}\\

$f_1\times f_2$ & $:\,\mathbb{R}^{k_1+k_2}\rightarrow\mathbb{R}^{\ell_1}\times\mathbb{R}^{\ell_2}, \ (z_1,\ldots,z_{k_1+k_2})\mapsto (f_1(z_1,\ldots, z_{k_1}), f_2(z_{k_1+1},$ $\ldots, z_{k_1+k_2}));$ the Cartesian product of $f_1 : \mathbb{R}^{k_1}\rightarrow \mathbb{R}^{\ell_1}$ and $f_2 : \mathbb{R}^{k_2}\rightarrow \mathbb{R}^{\ell_2}$.& \pageref{Intro:intext:doublingdimensions}\\

$\mathrm{Diff}^k(G)$ & $\coloneqq\{h : G\rightarrow\mathbb{R}^d\mid h: G\twoheadrightarrow h(G) \text{ is a $C^k$-diffeomorphism}\}$; the set of all functions $h\in C^k(G;\R^d)$ which are one-to-one with $h^{-1}\in C^k(h(G);\R^d)$, for $G\subseteq\R^d$ open. & \pageref{Intro:intext:generalisedComon}\\

$\mathrm{diag}_{i\in[d]}[a_i]$ & $\coloneqq (a_i\cdot\delta_{ij})_{ij}$; the diagonal $d\times d$-matrix with main diagonal $(a_1, \ldots, a_d)$. & \pageref{thm:NICA_stat:aux15}\\

$\nabla^{\times}$ & $\coloneqq\big\{(\lambda_\nu)\in\R^d\ \big| \ \exists\, i,j\in[d],\,i\neq j \, : \, \lambda_i=\lambda_j\big\}$; the set of all vectors in $\R^d$ whose coordinates are not pairwise distinct. & \pageref{def:log-regular}\\

$d(a,B)$ & $\equiv \operatorname{dist}_d(a,B)\coloneqq \inf\{d(a,b)\mid b\in B\}$, for $(M,d)$ a given metric space; the distance between a point $a\in M$ and a non-empty subset $B$ of $M$.& \pageref{sect:finsamlim:eq4}\\

$\mathfrak{C}_m$ & the set of all cross-shuffles $\mathfrak{C}\equiv\bigsqcup_{k=2}^d\mathcal{W}_k$ (Proposition \ref{prop:sig_cums}) of fixed word-length $m$. & \pageref{sect:capping}\\[0.5em]
\hline
\end{longtable}
\end{center} 
\end{scriptsize}\vspace{-2em}

\section{Comon's Framework of Linear Independent Component Analysis}\label{sect:comon} 
\noindent
Our approach to the problem of nonlinear Blind Source Separation \eqref{intext:ProblemOfBSS} for stochastic processes can be regarded as a natural extension of Comon's identifiability framework \cite{COM}. This section briefly recalls the main results of this classical framework as conceptional points of reference. 

\begin{theorem}[{Comon \citep[Theorem 11]{COM}}]\label{thm:Comon}
Let $S=(S^1, \cdots, S^d)^\intercal$ be a random vector in $\mathbb{R}^d$ with mutually independent, non-deterministic components $S^1, \ldots, S^d$ of which at most one is Gaussian. Let further $X=C\cdot S$ for an orthogonal matrix $C\in\mathbb{R}^{d\times d}$. Then, for any orthogonal matrix $\theta\in\mathbb{R}^{d\times d}$, we have the following characterisation:
\begin{equation}\label{thm:Comon:eq1}
\begin{gathered}
(\tilde{S}^1, \cdots, \tilde{S}^d)\coloneqq \theta\cdot X \, = \, \Lambda P\cdot S \quad\text{ for some } \quad (\Lambda, P)\in\Delta_d\times\mathrm{P}_d\\
\text{if and only if } \quad \tilde{S}^1, \ldots, \tilde{S}^d \ \quad \text{ are mutually independent}.
\end{gathered}
\end{equation} 
\end{theorem}    
\noindent  
The significance of Theorem \ref{thm:Comon} is that it characterises --- up to some minimal deviation, namely their scaling and re-ordering --- the independent sources $S^1, \ldots, S^d$ underlying an observable linear mixture $X=A\cdot (S^1, \ldots, S^d)^\intercal$ as precisely those transformations $\theta_\star\cdot X\eqqcolon(X_{\theta_\star}^1, \ldots, X_{\theta_\star}^d)$ of the data whose components $X^i_{\theta_\star}$ are mutually independent.
\begin{remark}\label{rem:thm_comon}
\begin{enumerate}[label=(\roman*)]
\item The orthogonality constraint of Theorem \ref{thm:Comon} imposes no loss of generality with regards to general linear mixtures since any invertible linear relation $X=A\cdot S$, $A\in\operatorname{GL}_d$, between $X$ and $S$ can be reduced to an orthogonal one by performing a principal component analysis on $X$. 
\item\label{rem:thm_comon_item2} The proof of Theorem \ref{thm:Comon} is based on the remarkable probabilistic fact that any two linear combinations of a family of statistically independent random variables can themselves be statistically independent only if each random variable of this family which has a non-zero coefficient in both of the linear combinations is Gaussian. (A result which is known as the Darmois-Skitovich theorem, see \citep{DAR,SKI}.) This accounts for the theorem's somewhat curious `non-Gaussianity' condition.
\item On a historical note, we thank Samuel Cohen for making us aware that the above works are in fact all predated by the earlier identifiability considerations \citep{reiersol1950} of Reiers{\o}l.    
\end{enumerate} 
\end{remark}  
\noindent
Theorem \ref{thm:Comon} enables the recovery of $S$ from $X$ by way of solving an optimisation problem. 
\begin{corollary}[\cite{COM}]\label{cor:Comon}
Let $X$ and $S$ be as in Theorem \ref{thm:Comon}. Then for any function\footnote{\ Here and in the following, $\mathcal{M}_1(V)\coloneqq \{ \mu : \mathcal{B}(V) \rightarrow [0,1] \mid \text{$\mu$ is a (Borel) probability measure}\}$ denotes the space of probability measures over the Borel $\sigma$-algebra $\mathcal{B}(V)\coloneqq\sigma(\mathcal{T})$ of a topological space $(V,\mathcal{T})$.} $\phi : \mathcal{M}_1(\R^d)\rightarrow \mathbb{R}_+$ such that $\phi(\mu)=0$ iff $\mu = \mu^1\otimes\cdots\otimes\mu^d$, it holds that\footnote{\ We write $\mu^i\coloneqq\mu\circ\pi_i^{-1}$ for the $i^{\mathrm{th}}$ marginal of a (Borel) measure $\mu$ on $\mathbb{R}^d$. We further abuse notation by writing $\phi(Z) \coloneqq \phi(\mathbb{P}_Z)$ for any random vector $Z : (\Omega,\mathscr{F}, \mathbb{P})\rightarrow(\R^d, \mathcal{B}(\R^d))$.}      
\begin{equation}\label{cor:Comon:eq1}
\left[\underset{\theta\in\Theta}{\operatorname{arg\ min}}\ \phi(\theta\cdot X)\right]\cdot X \ \subseteq \ \mathrm{M}_d\cdot S
\end{equation}where $\mathrm{M}_d\coloneqq\{\Lambda\cdot P\mid (\Lambda, P)\in\Delta_d\times\mathrm{P}_d\}$ is the subgroup of monomial matrices and $\Theta\subset\operatorname{GL}_d$ is the subgroup of orthogonal matrices.
\end{corollary}   
\noindent       
In other words: For $f$ linear and $S=(S^1, \cdots, S^d)^\intercal$ a random vector with mutually independent, non-Gaussian components, the constituent relationship \eqref{ICA-Relation} between the observable $X$ and its source $S$ can be inverted (up to a minimal deviation) by optimizing some independence criterion $\phi$ over a set of candidate transformations $\Theta$ applied to $X$.\\[-0.5em]      

Partially driven by their applicability \eqref{cor:Comon:eq1} to ICA, a variety of such criteria $\phi$, referred to in \cite{COM} as contrast functions, have been developed.\\[-0.5em] 

The `original' independence criterion $\phi_c$ proposed in \cite{COM} quantifies the statistical dependence between the components $Y^i$ of a random vector $Y=(Y^1, \cdots, Y^d)$ in $\R^d$ via the sum of the squares of all standardized cross-cumulants $\kappa_{i_1\cdots i_j}^Y$ of $Y$ up to $r^{\mathrm{th}}$-order (see \citep[Sect.\ 3.2]{COM} and cf.\ \eqref{rem:classic_cumulants:eq1}), i.e.\ via the quantity
\begin{equation}\label{intext:Comon:CF}
\phi_c(Y) \ \coloneqq \ \sum_{j = 2}^r{\sum_{i_1,\ldots,i_j}}^{\!\!\!\!\times}(\kappa_{i_1\cdots i_j}^Y)^2 \qquad (r\geq 2)   
\end{equation}
where the inner sum runs over the indices $i_1, \ldots, i_j\in[d]$ corresponding to \eqref{intext:coords_indep_class}.

\noindent 
Initially proposed in \cite{COM}, the statistic \eqref{intext:Comon:CF} originates from a truncated Edgeworth-expansion of mutual information in terms of the standardized cumulants of its argument.  

A variety of alternatives to \eqref{intext:Comon:CF} soon followed, including kernel-based independence measures \citep{BJO,GRE}, a variety of (quasi-) maximum-likelihood objectives, e.g.\ \citep{BES,MOU,PHG}, as well as mutual information and approximations thereof, e.g.\ \cite{CD3,COM,HYR,HYE}.\\

\noindent
While successfully achieving the separability of linear mixtures, Theorem \ref{thm:Comon} has its limitations: Being based on somewhat of a probabilistic curiosity (Rem.\ \ref{rem:thm_comon} \ref{rem:thm_comon_item2}), it might not be surprising that the characterisation \eqref{thm:Comon:eq1} cannot be generalised to guarantee the recovery of independent scalar sources from substantially more general nonlinear mixtures of them \cite{HYP}. Roughly speaking, the reason for this is that for a single random vector in $\R^d$, the statistical property of componental independence is too weak to characterise the nonlinear mixing transformations preserving this property as `trivial' in a sense made precise by Definition \ref{def:monomial_trafos} below. The following example illustrates this.\footnote{\ Ex.\ \ref{example:ComonAndNonlinearTrafos} is based on the `Box-Muller transform', a well-known subroutine from computational statistics. For a systematic way of constructing `unidentifiable' nonlinear mixtures of IC random vectors in $\R^d$, see \cite{HYP}.} 
\begin{example}[Comon's Criterion \eqref{thm:Comon:eq1} Does Not Apply to Nonlinearly Mixed Vectors in $\R^d$]\label{example:ComonAndNonlinearTrafos}
Let $S^1$ and $S^2$ be independent with $S^1$ Rayleigh-distributed of scale 1 and $S^1$ uniformly distributed over $(-\pi,\pi)$, and consider the nonlinear mixing transformation $f$ given by $f(u,v)\coloneqq(u\cos(v), u\sin(v))$ (transformation from polar to Cartesian coordinates). Then even though their functional relation $f$ to $S^1, S^2$ is `non-trivial' (i.e.\ $f$ is not monomial in the sense of Definition \ref{def:monomial_trafos}) the mixed variables $X^1$ and $X^2$ defined by $(X^1, X^2)\coloneqq f(S^1, S^2)$ are [normally distributed and] statistically independent.\footnote{\ Note that since the density $p_S$ of $(S^1, S^2)$ reads $p_S(s_1,s_2) = \frac{1}{2\pi}s_2 e^{-s_2^2/2}$, the (joint) density $p_X = (p_S\circ f)\cdot|\!\det J_f|^{-1}$ of $(X^1, X^2)$ factorizes, implying the independence of $X^1$ and $X^2$ as claimed.}  \hfill \bdiam      
\end{example}  

\section{Modelling Sources as Stochastic Processes}\label{sect:intro_BSS} 
\noindent
A central direction along which the blind recovery of the source $S$ from its nonlinear mixture $X$ can be controlled is the amount of statistical structure that $S$ carries: If the source $S$ is deterministic, then no additional information is given and a meaningful recovery of $S$ from $X$ is generally impossible, cf.\ Example \ref{example:intro}. If, on the other hand, the source $S$ were to be described merely as a random vector in $\R^d$, then a recovery of $S$ from $X$ is possible but in general only if $X$ is a linear function of $S$, cf.\ \citep{COM,HYP} and Example \ref{example:ComonAndNonlinearTrafos}. A key insight from \cite{HYM} is to go for the middle ground (see Remark \ref{prop:stoch_procs_interpolate}): if we demand the source $S$ to have a `non-degenerate temporal structure' and exploit this in a suitable manner, then the recovery of $S$ from even its nonlinear mixtures is possible.
To formalize such temporal statistical dependencies requires us to model the source $S$ as a stochastic process. To this end, we use this section to briefly recall foundational notions from stochastic analysis (Section \ref{subsect:stochprocs_interpolate}) and provide some basic notions and lemmas (Section \ref{sect:sources_stochprocesses}) that we will use for our subsequent identifiability results in Section~\ref{sect:core_theory}. 

\subsection{Stochastic Processes Interpolate Statistical Extremes}\label{subsect:stochprocs_interpolate}Here and throughout, let $\I$ be a compact interval, $d\in\N$ be some fixed integer, write $\mathcal{C}_d\equiv C(\I; \mathbb{R}^d) \coloneqq \{ x : \I \rightarrow \R^d \mid \text{the map }\I\ni  t \mapsto x(t)\eqqcolon x_t \text{ is continuous} \}$ for the space of continuous paths in $\R^d$, and let $(\Omega,\mathscr{F},\mathbb{P})$ denote a fixed probability space.
\begin{definition}[Source Model]\label{def:stochastic_process}
We call a \emph{continuous stochastic process in $\R^d$} any map
\begin{equation}\label{def:stochastic_process:eq1}
S \,:\, \Omega \rightarrow \mathcal{C}_d\quad\text{s.t.}\quad \omega\mapsto S(\omega)\equiv(S_t(\omega))_{t\in\I} \ \ \text{ is \ $(\mathscr{F}, \mathcal{B}(\mathcal{C}_d))$-measurable}, 
\end{equation}where $\mathcal{B}(\mathcal{C}_d)=\sigma(\pi_t\mid t\in\I)$ denotes the Borel $\sigma$-algebra on the Banach space $(\mathcal{C}_d, \|\cdot\|_\infty)$. Writing $S_t(\omega)\equiv(S^1_t(\omega), \cdots, S^d_t(\omega))^\intercal\in\R^d$ for each $\omega\in\Omega$, the scalar processes $S^i\equiv(S^i_t)_{t\in\I}$ ($i\in[d]$) are called \emph{the component processes} or \emph{the components} of $S\equiv(S^1, \cdots, S^d)$. We say that a stochastic process \emph{$S=(S^1,\cdots, S^d)$ has independent components}, or that \emph{$S$ is IC}, if its distribution $\mathbb{P}_S\coloneqq\mathbb{P}\circ S^{-1}$ satisfies the factor-identity\footnote{\ Strictly speaking, \eqref{def:stochastic_process:eq2}  reads $\mathbb{P}_{(S^1,\,\cdots,\,S^d)} \ = \ \big(\mathbb{P}_{S^1}\otimes\cdots\otimes\mathbb{P}_{S^d}\big)\circ\psi^{-1}$, an identity of measures on $\mathcal{B}(\mathcal{C}_d)$, where $\psi:\mathcal{C}_1^{\times d}\rightarrow\mathcal{C}_d$ is a canonical isometry defining the Cartesian identification $\mathcal{C}_d\cong\mathcal{C}_1^{\times d}$ (Remark \ref{appendix:sect:add_proofs_and_rems:cartesian_cd}).}
\begin{equation}\label{def:stochastic_process:eq2}
\mathbb{P}_{(S^1,\,\cdots,\,S^d)} \ = \ \mathbb{P}_{S^1}\otimes\cdots\otimes\mathbb{P}_{S^d}. 
\end{equation} 
\end{definition}
\begin{remark}\label{rem:stochprocs_as_randompaths}
From a more local perspective, Definition \ref{def:stochastic_process} is equivalent to the description of a continuous stochastic process $S$ as an $\I$-indexed family $S=(S_t)_{t\in\I}$ of random vectors\footnote{\ For us every random vector in $\R^d$ is Borel, i.e.\ $(\mathscr{F}, \mathcal{B}(\R^d))$-measurable.} $S_t$ in $\R^d$ such that the map $S(\omega)\ : \ \I\,\ni\,t \,\mapsto\, S_t(\omega)\,\in\,\R^d$ is continuous for each $\omega\in\Omega$; e.g.\ \citep[Sect.\ II.27]{rogerswilliams2000}. Consequently (cf.\ also Section \ref{appendix:sect:add_proofs_and_rems:cartesian_cd}), the independence condition \eqref{def:stochastic_process:eq2} is equivalent to  
\begin{equation}
\big(S^1_{t_1^{(1)}}, \cdots, S^1_{t_{k_1}^{(1)}}\big), \big(S^2_{t_1^{(2)}}, \cdots, S^2_{t_{k_2}^{(2)}}\big), \cdots, \big(S^d_{t_1^{(d)}}, \cdots, S^d_{t_{k_d}^{(d)}}\big) \quad \text{mutually $\mathbb{P}$-independent}  
\end{equation}for any finite selection of time-points $t_1^{(1)}, \ldots, t_{k_1}^{(1)}, \ldots, t_1^{(d)}, \ldots, t_{k_d}^{(d)}\in\I$, $k_1,\ldots, k_d\in\mathbb{N}_0$.   
\end{remark} 
\noindent  
Stochastic processes can be given a prominent role in the BSS-context, namely as natural interpolants between deterministic signals and random vectors. While the first type of signal is the unidentifiable default model for the source in \eqref{ICA-Relation}, the latter is the predominant source model in classical ICA-approaches. More specifically, the following is easy to see.

\begin{remark}[Stochastic Processes Interpolate Between Extremal Source Models]\label{prop:stoch_procs_interpolate}\phantom{.}\\
Let $S=(S_t)_{t\in\I}$ be a continuous stochastic process in $\R^d$ such that either 
\begin{enumerate}[label=(\alph*)]
\item $S_s$ and $S_t$ are independent for each $s,t\in\tilde{\I}$ with $s\neq t$, \quad or\\[-1em]
\item $S_s = S_t$ almost surely for each $s,t\in\tilde{\I}$,  
\end{enumerate} 
for some $\tilde{\I}\subset\I$ dense. Then $S$ is either a single path in $\mathcal{C}_d$ almost surely (i.e.\ $S$ is deterministic; `statistically trivial')\footnote{\ This implication is obtained from Kolmogorov's zero-one law (applied after a straightforward subsequence argument) and the sample continuity of $S$.} namely iff (a) holds, or the sample-paths of $S$ are constant almost surely (i.e.\ $S$ is a random vector; `temporally trivial') namely iff (b) holds.
\end{remark}
Remark \ref{prop:stoch_procs_interpolate} asserts that both deterministic signals (a) as well as random vectors (b) can be seen as degenerate stochastic processes, and that for a given stochastic process $S=(S_t)_{t\in\I}$ this degeneracy manifests on the level of its \nth{2}-order finite-dimensional distributions, i.e.\ on
\begin{equation}\label{2ndfdds}
\text{the distributions of} \qquad \big\{(S_s, S_t)\ \big| \ (s,t)\in\Delta_2(\I)\big\}
\end{equation}where the index set $\Delta_2(\mathbb{I})\coloneqq\{(s,t)\in\mathbb{I}^{\times 2}\mid s < t\}$ is the (relatively) open 2-simplex on $\I\times\I$. In the following, we refer to \eqref{2ndfdds} as the \emph{temporal structure} of a stochastic process $S=(S_t)_{t\in\I}$.\\

\noindent The following is essential: As mentioned above and illustrated in the next section, if the temporal structure of the IC source $S$ in \eqref{ICA-Relation} is `degenerate' in the sense of Remark \ref{prop:stoch_procs_interpolate} (a), (b), then $S$ is unidentifiable from $X$ unless $f$ is of a very specific form, e.g.\ linear (cf.\ Theorem \ref{thm:Comon}).   
Conversely, we will argue that if the source $S$ has a temporal structure which is `non-degenerate' (in some specified sense) and satisfies some additional regularity assumptions, then $S=(S^1,\cdots,S^d)$ will be identifiable from even its nonlinear mixtures up to a permutation and monotone scaling of its components $S^i$ (Theorems \ref{thm:NICA_stat}, \ref{cor:NICA_MainCor}, \ref{thm:optimisation}).         
 
\subsection{Stochastic Processes as Sources: Basic Notions and Assumptions}\label{sect:sources_stochprocesses}

\noindent
Recall that the BSS problem \eqref{intext:ProblemOfBSS} concerns the recovery of the source $S$ from its image $X$ under some mixing transformation $f$ on $\R^d$. It is thus clear that given $X$, the map $f$ can be analysed only on that part of its domain that is actually reached by $S$ during the time $X$ is observed. With this in mind, we introduce the `spatial support' of a stochastic process as the smallest closed subset of $\R^d$ which contains (the trace of) $\mathbb{P}$-almost each sample path of the process.\footnote{\ Analogous to how the support $D_Z\coloneqq\supp(Z)$ of a random vector $Z$ in $\R^d$ is the smallest closed subset of $\R^d$ within which $Z$ is contained with probability one.}  
\begin{definition}[Spatial Support]\label{def:spatial_support} 
For $Y=(Y_t)_{t\in\I}$ a (continuous) stochastic process in $\R^d$, the \emph{spatial support} of $Y$ is defined as the set
\begin{equation}\label{def:spatial_support:eq1}
D_Y \, = \, \overline{\bigcup_{t\in\I}\supp(Y_t)}
\end{equation}with $\supp(Y_t)\equiv\supp(\mathbb{P}_{Y_t})\eqqcolon D_{Y_t}$ denoting the support of the distribution of $Y_t$, and where the closure is taken wrt.\ the Euclidean topology on $\R^d$. 
\end{definition}\vspace{-0.5em}  
(Readers uncomfortable with \eqref{def:spatial_support:eq1} may for simplicity assume that $D_S=\R^d$ throughout.)\\[-0.5em]

\noindent
The following elementary properties of the set \eqref{def:spatial_support:eq1} will be useful to us.                  
\begin{restatable}{lemma}{lemspatsupp}\label{lem:spat_supp}
Let $Y=(Y_t)_{t\in\I}$ be a stochastic process in $\R^d$ which is continuous with spatial support $D_Y$. Then the following holds:
\begin{enumerate}[label=\upshape(\roman*)]
\item\label{lem:spat_supp:it1}if $f : D_Y\rightarrow \R^d$ is a homeomorphism onto $f(D_Y)$, then $D_{f(Y)} = \overline{f(D_Y)}$;
\item\label{lem:spat_supp:it2}the traces $\mathrm{tr}(Y(\omega))\coloneqq\{Y_t(\omega)\mid t\in\I\}$ are contained in $D_Y$ for $\mathbb{P}$-almost each $\omega\in\Omega$; 
\item\label{lem:spat_supp:it3}for each open subset $U$ of $D_Y$ there is some $t^\star\in\mathbb{I}$ with $\mathbb{P}(Y_{t^\star}\in U)>0$;
\item\label{lem:spat_supp:it4}if each random vector $Y_t$, $t\in\I$, admits a continuous Lebesgue density on $\R^d$, then $D_Y$ is the closure of its interior;
\item\label{lem:spat_supp:it5}if each random vector $Y_t$, $t\in\mathbb{I}$, admits a continuous Lebesgue density $\upsilon_t$ such that $\upsilon^x\,:\, \mathbb{I}\ni t\mapsto\upsilon_t(x)$ is continuous for each $x\in D_Y$, we for $\dot{D}_t\coloneqq \{\upsilon_t>0\}$ have that the set 
\begin{equation}\label{lem:TimeSeriesCopulaModel:eq1}
\bigcup_{(s,t)\in\Delta_2(\mathbb{I})}\dot{D}_s\cap\dot{D}_t \quad\text{is \quad dense \quad in}\quad D_Y.
\end{equation}
\end{enumerate} 
\end{restatable}
\begin{proof}
See Appendix \ref{pf:lem:spat_supp}. 
\end{proof} 
\noindent
Given the above, we can describe the mixing transformation $f$ mapping $S$ to $X$ via\footnote{\ Recall that $X = f(S)$ means: $X_t = f(S_t)$ for each $t\in\I$.} \eqref{ICA-Relation} as
\begin{equation}\label{intext:mixingtrafo_as_homoeomorphism}
\text{a homeomorphism}\footnotemark\quad f \,:\, D_S\rightarrow D_X,   
\end{equation}\footnotetext{\ Note that while the assumption of invertibility of $f$ is canonical, the additionally imposed bi-continuity of the mixing transformation $f$ is a technical condition to ensure that the sample-continuity of the considered processes is preserved under any of the operations that follow.}with the action of $f$ outside of $D_S$ and $D_X$ being irrelevant (and inaccessible) to us.\\

\noindent

We now introduce smoothness conditions on the density which we require later on.

\begin{definition}\label{def:density_reg}
A random vector $Z$ in $\R^n$ will be called \emph{$C^k$-distributed}, $k\in\N_0$, if its distribution admits a Lebesgue density $\varsigma\in C^k(G)$ for $G\coloneqq\mathrm{int}(\supp(\varsigma))$; if $\varsigma$ is $C^k$ on some open neighbourhood of $x_0\in\R^n$, then $Z$ will be called \emph{$C^k$-distributed around $x_0$}.   
\end{definition}
\begin{remark}\label{rem:TrafoProjProbDens}
We recall that for $\vartheta : \mathbb{R}^n\rightarrow\mathbb{R}^n$ a $C^{\ell}$-diffeomorphism, $\ell\geq 1$, the classical transformation formula for densities asserts that the image $\tilde{Z}\coloneqq \vartheta(Z)$ of a $C^k$-distributed random vector $Z$ with density $\varsigma$ is itself $C^{k\wedge(\ell-1)}$-distributed with density $\tilde{\varsigma}$ given by  
\begin{equation}\label{rem:TrafoProjProbDens:eq0A}
\tilde{\varsigma} \ = \ (\varsigma\circ\vartheta^{-1})\cdot|\det J_{\vartheta^{-1}}|.
\end{equation}  
\end{remark}
\noindent
The action of the mixing transformation \eqref{intext:mixingtrafo_as_homoeomorphism} on the source can be profitably captured by imposing the temporal structure \eqref{2ndfdds} of $S$ to meet the following analytical regularity condition:\\
 
\noindent
In the following, a stochastic process $Y=(Y_t)_{t\in\I}$ in $\R^d$ will be called
\begin{equation}
\text{\emph{$C^k$-regular at $(s,t)\in\Delta_2(\I)$} \quad if \quad the random vector $(Y_s,Y_t)$ is $C^k$-distributed;}
\end{equation} 
the process $Y$ will be called \emph{$C^k$-regular at $((s,t),y_0)\in\Delta_2(\I)\times\R^{2d}$} if the random vector $(Y_s,Y_t)$ is $C^k$-distributed around $y_0\in\R^{2d}$ and its density at $y_0$ is positive.
\begin{remark}\label{rem:lebesgue_boundary}
Note that if $Y$ is $C^k$-regular at $(s,t)$, then the boundary of the support of the joint density of $(Y_s, Y_t)$ is a Lebesgue nullset. (A direct consequence of Sard's theorem.) 
\end{remark}
\noindent
The theory of ICA knows two prominent `exceptional cases' for which the recovery of an IC random vector $S$ in $\R^d$ from even its linear mixtures $X$ cannot be guaranteed without further assumptions, namely the cases in which
\begin{enumerate}[label=(\roman*)]
\item\label{intext:ICA_pathologies1}more than one of the components of $S$ is Gaussian (cf.\ Theorem \ref{thm:Comon}), or
\item\label{intext:ICA_pathologies2}the source $S$ is `statistically trivial' in the sense of Remark \ref{prop:stoch_procs_interpolate} (a).     
\end{enumerate}   
As it turns out, a generalised version of these pathologies carries over to the first and more `static' of our separation principles (Theorem \ref{thm:NICA_stat}), owing to the fact that certain analytical forms of the joint distributions constituting \eqref{2ndfdds} will be `too simple' to guarantee nonlinear identifiability even for sources whose temporal structure \eqref{2ndfdds} is not otherwise degenerate.\\[-0.5em] 

Generalising (i) and (ii) from `spatial' to `inter-temporal statistics', these exceptional types of joint distributions\footnote{\ Distributional pathologies similar to Definition \ref{def:PseudoGaussian} have been first described in \cite{HYM}. More specifically, the above notions of (strict) non-separability and pseudo-Gaussianity generalise the notions \citep[Def.\ 1 and Def.\ 2]{HYM}, respectively, see Section \ref{rem:hyvarinenmorioka}.} will be named `pseudo-Gaussian' and `separable', respectively:       

\begin{definition}[Non-Gaussian, (Regularly) Non-Separable]\label{def:PseudoGaussian}
A function $\varsigma : G\rightarrow\mathbb{R}$, $G\subseteq\R^2$ open, will be called \emph{pseudo-Gaussian} if there are functions $\varsigma_1, \varsigma_2, \varsigma_3 : \R\rightarrow\R$ for which  
\begin{equation}\label{def:PseudoGaussian:eq1}
\varsigma(x,y) \ = \ \varsigma_1(x)\cdot \varsigma_2(y)\cdot\exp(\pm\,\varsigma_3(x)\cdot\varsigma_3(y))
\end{equation}holds on all of $G$; the function $\varsigma$ will be called \emph{separable} if the above holds for $\varsigma_3\equiv 0$. 
The function $\varsigma : G\rightarrow\R$ will be called \emph{strictly non-Gaussian} if it is such that 
\begin{equation}\label{def:PseudoGaussian:eq2}
\left.\varsigma\right|_{\mathcal{O}} \ \text{ is not pseudo-Gaussian}, \quad \text{for every open subset $\mathcal{O}$ of $G$}\,;
\end{equation}the property of $\varsigma$ being \emph{strictly non-separable} is declared mutatis mutandis. Furthermore, the function $\varsigma : G \rightarrow \R$ will be called \emph{almost everywhere non-Gaussian} if
\begin{equation}    
\text{there is a closed nullset \quad $\mathcal{N} \subset G$ \ \quad s.t. }\quad\left.\varsigma\right|_{(G\setminus\mathcal{N})} \ \text{ is strictly non-Gaussian}\,;
\end{equation}the notion of $\varsigma$ being \emph{a.e.\ non-separable} is defined analogously.

\noindent
Finally, a twice continuously differentiable function $\varsigma : \tilde{U}\times\tilde{U}\rightarrow \R_{>0}$, with $\tilde{U}\subseteq\R$ open, will be called \emph{regularly non-separable} if 
\begin{equation}\label{def:regNonSep:eq1}
\varsigma \ \text{ is a.e.\ non-separable} \quad\text{ and }\quad \left.\big(\partial_x\partial_y\log\varsigma\big)\right|_{\Delta_{\tilde{U}}}\neq 0 \ \text{ a.e.\ on } \Delta_{\tilde{U}} 
\end{equation}where $\Delta_{\tilde{U}}\coloneqq\{(x,x)\mid x\in\tilde{U}\}$ denotes the diagonal over $\tilde{U}$. 
\end{definition}\vspace{-0.5em}
(Clearly, if $\varsigma$ is [strictly/a.e.] non-Gaussian then it is also [strictly/a.e.] non-separable.)\\[-0.5em]

\noindent
It will be convenient for us to have an analytical characterisation of these `pathological' types of densities at hand. Such a characterisation is provided by Lemma \ref{lem:C2PseudoGaussian} in Section \ref{pf:lem:C2PseudoGaussian}. 

\begin{remark}
\begin{enumerate}[label=(\roman*)]
\item In light of Lemma \ref{lem:C2PseudoGaussian} (ii), the assumption of regular non-separability can be regarded as a minimal extension of the above notion of strict non-separability. The necessity of this extension will become clear in Section \ref{sect:main_thm}. 
\item The log-derivative condition of \eqref{def:regNonSep:eq1} is non-vacuous as there are (strictly) non-separable functions whose mixed log-derivatives vanish on the diagonal, see Example \ref{example:strictlynonsep_not_regnonsep}.    
\end{enumerate}
\end{remark}

\section{An Identifiability Theorem for Nonlinearly Mixed Independent Sources}\label{sect:core_theory}
\noindent
We are now ready to present the mathematical core behind our identifiability results for nonlinearly mixed time-dependent sources. Following an overview of our strategy (Section \ref{sect:overview}), we state and prove our main results (Sections \ref{sect:main_thm} and \ref{sect:beta_and_gamma}) and conclude with a comparison with related work (Section \ref{rem:hyvarinenmorioka}). \\[-0.5em]

\noindent
Throughout, let $S$ and $X$ be two continuous stochastic processes in $\R^d$ that are related via
\begin{equation}\label{intext:BSS_relation}
X \ = \ f(S)
\end{equation} 
for a mixing transformation $f$ which is $C^{2}$-invertible on some open superset of $D_S$.\\[-0.5em] 

\noindent
Here, we say that $f : \R^d\rightarrow\R^d$ is \emph{$C^{k}$-invertible} on an open set $G$ of $\R^d$, in symbols: $f\in C^{k}(G)$, if the restriction $\restr{f}{G}$ is a $C^k$-diffeomorphism (with $C^k$-inverse $f^{-1} : f(G)\rightarrow G$).\\[-0.5em]

\noindent
Throughout the rest of this paper, we operate under the following convenience assumption: 
\begin{assumption}\label{sect:assumption_convexity}
For the source $S$ in \eqref{intext:BSS_relation}, every connected component of $D_S$ is convex.
\end{assumption} 
\begin{remark}While Assumption \ref{sect:assumption_convexity} holds for most conventional process models (including the examples in Sections \ref{sect:example_sources} and \ref{sect:experiments} below), it can be dropped immediately at the only price that for a given realisation $\mathfrak{s}\equiv S(\omega)$ of $S$ the scales $\alpha_i\equiv \alpha_i(\mathfrak{s})$ in \eqref{Intro:intext:generalisedComon:eq0} may\footnote{\ But even for non-convex geometries of $D_S$ this price is not necessarily incurred, see Example \ref{sect:nonconvexsupport:example}. } then vary with each maximally convex subset of $D_S$ that the trace of $\mathfrak{s}$ passes through, see Section \ref{sect:nonconvexsupport}.\footnote{\ Even without Assumption \ref{sect:assumption_convexity} and except for the incurred $\mathfrak{s}$-dependence of the multiples $\alpha_i$, the last identity in \eqref{Intro:intext:generalisedComon:eq0} below continues to hold as stated with the permutation $P$ depending only on the connected component of $D_S$ that the realisation $\mathfrak{s}$ of $S$ is [almost surely] contained in, cf.\ Theorem \ref{sect:nonconvexsupport:thm} in Section \ref{sect:nonconvexsupport}.}   
\end{remark}  

\subsection{Overview}\label{sect:overview}Starting from \eqref{intext:BSS_relation} with the coordinates $(S^1_t)_{t\in\I}, \ldots, (S^d_t)_{t\in\I}$ of the source $S=(S_t^1, \cdots, S^d_t)_{t\in\I}$ assumed mutually independent, we seek to identify $S$ from $X$ by exploiting the main dimensions of our model, space and time, via their statistical synthesis \eqref{2ndfdds}, the temporal structure of $S$. This will be done along the following lines.\\  

\noindent
Given $(s,t)\in\Delta_2(\I)$ we first double the available spatial degrees of freedom by lifting the mixing identity \eqref{intext:BSS_relation} to an associated identity in the factor-space $\R^d\times\R^d$, namely 
\begin{equation}\label{Intro:intext:doublingdimensions}
(X_s, X_t) = (f\times f)(S_s, S_t).
\end{equation}The lifted mixing identity \eqref{Intro:intext:doublingdimensions}, which directly involves the temporal structure \eqref{2ndfdds} of the source, now allows for the following statistical comparison in the spirit of \cite{HYM}:\\[-0.5em] 

\noindent
For $X_t^\ast$ an independent copy of $X_t$, consider the intertemporal features $Y\coloneqq(X_s,X_t)$ and $Y^\ast\coloneqq(X_s,X_t^\ast)$ of the observable $X$ at fixed $(s,t)$ together with their random combination       
\begin{equation}
\bar{Y}\coloneqq C\cdot Y + (1-C)\cdot Y^\ast
\end{equation}for an equiprobable $\{0,1\}$-valued random variable $C$ independent of $Y, Y^\ast$. Combining \eqref{Intro:intext:doublingdimensions} with the fact that $S$ is IC, we obtain for the (deterministic) functional $L(Y,Y^\ast)\coloneqq\psi\circ\rho$ with $\rho(y)\coloneqq\mathbb{E}[C\mid\bar{Y}=y]$ and $\psi(p)\coloneqq\log(p/(1-p))$ a contrast identity of the form      
\begin{equation}\label{Intro:intext:centrallink}
L(Y,Y^\ast) = R(f,(S_s, S_t))
\end{equation}for a function $R\equiv R(f,(S_s,S_t))$ which depends exclusively on $f$ and the distribution of $(S_s,S_t)$. In other words, \eqref{Intro:intext:centrallink} relates $X$ to $S$ by way of the source's temporal structure \eqref{2ndfdds}.\\[-0.5em]  

\noindent
Since the LHS $L\equiv L(Y,Y^\ast)$ in \eqref{Intro:intext:centrallink} is a function of the (joint) distribution of $(Y,Y^\ast)$ -- and thus of the mixture $X$ -- only, we for any alternative pair $(\tilde{f}, \tilde{S})$ with $\tilde{f}(\tilde{S})=X$ and $\tilde{f}\in C^{2}$ and $\tilde{S}$ IC analogously obtain that $L(Y,Y^\ast)=\tilde{R}(\tilde{f}, (\tilde{S}_s,\tilde{S}_t))$ and hence 
\begin{equation}\label{Intro:intext:centrallink2}
\tilde{R}(\tilde{f}, (\tilde{S}_s,\tilde{S}_t)) = R(f,(S_s,S_t))
\end{equation}       
by \eqref{Intro:intext:centrallink}, where again $\tilde{R}\equiv\tilde{R}(\tilde{f}, (\tilde{S}_s,\tilde{S}_t))$ is some function which depends only on $\tilde{f}$ and the distribution of $(\tilde{S}_s,\tilde{S}_t)$. Using the $C^{2}$-invertibility of $\tilde{f}$, the IC-properties of both $\tilde{S}$ and $S$ allow us to derive from \eqref{Intro:intext:centrallink2} via \eqref{rem:TrafoProjProbDens:eq0A} a (deterministic) system of functional equations 
\begin{equation}\label{Intro:intext:centrallink3}
\Gamma(\varrho, (\tilde{S}_s,\tilde{S}_t),(S_s,S_t)) = 0 \qquad\text{for}\qquad\varrho\coloneqq\restr{(\tilde{f}^{-1}\circ f)}{D_S}
\end{equation}
which involves the partial derivatives of the `mixing residual' $\varrho$ and is otherwise completely determined by the distributions of $(\tilde{S}_s,\tilde{S}_t)$ and $(S_s,S_t)$.\\[-0.5em]

\noindent
The assumed distributional properties of $(S_s,S_t)$, i.e.\ the temporal structure of $S$ as specified by Definition \ref{def:psG_nonsep_stochproc}, together with the required IC-property of $\tilde{S}$ are then sufficient to infer from  \eqref{Intro:intext:centrallink3} that the residual $\varrho$ must be `monomial' in the sense of Definition \ref{def:monomial_trafos}.\\[-0.5em]

In other words, we obtained the following: Given a $C^{2}$-invertible map $\tilde{f}$, we have that:
\begin{gather}\label{Intro:intext:generalisedComon:eq0}
(\tilde{S}^1, \cdots, \tilde{S}^d)\equiv\tilde{S}= \tilde{f}^{-1}(X) \, = \, \big[P\circ(\alpha_1\times\cdots\times \alpha_d)\big](S)\\\label{Intro:intext:generalisedComon}
\text{for some $P\in\mathrm{P}_d$ \ and monotone } \ \alpha_1,\ldots, \alpha_d \ \quad \text{ \emph{if and only if}}\\
\quad \text{the component processes } \quad \tilde{S}^1, \ldots, \tilde{S}^d \ \quad \text{ are mutually independent}.
\end{gather} 
The characterisation \eqref{Intro:intext:generalisedComon}, formulated as Theorem \ref{thm:NICA_stat}, can thus be read as a natural extension of Comon's classical independence criterion \eqref{thm:Comon:eq1} to nonlinear mixtures of IC stochastic processes whose temporal structure is sufficiently regular.\\[-0.5em]

\noindent
Additional source conditions that qualify $S$ for the characterisation \eqref{Intro:intext:generalisedComon} are obtained by `unfreezing' the above time pair $(s,t)\in\Delta_2(\I)$, see Theorem \ref{cor:NICA_MainCor} in Section \ref{sect:beta_and_gamma}.\\[-0.5em]  

\noindent
Analogous to how Comon's criterion \eqref{thm:Comon:eq1} became practically applicable by way of \eqref{cor:Comon:eq1}, our extended criterion \eqref{Intro:intext:generalisedComon} is clearly equivalent to the optimisation-based procedure (cf.\ Thm.\ \ref{thm:optimisation})
\begin{equation}\label{Intro:intext:generalised_contrast}
\begin{gathered}
\left[\underset{\tilde{g}\in\Theta}{\operatorname{arg\ min}}\ \phi\big(\tilde{g}(X)\big)\right]\cdot X \ \subseteq \ \mathrm{DP}_d\cdot S \quad \text{ for any }\quad \phi : \mathcal{M}_1(\mathcal{C}_d)\rightarrow\mathbb{R}_+\\[-0.5em] 
\text{ such that: }\quad \phi(\mu)=0 \quad\text{iff}\quad \mu=\mu^1\otimes\cdots\otimes\mu^d,
\end{gathered} 
\end{equation}for $\Theta$ some `large enough' family of $C^{2}$-invertible candidate transformations, and $\DP$ a nonlinear analogon of the family of monomial matrices $\mathrm{M}_d$ (Definition \ref{def:monomial_trafos}). 

Based on a `moment-like' coordinate description for (the laws of) stochastic processes, we propose an efficiently computable such objective $\phi$ that generalises Comon's original contrast \eqref{intext:Comon:CF} from random vectors to stochastic processes (Section \ref{sect:independence_criterion}).           
\subsection{Main Theorem}\label{sect:main_thm}This section forms the heart of our identifiability theory.\\[-0.5em]

\noindent
We seek to recover the source $S=(S^1,\cdots,S^d)$ from its nonlinear mixture $X$ in \eqref{intext:BSS_relation} up to a minimal deviation, namely a permutation and monotone scaling of its coordinates $S^1,\ldots, S^d$.

The following nonlinear analogue of the family of monomial matrices makes this precise.          
\begin{definition}[Monomial Transformations]\label{def:monomial_trafos}
Given a subset $G$ of $\R^d$, a map $\varrho : \R^d\rightarrow\R^d$ will be called \emph{monomial on $G$} if for each connected component $\tilde{G}$ of $G$ we have that  
\begin{equation}\label{Intro:intext:monomial}
\restr{\varrho}{\tilde{G}} \, = \, P\circ(\alpha_1\times\cdots\times \alpha_d) \quad\text{for \ $P\in\mathrm{P}_d$ \ and \ $\alpha_i\in \mathrm{Diff}^1(\pi_i(\tilde{G}))$}.  
\end{equation}
(The above differentiability condition is considered void at isolated points of $\pi_i(\tilde{G})$.) We write $\DP(G)$ for the family of all functions on $\R^d$ which are monomial on $G$.  
\end{definition}\vspace{-0.5em}    
Accordingly, we say that any two paths $\tilde{x}$ and $x$ in $\mathcal{C}_d$ coincide up to a permutation and monotone scaling of their coordinates, in symbols: 
\begin{equation}\label{intext:monomial_orbit}
\tilde{x} \ \in \ \DP\cdot x,
\end{equation}if $(\tilde{x}_t)_{t\in\I} = (\varrho(x_t))_{t\in\I}$ for some $\varrho\in\DP(\mathrm{tr}(x))$, where $\mathrm{tr}(x)\equiv\bigcup_{t\in\I}x_t$ is the \emph{trace} of $x$.\\[-0.5em] 

\noindent
Definition \ref{def:PseudoGaussian} describes analytical forms that  need to be avoided by `sufficiently many' of the distributions constituting its temporal structure \eqref{2ndfdds} if the source $S$ is to be identifiable from $X$ up to a monomial transformation. Sources for which this is the case will be given the following label of regularity (or `non-degeneracy').      
\begin{definition}[$\alpha$-Contrastive]\label{def:psG_nonsep_stochproc}
A continuous stochastic process $S\equiv(S^1_t,\cdots, S^d_t)^\intercal_{t\in\I}$ in $\R^d$ with spatial support $D_S$ will be called \emph{$\alpha$-contrastive} if $S$ is IC and there is a collection of time-pairs $\mathcal{P}$ in $\Delta_2(\I)$ and an associated collection $(D_\p)_{\p\in \mathcal{P}}$ of open subsets of $\R^d$ such that 
\begin{enumerate}[label=(\roman*)]
\item\label{def:psG_item1}the union $\bigcup_{(s,t)\in \mathcal{P}}D_{(s,t)}$ is dense in $D_S$, and
\item\label{def:psG_item2}for each $(i,(s,t))\in [d]\times \mathcal{P}$ it holds that $S^i$ is $C^2$-regular at $(s, t)$ with density $\zeta^i_{s,t}$, and 
\begin{equation*}
\begin{aligned} 
&\restr{\zeta^i_{s, t}}{D_{(s,t)}^{\times 2}} \text{ is regularly non-separable for all $i\in[d]$, \quad and}\\[-0.5em]
&\restr{\zeta^i_{s, t}}{D_{(s,t)}^{\times 2}} \text{ is almost everywhere non-Gaussian for all but at most one $i\in[d]$,} 
\end{aligned}
\end{equation*}where the above restrictions of the densities are understood wrt.\ the abuse of notation $\zeta^i_{s,t}(x)\coloneqq\zeta^i_{s,t}(x_i, x_{i+d})$ for $x=(x_\nu)\in\mathbb{R}^{2d}$. (For notational convenience, this abuse of notation is kept throughout the following.)      
\end{enumerate}  
\end{definition} 
\noindent
Notice that the conditions in Definition \ref{def:psG_nonsep_stochproc} \ref{def:psG_item2} reflect the classical pathologies \ref{intext:ICA_pathologies2} and \ref{intext:ICA_pathologies1} from p.\ \pageref{intext:ICA_pathologies1}. Further below we will see how the assumptions of Definition \ref{def:psG_nonsep_stochproc} are linked to related works (Section \ref{rem:hyvarinenmorioka}) and that they are satisfied for a number of popular copula-based time series models (Section \ref{sect:copulas}). Recall that the following operates under Assumption \ref{sect:assumption_convexity}.   
\begin{theorem}\label{thm:NICA_stat}
Let the process $S$ in \eqref{intext:BSS_relation} be $\alpha$-contrastive. Then for any transformation $h$ which is $C^{2}$-invertible on some open superset of $D_X$, we have with probability one that:  
\begin{equation}\label{thm:NICA_stat:eq3}
h(X) \ \in \ \mathrm{DP}_{\!d}\cdot S \quad\text{if and only if}\quad h(X) \ \text{ has independent components}.
\end{equation}  
\end{theorem} 
\begin{proof} 
The `only-if'-direction in \eqref{thm:NICA_stat:eq3} is clear, so we only need to show the converse implication. To this end, we in total prove the slightly stronger assertion that 
\begin{equation}\label{thm:NICA_stat:aux0}
\text{If } \quad h(X)\text{ is IC } \ \text{ and } \ D\equiv D_{(s,t)} \text{ as in Def.\ \ref{def:psG_nonsep_stochproc}},\quad\text{ then }\quad \left\{J_{h\circ f}(u)\mid u\in D\right\}\subseteq \operatorname{M}_d. 
\end{equation}Given \eqref{thm:NICA_stat:aux0} (and Definition \ref{def:psG_nonsep_stochproc} \ref{def:psG_item1}), the assertion \eqref{thm:NICA_stat:eq3} follows by way of Lemma \ref{lem:monomial_trafos} \ref{lem:monomial_trafos:it2} and the fact that the trace of almost every realisation of $S$ is contained in a connected component of $D_S$ (Lemma \ref{lem:spat_supp} \ref{lem:spat_supp:it2}) which in turn is convex by Assumption \ref{sect:assumption_convexity}.\\[-0.5em]      

\noindent
Let now $(s,t)\in\Delta_2(\I)$ be as in Definition \ref{def:psG_nonsep_stochproc} \ref{def:psG_item2}, i.e.\ suppose that $(S_s, S_t) = \pi_{(s,t)}(S)$ admits a (joint) $C^2$-density $\zeta=\zeta_1\cdots\zeta_d$ (where $\zeta_i\equiv\zeta^i_{s,t}$) with a support $\bar{D}\coloneqq\mathrm{supp}(\zeta)\subseteq\mathbb{R}^{2d}$ whose boundary $\partial\bar{D}$ is a Lebesgue nullset (cf.\ Remark \ref{rem:lebesgue_boundary}). 

Moreover, let $X_t^\ast$ be a copy of $X_t$ which is independent of $(X_s, X_t)$, and denote
\begin{equation}\label{thm:NICA_stat:aux1}
Y\coloneqq (X_s, X_t)\qquad\text{ and }\qquad Y^\ast\coloneqq(X_s, X_t^\ast).
\end{equation}
For $C\sim\operatorname{Ber}(1/2)$ and independent of $Y$ and $Y^\ast$, consider further
\begin{equation}
\bar{Y}\coloneqq C\cdot Y + (1-C)\cdot Y^\ast
\end{equation}
(so that $\mathbb{P}_{\bar{Y}} = \tfrac{1}{2}\mathbb{P}_Y + \tfrac{1}{2}\mathbb{P}_{Y^\ast}$) together with the associated regression function
\begin{equation}\label{thm:NICA_stat:aux3}
\rho\,:\,\mathbb{R}^{2d}\rightarrow [0,1] \quad\text{given by}\quad \rho(y)\coloneqq\mathbb{E}[C\,|\,\bar{Y}=y].
\end{equation}
The function $\rho$ then satisfies the following central equation. 
\begin{restatable}{lemma}{lemLogReg}\label{lem:LogReg}For $\mu$ the probability density of $Y$, and $\mu^\ast$ the probability density of  $Y^\ast$,
\begin{equation}\label{thm:NICA_stat:aux4}
\psi\circ\rho \ = \ \log \mu - \log \mu^\ast \quad\text{ a.e. on }\quad \tilde{D}\coloneqq\supp(\mu)
\end{equation} 
for the logit-function $\psi(p)\coloneqq\log(p/(1-p))$.
\end{restatable}

\noindent
The proof of Lemma \ref{lem:LogReg} is given in Appendix \ref{pf:lem:LogReg}. Recalling now that the components of $S$ are mutually independent, we obtain from the transformation formula for densities \eqref{rem:TrafoProjProbDens:eq0A} that for the inverse $g\equiv(g_1,\cdots,g_d)\coloneqq f^{-1}$ and the density $\zeta^i_1$ of $S^i_s$, resp.\ the density $\zeta^i_2$ of $S^i_t$, 
\begin{equation}\label{thm:NICA_stat:aux7}
\log \mu - \log \mu^\ast \ = \ \sum_{i=1}^d\big[\log\zeta_i\circ(g_i\times g_i) - \log \zeta^i_1\circ g_i(u) - \log\zeta^i_2\circ g_i(v)\big]
\end{equation}  
almost everywhere on $\tilde{D}\,(=(f\times f)(\bar{D}))$. Using \eqref{thm:NICA_stat:aux4}, it follows that
\begin{equation}\label{thm:NICA_stat:aux8}
\psi\circ\rho \ = \ \sum_{i=1}^d P_i\circ(g_i\times g_i)\quad\text{ for }\quad P_i\coloneqq\log\zeta_i - \sum_{\nu=1,2}\log\zeta^i_\nu\circ\pi_\nu.
\end{equation} 
Let now $h\equiv(h_1, \cdots, h_d)\in\operatorname{Diff}^2(\mathcal{O}_X)$, for some $\mathcal{O}_X\supseteq D_X$ open, be such that the process $\tilde{S}\coloneqq h(X)$ has independent components. Using that the above function $\psi\circ\rho$ depends on the observable $X$ only, we due to $(\tilde{S}_s, \tilde{S}_t) = (h\times h)(X_s, X_t)$ and \eqref{rem:TrafoProjProbDens:eq0A} obtain that
\begin{equation}\label{thm:NICA_stat:aux9}
\psi\circ\rho \ = \ \sum_{i=1}^d Q_i\circ(h_i\times h_i) \quad \text{ a.e. on } \ \tilde{D}
\end{equation}  
analogous to \eqref{thm:NICA_stat:aux8}, where the functions\footnote{\ Note that here, we employ the abuse of notation $Q_i(x) \equiv Q_i(x_i, x_{i+d})$ for $x=(x_\nu)\in \bar{D}'$.} $Q_i\in C^2(\bar{D}')$, $i\in[d]$, are given as 
\begin{equation}\label{thm:NICA_stat:aux9.1}
Q_i\coloneqq \log\tilde{\zeta}_i - \sum_{\nu=1,2}\log\tilde{\zeta}^i_\nu\circ\pi_\nu\quad\text{ with }\quad \tilde{\zeta}_i\coloneqq\frac{\mathrm{d}\mathbb{P}_{({\tilde{S}_s}^i, {\tilde{S}_t}^i)}}{\mathrm{d}(u,v)}\ \text{ and } \ \tilde{\zeta}^i_\nu\coloneqq\frac{\mathrm{d}\mathbb{P}_{{{\tilde{S}_{r_\nu}}^{\hspace{-0.5em}i}}}}{\mathrm{d}u} 
\end{equation}for $r_1\coloneqq s$ and $r_2\coloneqq t$, and where $\bar{D}'\subseteq\mathbb{R}^{2d}$ denotes the support of $\tilde{\zeta}\equiv\tilde{\zeta}_1\cdots\tilde{\zeta}_d$. 

Note that the $Q_i$ are indeed twice continuously differentiable: By \eqref{rem:TrafoProjProbDens:eq0A} we have
\begin{equation*}  
\tilde{\zeta} = \frac{\mathrm{d}\mathbb{P}_{(\tilde{S}_s, \tilde{S}_t)}}{\mathrm{d}(u,v)} = |\!\det(J_{\phi})|\cdot\big[\zeta\circ\phi \big]\ \in \ C^1(\bar{D}')
\end{equation*}   
for the $C^2$-density $\zeta$ and for $\phi\coloneqq((h\circ f)\times(h\circ f))^{-1}\in\operatorname{Diff}^2(\mathcal{O}_{\tilde{S}}^{\times 2};\mathcal{O}_{S}^{\times 2})$, with $\mathcal{O}_{\tilde{S}}\coloneqq h(\mathcal{O}_S)$; reading off the marginal densities $\tilde{\zeta}_i$, $\tilde{\zeta}^i_\nu$, cf.\ \eqref{rem:TrafoProjProbDens:eq1}, we see that the Jacobians appearing in \eqref{thm:NICA_stat:aux9.1} cancel out as they did in \eqref{thm:NICA_stat:aux7}, giving us $Q_i\in C^2(\bar{D}')$ as desired. 

Combining the identities \eqref{thm:NICA_stat:aux8} and \eqref{thm:NICA_stat:aux9} yields that 
\begin{equation}\label{thm:NICA_stat:aux10}
\sum_{i=1}^d Q_i\circ(h_i\times h_i) \ = \ \sum_{i=1}^d P_i\circ(g_i\times g_i)
\end{equation}  
everywhere on the dense open subset $D_\mu\coloneqq\{\mu>0\}$ of $\tilde{D}$.

Therefore, the desired implication \eqref{thm:NICA_stat:aux0} -- and hence the assertion of the theorem (see the initial remarks of this proof) -- holds if we can show \eqref{thm:NICA_stat:aux10} to imply that for $\varrho \coloneqq h\circ f$ we have 
\begin{equation}\label{thm:NICA_stat:aux11}
\{J_\varrho(u)\mid u\in D\} \ \subseteq \ \operatorname{M}_d \quad \text{for each open \ $D\subseteq D_S$ \ as in Def.\ \ref{def:psG_nonsep_stochproc} \ref{def:psG_item2}},
\end{equation}i.e.\ for any (non-empty) open subset $D$ of $\R^d$ for which $\restr{\zeta^i}{D^{\times 2}}$ is regularly non-separable for all $i\in[d]$, and a.e.\ non-Gaussian for all but at most one $i\in[d]$. 
Let any such $D$ be fixed. 

The remainder of this proof is aimed at deriving \eqref{thm:NICA_stat:aux11} from \eqref{thm:NICA_stat:aux10}. To this end, notice that since \eqref{thm:NICA_stat:aux10} can be equivalently written as  
\begin{equation*} 
Q\circ(h\times h) \ = \ P\circ(g\times g)
\end{equation*}    
for $Q\coloneqq \varsigma\circ(Q_1\times\cdots\times Q_d)\circ\tau$ and $P\coloneqq \varsigma\circ(P_1\times\cdots\times P_d)\circ\tau$ with $\varsigma(y_1, \ldots, y_d)\coloneqq \sum_{i=1}^d y_i$ and $\tau(x_1, \ldots, x_{2d})\coloneqq(x_1, x_{d+1}, x_2, x_{d+2}, \ldots, x_d, x_{2d})$, we obtain that \eqref{thm:NICA_stat:aux10} is equivalent to $Q\circ (\varrho\times\varrho) = P$, i.e.\ to the ($D_\zeta\coloneqq\{\zeta>0\}$--everywhere) identity\footnote{\ Once more, we abuse notation by writing $P_i(x)\equiv P_i(x_i, x_{i+d})$ ($x\in D$) for the RHS of \eqref{thm:NICA_stat:aux12}.}
\begin{equation}\label{thm:NICA_stat:aux12}
\sum_{i=1}^d Q_i\circ(\varrho_i\times\varrho_i) \ = \ \sum_{i=1}^d P_i.
\end{equation} 
The above is an identity between two twice-continuously-differentiable functions in the arguments $(u_1, \ldots, u_d, v_1, \ldots, v_d)\in D_\zeta\subseteq\mathbb{R}^{2d}$, so we can apply the cross-derivatives $\partial_{u_j}\partial_{v_k}$ to both sides of \eqref{thm:NICA_stat:aux12} to arrive at the identities     
\begin{equation}\label{thm:NICA_stat:aux13}
\sum_{i=1}^d\big[q_i\circ(\varrho_i\times\varrho_i)\big]\cdot\partial_{u_j}\varrho_i\cdot\partial_{v_k}\varrho_i \ = \ \sum_{i=1}^d \xi_i\cdot\delta_{ijk}\qquad(j,k\in[d])
\end{equation}
where the $\varrho_i$ are the components of \eqref{thm:NICA_stat:aux11} and the functions $q_i$ and $\xi_i$ are given as 
\begin{equation}\label{thm:NICA_stat:aux14.0}
q_i\,\coloneqq\,\partial_{u_i}\partial_{v_i}Q_i\qquad\text{ and }\qquad \xi_i\,\coloneqq\,\partial_{u_i}\partial_{v_i}P_i = \partial_{u_i}\partial_{v_i}\!\log\zeta_i, 
\end{equation}
respectively. (Note that $\partial_{u_j}\partial_{v_k}R_i = r_i\cdot\delta_{ijk}$ ($(R,r)\in\{(Q,q), (P,\xi)\}$) by the Cartesian product-form of the functions \eqref{thm:NICA_stat:aux8} and \eqref{thm:NICA_stat:aux9.1}.) 
Observe now that the system of equations \eqref{thm:NICA_stat:aux13} can be equivalently expressed as the congruence relation
\begin{equation}
J_\varrho^\intercal\cdot\Lambda_q\cdot J_\varrho \ = \ \Lambda_\xi \qquad\left(:\Leftrightarrow\ J_\varrho^\intercal(u)\cdot\Lambda_q(u,v)\cdot J_\varrho(v) = \Lambda_\xi(u,v)\right)
\end{equation}
for $J_\varrho$ the Jacobian of $\varrho$ and for $\Lambda_q, \Lambda_\xi$ defined as the matrix-valued functions
\begin{equation}
\Lambda_q\coloneqq\mathrm{diag}_{i=1,\ldots,d}[q_i\circ(\varrho_i\times\varrho_i)] \quad\text{ and }\quad\Lambda_\xi\coloneqq\mathrm{diag}_{i=1,\ldots,d}[\xi_i].   
\end{equation} 
Since $\varrho$ is a diffeomorphism over $\bar{D}$, its Jacobian $J_\varrho$ is invertible and hence 
\begin{equation}\label{thm:NICA_stat:aux15}
\Lambda_q \ = \ B_{\varrho}^\intercal\cdot\Lambda_\xi\cdot B_{\varrho}\quad\text{ on }\quad D_\zeta,\qquad\text{for }\quad B_\varrho\coloneqq J_\varrho^{-1}.
\end{equation}       
Since $B_{\varrho} = J_{\varrho^{-1}}\circ\varrho$ by the inverse function theorem, the matrix-valued function $B_\varrho$ is clearly continuous. Hence\footnote{\ Notice that $D^{\times 2}\subset\bar{D}\equiv\supp(\zeta)$ (and hence $D^{\times 2}\subseteq D_\zeta$, as $D$ is open) since $\restr{\zeta}{D^{\times 2}}>0$ a.e.\ by the fact that $\restr{\zeta}{D^{\times 2}}$ is a.e.\ non-separable (and hence a.e.\ non-zero in particular).} we can apply Lemma \ref{lem:PermDiagForThm} below to from \eqref{thm:NICA_stat:aux15} and the assumptions of Definition \ref{def:psG_nonsep_stochproc} \ref{def:psG_item2} obtain as desired that 
\begin{equation}\label{thm:NICA_stat:aux16}
\{J_{\varrho}(u)\mid u\in D\} \ \subseteq \ \operatorname{M}_d. 
\end{equation} 
Indeed, since the above open set $D\subseteq D_S$ has been chosen such that the (positive) functions $\restr{\zeta^i}{D^{\times 2}}$ are regularly non-separable for each $i\in[d]$ and a.e.\ non-Gaussian for all but at most one $i\in[d]$ (Definition \ref{def:psG_nonsep_stochproc} \ref{def:psG_item2}), Lemma \ref{lem:PermDiagForThm} is clearly applicable to the system \eqref{thm:NICA_stat:aux15}, providing \eqref{thm:NICA_stat:aux16} as required. But since the above set $D$ was chosen without further restrictions, \eqref{thm:NICA_stat:aux16} amounts to \eqref{thm:NICA_stat:aux11} and hence proves Theorem \ref{thm:NICA_stat} as desired.
\end{proof} 
The following section extends the above line of argument to additional types of sources.

\subsection{An Extension to Sources of Alternative Temporal Structures}\label{sect:beta_and_gamma} 
We can generalise the strategy behind Theorem \ref{thm:NICA_stat} by `unfreezing' its usage of the temporal structure \eqref{2ndfdds}, that is by allowing the considered time-pairs $(s,t)$ to `vary more freely' across $\Delta_2(\I)$; see Lemma \ref{lem:jacobian_system}. This qualifies additional source classes for nonlinear identification via the characterisation \eqref{thm:NICA_stat:eq3}. As before, the technical key for this is to make the Jacobian $J_\varrho$ of the mixing residual (cf.\ \eqref{thm:NICA_stat:aux15}) serve as change of basis for a source-dependent matrix function with non-degenerate eigenspectrum. The next definition formulates two sufficient conditions for this.\\[-0.5em]         

\noindent
Define $\displaystyle\psi(x,y,z)\coloneqq x^{-2}yz$, and denote by $\nabla^{\times}\coloneqq\big\{(\lambda_\nu)\in\R^d\ \big| \ \exists\, i,j\in[d],\,i\neq j \ : \ \lambda_i=\lambda_j\big\}$ the set of all vectors in $\R^d$ whose coordinates are not pairwise distinct.
\begin{definition}[$\{\beta, \gamma\}$-Contrastive]\label{def:log-regular}
A continuous stochastic process $S=(S^1_t,\ldots,S^d_t)_{t\in\I}$ in $\R^d$ with independent components and spatial support $D_S$ will be called
\begin{itemize}
\item\mbox{\emph{$\beta$-contrastive} if $D_S$ is the closure of its interior and for any open subset $U$ of $D_S$ there is}
\begin{gather}\notag
\text{an open subset $\tilde{U}$ of $U$ \quad and \quad $\p\equiv(s,t),\,\p'\in\Delta_2(\I)$ \quad such that, \quad for all $i\in[d]$,}\\[-0.5em] \notag
\text{the density $\zeta^i_{s,t}$ of $(S^i_s, S^i_t)$, \ likewise $\zeta^i_{\p'}$, \ exists with \ $\zeta^i_\p, \zeta^i_{\p'}\in C^2(\tilde{U}^{\times 2})$ \ and}\\ \label{def:log-regular:eq1}
\begin{align}
\xi_{s,t}^{i\,|\,\tilde{U}}\coloneqq\big[\partial_{x_i}\partial_{x_{i+d}}\log \zeta_{s,t}^i\big]\circ\iota_{\tilde{U}} \, \neq \, 0 \ \quad &\text{and}\quad \xi_{\p'}^{i\,|\,\tilde{U}} \neq 0 \quad \text{ (a.e.)}, \qquad\text{ and}\\[-0.5em]
\xi_{\p'}^{i\,|\,\tilde{U}} \, \notin \, \big\langle\xi_{\p}^{i\,|\,\tilde{U}}\big\rangle_\R &\coloneqq \big\{c\cdot\xi_{\p}^{i\,|\,\tilde{U}} \, \big| \, c\in\R\big\}\label{def:log-regular:eq2} 
\end{align} 
\end{gather}
with $\iota_{\tilde{U}} : \tilde{U}\ni u \mapsto (u,u)\in \Delta_{\tilde{U}}$ and both $U,\tilde{U}$ non-empty;\footnote{\ Here as before, we abuse notation by writing $\zeta_{s,t}^i(x)=\zeta_{s,t}^i(x_i,x_{i+d})$ for $x=(x_\nu)\in\R^{2d}$.}\vspace{0.5em}
\item\emph{$\gamma$-contrastive} if there is a dense open subset $\mathcal{U}$ of $D_S$ for which the following holds: 
\begin{gather}\notag
\label{def:gamma_contrastive:eq1}\text{for each $u\in\mathcal{U}$ \quad there exists $(v,\p_0,\p_1,\p_2)\in\R^d\times\Delta_2(\I)^{\times 3}$ \quad such that}\\[-0.5em]
S \text{ is $C^2$-regular around $(\p_0,(u,v)), (\p_1,(u,u))$ and $(\p_2,(v,v))$, \quad and}\\[-0.5em]\label{def:gamma_contrastive:eq2} 
\big(\psi(\xi^i_{\p_0}(u,v),\xi^i_{\p_1}(u,u),\xi^i_{\p_2}(v,v))\big)_{i\in[d]} \in \big(\mathbb{R}^d\setminus\nabla^{\times}\big),
\end{gather}where $\xi^i_\p\coloneqq\partial_{x_i}\partial_{x_{i+d}}\log\zeta_{\p}$ is the mixed log-derivatives of the $C^2$-density $\zeta_\p^i$ of $(S_s^i,S_t^i)$. 
\end{itemize} 
\end{definition} 
We will see that the assumptions of $\gamma$-contrastivity are satisfied for a number of popular stochastic processes (Section \ref{sect:Gaussian_processes}).
\begin{remark}[Relation Between $\alpha$-, $\beta$- and $\gamma$-Contrastive Sources]Notice that every $\alpha$-contrastive process is also $\gamma$-contrastive (for $\p_0=\p_1=\p_2$, as the proof of Theorem \ref{thm:NICA_stat} shows), while $\beta$-contrastivity does not imply---nor is it implied by---either $\alpha$- or $\gamma$-contrastivity.
\end{remark}
Recall that the following theorem operates under Assumption \ref{sect:assumption_convexity}.
\begin{theorem}\label{cor:NICA_MainCor}
Let the process $S$ in \eqref{intext:BSS_relation} be $\beta$- or $\gamma$-contrastive. Then for any transformation $h$ which is $C^{2}$-invertible on an open superset of $D_X$, we have with probab.\ one that: 
\begin{equation}\label{cor:NICA_MainCor:eq1}
h(X) \ \in \ \mathrm{DP}_d\cdot S \quad\text{if and only if}\quad h(X) \text{ has independent components}.
\end{equation}  
\end{theorem}
\begin{proof} 
Let $h$ be $C^{2}$-invertible on some open superset of $D_X$ and such that $h(X)$ has independent components; the proof of \eqref{cor:NICA_MainCor:eq1} is an extension of the proof of Theorem \ref{thm:NICA_stat}, so let us adopt the set-up and notation of the latter (as done in Lemma \ref{lem:jacobian_system}). Recall from there (cf.\ \eqref{thm:NICA_stat:aux11}) that \eqref{cor:NICA_MainCor:eq1} follows if we can find a dense open subset $\mathcal{D}$ of $D_S$ such that      
\begin{equation}\label{cor:NICA_MainCor:aux3}
B_\varrho(u)\, \in \, \mathrm{M}_d \quad\text{ for each } \ u\in\mathcal{D}.
\end{equation} 
Suppose first that $S$ is $\gamma$-contrastive. In this case, we claim that \eqref{cor:NICA_MainCor:aux3} holds for the dense subset $\mathcal{D} \equiv\mathcal{U}$ of $D_S$ postulated by Def.\ \ref{def:log-regular}. To see that this is true, fix any $u\in\mathcal{U}$ and recall that, by Lemma \ref{lem:jacobian_system} and the previous discussions, the Jacobian $B_\varrho(u)$ of $\varrho\equiv h\circ f$ at $u$ is monomial if there is $v\in\mathbb{R}^d$ and $\p_0, \p_1, \p_2\in\Delta_2(\mathbb{I})$ with $(u,v)\in\{\xi^1_{\p_0}\neq 0, \ldots,\xi^d_{\p_0}\neq 0\}$ for which the diagonal matrix $\bar{\Lambda}_{u,v}\equiv\bar{\Lambda}_{\p_0,\p_1,\p_2}(u,v)$ given by \eqref{lem:jacobian_system:eq2} has pairwise distinct eigenvalues. (Note that $(u,v)\in\{\xi^1_{\p_0}\neq 0, \ldots,\xi^d_{\p_0}\neq 0\}$ if $\{\psi(\xi^i_{\p_0}(u,v),\alpha,\beta)\mid i\in[d]\}\subset\R$ for some $\alpha,\beta\in\R$.) Since the diagonal of $\bar{\Lambda}_{u,v}$ equals the vector $(\psi(\xi^i_{\p_0}(u,v),\xi^i_{\p_1}(u,u),\xi^i_{\p_2}(v,v))_{i\in[d]}$, choosing $(v,\p_0,\p_1,\p_2)$ as in \eqref{def:gamma_contrastive:eq1}, \eqref{def:gamma_contrastive:eq2} thus yields $B_\varrho(u)\in\mathrm{M}_d$ as claimed. As $u\in\mathcal{U}$ was arbitrary, we obtain $\restr{B_\varrho}{\mathcal{U}}\subset\mathrm{M}_d$ as desired in \eqref{cor:NICA_MainCor:aux3}. This proves \eqref{cor:NICA_MainCor:eq1} for $\gamma$-contrastive sources.

The $\beta$-contrastive case is more technical and hence deferred to Appendix \ref{cor:NICA_MainCor:beta}.              
\end{proof} 

\subsection{Related Work}\label{rem:hyvarinenmorioka}
We remark that the above assumption of $\alpha$-contrastivity is strictly weaker than the earlier identifiability conditions given in \citep{HYM} which served us as motivation. Indeed: the latter are defined for densities on $G=\R^2$ only, and if such a density $\varsigma$ is ``uniformly dependent'' in the sense of \citep[Def.\ 1]{HYM} then $\varsigma$ is also strictly (and regularly) non-separable on $G=\R^2$ by Lemma \ref{lem:C2PseudoGaussian} \ref{lem:C2PseudoGaussian:it2}; as to \citep{HYM}'s complementary notion of $\varsigma$ being ``quasi-Gaussian'' \citep[Def.\ 2]{HYM}, we thank one of our referees for drawing our attention to the fact\footnote{\ The insufficiency of \citep[Thm.\ 1, Assmpt.\ 3.]{HYM} (for \citep[Def.\ 2 eq.\ (4)]{HYM} as stated) was also conjectured in \citep[(end of) Section 4.3]{halva2021disentangling}; we prove this conjecture true in Appendix \ref{rem:hyvarinenmorioka2}.} that, as stated in loc.\ cit., the corresponding non-separability condition \citep[Thm.\ 1, Assmpt.\ 3.]{HYM} fails to ensure the validity of \citep[Thm.\ 1]{HYM}, see Appendix \ref{rem:hyvarinenmorioka2}; this deficiency can be remedied, however, if one weakens the excluding notion of quasi-Gaussianity \citep[Def.\ 2]{HYM} by imposing its defining factorisation condition \citep[Def.\ 2 eq.\ (4)]{HYM} to hold merely on some open subset of $\R^2$ instead of globally on all of $\R^2$ (cf.\ Appendix \ref{rem:hyvarinenmorioka2}), as is done -- upon logical negation -- in Definition \ref{def:PseudoGaussian}, eq.\ \eqref{def:PseudoGaussian:eq2}, and also in the later work \citep[Theorem 2]{halva2021disentangling}. With \citep[Def.\ 2]{HYM} thus weakened, the (thus strengthened) identifiability condition \citep[Thm.\ 1, Assmpt.\ 3.]{HYM} then becomes a special case of our assumption of pseudo-Gaussianity (Definition \ref{def:PseudoGaussian}) by Lemma \ref{lem:C2PseudoGaussian} \ref{lem:C2PseudoGaussian:it3}. Consequently, if a source $S\equiv(S^i)$ satisfies \citep[Hypotheses 1., 2.\ $\&$ 3.\ of Theorem 1]{HYM} -- that is if $S$ is stationary and $C^2$-regular at some point $(s_0,t_0)\equiv(t-1,t)$ with $D_{(s_0,t_0)}=\R^d\,(=D_S)$ such that the densities $\varsigma^i\equiv\varsigma^i_{s_0,t_0}$ of $S^i$ are all uniformly dependent (hence all regularly non-separable) and none quasi-Gaussian in the above, corrected sense (cf.\ also \citep[Assmpt.\ B2]{halva2021disentangling}) -- then $S$ is clearly $\alpha$-contrastive in particular, and the converse is clearly not true in general.\\[-0.5em] 

\noindent
When contrasted with the few prior works in the area that allow for a theoretical comparison, most notably \citep{TCL,HYM}, we see that our approach provides a strict generalisation of previously attained results, see above, or yields stronger conclusions while operating under assumptions which are much less restrictive; for example, we do not require the source to belong to a predefined distributional family as, e.g., in \citep{TCL}.\\[-0.5em]

\noindent
With regards to methodology, we recall that \cite{HYM} propose to estimate the demixing nonlinearity by training a universal approximator (typically a neural network) to distinguish between vectors excerpting originally-ordered data and vectors excerpting data whose initial sequential order has undergone a random permutation. By implementing this classification task via logistic regression, an approximation of the demixing transformation is then obtained as an optimally trained configuration of the classifying universal approximator provided that the employed regression function is of a certain composite functional form. 

In contrast, our approach approximates the demixing nonlinearity more directly via a dependence minimisation task in the classical spirit of Comon \citep{COM}, which we propose to perform by optimising an explicitly defined, universally applicable contrast function derived from novel signature-based statistics for multidimensional stochastic processes (Section \ref{sect:independence_criterion}). Not only is our method thus guaranteed to work under much weaker assumptions than \cite{HYM} --- see the above discussion and the facts that our method is fully applicable to the (non-stationary) discrete- and continuous-time case and free of assumptions on the functional form of any approximating auxiliary nonlinearities; its equivalence to a simple-to-formulate optimisation problem also makes our method straightforward to implement and more directly accessible to a theoretical analysis of its statistical properties, cf.\ Sections \ref{sect:consistency} and \ref{sect:experiments}.\\[-0.5em] 

\noindent
We also note that a slightly weaker technical modification of our assumptions $\mathrm{(i)\,\&\,(ii)}$ from Definition \ref{def:psG_nonsep_stochproc} is given and used in the later work \citep[Theorem 2]{halva2021disentangling}, where the problem of nonlinear blind source separation is studied in the presence of independent additive noise. To the best of our knowledge, our notions of $\beta$- or $\gamma$-constrastivity (Definition \ref{def:log-regular} (and \ref{def:alpha_bar_contrastive})) bear no evident resemblance to the conditions proposed in this or other works.\\[-0.5em]    

\section{Examples of Applicable Sources}\label{sect:example_sources}
\noindent
The statistical non-degeneracy assumptions of $\alpha$-, $\beta$- or $\gamma$-contrastivity hold for a number of well-established models for stochastic signals, among them most popular copula-based time series models (Section \ref{sect:copulas}) as well as a variety of Gaussian processes and Geometric Brownian Motion (Section \ref{sect:Gaussian_processes}).       
\subsection{Popular Copula-Based Source Models Are $\alpha$-Contrastive}\label{sect:copulas}
It is well-known (e.g.\ \citep[Sect.\ 2.10]{nelsen2006}, \cite{darsow1992}) that the temporal structure \eqref{2ndfdds} of a scalar stochastic process $S=(S_t)_{t\in\mathbb{I}}$ can be given an analytical representation of the form
\begin{equation}\label{TimeSeriesCopulaModel:eq1}
\zeta_{s,t}(x,y) \ = \  \zeta_s(x)\zeta_t(y)\cdot c_{s,t}(F^S_s(x), F^S_t(y)) \qquad \big((s,t)\in\Delta_2(\I)\big), 
\end{equation}
where $\zeta_{s,t}$ is the probability density of $(S_s,S_t)$, $F^S_r$ is the cdf of the vector $S_r$ with $\zeta_r$ its density, and $c_{s,t}:[0,1]^{\times 2}\rightarrow\R$ is the uniquely determined copula density of $(S_s,S_t)$.

\begin{proposition}\label{prop:TSCopulaProp1}
Let $S\equiv(S_t)_{t\in\mathbb{I}}\equiv(S^1, \cdots, S^d)$ be an IC stochastic process in $\R^d$ such that $S_t$ admits a $C^2$-density $\zeta_t$ for each $t\in\mathbb{I}$ with the property that $t\mapsto\zeta_t(x)$ is continuous for each $x\in\mathbb{R}^d$. Suppose further that for some $\mathcal{P}\subseteq\Delta_2(\I)$ with $\bigcup_{(s,t)\in\mathcal{P}}\{\zeta_s\cdot\zeta_t>0\}$ dense in $D_S$,\footnote{\ Lemma \ref{lem:spat_supp} \ref{lem:spat_supp:it5} guarantees that such a set $\mathcal{P}$ exists.} it holds that the copula densities $\{c^i_{s,t}\mid (s,t)\in\mathcal{P}\}$ of $S^i$ (cf.\ \eqref{TimeSeriesCopulaModel:eq1}) are such that  
\begin{equation}\label{prop:TSCopulaProp1:aux3}
\text{$c^i_{s,t}$ \ are positive and strictly non-Gaussian \quad and \quad $\partial_x\partial_y \log c^i_{s,t}$ vanishes nowhere,} 
\end{equation}
for each $i\in[d]$. Then the process $S$ is $\alpha$-contrastive.
\end{proposition}
\begin{proof}
See Appendix \ref{pf:prop:TSCopulaProp1}.         
\end{proof} 

\noindent
A popular approach in finance, insurance economy and other fields is to read \eqref{TimeSeriesCopulaModel:eq1} as a semi-parametric stationary model for $S=(S_t)_{t\in\I}$ by assuming the existence of some $\mathcal{I}\subset\I$ discrete (`set of observations') such that $\zeta_r\equiv\zeta$ with cdf $F_\zeta$ for each $r\in\mathcal{I}$, and $D_S = \supp(\zeta)$ and $c_{s,t}\equiv c_\theta$ uniformly parametrized for all $(s,t)\in\mathcal{P}\coloneqq \mathcal{I}^{\times 2}\cap\Delta_2(\I)$, see e.g.\ \citep[Sect.\ 2]{chenfanCopula}, \cite{emura2017}:
\begin{equation}\label{TimeSeriesCopulaModel:eq2}
\zeta_{s,t}(x,y) \ = \ \zeta(x)\zeta(y)\cdot c_\theta(F_\zeta(x), F_\zeta(y)), \qquad (s,t)\in\mathcal{P}.
\end{equation} 
We verify exemplarily that a source $S=(S^1, \cdots, S^d)$ in $\R^d$ whose components $S^i$ are modelled according to \eqref{TimeSeriesCopulaModel:eq2} is $\alpha$-contrastive for a number of popular \mbox{copula densities $c_\theta$.}\\[-1em]
\begin{corollary}\label{cor:TSCopulaProp}
Let $S=(S^1,\cdots,S^d)$ be a stochastic process whose independent components $S^i$ are modelled according to \eqref{TimeSeriesCopulaModel:eq2} for each $i\in[d]$ with copula-density $c_i$ belonging to one of the following popular classes:
\begin{enumerate}[label=\upshape(\roman*)]
\item\label{prop:TSCopulaProp1:item1}\emph{(Clayton)}\qquad  
$\begin{gathered}
\displaystyle \ \ \, c_i(x,y) = (1+\theta)(xy)^{(-1-\theta)}(-1 +x^{-\theta} + y^{-\theta})^{(-2-1/\theta)}
\end{gathered}$\\[0.5em]where $\theta\in(-1,\infty)\setminus\{0, -\tfrac{1}{2}\}$;\\[-0.5em]
\item \emph{(Gumbel)} \qquad 
$\begin{gathered}[t]\displaystyle c_i(x,y) = 1+\theta(1-2x)(1-2y), \qquad\theta\in[-1,1]\setminus\{0\};
\end{gathered}$\\[-0.25em]
\item \emph{(Frank)} \qquad
$\begin{gathered}
\displaystyle \ \ \, c_i(x,y) = \frac
{\theta e^{\theta(x+y)}(e^{\theta}-1)}
{(e^\theta-e^{\theta x}-e^{\theta y}+e^{\theta(x+y)})^2}, \qquad \theta\in\mathbb{R}\setminus\{0\}.
\end{gathered}$
\end{enumerate}Then $S$ is $\alpha$-contrastive. 
\end{corollary}
\begin{proof}
This is a direct consequence of Proposition \ref{prop:TSCopulaProp1} upon checking that each of the copula densities (i), (ii) and (iii) satisfies \eqref{prop:TSCopulaProp1:aux3}. This, however, follows from inspection and a straightforward computational verification.
\end{proof} 

\subsection{Popular Gaussian Processes and Geometric Brownian Motion are $\gamma$-Contrastive}\label{sect:Gaussian_processes}
\noindent
Given an interval $\mathbb{I}$ and functions $\mu:\mathbb{I}\rightarrow\mathbb{R}^d$ and $\kappa : \mathbb{I}^{\times 2}\rightarrow\operatorname{GL}_d(\mathbb{R})$, we write $S\sim\mathcal{GP}_{\mathbb{I}}(\mu,\kappa)$ to denote that $S=(S_t)_{t\in\mathbb{I}}$ is a Gaussian Process in $\R^d$ with mean $\mu=(\mu_i)$ and covariance $\kappa=(\kappa^{ij})$. We assume that any pair $(\mu,\kappa)$ we consider in the following is such that each process $S\sim\mathcal{GP}(\mu,\kappa)$ admits a version with continuous sample paths.\\[-0.5em]

\noindent
(The proofs of the below results are given in Appendices \ref{pf:prop:GPsAreVaried}, \ref{pf:cor1:GPsAreVaried} and \ref{pf:prop:SDE_GBM}, respectively.)    

\begin{lemma}\label{prop:GPsAreVaried}
Let $S\sim\mathcal{GP}_{\mathbb{I}}(\mu,\kappa)$ be a (continuous) Gaussian process in $\mathbb{R}^d$ with diagonal covariance function $\kappa\equiv(\kappa^{ij})=(\kappa^{ij}\delta_{ij})$. Then $S$ is $\gamma$-contrastive if and only if there exist pairs $\p_0,\p_1,\p_2\in\Delta_2(\I)$ such that
\begin{equation}\label{prop:GPsAreVaried:eq1}
\left(\frac{\kappa^i_{\p_1}\cdot\kappa^i_{\p_2}\cdot[k^i_{\p_0}-(\kappa^i_{\p_0})^2]^2}{[k^i_{\p_1}-(\kappa^i_{\p_1})^2]\cdot[k^i_{\p_2}-(\kappa^i_{\p_2})^2]\cdot(\kappa^i_{\p_0})^2}\right)_{i\in[d]} \in \ (\mathbb{R}^d\setminus\nabla^{\times})   
\end{equation}
for the auxiliary functions 
\begin{equation}\label{prop:GPsAreVaried:eq2}
\kappa_{s,t}^i\coloneqq \kappa^{ii}(s,t) \quad\text{ and }\quad k^i_{s,t}\coloneqq\kappa^{ii}(s,s)\cdot\kappa^{ii}(t,t).
\end{equation} 
\end{lemma}   
\begin{remark}\label{rem:Gaussian_processes} 
The above lemma asserts that IC Gaussian processes are `generically identifiable', namely if the function \eqref{prop:GPsAreVaried:eq1} of their autocovariances avoids the nullset $\nabla^{\times}$ for some time pairs $\p_0,\p_1,\p_2$.  Compare this to the well-known result \cite{BCM} that an IC Gaussian process $S$ is identifiable from its linear mixtures -- via joint diagonalisation of the covariance matrices of such mixtures at one or several time lags -- if the (vector whose components are the) autocovariances of $S$ themselves avoids the nullset $\nabla^{\times}$ at one of these time lags.                
\end{remark}    
\noindent  
We verify the above contrastivity condition for a number of popular Gaussian processes. 
\begin{proposition}\label{cor1:GPsAreVaried}
Let $S=(S^1_t, \cdots, S^d_t)_{t\in\mathbb{I}}$ be an IC stochastic process in $\R^d$ with $S^i\sim\mathcal{GP}(\mu_i, \kappa_i)$ for each $i\in[d]$. Then $S$ is $\gamma$-contrastive in each of these four classical cases. 
\begin{enumerate}[label=\upshape(\roman*)]
\item\label{cor1:GPsAreVaried:item1} For each $i\in[d]$, the componental autocovariance functions \eqref{prop:GPsAreVaried:eq2} of $S$ are of the form  
\begin{equation}\label{cor1:GPsAreVaried:eq1}
\kappa^i(s,t)=\exp\!\left(-\left[\frac{|t-s|}{\alpha_i}\right]^{\gamma_i}\right)
\end{equation} 
with $\gamma\equiv(\gamma_i)_{i\in[d]}\in(0,2]^d$ and $\alpha\equiv(\alpha_i)_{i\in[d]}\in(\mathbb{R}_{\times})^{\times d}\setminus\mathcal{N}_\gamma$, where $\mathcal{N}_\gamma\subset\mathbb{R}^d$ is a Lebesgue nullset defined in the proof below.\footnote{\ This includes the family of $\gamma$-exponential processes, cf.\ \citep[Sect.\ 4.2 (pp.~84 ff.)]{rasmuwillGP2006}.}  
\item\label{cor1:GPsAreVaried:item2} Each component process $S^i$ of $S$ is an Ornstein-Uhlenbeck process
\begin{equation}\label{cor1:GPsAreVaried:eq2}
\mathrm{d}S^i_t = \theta_i\cdot(\mu_i - S^i_t)\,\mathrm{d}t \, + \, \sigma_i\,\mathrm{d}B^i_t, \quad S^i_0=a_i, \qquad (i\in[d]) 
\end{equation}
with $a_i, \mu_i\in\mathbb{R}$ and $\sigma\equiv(\sigma_i)_{i\in[d]}\in\mathbb{R}^d_{>0}$ and $\theta\equiv(\theta_i)_{i\in[d]}\in\mathbb{R}_{>0}^d\setminus\tilde{\mathcal{N}}$, where $\tilde{\mathcal{N}}\subset\mathbb{R}^d$ is a Lebesgue nullset defined in the proof below. 
\item\label{cor1:GPsAreVaried:item3} The component processes of $S$ are fractional Brownian motions with pairwise distinct Hurst indices, that is their autocovariance functions \eqref{prop:GPsAreVaried:eq2} take the form
\begin{equation}
\kappa^i(s,t) = \frac{1}{2}(|t|^{2H_i} + |s|^{2H_i} - |t-s|^{2H_i}) \qquad (i\in[d])
\end{equation}for some $(H_i)_{i\in[d]}\in(0,1)^d\setminus\nabla^{\times}$.   
\item\label{cor1:GPsAreVaried:item4} Denoting $s\wedge t\coloneqq \min(s,t)$, the autocovariance functions \eqref{prop:GPsAreVaried:eq2} of the $S^i$ are of the form 
\begin{equation}
\kappa^i(s,t)= \int_0^{s\wedge t}\!\eta_i(r)\,\mathrm{d}r \qquad \text{for each } \ i\in[d], 
\end{equation}
with functions $\eta_1, \ldots, \eta_d : \mathbb{I}\rightarrow\mathbb{R}$ for which there are $r_0, r_1\in\mathbb{I}$ such that the products $\{\eta_i(r_0)\cdot\eta_j(r_1)\mid i,j\in[d]\}$ are pairwise distinct. This is includes deterministic signals perturbed by white noise, i.e.\ signals $S=(S^1_t,\cdots,S^d_t)_{t\in\I}$ which, for $(B^i_t)_{t\geq 0}$ some standard Brownian motion in $\R^d$, are given by
\begin{equation}
\mathrm{d}S^i_t = \mu_i(t)\,\mathrm{d}t \, + \, \sigma_i(t)\,\mathrm{d}B^i_t \qquad \text{for each } \ i\in[d] 
\end{equation}with $\mu_i, \sigma_i:\mathbb{I}\rightarrow\mathbb{R}$ integrable and continuous such that the entries of $(\sigma_i^2(r_0)\cdot\sigma_j^2(r_1))_{i,j\in[d]}$ are pairwise distinct for some $r_0, r_1\in\mathbb{I}$.       
\end{enumerate}
\end{proposition}   
The proposition below concludes our short compilation of applicable source models.
\begin{restatable}{proposition}{propGBM}\label{prop:SDE_GBM}
Let $S=(S_t)_{t\geq 0} = (S^1, \cdots, S^d)$ be an IC geometric Brownian motion in $\mathbb{R}^d$, i.e.\ suppose that there is a standard Brownian motion $B=(B^1_t, \cdots, B^d_t)_{t\geq 0}$ such that 
\begin{equation}\label{SDE_GBM:eq1}
\mathrm{d}S_t^i \ = \ S_t^i\cdot\big(\mu_i(t)\,\mathrm{d}t \ + \ \sigma_i(t)\,\mathrm{d}B_t^i\big), \quad S_0^i = s^i_0 \qquad(i\in[d])
\end{equation} 
for some $s^i_0>0$ and continuous functions $\mu_i : \mathbb{I}\rightarrow\mathbb{R}$ and $\sigma_i:\mathbb{I}\rightarrow\mathbb{R}_{>0}$. Then $S$ has spatial support $D_S = \mathbb{R}_+^d$, and $S$ is $\gamma$-contrastive if there are $r_0, r_1\in\mathbb{I}$ for which the numbers $\{\sigma_i^2(r_0)\cdot\sigma_j^2(r_1)\mid (i,j)\in[d]\times[d]\}$ are pairwise distinct.     
\end{restatable}  

\section{Signature Cumulants as Contrast Function}\label{sect:independence_criterion}This section uses the identifiability results of Section \ref{sect:core_theory} to reformulate the problem of nonlinear blind source separation as an optimisation task in the spirit of Corollary \ref{cor:Comon}. Central to this is the concept of an IC-characterising contrast function on stochastic processes. We propose such a function by means of signature cumulants, which we introduce as a natural extension of classical (multivariate) cumulants to multidimensional stochastic processes.      

\begin{remark}\label{rem:bv_for_sig}In this section, we restrict our exposition to stochastic processes whose sample paths are smooth [i.e., of bounded variation\footnote{A path $x=(x_t)_{t\in[0,1]}\in\mathcal{C}_d$ is called \emph{of bounded variation} if its variation norm $\|x\|_{1\mathrm{\text{-}var}}\coloneqq |x_0| + \sup\sum|x_{t_{i+1}} - x_{t_i}|$ is finite, where the supremum is taken over all finite partitions $\{0\leq t_1\leq\cdots\leq t_n\leq 1\}$ $(n\in\mathbb{N})$ of $[0,1]$; cf.\ also definition \eqref{rem:p-varseminorm} and Section \ref{rem:pwlinterpol}.}], and further assume that the expected signature of these processes (defined below) exists and characterizes their law.
  These assumptions can be avoided by using rough integration and tensor normalization, but since this requires background in rough path theory and is not central to our methodology, we simply refer the interested reader to \citep{FVI,FLO} and \cite{CHL,CHO}, respectively.
Let further $\I=[0,1]$ wlog.
\end{remark}  

\subsection{Signature Cumulants}Many results in statistics, including Corollary \ref{cor:Comon} via \eqref{intext:Comon:CF}, are based on the well-known facts that laws of $\R^d$-valued random variables are often characterised by their moments, and that statistical independence turns into simple algebraic relations when expressed in terms of cumulants.
Our main object of interest are $\mathcal{C}_d$-valued random variables (stochastic processes), for which the so-called expected signature \cite{CHL} provides a natural generalisation of the classical moment sequence. Similar to classical moments, these signature moments form multi-indexed collections of numbers that can characterize the laws of stochastic processes. Similar still, upon their `logarithmic compression' these number collections give rise to signature cumulants that quantify the statistical dependencies within multidimensional stochastic processes (that is, between their coordinates and over time).\\[-0.5em]    
    
\noindent
Denote by $[d]^\star \coloneqq \bigcup_{m \ge 0} [d]^{\times m}$ the set of all multi-indices\footnote{\ We define $[d]^{\times 0}\coloneqq \{\emptyset\}$ with $\emptyset$ the empty set, and let $\{k\}^\star\coloneqq \bigcup_{m\geq 0}\{k\}^{\times m} \, \big(= \{\emptyset, k, kk, kkk, \ldots\}\big)$.} with entries in $[d]=\{1,\ldots,d\}$.
\begin{definition}[Expected Signature]\label{def:expected_signature} 
For $Y=(Y^1_t, \cdots, Y^d_t)_{t\in[0,1]}$ a stochastic process in $\R^d$ with
sample-paths of bounded variation, the collection of real numbers (if it exists) $\mathfrak{S}(Y) \coloneqq \left(\sigma_{\bm{i}}(Y)\right)_{\bm{i}\in[d]^\star}$ defined by the expected iterated Stieltjes integrals
\begin{equation}\label{def:expected_signature:eq1}
\sigma_{\bm{i}}(Y)\coloneqq 
\mathbb{E}\!\left[\int_{{0\leq t_1 \leq t_2 \leq \cdots \leq t_m \leq 1}}\!\mathrm{d}Y^{i_1}_{t_1}\mathrm{d}Y^{i_2}_{t_2}\cdots\mathrm{d}Y^{i_m}_{t_m}\right] \quad \text{ for } \ \ \bm{i}=(i_1,\ldots,i_m), 
\end{equation}
with $\sigma_\emptyset(Y)\coloneqq 1$, is called \emph{the expected signature of $Y$}.
\end{definition} 

\noindent
The expected signature is to a stochastic process roughly what the sequence of moments is to a vector-valued random variable, and analogous to the case of classical moments, for many statistical purposes the concept of cumulants is better suited. This leads to the notion of signature cumulants \cite{bonnier2019signature} below. (See Remark \ref{def:sigcumulant:rem} and Sections \ref{rem:sigcumulants} and \ref{sect:expected_signature_moments} for details.)   
\begin{definition}[Signature Cumulants]\label{def:sigcumulant}
For $Y$ a stochastic process in $\R^d$ with sample-paths of bounded variation, the collection of real numbers\footnote{\ The $\log$ in \eqref{def:sigcumulant:eq1} denotes the logarithm on the space of formal power series, see Section \ref{sect:sigmoments:logtransform} and \cite{bonnier2019signature}.}
\begin{equation}\label{def:sigcumulant:eq1}
\left( \kappa_{\bm{i}}(Y) \right)_{\bm{i} \in [d]^\star} \coloneqq \log[\mathfrak{S}(Y)]
\end{equation}is called \emph{the signature cumulant}
of $Y$. We further define 
\begin{equation}\label{def:sigcumulant:eq2}
\bar{\kappa}_{\bm{i}}(Y) \coloneqq \frac{\kappa_{\bm{i}}(Y)}{\kappa_{11}(Y)^{\eta_1(\bm{i})/2}\cdot\ldots\cdot\kappa_{dd}(Y)^{\eta_d(\bm{i})/2}} \quad \text{ for } \ \ \bm{i}=(i_1,\ldots,i_m)\in[d]^\star,
\end{equation}
where $\eta_\nu(\bm{i})$ denotes the number of times the index-value $\nu$ appears in $\bm{i}$. 
We refer to $(\bar \kappa_{\bm{i}}(Y))_{\bm{i} \in [d]^\star}$ as the \emph{standardized signature cumulant} of $Y$.     
\end{definition} 

\begin{remark}\label{def:sigcumulant:rem}The signature cumulant of a process $Y$ gives an efficiently computable \cite{KGL}, informationally condensed and hierarchically graded [cf.\ Sect.\ \ref{rem:sig_cumulants_generalise}] compression of the statistical information contained in (the distribution of) $Y$ [cf.\ Sects.\ \ref{rem:sigcumulants} and \ref{sect:subsect:sigmoments}], which enjoys a broad variety of excellent practical and theoretical features \cite{CHO}. Just as for standardized classical cumulants, the normalisation \eqref{def:sigcumulant:eq2} contributes the additional benefit of scale invariance which facilitates our below usage of signature cumulants as a contrast function.
\end{remark}      

\subsection{Signature Contrasts for Nonlinear ICA}Similar to how classical cumulants are traditional in linear ICA, cf.\ page \pageref{intext:Comon:CF}, the usage of signature cumulants in our present ICA-context is due to the following observation: Recall that a random vector $Y$ in $\R^d$ has independent components if and only if all of its cross-cumulants vanish, that is iff, in the notation of \eqref{intext:Comon:CF} and for $\ast$ the concatenation of indices,
\begin{equation}\label{intext:coords_indep_class}
\kappa_{\bm{q}}^Y \ = \ 0  \quad \text{ for all } \quad \bm{q}\in\bigsqcup_{k=2}^d\big\{\bm{i}\ast\bm{j} \, \ \big| \ \, \bm{i}\in [k-1]^\star\setminus\{\emptyset\}, \ \bm{j}\in\{k\}^\star\setminus\{\emptyset\}\big\}. 
\end{equation}
Now in the same way that the expected signature generalises the classical concept of moments, cf.\ Remark \ref{rem:sigcumulants}, it was shown in \cite{bonnier2019signature} that signature cumulants generalise this classical relation \eqref{intext:coords_indep_class} to an algebraic characterisation of statistical independence between [the components of] stochastic processes, cf.\ also Remark \ref{rem:classic_cumulants}. This is particularly useful in our context as it yields a natural and explicitly computable contrast function for path-valued random variables (Proposition \ref{prop:sig_cums}) as desired for nonlinear ICA.\\[-0.5em] 

Algebraically, cf.\ Remark \ref{rem:hopf}, the \eqref{def:sigcumulant:eq1}-based extension of the characterisation \eqref{intext:coords_indep_class} to stochastic processes requires us to replace the simple operation $\ast$ of index concatenation by a slightly more involved combinatorial operation on $[d]^\star$. This operation is defined next.           

\begin{notation}\label{notation:index_sum}
For convenience, we denote by $[d]^\star_+$ the family of all finite sums of indices in $[d]^\star$, and for any such sum $\bm{i}\equiv\bm{i}_1 + \ldots + \bm{i}_\ell\in[d]^\star_+$ define $\kappa_{\bm{i}}\coloneqq \kappa_{\bm{i}_1}+\ldots+\kappa_{\bm{i}_\ell}$.
\end{notation}

The \emph{shuffle product} of two multi-indices $\bm{i}=(i_1,\ldots,i_m)$ and $\bm{j}=(i_{m+1},\ldots,i_{m+n})$ in $[d]^\star$ is defined as the element of $[d]^\star_+$ which is given by  
\begin{align}\label{shuffle_product}
\bm{i} \shuffle  \bm{j} \ \coloneqq \ \sum_\tau (i_{\tau(1)},\ldots,i_{\tau(m+n)}) \ \in \ [d]^\star_+ 
\end{align}where the sum is taken over the family of permutations 
\begin{equation}
\{\tau\in S_{m+n} \ | \  \tau(1)<\cdots<\tau(m) \ \text{ and } \ \tau(m+1)<\cdots<\tau(m+n)\}.
\end{equation}
This enables us to formulate the following central observation.  
\begin{proposition}\label{prop:sig_cums}
For any stochastic process $Y = (Y^1, \cdots, Y^d)$ in $\R^d$ whose expected signature exists, the component processes $Y^1, \ldots, Y^d$ are mutually independent if and only if 
\begin{equation}\label{prop:sig_cums:eq1}
\bar{\kappa}_{\mathrm{IC}}(Y) \ \coloneqq \ \sum_{k=2}^{d}\sum_{\bm{q}\in\W_k}\bar{\kappa}_{\bm{q}}(Y)^2 \ = \ 0
\end{equation}
where $\mathcal{W}_k \ \coloneqq \ \big\{\bm{i}\shuffle \bm{j} \, \mid \, \bm{i}\in [k-1]^\star\setminus\{\emptyset\}, \ \bm{j}\in\{k\}^{\times m}, \ m \geq 1\big\} \ \subset \ [d]^\star_+$.
\end{proposition} 
\begin{proof}Observe that the component processes $Y^1, \ldots, Y^d$ are mutually independent iff: 
\begin{equation}\label{lem:mutindep_pwindep:eq1}
\text{for each } \ 2\leq k \leq d, \quad \text{ the process } \ Y^k \ \text{ is independent of } \ (Y^1,\cdots, Y^{k-1}).
\end{equation}
The asserted characterisation is a direct consequence of this and \citep[Theorem 1.2~(iii)]{bonnier2019signature}.
\end{proof}

\noindent
We may now combine Proposition \ref{prop:sig_cums} with Theorems \ref{thm:NICA_stat} and \ref{cor:NICA_MainCor} to obtain the following instance of \eqref{Intro:intext:generalised_contrast} for the inversion `$X\mapsto S$' that is desired in \eqref{intext:ProblemOfBSS} (cf.\ Corollary \ref{cor:Comon}). 

\hfill (Recall Remark \ref{rem:bv_for_sig} for the well-definedness of the signature statistics featured in \eqref{thm:optimisation:eq1}.)      
\begin{theorem}\label{thm:optimisation} 
Let the process $S$ in \eqref{intext:BSS_relation} be $\alpha$-, $\beta$- or $\gamma$-contrastive with sample-paths of bounded variation.
Then it holds with probability one that 
\begin{equation}\label{thm:optimisation:eq1}
\left[\underset{h\in\Theta}{\operatorname{arg\ min}}\ \bar{\kappa}_{\mathrm{IC}}\big(h(X)\big)\right]\cdot X \ \subseteq \ \mathrm{DP}_{\!d}\cdot S
\end{equation}for any family of transformations $\Theta\subseteq C^{2,2}(D_X)$ with $\Theta\cap\big(\restr{\DP(D_S)\cdot f^{-1}\big)}{D_X}\neq\emptyset$.      
\end{theorem} 

This theorem states that the initial problem \eqref{intext:ProblemOfBSS} of nonlinear blind source separation can be reformulated as a problem of optimisation-based function approximation. More specifically, statement \eqref{thm:optimisation:eq1} says that the desired demixing transformations of the data can be found as minimizers of the energy-like functional \eqref{prop:sig_cums:eq1}. We conclude with a few practical remarks.   
\begin{remark}\begin{enumerate}[label=(\roman*)]\label{rem:thm_optimisation}
\item For $\Theta\subseteq\operatorname{GL}_d$ and under the temporally degenerate hypothesis of Theorem \ref{thm:Comon}, the procedure \eqref{thm:optimisation:eq1} reduces to Comon's optimisation \eqref{cor:Comon:eq1} for $\phi=\phi_c$ since 
\begin{equation}
\bar{\kappa}_{\mathrm{IC}}\big((Y\cdot t)_{t\in[0,1]}\big) \, = \, \phi_c  (Y) \quad\text{if $Y$ is a random vector in $\R^d$ \quad (cf.\ Remark \ref{rem:classic_cumulants})}.
\end{equation}    
\item\label{rem:thm_optimisation:it0.2} Regarding implementations of \eqref{thm:optimisation:eq1}, one may choose to realise the above domain $\Theta$ by way of an Artificial Neural Network, see e.g.\ Section \ref{sect:experimentsII}. This choice is mathematically justified by the fact that neural networks can be designed as universal approximators to $C^{2,2}(D_X)$ \cite{teshima2020} with a favourable convergence topology \citep{petersen2020} (cf.\ also Remarks \ref{sect:capping:assumptions:rem}, \ref{rem:method_in_practice}).          
\item\label{rem:thm_optimisation:it1}In practice, only discrete-time observations $(X_t)_{t\in\mathcal{I}}$ of $X$ for a finite $\mathcal{I}\subset\I$ are available. Our framework covers these discretised observations as well, as we can naturally identify the data $(X_t)_{t\in\mathcal{I}}$ with a continuous bounded variation process in $\R^d$ via piecewise-linear interpolation of the points $\{X_t\mid t\in\mathcal{I}\}$. The identifiability procedure of Theorem \ref{thm:optimisation} is robust under this discretisation, see Theorem \ref{thm:consistency} and Section \ref{rem:discrete} \ref{rem:discreteobs} in particular. A quick inspection of Theorems \ref{thm:NICA_stat} \& \ref{cor:NICA_MainCor} further reveals that the identifiability approach of the preceding sections can be immediately extended to discrete time-series that are not necessarily generated from continuous-time processes, see Appendix \ref{appendix:NICA_discrete} for details.            
\item\label{rem:thm_optimisation:it2}  
The contrast function $\bar{\kappa}_{\mathrm{IC}}$ can be efficiently approximated by restricting the summation in \eqref{prop:sig_cums:eq1} to multindices $(i_1,\ldots,i_m)$ up to a maximal order $m \leq m_0$ and estimating these remaining summands using the unbiased minimum-variance estimators for signature cumulants introduced in \citep[Section 4]{bonnier2019signature}.
For the latter, a more naive but straightforward approach that is sufficient for our experiments is to just use the Monte-Carlo estimator, see Sections \ref{sect:capping} \& \ref{sect:ergodicity} for details. 
\item\label{rem:thm_optimisation:it5} For a fixed and finite data set, lower-order summands in the above (capped) approximation of the contrast $\bar{\kappa}_{\mathrm{IC}}$ are typically estimated more accurately than higher-order summands. In practical applications this may be accounted for by applying weights to the estimated summands of the contrast, leading one to estimate the alternative objective
\begin{equation}
\textstyle \sum_{m=2}^{m_0}\sum_{\bm{q}\in\mathfrak{C}_m} w_{\bm{q}}\cdot\bar{\kappa}_{\bm{q}}(Y)^2 \quad\text{ for}\quad 0 < w_{\bm{q}}\equiv w_{\bm{q}}(Y) \ \text{ decreasing in the order of } \bm{q},
\end{equation}
where we used the notation of \eqref{sect:capping:normalised_seriescapped} for convenience. Appropriate choices of weights $(w_{\bm{q}})$ will generally depend on $m_0$ and the respective domain of $\bar{\kappa}_{\mathrm{IC}}$, but may otherwise be arbitrary provided that the $\operatorname{arg\,min}$ of the thus-weighted objective coincides with the $\operatorname{arg\,min}$ of the default case $w_{\bm{q}}\equiv 1$ (as, e.g., is guaranteed by the final assumption of Theorem \ref{thm:optimisation}).                 
\end{enumerate}     
\end{remark}  

\section{Statistical Consistency}\label{sect:consistency}
\noindent
In practice, the mixture $X$ is usually not observed as a whole stochastic process with fully known distribution but rather as a time-discretized sample consisting of finitely many data points in $\R^d$, often taken over non-equally spaced time grids. Further, often only a single such (time-discretized) sample path of the process is available rather than many independent sample realisations,\footnote{\ Note, however, that the latter can be regarded as a special case of the former by concatenating the available (independent) sample observations into a single long observation.} for example in the classical cocktail party problem. In this section we demonstrate that our proposed nonlinear ICA method is stable under such discretisations. More precisely, we formalise the usual observation schemes (Section \ref{sect:sampling}) and prove that if these sampling discretisations of $X$ get `finer' in some natural sense, then our ICA method produces a signal that gets uniformly closer to the unobserved source $S$ underlying $X$ (Section \ref{sect:consistlim}). The main steps towards this are outlined in Section \ref{sect:consoverview}.\\[-0.5em] 

We assume throughout this section that the mixture $X$ is a continuous-time signal. As before, the (simpler) case where the underlying mixture $X$ is discrete in time is covered analogously up to some minor modifications, as we explain in detail in Section \ref{rem:discrete:consistency}.  

\subsection{Sampling}\label{sect:sampling} 
In practical applications, observations of the mixture $X$ are typically not available as continuous paths, i.e.\ elements of $\mathcal{C}_d$, but rather as discrete, sequentially ordered collections of data points in $\R^d$. Formally, such data can be modelled as a discrete time series  
\begin{equation}\label{sect:finsamlim:eq1}
\mathfrak{x} \,\coloneqq\, (X_t(\omega))_{t\in\mathcal{I}} \quad\text{ with }\quad \mathcal{I}\coloneqq\{0=t_0<t_1<\ldots<t_{n-1}=1\} 
\end{equation}
for $\omega\in\Omega$ any fixed elementary event in the probability space $(\Omega, \mathscr{F}, \mathbb{P})$ underlying  $X$. (In empirical language, the above dissection $\mathcal{I}$ of $[0,1]$ can then be regarded as a `protocol' describing a sequence of measurements of $X$ carried out per unit time with frequency $n$.)\\[-0.5em]  

\noindent
The time series data \eqref{sect:finsamlim:eq1} is then typically collected over not just one but several ($\nu\in\N$) time intervals (per observation), and different observations of $X$ $(k\in\N)$ may vary in the frequency at which their fixed-time measurements are made. This gives rise to the data scheme  
\begin{equation}\label{sect:finsamlim:eq2} 
\mathfrak{x}^{(k)}\equiv\big(\mathfrak{x}_1^{(k)},\mathfrak{x}_2^{(k)}, \cdots\big)\coloneqq(X_t(\omega))_{t\in\mathcal{J}_k}, \ \ \text{ for } \ \ \mathcal{J}_k \coloneqq \mathcal{I}_1^{(k)}\,\sqcup\,\mathcal{I}_2^{(k)}\,\sqcup\,\ldots \,=\,\bigsqcup\nolimits_{\nu\in\N}\mathcal{I}_\nu^{(k)}
\end{equation} 
a dissection\footnote{\ We call $\mathcal{J}_T\coloneqq\{t_0 < t_1 < \ldots\}$ a \emph{dissection of $[0,\infty)$} if $t_0=0$ and $t_j\nearrow\infty$ as $j\rightarrow\infty$. Also, we then assume $X$ to be defined over the full positive time-axis $[0,\infty)$, see \eqref{sect:finsamlim:eq3}.} of $[0,\infty)$ such that $\mathcal{I}_1^{(k)} < \mathcal{I}_2^{(k)} < \ldots $ and $\hat{n}_k\coloneqq\sup_{\nu\in\N}|\mathcal{I}_\nu^{(k)}|<\infty$,\footnote{\ Notice that $|\mathcal{I}|$ denotes the cardinality of a set $\mathcal{I}\subset\R$ (i.e.\ the number of its elements) while the \emph{mesh-size} of $\mathcal{I}$ is denoted $\|\mathcal{I}\|$, cf.\ \eqref{rem:pwlinterpol:eq1}. (Consequently $|\mathcal{I}| \geq (\max\mathcal{I} - \min\mathcal{I})/\|\mathcal{I}\| + 1$.)} with $\mathcal{I}_1^{(k)}$ a dissection of $[0,1]$ and $\mathfrak{x}^{(k)}_\nu=(X_t(\omega))_{t\in\mathcal{I}^{(k)}_\nu}$ for each $\nu\in\N$. Any such family $(\mathcal{J}_k)_{k\in\N}$ will be called a \emph{protocol}, where the index $k\in\N$ enumerates the different observations of $X$.\\[-0.5em]      
 
\noindent 
Adopting the ergodicity perspective common in time-series analysis and signal processing, the relevant statistical information of $X$ (in our case: the signature-cumulant coordinates \eqref{def:sigcumulant:eq1}) may be averaged from the discretely-sampled observation $\mathfrak{x}^{(k)}$ of a single, sufficiently long realisation of $X$. 
For this, we may generalise established observation schemes\footnote{\ Observations schemes are typically required to be equispaced, see e.g.\ \citep[Sect.\ 3]{fan2005overview} and \citep{yu2015sampling} for an overview.} by assuming that the data $\mathfrak{x}^{(k)}$ in \eqref{sect:finsamlim:eq2} be obtained according to nothing else but the assumptions                   
\begin{equation}\label{sect:finsamlim:eq3} 
X\,=\, (X_t)_{t\geq 0} \quad \text{ and }\quad \lim_{k\rightarrow\infty}\big\|\mathcal{I}^{(k)}_1\big\|=0\,,  
\end{equation} 
that is, the requirements that the continuous observable $X$ be `infinitely long' (i.e., defined over $[0,\infty)$) and the frequency $|\mathcal{I}_1^{(k)}|$ of observations per initial interval be going to infinity.\\[-0.5em]   

\noindent
A protocol $(\mathcal{J}_k)_{k\in\N}$ as in \eqref{sect:finsamlim:eq2} \& \eqref{sect:finsamlim:eq3} will be called \emph{exhaustive} with \emph{base lengths} $n_k\coloneqq|\mathcal{I}_1^{(k)}|$ and \emph{maximal (observation) length} $\hat{n}_k$; we called it \emph{balanced} if $|\mathcal{I}^{(k)}_\nu|=n_k$ for all $\nu\in\N$. 

\begin{remark}\label{rem:finite_samples_flexibility}
Note that the sampling scheme \eqref{sect:finsamlim:eq2} allows for the units $\mathcal{I}^{(k)}_\nu\eqqcolon\big\{t^{(k|\nu)}_0<\cdots<t^{(k|\nu)}_{n_{k,\nu}-1}\big\}$ to span observation intervals $\big[t^{(k|\nu)}_0,\, t^{(k|\nu)}_{n_{k,\nu}-1}\big]$ of different lengths and dissect them at different, non-constant frequencies. Notice further that while the mesh $\mathcal{J}_k$ in \eqref{sect:finsamlim:eq2} is infinite, this does not restrict us to considering observations $\mathfrak{x}^{(k)}$ that contain an infinitude of information $(\mathfrak{x}^{(k)}_\nu)$. In fact, everything presented in this section works as stated if we relax the above definition of a protocol by replacing the observable $X$ in \eqref{sect:finsamlim:eq2} with the cutoff
\begin{equation}
X\cdot\mathbbm{1}_{[0, T_k]} \quad\text{ for some finite observation horizon } \ T_k\geq 0 \text{ with } \lim_{k\rightarrow\infty} T_k=\infty,  
\end{equation}
thus extending the class of permissible data $(\mathfrak{x}^{(k)})$ in  \eqref{sect:finsamlim:eq2} to finite (eventually zero) sequences. 
\end{remark}  
 
\subsection{Section Overview}\label{sect:consoverview}
Let $(\mathfrak{x}^{(k)})$ be data associated to an exhaustive observation $\big(X, (\mathcal{J}_k)\big)$ via the sampling scheme \eqref{sect:finsamlim:eq2}. The aim of this section is to establish conditions under which the optimisation procedure of Theorem \ref{thm:optimisation}, when applied to $(\mathfrak{x}^{(k)})$ in lieu of $X$, yields a sequence of approximations $(\hat{\theta}_k)$ of the true demixing inverse $f^{-1}$ which is statistically consistent in the sense that 
\begin{equation}\label{sect:finsamlim:eq4}
\lim_{k\rightarrow\infty}\mathrm{dist}\big(\hat{\theta}_k(X),\,\mathrm{DP}_d\cdot S\big) \ = \ 0
\end{equation} 
almost surely or in probability, where the distance is taken with respect to the uniform norm on $\mathcal{C}_d$.\footnote{\ The trivial case $\mathrm{dist}\big(\hat{\theta}_k(X),\,\mathrm{DP}_d\cdot S\big) \leq \|\hat{\theta}_k(X) - \alpha_k(S)\| \leq \|\hat{\theta}_k(X)\| + \|\alpha_k(S)\| \stackrel{!}{\rightarrow} 0 \ \ (k\rightarrow\infty)$ is automatically excluded if ($X$ is non-trivial and) $(\hat{\theta}_k)\subseteq\Theta$ is bounded away from zero. This is guaranteed by the below assumption, in Theorem \ref{thm:consistency}, of $\Theta$ being a compact subset of $C^{1,1}(D_X)$.} Since the original optimisation \eqref{thm:optimisation:eq1} is composed of three (`limiting') operations that each involve an `infinite amount of information', namely the \emph{infinite series} $\bar{\kappa}_{\mathrm{IC}}$ from \eqref{prop:sig_cums:eq1} whose summands \eqref{def:sigcumulant:eq2} are each defined by taking \emph{expectations} of nonlinear functionals \eqref{def:expected_signature:eq1} of the \emph{continuous-time} stochastic processes $Y=\theta(X)$, one may expect the consistency \eqref{sect:finsamlim:eq4} to result as a combination of the following three sublimits:   
\begin{itemize}
\item \emph{Capping Limit} (Section \ref{sect:capping}). In practice, only finitely many summands of the infinite statistics $\bar{\kappa}_{\mathrm{IC}}$ from \eqref{prop:sig_cums:eq1} can be computed from the data. This is to say that the series  
\begin{equation}\label{sect:capping:normalised_series}
\bar{\kappa}_{\mathrm{IC}}(Y) \ = \ \sum_{m=2}^\infty\sum_{\bm{q}\in\mathfrak{C}_m}\bar{\kappa}_{\bm{q}}(Y)^2   
\end{equation}
with $\mathfrak{C}_m\subset [d]^\star_+$ denoting the set of all cross-shuffles $\mathfrak{C}\coloneqq\bigsqcup_{k=2}^d\mathcal{W}_k$ of word-length $m\in\N$ (see Prop.\ \ref{prop:sig_cums}), needs to be capped at some index $m=m_0$. Denoting this capped series by 
\begin{equation}\label{sect:capping:normalised_seriescapped}
\bar{\kappa}^{[m_0]}_{\mathrm{IC}}(Y) \ \coloneqq \ \sum_{m=2}^{m_0}\sum_{\bm{q}\in\mathfrak{C}_m}\bar{\kappa}_{\bm{q}}(Y)^2 
\end{equation}we show that in the \emph{capping limit} $m_0\rightarrow\infty$ the minimizers of $\theta\mapsto\bar{\kappa}^{[m_0]}_{\mathrm{IC}}\!\big(\theta(X)\big)$ approach those of \eqref{sect:capping:normalised_series} with respect to a naturally chosen topology on $\Theta$; this provides the first ingredient for the consistency limit \eqref{sect:finsamlim:eq4}.\\[-0.75em]
 
\item \emph{Interpolation Limit} (Section \ref{sect:interpollim}). The mixture $X$ is usually observed along a discrete set of time-points $\mathcal{I}$ rather than continuously over time, as mentioned in Sect.\ \ref{sect:sampling}. By way of their piecewise-linear interpolation $\hat{X}_{\mathcal{I}}$, these discrete observations $(X_t)_{t\in\mathcal{I}}$ can be reinterpreted as $\mathcal{C}_d$-valued data, which then allows to approximate the summands in \eqref{sect:capping:normalised_seriescapped} via
\begin{equation}\label{sect:finsamlim:eq5}
\bar{\kappa}_{\bm{q}}\big(\theta(X)\big) \, \approx \, \bar{\kappa}_{\bm{q}}(\hat{X}_{\mathcal{I}}^\theta), \quad \text{ for } \ \hat{X}^\theta_{\mathcal{I}} \ \text{ the linear interpolant of } \ \big(\theta(X_t)\big)_{t\in\mathcal{I}}. 
\end{equation}
By showing that \eqref{sect:finsamlim:eq5} defines a $\Theta$-uniform approximation as $\|\mathcal{I}\|\rightarrow 0$, we obtain that the minimizers of $\theta \mapsto \bar{\kappa}_{\mathrm{IC}}^{[m_0]}(\hat{X}^{\theta}_{\mathcal{I}})$ converge to those of \eqref{sect:capping:normalised_seriescapped}; our second ingredient for \eqref{sect:finsamlim:eq4}.\\[-0.75em]

\item \emph{Ergodicity Limit} (Section \ref{sect:ergodicity}). Finally, as the data \eqref{sect:finsamlim:eq2} that is actually available is but a single realisation of the discrete time-series $(X_t)_{t\in\mathcal{J}_k}$, we propose to approximate the above approximations \eqref{sect:finsamlim:eq5} by estimating their constituent signature moments \eqref{def:expected_signature:eq1} via 
\begin{equation}\label{sect:finsamlim:eq6}
\sigma_{\bm{i}}\big(\hat{X}^{\theta}_{\mathcal{I}^{(k)}_1}\big) \, \approx \, \frac{1}{T}\sum_{\nu=1}^T\sig_{\bm{i}}\big(\hat{\mathfrak{x}}_\nu^{\theta|k}\big), \quad \text{ for } \ \hat{\mathfrak{x}}_\nu^{\theta|k} \ \text{ the linear interpolant of } \ \theta(\mathfrak{x}^{(k)}_\nu)   
\end{equation}
and where $\sig_{\bm{i}}(Y)$ denotes the iterated integrals inside the expectation \eqref{def:expected_signature:eq1} (cf.\ Sect.\ \ref{sect:subsect:sigmoments}, eq.\ \eqref{sect:sigmoments:eq1}\,$\&$\,\eqref{sect:sigmoments:signature}). Showing that for many popular time-series models and stochastic signals the above estimation scheme \eqref{sect:finsamlim:eq6} for $\bar{\kappa}_{\bm{q}}(\hat{X}_{\mathcal{I}}^\theta)$ is $\Theta$-uniformly consistent as $T\rightarrow\infty$, we obtain our third and final ingredient for \eqref{sect:finsamlim:eq4}.                        
\end{itemize} 
In Section \ref{sect:consistlim}, these three sublimits are then combined to prove that our nonlinear ICA method (Theorem \ref{thm:optimisation}) gives rise to statistically consistent estimators of the sources underlying the data, see Theorem \ref{thm:consistency} which also includes the consistency limit \eqref{sect:finsamlim:eq4} as a special case. The resulting approach is condensed into a readily implementable source estimator in Section \ref{sect:algorithm}.\\[-0.5em]  

The majority of the proofs for this section are deferred to Appendix \ref{appendix:sect:proof_section8} as they are mostly technical and independent of the argumentation developed in the main body of this work.                            

\subsection{The Capping Limit}\label{sect:capping}Let the subset $\mathfrak{C}_m$ of $[d]^\star_+$ denote the set of all cross-shuffles $\mathfrak{C}\coloneqq\bigsqcup_{k=2}^d\mathcal{W}_k$ of fixed word-length $m$,\footnote{\ The \emph{word-length} of an element $\bm{i}\in[d]^\star_+$ is defined as the maximal order of the (finitely many) indices in $[d]^\star_+$ whose formal sum is $\bm{i}$ (cf.\ Notation \ref{notation:index_sum}). Thus $\mathfrak{C}_m= V_m \cap \bigsqcup_{k=2}^d\mathcal{W}_k$ in the language of Section \ref{rem:sig_cumulants_generalise:subsect:coordspace}.} for $\mathcal{W}_2, \ldots, \mathcal{W}_d$ as in Proposition \ref{prop:sig_cums} and $m\in\N$.\\[-0.5em]   

\noindent
Let further $\Theta$ be a given set of nonlinearities (as specified below), and for $(\kappa_q(Y) \mid q\in [d]^\star)$ as in \eqref{def:sigcumulant:eq1} and $\theta\in\Theta$, consider the (in)finite cumulant series        
\begin{equation}\label{sect:capping:qobjectives}
Q(\theta) \,\coloneqq \, \sum_{\nu=2}^\infty\sum_{\bm{q}\in\mathfrak{C}_\nu}c_{\bm{q}}^{-1}\cdot\kappa_{\bm{q}}\!\big(\theta(X)\big)^{\!2} \quad \ \text{ and } \ \quad Q_m(\theta)\,\coloneqq\, \sum_{\nu=2}^m\sum_{\bm{q}\in\mathfrak{C}_\nu}c_{\bm{q}}^{-1}\cdot\kappa_{\bm{q}}\!\big(\theta(X)\big)^{\!2} 
\end{equation} 
for $m\geq 2$, where $c_{\bm{q}}$ denotes the number of monomials in $\bm{q}$ (cf.\ Remark \ref{rem:cross_shuffles}).\\[-0.5em]

\noindent 
We make following technical compatibility assumptions on $\Theta$ and $X$. (For notation, \ref{rem:sig_cumulants_generalise:subsect:coordspace}.) 

\begin{assumption}\label{sect:capping:assumptions}
Let $\Theta \,\subseteq\, C(D_X;\R^d)$ be equipped with the topology of compact convergence and suppose that $\Theta$ is compact, satisfies $\Theta\cdot X \subseteq \mathcal{BV}$ and is such that it holds with probability one that for every convergent sequence $(\theta_j)_{j\geq 1}$ in $\Theta$ there is some $p\in[1,2)$ with
\begin{equation}\label{sect:capping:thetax_reg1}
\sup\nolimits_{j\geq 1}\|\theta_j(X)\|_{p\text{-$\mathrm{var}$}} \ < \ \infty
\end{equation}
where $\|\cdot\|_{\text{$p$-$\mathrm{var}$}}$ denotes the $p$-variation   seminorm \eqref{rem:p-varseminorm}. 
On side of the signature moments \eqref{def:expected_signature:eq1}, suppose that the expected signatures $\mathfrak{S}\!\big(\theta(X)\big) \equiv \mathbb{E}\big[\mathfrak{sig}\big(\theta(X)\big)\big]$, $\theta\in\Theta$, exist and characterize the law of their arguments, that their collection $\{\mathfrak{S}(\theta(X))\mid \theta\in\Theta\}$ is $\vertiii{\cdot}_\lambda$-bounded\footnote{\ As the source $S$ can be recovered up to (a componental permutation and) monotone scaling only, we can and will assume wlog (cf.\ Lemma \ref{sect:sigmoments:lem1} \ref{sect:sigmoments:lem1:it5}) that the set $\{\mathfrak{S}(\theta\cdot X)-1\mid\theta\in\Theta\}\subseteq V_{(0)}$ is $\vertiii{\cdot}_\lambda$-bounded by 1.} for some $\lambda>2$, and that for each $m\geq 1$ (with $\sig_m\coloneqq\pi_m\circ\sig$, $\sig$ as in \eqref{sect:sigmoments:signature}),\footnote{\ Some of the $\sup_\theta$-related (or $\mathrm{dist}(\,\cdot\,,\mathrm{DP}_d\cdot S)$-related) expressions in the following sections may be non-measurable, in which case any probability statements involving these expressions are to be understood in terms of outer measure (cf.\ \citep{vdv1998, WEL}).} 
\begin{equation}\label{sect:capping:thetax_reg2}
\E\!\left[\sup_{\theta\in\Theta}\big\|\sig_m\!\big(\theta(X)\big)\big\|_m\right] \, < \, \infty.
\end{equation}
(To avoid potential measurability problems, we may as well replace \eqref{sect:capping:thetax_reg2} by the (weaker) requirement that $\big[\E\!\left[\sup_{\theta\in\mathcal{F}}\big\|\sig_m\!\big(\theta(X)\big)\big\|_m\right] \, < \, \infty, \ \ \forall\,\text{ countable } \mathcal{F}\subseteq\Theta\big]$ if desired.)
\end{assumption}       
\begin{remark}\label{sect:capping:assumptions:rem}
Notice that the above conditions on $\Theta$ and $X$ are quite natural and well-established in the contexts of Artificial Neural Networks (ANNs) and Stochastic Analysis. Indeed: The topological requirement of compact convergence is typically met if $\Theta$ is given as the realisation space of an ANN, see e.g.\ \citep{berner2019, petersen2020}, while the assumptions $\Theta\cdot X\subseteq\mathcal{BV}$ resp.\ \eqref{sect:capping:thetax_reg1} hold for instance if $\Theta\subseteq C^1(D_X;\R^d)$ resp.\ if the elements of $\Theta$ are continuously differentiable with uniformly bounded Jacobians. 
The growth assumptions on the signature coordinates (including \eqref{sect:capping:thetax_reg2}), on the other hand, have been extensively studied, established and applied in the context of rough path analysis and statistics, see e.g.\ \citep{CHL} and \citep{bonnier2019signature, CHO}.                      
\end{remark}  
The following result is but a reformulation of Lemma \ref{sect:capping:lem1} in terms of the standardized signature cumulants \eqref{prop:sig_cums:eq1}. It also anticipates the consistency assertion in Theorem \ref{thm:consistency} below.     
\begin{proposition}[Capping Limit]\label{prop:capping_limit}
Let $X$ and $\Theta\subseteq C^{2,2}(D_X)$ fulfil Assumption \ref{sect:capping:assumptions}, and suppose that there is a unique $\theta_\star\in\Theta$ such that $\theta_\star(X)$ is IC. Let $\bar{\kappa}^{[m]}_{\mathrm{IC}}$ be as in \eqref{sect:capping:normalised_seriescapped}. Then for any sequence of minimizers $(\theta_m^\star)$ in $\Theta$ such that 
\begin{equation}\label{prop:capping_limit:eq1}
\bar{\kappa}_{\mathrm{IC}}^{[m]}\!\big(\theta_m^\star(X)\big)\ \leq \  \min_{\theta\in\Theta}\,\bar{\kappa}_{\mathrm{IC}}^{[m]}\!\big(\theta(X)\big) \ + \ \eta_m
\end{equation}
for some $(\eta_m)$ in $\R_+$ with $\lim_{m\rightarrow\infty}\eta_m=0$ a.s., it holds with probability one that 
\begin{equation}\label{prop:capping_limit:eq2}
\lim_{m\rightarrow\infty}\mathrm{dist}_{\|\cdot\|_\infty}\!\big(\theta^\star_m(X),\,\mathrm{DP}_d\cdot S\big) \ = \ 0.
\end{equation}       
\end{proposition}  
\begin{proof} 
Since for any $\bm{q}\in\mathfrak{C}$ we have that $\kappa_{\bm{q}}(Y) = 0$ iff $\bar{\kappa}_{\bm{q}}(Y)=0$ (recall \eqref{def:sigcumulant:eq2} and Notation \ref{notation:index_sum}), it holds that $\operatorname{arg\,min}_{\theta\in\Theta}Q_m(\theta) = \operatorname{arg\,min}_{\theta\in\Theta}\bar{\kappa}_{\mathrm{IC}}^{[m]}\!\big(\theta(X)\big)$ for each $m\geq 2$. The convergence \eqref{prop:capping_limit:eq2} is thus immediate by Lemma \ref{sect:capping:lem1} \ref{sect:capping:lem1:it3} and Prop.\ \ref{prop:sig_cums}/\,Thm.\ \ref{thm:optimisation}.
\end{proof}  

\subsection{The Interpolation Limit}\label{sect:interpollim}
Let $\I$ be a compact interval; say $\I=[0,1]$ wlog as above.\\[-0.5em]
 
\noindent
A finite subset $\mathcal{I}$ of $\I$ is called a \emph{dissection of $\I$} if it contains the boundary points of $\I$, and a sequence $(\mathcal{I}_\mu)_{\mu\geq 1}$ of dissections $\mathcal{I}_\mu\equiv\big\{t^{(\mu)}_0,\ldots, t^{(\mu)}_{n_\mu-1} \ \big| \ t^{(\mu)}_0 < t^{(\mu)}_1 < \ldots < t^{(\mu)}_{n_\mu-1}\big\}$ of $\I$ is called \emph{refined} if the maximal distance $\|\mathcal{I}_\mu\|$ between two successive points in $\mathcal{I}_\mu$, the so-called \emph{mesh-size} of $\mathcal{I}_\mu$, goes to zero as $\mu\rightarrow\infty$; in symbols: 
\begin{equation}\label{rem:pwlinterpol:eq1}
\|\mathcal{I}_\mu\|\coloneqq\max_{j\in[n_\mu-1]}\big|t^{(\mu)}_j - t^{(\mu)}_{j-1}\big| \ \longrightarrow \ 0 \qquad \text{as} \quad \mu\rightarrow\infty.   
\end{equation} 
Writing $\hat{X}^\theta_{\mathcal{I}}$ for the piecewise linear interpolant\footnote{\ For a formal definition of this operation see Appendix \ref{rem:pwlinterpol}, where a unified notation for the projection of continuous-time data to discrete time series -- and, conversely, the embedding (via interpolation) of the latter type of data into $C(\I;\R^d)$ -- is provided.} of the transformed data $X_{\mathcal{I}}^\theta\coloneqq\big(\theta(X_t)\big)_{t\in\mathcal{I}}$, $\theta\in\Theta$, the next lemma shows that, as $\|\mathcal{I}\|\rightarrow 0$, the statistic (cf.\ \eqref{sect:capping:qobjectives})
\begin{equation}\label{sect:interpollim:data:eq1}
\widehat{Q}_m(Y)\coloneqq\sum_{\nu=2}^m\sum_{\bm{q}\in\mathfrak{C}_\nu}c_{\bm{q}}^{-1}\cdot\kappa_{\bm{q}}(Y)^2 \quad \text{ with } \quad Y\coloneqq\hat{X}^\theta_{\mathcal{I}} \qquad (m\geq 2)
\end{equation}  
yields a $\Theta$-uniform approximation of the contrast $Q_m$ from \eqref{sect:capping:qobjectives}.       
             
\begin{restatable}[Interpolation Limit]{lemma}{leminterpolim}\label{lem:interpolim} 
Let $\Theta$ and $X$ be as in Assumption \ref{sect:capping:assumptions}, $\widehat{Q}_m$ as in \eqref{sect:interpollim:data:eq1} and $Q_m$ as in \eqref{sect:capping:qobjectives}. Then for $(\mathcal{I}_n)_{n\in\N}$ any refined sequence of dissections of $\I$ and any $m\in\N_{\geq 2}$,  
\begin{equation}\label{lem:interpolim:eq1}
Q_m(\theta) \ = \ \lim_{n\rightarrow\infty}\widehat{Q}_m(\hat{X}_{\mathcal{I}_n}^\theta) \quad \text{uniformly on \ $\Theta$}.
\end{equation}   
\end{restatable} 
\begin{proof} 
See Appendix \ref{pf:interpolim}. 
\end{proof} 

\subsection{The Ergodicity Limit}\label{sect:ergodicity}We formalise the estimation scheme \eqref{sect:finsamlim:eq6} and show that it holds uniformly on $\Theta$ for a large class of time-series models and stochastic processes.    

\begin{notation}Let $Z\coloneqq\R^d$. Given $z\equiv(z_j)_{j\in\N}$ and $J\subset\N$, we write $z_J\coloneqq(z_j)_{j\in J}$ and $z_{(\ell_1:\,\ell_2]}\coloneqq (z_{\ell_1+1}, z_{\ell_1+2}, \ldots, z_{\ell_2})$ for $\ell_1, \ell_2\in\N_0$ with $\ell_1<\ell_2$, and denote by $\hat{z}_{\mathcal{E}_n}\equiv\hat{\iota}_{\mathcal{E}_n}(z)$ the piecewise-linear interpolation of $z\equiv (z_j)_{j\in[n]}\in Z^{\times n}$ along the equidistant dissection $\mathcal{E}_n\coloneqq\{(\nu-1)/(n-1)\mid\nu\in[n]\}$ of $[0,1]$ (cf.\ Appendix \ref{rem:pwlinterpol}). For $X_\ast\equiv(X_j)_{j\in\N}$ a discrete time-series in $\R^d$, we denote by $D_{X_\ast}\coloneqq\overline{\bigcup_{j\in\N}\supp(X_j)}^{|\cdot|}$ its \emph{spatial support}. 

Set further $\sig_{[m]}\coloneqq\pi_{[m]}\circ\sig$ for the signature capped at level $m\geq 2$, cf.\ Section \ref{rem:sig_cumulants_generalise:subsect:coordspace}.
\end{notation} 
\noindent
The signature transform \eqref{sect:sigmoments:signature}, and thereby its cumulants \eqref{def:sigcumulant:eq1}, \eqref{def:sigcumulant:eq2}, are invariant under time-domain reparametrisations of $X$, see Lemma \ref{sect:sigmoments:lem1} \ref{sect:sigmoments:lem1:it2.1}. Hence, the statistics \eqref{sect:capping:normalised_seriescapped} of an interpolant $Y\equiv\hat{X}_{\mathcal{I}}$ depend only on the time series $X_{\mathcal{I}}$ --- i.e.\ on the random variables $Z_1\coloneqq X_{t_0}, \ldots, Z_{n}\coloneqq X_{t_{n-1}}$ and their sequential order --- and not on the dissection along which $X_{\mathcal{I}}=(Z_j)_{j\in[n]}$ is interpolated. In symbols, see Appendix \ref{rem:pwlinterpol} \eqref{rem:pwlinterpol:inj} for notation,   
\begin{equation}\label{rem:def:sigergodicity:eq1}
\sig(\hat{X}_{\mathcal{I}}) \ = \ \sig(\hat{\iota}_{\mathcal{J}}(Z_1, \ldots, Z_n)) \qquad\text{for any dissection $\mathcal{J}$}
\end{equation} 
with cardinality $|\mathcal{J}|=|\mathcal{I}|$. This justifies to abstract from the topology of the time-indices $t\in\mathcal{I}$ in \eqref{sect:finsamlim:eq2}, as done in the formulation of Definition \ref{def:sigergodicity} below.\\[-0.5em]       

All expectations in the following definition are assumed to exist.
\begin{definition}[Signature Ergodicity]\label{def:sigergodicity}
Let $X_\ast=(X_j)_{j\in\N}$ be a discrete time-series in $\R^d$ and $n, m\in\N$. We call $X_\ast$ \emph{$m^{\mathrm{th}}$-order signature ergodic to length $n$} if, almost surely,
\begin{equation}\label{def:sigergodicity:eq2}
\E\big[\phi(X_{[n]})\big] \ = \ \lim_{T\rightarrow\infty}T^{-1}\sum_{j=1}^T \phi(X_{(n(j-1):\,nj]}) \quad \text{for} \quad  \phi(z)\coloneqq\sig_{[m]}(\hat{z}_{\mathcal{E}_n})\,,   
\end{equation} 
and $X_\ast$ will be called \emph{weakly} $m^{\mathrm{th}}$-order signature ergodic to length $n$ if \eqref{def:sigergodicity:eq2} holds in probability.   
We call the process $X_\ast$ [weakly] \emph{signature ergodic} to length $n$ if $X$ is [weakly] $m^{\mathrm{th}}$-order signature ergodic to length $n$ for every $m\geq 1$.  

Given $\Theta\subseteq C(D_{X_\ast};\R^d)$, we call $X_\ast$ [weakly/\,$m^{\mathrm{th}}$-order] signature ergodic to length $n$ \emph{on $\Theta$} if the respective property holds for each $\theta(X_\ast)\coloneqq\big(\theta(X_j)\big)_{\!j\in\N}$, $\theta\in\Theta$.        
\end{definition} 
We refer to the LHS of \eqref{def:sigergodicity:eq2} as the \emph{$[m]^{\mathrm{th}}$-signature moment} of the batch $(X_1, \ldots, X_n)$. 
\begin{remark}\label{def:sigergodicity:rem}
\begin{enumerate}[label=(\roman*)]
\item\label{def:sigergodicity:rem:it1} In other words, the time-series $X=(X_j)_{j\in\N}$ is [weakly] $m^{\mathrm{th}}$-order signature ergodic to length $n$ iff the sequence of empirical path-space measures (on $\mathcal{B}(\mathcal{C}_d)$)
\begin{equation} 
\hat{\mu}_T \ \coloneqq \ \frac{1}{T}\sum_{j=1}^T\delta_{\hat{X}_j} \quad\text{ for }\quad \hat{X}_j\coloneqq\hat{\iota}_{\mathcal{E}_n}(X_{n(j-1)+1}, \ldots, X_{nj}) 
\end{equation}
yields a consistent estimator for the expected signature $\mathfrak{S}_{[m]}(\hat{X}_1)$ of $\hat{X}_1$, that is iff
\begin{equation}\label{def:sigergodicity:rem:eq2}
\mathfrak{S}_k(\hat{X}_1) \ = \ \lim_{T\rightarrow\infty}\int_{\mathcal{C}_d}\!\sig_k(x)\,\hat{\mu}_T(\mathrm{d}x) \quad \text{a.s.\ \ [in probab.]}  
\end{equation} 
for each $1\leq k\leq m$. Notice that due to \eqref{rem:def:sigergodicity:eq1}, the equidistant dissection $\mathcal{E}_n$ in \eqref{def:sigergodicity:eq2} may be replaced by any other $[0,1]$-dissection of the same cardinality.
\item\label{def:sigergodicity:rem:it2} A time-series $(X_t)_{t\in\mathcal{J}_k}$ for $\mathcal{J}_k$ as in \eqref{sect:finsamlim:eq2}, is called \emph{$m^{\mathrm{th}}$-order signature ergodic} if a.s.\   
\begin{equation}\label{def:sigergodicity:rem:eq2.1}
\E\big[\tilde{\phi}(\hat{X}_{\mathcal{I}^{(k)}_1})\big] \ = \ \lim_{T\rightarrow\infty}T^{-1}\sum_{\nu=1}^T \tilde{\phi}(\hat{X}_{\mathcal{I}_\nu^{(k)}}) \quad \text{for} \quad  \tilde{\phi}\coloneqq\sig_{[m]}(\,\cdot\,)\,, 
\end{equation}
for $\hat{X}_{\mathcal{I}}\coloneqq\hat{\iota}_{\mathcal{I}}(X_{\mathcal{I}})$ the piecewise-linear interpolation of $X_{\mathcal{I}}\equiv(X_t)_{t\in\mathcal{I}}$ along $\mathcal{I}\subset\R$. The remaining notions of Definition \ref{def:sigergodicity} carry over analogously. Notice that in consequence of Lemma \ref{sect:sigmoments:lem1} \ref{sect:sigmoments:lem1:it2.1}, the above notions \eqref{def:sigergodicity:rem:eq2.1} of signature ergodicity for protocol-indexed time-series are in fact a special case of Definition \ref{def:sigergodicity}: see Lemma \ref{def:sigergodicity:rem:lem_unnecessary}. Hence also for time-series of this protocol-indexed kind, the results of this section all apply as stated upon replacing their respective ergodicity assumptions by their \eqref{def:sigergodicity:rem:eq2.1}-type counterparts.               
\end{enumerate}            
\end{remark} 
  
\noindent
Let as before the space $C(D_{X_\ast};\R^d)$ be endowed with the compact-open topology. Using the universality of the signature transform (Lemma \ref{sect:sigmoments:lem1} \ref{sect:sigmoments:lem1:it6}), we find that the [weak] signature ergodicity of $X_\ast=(X_j)_{j\in\N}$ is passed onto $\theta(X_\ast)=\big(\theta(X_j)\big)_{\!j\in\N}$ for any $\theta\in C(D_{X_\ast};\R^d)$.\footnote{\ Proposition \ref{lem:ergodicity_theta} and Lemma \ref{lem:ergodicity_uniformconv} are proved in the Appendices \ref{pf:ergodicity_theta} and \ref{pf:ergodicity_uniformconv}, respectively.}    
\begin{restatable}{proposition}{propergodicitytheta}\label{lem:ergodicity_theta}
Let $\Theta\subseteq C(D_{X_\ast};\R^d)$ and $X_\ast=(X_j)_{j\in\N}$ be a discrete time-series in $\R^d$ with compact spatial support and such that for each $\theta\in\Theta$ the expectations 
\begin{equation}
\E[\sig_m(\hat{X}^\theta_1)] \ \text{ exist \ for all $m\geq 1$, \ with } \hat{X}_1^\theta \text{ the interpolant of } \theta(X_1), \ldots, \theta(X_n). 
\end{equation}It then holds that: if $X_\ast$ is \emph{[}weakly\emph{]} signature ergodic to length $n$, then $X_\ast$ is \emph{[}weakly\emph{]} signature ergodic to length $n$ on $\Theta$.   
\end{restatable} 

\noindent
Using a Glivenko-Cantelli type result yields the \mbox{following observation of uniform convergence.}\\[-0.5em] 

For the lemma below, let $\Theta$ be as in Assumption \ref{sect:capping:assumptions} and $\kappa^{[m]}_{\mathrm{IC}}$ as in \eqref{sect:capping:normalised_seriescapped}, and for $n\in\N$ denote by $(\mathcal{I}_{n|j})_{j\in\N}$ any fixed sequence of $[0,1]$-dissections with $|\mathcal{I}_{n|j}|=n$ for all $j\in\N$.\\[-1.25em]     
\begin{restatable}[Ergodicity Limit]{lemma}{lemergodicityuniformconv}\label{lem:ergodicity_uniformconv}
Let $X_\ast=(X_j)_{j\in\N}$ be a discrete time-series which for some $m$ is \emph{[}weakly\emph{]} $m^{\mathrm{th}}$-order signature ergodic to some length $n\in\N$ on $\Theta$, and denote 
\begin{align}\label{lem:ergodicity_uniformconv:eq1}
{\mathfrak{K}}^{m|n|T}\!(\theta)\,&\coloneqq\, \log_{[m]}(\hat{\mathfrak{S}}^{m|n}_T(\theta)) \quad\text{ for }\quad \hat{\mathfrak{S}}^{m|n}_T(\theta)\coloneqq \frac{1}{T}\sum_{j=1}^T\sig_{[m]}(\hat{X}_j^\theta),\\ \label{lem:ergodicity_uniformconv:eq1.1}
\text{and}\footnotemark \quad \bar{\mathfrak{K}}^{m|n|T}_{\bm{i}}\!(\theta)&\coloneqq\frac{{\mathfrak{K}}^{m|n|T}_{\bm{i}}\!(\theta)}{\big({\mathfrak{K}}^{m|n|T}_{11}\!(\theta)\big)^{\!\eta_1(\bm{i})/2}\cdot\ldots\cdot\big({\mathfrak{K}}^{m|n|T}_{dd}\!(\theta)\big)^{\!\eta_d(\bm{i})/2}}, \ \ \bm{i}\,\in\,{[d]}^\star,      
\end{align}\footnotetext{\ Where $\eta_\nu(\bm{i})$ denotes the number of times the index-entry $\nu$ appears in $\bm{i}$; cf.\ \eqref{def:sigcumulant:eq2}.}  
for each $\theta\in\Theta$, where $\hat{X}^\theta_j$ is the interpolant of $\theta(X_{n(j-1)+1}), \ldots, \theta(X_{nj})$ along $\mathcal{I}_{n|j}$. 

\noindent 
For any $m\geq 2, n,T\in\N$ and $\theta\in\Theta$, denote further (recalling Notation \ref{notation:index_sum})  
\begin{equation}\label{lem:ergodicity_uniformconv:eq1.2}
\hat{\kappa}^{m|n}_T(\theta) \, \coloneqq \, \sum_{\nu=2}^m\sum_{\bm{q}\in\mathfrak{C}_\nu}{\bar{\mathfrak{K}}}^{m|n|T}_{\bm{q}}\!(\theta)^2.
\end{equation} 
Provided that $\E\big[\!\sup_{\theta\in\Theta}\big\|\sig_k(\hat{X}^\theta_1)\big\|_k\big] < \infty$ for each $k\in[m]$, it then holds that
\begin{equation}\label{lem:ergodicity_uniformconv:eq2}
\bar{\kappa}^{[m]}_{\mathrm{IC}}(\hat{X}^\theta_1) \, = \, \lim_{T\rightarrow\infty}\hat{\kappa}^{m|n}_T(\theta) \quad \text{uniformly on \ $\Theta$} \quad \text{ a.s.\ \ \emph{[}in probability\emph{]}}.   
\end{equation}     
\end{restatable}   

\noindent
A detailed study of the class of (weakly) signature-ergodic stochastic processes is beyond the scope of this article, but Section \ref{pf:sigergodicity1} and the examples below show that the ergodicity assumption \eqref{def:sigergodicity:rem:eq2} is met for many popular time series models and stochastic processes.
  
\begin{definition}[Ergodic Observations]\label{def:ergodic_observation}
For $\tilde{X}=(\tilde{X}_t)_{t\geq 0}$ a continuous stochastic process in $\R^d$ and $\mathcal{J}=(\mathcal{J}_k)_{k\in\N}\subset 2^{[0,\infty)}$ an exhaustive protocol with base lengths $(n_k)_{k\in\N}$, 
\begin{enumerate}[label=(\roman*)]
\item the pair $(\tilde{X},\mathcal{J})$ will be called an \emph{ergodic observation} if for almost all $k\in\N$, 
\begin{equation}\label{def:ergodic_observation:eq1}
(\tilde{X}_t)_{t\in\mathcal{J}_k} \ \ \text{ is} \quad \text{signature ergodic to length $n_k$},
\end{equation}
and $(\tilde{X}, \mathcal{J})$ will be called \emph{ergodic*} if in addition the spatial support of $\tilde{X}$ is compact;
\item the pair $(\tilde{X},\mathcal{J})$ will be called a \emph{weakly} ergodic observation if for almost all $k\in\N$, 
\begin{equation}\label{def:ergodic_observation:eq2}
(\tilde{X}_t)_{t\in\mathcal{J}_k} \ \ \text{ is} \quad \text{weakly signature ergodic to length $n_k$} 
\end{equation}and the running maximum of $|\tilde{X}|$ has finite expectation,\footnote{\ The integrability of the running maximum of $|\tilde{X}|$ is discussed in, e.g., \citep[Chapter 13]{boucheron2013} and \citep{marcus1972sample}} i.e.\ $\E[\sup_{t\in[0,1]}\!|\tilde{X}_t|]<\infty$. 
\end{enumerate}
Given a finite-horizon process $X=(X_t)_{t\in\I}$, we call a pair $(\tilde{X}, \mathcal{J})$ a [weakly] $\text{ergodic}^{[\ast]}$ observation \emph{of $X$} if it is a [weakly] $\text{ergodic}^{[\ast]}$ observation and $(\tilde{X}_t)_{t\in\I}=X$ almost surely, and we call $(\tilde{X}, \mathcal{J})$ a [weakly] $\text{ergodic}^{[\ast]}$ observation of $X$ \emph{on $\Theta\subseteq C(D_{\!\tilde{X}};\R^d)$} if in addition the pair $\big(\theta(\tilde{X}), \mathcal{J}\big)$ is a [weakly] $\text{ergodic}^{[\ast]}$ observation for each $\theta\in\Theta$ individually.            
\end{definition}   

\begin{examples} Lemma \ref{lem:sigergodicity1} implies that the [strong resp.\ weak] ergodicity assumptions \eqref{def:sigergodicity:eq2} resp.\ \eqref{def:sigergodicity:rem:eq2.1} are satisfied by a large number of time series and continuous stochastic processes $\tilde{X}$ (and with it by $\theta(\tilde{X})$ for $\theta\in\Theta$), for the latter by way of \eqref{def:ergodic_observation:eq1} resp.\ \eqref{def:ergodic_observation:eq2} and Lemma \ref{def:sigergodicity:rem:lem_unnecessary} via any protocol $\mathcal{J}$ chosen such that $\tilde{X}_{\mathcal{J}}$ is appropriately stationary. These include, adequate stationarity provided (cf.\ e.g.\ Definition \ref{appendixdef:seasonal_increments}),       
\begin{enumerate}[label=\arabic*.] 
\item (trivially) all $q$-dependent time series (e.g.\ all moving-average processes of finite degree);
\item various linear and related processes such as certain [MC]ARMA, ARCH and GARCH models (see, e.g., \citep{fryzlewicz2011mixing, mokkadem1988,pham1985,schlemm2012});
\item many Markov processes, diffusions and stochastic dynamical systems \mbox{(e.g.\ \citep{denker1989,masuda2007ergodicity,rosenblatt1971,xiaohong2010});}    
\end{enumerate}
see e.g.\ \cite{BR2} for an overview. In practice however, infringements of the above (sufficient) conditions for signature ergodicity may typically be innocuous, cf.\ Section \ref{sect:experiments}.     
\end{examples}    

\subsection{The Consistency Limit}\label{sect:consistlim}The considerations of Subsections \ref{sect:consoverview} to \ref{sect:ergodicity} combine to the following consistency result for our ICA-method (Theorem \ref{thm:optimisation}).   
\begin{theorem}[Consistency]\label{thm:consistency}
Let $X, S$ and $\Theta$ be as in Theorem \ref{thm:optimisation} and Assumption \ref{sect:capping:assumptions}, and let $(\tilde{X},(\mathcal{J}_k)_{k\in\N})$ be an ergodic* \emph{[}resp.\ weakly ergodic on $\Theta$\emph{]} observation of $X$ with base lengths $(n_k)_{k\in\N}$. Suppose that there is $\theta_\star\in\Theta$ such that $\theta_\star(X)$ is IC. Then for any error bound $\varepsilon>0$ there exists a capping threshold $m_0\geq 2$ such that for any fixed $m\geq m_0$ the following holds: There is a mesh-index $k_0=k_0(m)\in\N$ such that for any $k\geq k_0$ and any sequence $(\hat{\theta}^\star_T)$ in $\Theta$ with the property that, for $\hat{\kappa}^{m|n_k}_{T}$ as in \eqref{lem:ergodicity_uniformconv:eq1.2} but computed from $X_\ast \equiv \tilde{X}_{\mathcal{J}_k}$,
\begin{equation}\label{thm:consistency:eq1}
\hat{\kappa}^{m|n_k}_{T}\!\big(\hat{\theta}_T^\star\big)\ \leq \  \min_{\theta\in\Theta}\,\hat{\kappa}^{m|n_k}_{T}(\theta) \ + \ \eta_T \qquad (T\in\N)
\end{equation} 
for some $(\eta_T)\subset\R_+$ with $\lim_{T\rightarrow\infty}\eta_T=0$ almost surely \emph{[}resp.\ in probability\emph{]}, it holds that       
\begin{equation}\label{thm:consistency:eq2}
\lim_{\tau\rightarrow\infty}\max\left\{\sup_{T\geq\tau}\Big[\mathrm{dist}_{\|\cdot\|_\infty}\!\big(\hat{\theta}^\star_T(X),\,\mathrm{DP}_d\cdot S\big)\Big], \, \varepsilon\right\} \ = \ \varepsilon    
\end{equation}  
almost surely \emph{[}resp.\ in probability\emph{]}. If $(\tilde{X}, (\mathcal{J}_k)_{k\in\N})$ is ergodic on $\Theta$ and the spatial support of $X$ is not necessarily compact, then \eqref{thm:consistency:eq2} holds almost surely with the above threshold $m_0$ depending on the realisation of $\tilde{X}$.              
\end{theorem} 
The above theorem shows that the optimality-based inversion scheme \eqref{thm:optimisation:eq1} is provably robust under a variety of approximations arising in statistical practice.
\begin{remark}In particular,\footnote{\ Since the maximum norm and the Euclidean norm on $\R^d$ are equivalent.} Theorem \ref{thm:consistency} gives the following guarantee: Provided that the hyperparameters $m$ (capping threshold) and $k$ (observation frequency) are chosen large enough, the minimizers \eqref{thm:consistency:eq1} of the empirical signature contrasts $\hat{\kappa}^{m|n}_T$ from \eqref{lem:ergodicity_uniformconv:eq1.2} will, in the infinite-data limit $T\rightarrow\infty$, recover each component of the source to an arbitrarily high $\I$-uniform precision (up to order and monotone scaling). In other words, as the grid of observational time-points gets finer $(k\rightarrow\infty)$ and the length of the observed time series increases $(T\rightarrow\infty)$, our method produces a signal that gets uniformly closer to the unobserved source.
\end{remark}     
    
\begin{proof}[Proof of Theorem \ref{thm:consistency}] 
For brevity, only the statement for $(\tilde{X}, (\mathcal{J}_k))$ ergodic* is proved here; a proof of the remaining non-compact [weakly] ergodic case is given in Appendix \ref{pf:thm:consistency:add}.   
So let $(\tilde{X}, (\mathcal{J}_k))$ be an ergodic observation of $X$ and assume that $D_X$ is compact.   

Making the $(m,k,T)$-dependence of each minimizer $\hat{\theta}_T^\star$ in \eqref{thm:consistency:eq1} explicit by writing $\hat{\theta}_T^\star\eqqcolon\theta^{m|k}_T$, fix any $\theta_\star\in\Theta$ with $\theta_\star(X)$ IC and observe that assertion \eqref{thm:consistency:eq2} follows from the claim: 
\begin{equation}\label{thm:consistency:aux1}
\begin{gathered}
\forall\,\tepsilon>0\,:\, \exists\, m_0\geq 2\,:\, \text{for each } m\geq m_0 \text{ there is } k_0\equiv k_0(m) \text{ such that\,:}\\
\lim_{\tau\rightarrow\infty}\alpha_\tau^{m|k} \vee \tepsilon = \tepsilon \text{ \ \ a.s.\ \ \ with \ \ } \alpha_\tau^{m|k}\coloneqq\sup\nolimits_{T\geq\tau}\tilde{d}(\theta_T^{m|k}\!, \Theta_\star), \  \text{ for each } k\geq k_0;    
\end{gathered} 
\end{equation}here: $a\vee b\coloneqq\max\{a,b\}$ and $\tilde{d}$ denotes the topology-inducing metric \eqref{pf:thm:consistency:add:eq1} on $\Theta$, and
\begin{equation}
\Theta_\star \,\coloneqq\, \left\{\beta\circ\theta_\star \ \middle| \ \beta\in\mathfrak{M}_\Theta\right\} \quad \text{for}\quad \mathfrak{M}_{\Theta}\coloneqq\DP(D_{\theta_\star(X)})\cap\big[\Theta\circ\theta_\star^{-1}\big].
\end{equation}   
Indeed: Note first that for each $\theta\in\Theta$, we have that $\theta(X)$ is IC iff $\theta\in\mathfrak{M}_\Theta\cdot\theta_\star$. (For this, recall that if $\theta(X) \equiv (\theta\circ f)(S)$ is IC then $\tilde{\beta}\coloneqq\theta\circ f\in\DP(D_S)$ by [the respective proofs of] Theorems \ref{thm:NICA_stat} and \ref{cor:NICA_MainCor}; thus also the residual $\beta_\star\coloneqq\theta_\star\circ f : D_S \rightarrow D_{\theta_\star(X)}$ (cf.\ Lemma \ref{lem:spat_supp} \ref{lem:spat_supp:it1}) is in $\DP(D_S)$, whence the map $\beta\coloneqq\theta\circ\theta_\star^{-1} = \tilde{\beta}\circ\beta_\star^{-1}$ is both monomial and in $\Theta\circ\theta_\star^{-1}$, i.e.\ $\beta\in\mathfrak{M}_\Theta$, as claimed.) In other words (recall Proposition \ref{prop:sig_cums}), we have that 
\begin{equation}\label{thm:consistency:aux_new2}
\Theta_\star = {\operatorname{arg\,min}}_{\theta\in\Theta}\ \bar{\kappa}_{\mathrm{IC}}(\theta(X)).
\end{equation}
(In the following, we import the setting and notation of Subsection \ref{subsubsect:theta_metrizable} for usage below.)

Provided now that \eqref{thm:consistency:aux1} holds, we find that for every $\varepsilon>0$ there is $m_0\geq 2$ with the property that each capping index $m\geq m_0$ comes with a mesh-threshold $k_0=k_0(m)\in\N$ which is such that for each observation at mesh-level $k\geq k_0$ we have with probability one that there is a sequence of transformations $(\theta_T)_{T\in\N}$ in $\Theta_\star$ such that:
\begin{equation}\label{thm:consistency:aux2}
\big\|\theta^{m|k}_T(X) - \theta_T(X)\big\|_\infty \ \leq \ \varepsilon \quad \text{ for almost all $T\in\N$}, \quad \text{where } \ \big(\theta_T(X)\big)\subset\mathrm{DP}_d\cdot S   
\end{equation}with probability one due to \eqref{thm:consistency:aux_new2} and Theorem \ref{thm:optimisation}. This readily implies \eqref{thm:consistency:eq2} as desired. 

To derive \eqref{thm:consistency:aux2} from \eqref{thm:consistency:aux1}, let $\varepsilon>0$ be arbitrary, assuming $\varepsilon<1$ wlog, and note that $D_X\subseteq K_{\nu_0}$ for some $\nu_0\in\N$ as $D_X$ is compact. Observe then that \eqref{pf:thm:consistency:add:eq1.1} provides the inclusion $B^{d_{\nu_0}}_\varepsilon(\theta)\supseteq B^{\rho_{\nu_0}}_{\varepsilon/2}(\theta)$ for each $\theta\in\Theta$, while the definition of $\tilde{d}$ yields $\tilde{B}_{\tepsilon}(\theta)\subseteq B^{\rho_{\nu_0}}_{\varepsilon/2}(\theta)$ for $\tepsilon\coloneqq 2^{-\nu_0}\varepsilon/2$, cf.\ \eqref{pf:thm:consistency:add:eq1}. Given \eqref{thm:consistency:aux1}, this $\tepsilon$ comes with associated $m_0, m, k_0, k\in\N$ and a $\mathbb{P}$-full set $\Omega'\equiv\Omega'_{m,k}\in\footnote{\ For convenience, we may assume the underlying probability space $(\Omega, \mathscr{F}, \mathbb{P})$ to be complete.}\mathscr{F}$ such that                   
\begin{equation}\label{thm:consistency:aux2.1}
\alpha^{m|k}(\omega)\coloneqq\lim_{\tau\rightarrow\infty}\sup\nolimits_{T\geq\tau}\tilde{d}(\theta_T^{m|k}(\omega), \Theta_\star) \ \leq \ \tepsilon/2  \qquad \text{for each } \ \omega\in\Omega'
\end{equation}           
(note: $\alpha^{m|k}$ exists as $(\alpha^{m|k}_\tau)_{\tau\in\N}$ is monotone and bounded). Thus, for each $\omega\in\Omega'\cap\Omega''$ (for $\Omega''\in\mathscr{F}$ the $\mathbb{P}$-full set on which the traces of $X$ are all contained in $D_X$; Lemma \ref{lem:spat_supp} \ref{lem:spat_supp:it2}) there is $\tau_0\,(=\tau_0(\omega))\in\N$ together with a sequence $(\theta_T)_{T\in\N}\,(\equiv(\theta_T(\omega))_{T\in\N})\subset\Theta_\star$ such that: $\theta_T^{m|k}(\omega)\subseteq \tilde{B}_{\tepsilon}(\theta_T)\, \big(\subseteq B_{\varepsilon/2}^{\rho_{\nu_0}}(\theta_T) \subseteq B^{d_{\nu_0}}_\varepsilon(\theta_T)\big)$ for each $T\geq\tau_0$, and hence  
\begin{equation}\label{thm:consistency:aux2.2}
\big\|\theta^{m|k}_T(\omega)\big(X(\omega)\big) - \theta_T\big(X(\omega)\big)\big\|_\infty \leq \big\|\theta^{m|k}_T(\omega) - \theta_T\big\|_{K_{\nu_0}} \!= d_{\nu_0}(\theta^{m|k}_T(\omega), \theta_T) \ \leq \ \varepsilon    
\end{equation}          
for each $T\geq\tau_0$, which gives \eqref{thm:consistency:aux2} as desired. To prove \eqref{thm:consistency:aux1} next, let $\tepsilon>0$ be arbitrary. 

Then for $\kappa_{m,k}(\theta)\coloneqq\widehat{Q}_m(\hat{X}^\theta_{\mathcal{I}_1^{(k)}})$ as in \eqref{sect:interpollim:data:eq1} with $\mathcal{J}\eqqcolon(\mathcal{J}_k\equiv\sqcup_{\nu=1}^\infty\mathcal{I}_\nu^{(k)}\mid k\in\N)$ the given protocol under consideration, there is $m_0\geq 2$ such that for each $m\geq m_0$ it holds that    
\begin{equation}\label{thm:consistency:aux3} 
\exists\, k_0\,(\equiv k_0(m))\in\N \ : \quad \underset{\theta\in\Theta}{\operatorname{arg\,min}}\ \kappa_{m,k}(\theta) \, \subseteq \, \Theta_\star^{\tepsilon/2}\coloneqq\bigcup\nolimits_{\theta\in\Theta_\star}\!\tilde{B}_{\tepsilon/2}(\theta), \quad \forall\, k\geq k_0.    
\end{equation} 
Indeed: The identity \eqref{thm:consistency:aux_new2} (together with the fact that $\bar{\kappa}_{\mathrm{IC}}(\theta_\star(X)) = 0$) implies that $\zeta_\tepsilon\coloneqq\min_{\theta\in\Theta\setminus C_\tepsilon}Q(\theta)>0$ for\footnote{\ Provided that $\tepsilon$ is small enough such that $C_\tepsilon\subsetneq\Theta$, which can be assumed without loss of generality.} $C_\tepsilon\coloneqq \Theta^{\tepsilon/2}_\star\cap\Theta$ and $Q$ as in \eqref{sect:capping:qobjectives}, while Lemma \ref{sect:capping:lem1} \ref{sect:capping:lem1:it2} provides an $m_0\geq 2$ with $\sup_{m\geq m_0}\|Q - Q_m\|_{\Theta} < \zeta_\tepsilon/2$. Fixing any $m\geq m_0$ and using that $(\mathcal{J}_k)_{k\in\N}$ is exhaustive (whence the sequence $(\mathcal{I}_1^{(k)})_{k\in\N}$ is refined), Lemma \ref{lem:interpolim} yields some $k_0\equiv k_0(m)\in\N$ with $\sup_{k\geq k_0}\|Q_m - \kappa_{m,k}\|_\Theta<\zeta_\tepsilon/2$. Hence for any fixed $k\geq k_0$ and each $\tilde{\theta}\in\Theta$ with $\kappa_{m,k}(\tilde{\theta}) = \min_{\theta\in\Theta}\kappa_{m,k}(\theta)$, it holds that 
\begin{equation}
Q(\tilde{\theta}) = \kappa_{m,k}(\tilde{\theta}) + (Q-\kappa_{m,k})(\tilde{\theta})\,<\, \zeta_\tepsilon/2 + \zeta_\tepsilon/2 = \zeta_\tepsilon \quad \text{and hence} \quad \tilde{\theta}\in \Theta^{\tepsilon/2}_\star,  
\end{equation}    
where we used the fact that $0\leq \min_{\theta\in\Theta}\kappa_{m,k}(\theta)\leq\kappa_{m,k}(\theta_\star) = 0$ (i.e., the argmins of $\kappa_{m,k}$ on $\Theta$ coincide with its roots), where this last property follows from Proposition \ref{prop:sig_cums} and the obvious fact that if $\theta_\star(X)$ is IC then so is $\hat{X}_{\mathcal{I}_1^{(k)}}^{\theta_\star}$. Together with \eqref{thm:consistency:aux3} the identity $\operatorname{arg\,min}_{\Theta}(\kappa_{m,k}) = \{\theta\in\Theta\mid \kappa_{m,k}(\theta)=0\}\eqqcolon\mathcal{N}(\kappa_{m,k})$ now implies that, for $\bar{\kappa}_{m,k}(\theta)\coloneqq\bar{\kappa}^{[m]}_{\mathrm{IC}}\big(\hat{X}^\theta_{\mathcal{I}^{(k)}_1}\big)$ as in \eqref{sect:capping:normalised_seriescapped},    
\begin{equation}\label{thm:consistency:aux5}
\mathcal{M}\coloneqq\underset{\theta\in\Theta}{\operatorname{arg\,min}}\ \bar{\kappa}_{m,k}(\theta) \, \subseteq \, \Theta^{\tepsilon/2}_\star \quad \text{ for each } \ k\geq k_0,
\end{equation}  
because $\operatorname{arg\,min}_\Theta\bar{\kappa}_{m,k}=\mathcal{N}(\bar{\kappa}_{m,k})$ (as above) and the zero sets $\mathcal{N}(\bar{\kappa}_{m,k})$ and $\mathcal{N}(\kappa_{m,k})$ coincide (by Definition \ref{def:sigcumulant}). Next we claim that, for $m$ and $k$ as above, 
\begin{equation}\label{thm:consistency:aux6} 
\lim_{\tau\rightarrow\infty}\sup\nolimits_{T\geq\tau}\tilde{d}\big(\theta^{m|k}_T\!,\, \mathcal{M}\big) \, = \, 0 \qquad \text{almost surely}, 
\end{equation} 
which by way of \eqref{thm:consistency:aux5} implies \eqref{thm:consistency:aux1} as desired. To see \eqref{thm:consistency:aux6}, observe first that 
\begin{equation}\label{thm:consistency:aux7}
\lim_{T\rightarrow\infty}\bar{\kappa}_{m,k}(\theta^\star_T) \, = \, 0 \quad \text{almost surely}, \quad \text{with} \ \ (\theta^\star_T)\equiv(\theta^{m|k}_T)
\end{equation}   
as in \eqref{thm:consistency:aux6}, which due to $\left.\eqref{lem:ergodicity_uniformconv:eq2}\right|_{X_\ast\equiv\tilde{X}_{\mathcal{J}_k}}$ follows by the same arguments that led to \eqref{sect:capping:lem1:aux9}. Invoking a proof by contradiction, assume now that \eqref{thm:consistency:aux6} does not hold. Then, pointwise on an event of positive probability, there will be $\delta_0>0$ together with a subsequence $(T_j)_{j\in\N}\subseteq\N$ such that $\tilde{d}(\theta^\star_{T_j}, \mathcal{M})\geq\delta_0$ for each $j\in\N$. But as $\Theta$ is compact, we (upon passing to a convergent subsequence) may assume that $(\theta^\star_{T_j})_{j\in\N}$ converges to some $\theta_0\in\Theta$. Then by continuity $\bar{\kappa}_{m,k}(\theta_0)=\lim_{j\rightarrow\infty}\bar{\kappa}_{m,k}(\theta^\star_{T_j})$, whence $\bar{\kappa}_{m,k}(\theta_0)=0$ by \eqref{thm:consistency:aux7} and thus $\theta_0\in\mathcal{M}$. The latter is a contradiction, however, as $(\theta^\star_{T_j})_{j\in\N}$ is bounded away from $\mathcal{M}$ (by $\delta_0$), proving \eqref{thm:consistency:aux6}.                      
\end{proof} 

\newpage
\subsection{Algorithm}\label{sect:algorithm}The computational procedures of this section can be summarized into the following practical algorithm whose consistency is established by Theorem \ref{thm:consistency}.\\[-0.5em] 

\begin{algorithm}[H]
\vspace{0.5em}
\begin{center} 
\begin{minipage}[c]{0.925\textwidth}
\begin{enumerate}
\reqnomode
\item[\phantom{\textbf{1.}}]\textbf{Goal:} For $X = f(S)$ with a contin.-/discrete-time process $S$ in $\R^d$, invert $X$ for $S$.\\[-0.75em] 
\item[\phantom{\textbf{1.}}]\emph{Hyperparameters:} candidate nonlinearities $\Theta$, capping order $m_0$ (as in \eqref{sect:capping:normalised_seriescapped}),
\item[\phantom{\textbf{1.}}]\phantom{Hyperparameters:} base length $n$ ($\coloneqq |\mathcal{I}_1^{(k)}|$, as in \eqref{sect:finsamlim:eq2}), observation horizon $T$ (as\\[-1em]
\item[\phantom{\textbf{1.}}]\phantom{Hyperparameters:} in \eqref{lem:ergodicity_uniformconv:eq1.2}, or $T\equiv \max\{j\geq 1\mid \mathcal{I}^{(k)}_j\leq T_k\}$ for $T_k$ as in Rem.\ \ref{rem:finite_samples_flexibility}).\footnote{\ Optional: weights $(w_{\bm{q}})$ as in Remark \ref{rem:thm_optimisation} \ref{rem:thm_optimisation:it5}.}\\[-0.75em]  
  
\item[\textbf{1.}]\textbf{Input:} sample observation $\mathfrak{x}\equiv(\mathfrak{x}_j)$ of $X$ (as in \eqref{sect:finsamlim:eq2}, with index $(k)$ omitted). 

\item[\textbf{2.}] Compute the estimated contrast $\hat{\phi}\coloneqq \hat{\kappa}^{m_0|n}_T$ as in \eqref{lem:ergodicity_uniformconv:eq1.2}, that is compute 
\begin{gather}
\hat{\phi}(\theta) = \sum_{\nu=2}^{m_0}\sum_{\bm{q}\in\mathfrak{C}_\nu}\bar{\varphi}_{\bm{q}}(\theta)^2  \quad\text{ with\footnotemark }\quad \bar{\varphi}_{\bm{i}}(\theta)\coloneqq\frac{\varphi_{\bm{i}}(\theta)}{\varphi_{11}(\theta)^{\!\tfrac{\eta_1(\bm{i})}{2}}\!\!\cdots\,\varphi_{dd}(\theta)^{\!\tfrac{\eta_d(\bm{i})}{2}}} 
\end{gather}
\begin{subequations}  
\begin{align}
\text{where }\qquad
\varphi_{\bm{i}}(\theta)&\coloneqq\big\langle \hat{\mathscr{C}}(\theta),\bm{i}\big\rangle \quad (\bm{i}\in[d]^\star)\label{Algo:SigNICA:eq1.1}\\
\text{and \hspace{0.5em}}\qquad\hat{\mathscr{C}}(\theta)&\coloneqq \log_{[m_0]}\!\!\left[T^{-1}\sum_{j=1}^T\sig_{[m_0]}\big(\hat{\mathfrak{x}}_j^\theta\big)\right]\!,\label{Algo:SigNICA:eq1.2}
\end{align}  
\end{subequations}
for $\hat{\mathfrak{x}}_j^\theta$ the piecewise-linear interpolation of the data $\theta(\mathfrak{x}_j)\equiv\big(\theta(X_t(\omega))\mid t\in\mathcal{I}_j\big)$. 

\item[\textbf{3.}] Compute a minimiser $\theta_\star$ of $\hat{\phi}$ over $\Theta$, that is find 
\begin{equation}\label{Algo:SigNICA:eq2}  
\theta_\star\,\in\,\underset{\theta\in\Theta}{\operatorname{arg\ min}}\ \hat{\phi}(\theta).
\end{equation} 

\item[\textbf{4.}] Compute the estimated source realisation $\hat{\mathfrak{s}}\coloneqq\theta_\star(\mathfrak{x})\equiv\big(\theta_\star(X_t(\omega))\mid t \in \mathcal{I}_1\big)$.\\[-1em]  

\item[\textbf{5.}]\textbf{Output:} $\hat{\mathfrak{s}}$ and $\theta_\star$.
\leqnomode 
\end{enumerate}
\end{minipage} 
\end{center}
    \caption{Nonlinear ICA via Signature Cumulants}\label{Algo:SigNICA}  
\end{algorithm}
\footnotetext{\ Recall Notation \ref{notation:index_sum} and that $\eta_\nu(\bm{i})$ equals the number of times the index-value $\nu$ appears in $\bm{i}$, for example $\eta_3(123433235) = 4$. Recall further that $\langle\hat{\mathscr{C}}(\theta),\bm{i}\rangle$ denotes the $\bm{i}^{\mathrm{th}}$ entry of the multiindexed list $\hat{\mathscr{C}}(\theta)$ (cf.\ \eqref{rem:free_algebra}).}{\ }\\[-1em]   

\noindent
The above algorithm involves two independent subroutines, namely the computation of the free logarithm of averages of signatures of piecewise-linearly interpolated data batches \eqref{Algo:SigNICA:eq1.2} followed by the subsequent extracion of its relevant [cross-shuffle-indexed] coefficients \eqref{Algo:SigNICA:eq1.1}, and the optimisation \eqref{Algo:SigNICA:eq2} of the contrast $\hat{\phi}$ over a given set $\Theta$ of candidate demixing transformations. The first of these routines can be conveniently implemented by use of the functionality provided with specialised signature libraries such as \cite{KGL}, while the second task can be performed with great flexibility by choosing $\Theta$ as an artificial neural network that has $\hat{\phi}$ as its loss function. The minimiser $\theta_\star$ can then be learnt as an optimal network configuration reached by training $(\Theta, \hat{\phi})$ via backpropagation, cf.\ e.g.\ Remark \ref{rem:thm_optimisation} \ref{rem:thm_optimisation:it0.2} and Section \ref{sect:experimentsII}.\\[-0.5em] 

Several example applications of the above method are detailed in Section \ref{sect:experiments} and on the public repository \cite{githubSigNICA}, where an implementation of the above algorithm, including a differentiable (i.e.\ backpropagatable) implementation of the contrast function $\hat{\phi}$, is also provided.

\section{Numerical Experiments}\label{sect:experiments}
\noindent
We present a series of numerical examples to illustrate the practical applicability of our ICA method on discrete- and continuous-time signals. A complete account of the following experiments and results, including their full parameter settings and all relevant implementations and estimates, is provided on the public repository \cite{githubSigNICA}.       

\subsection{A Performance Index for Nonlinear ICA}
As before, we consider stochastic processes $X$ and $S$ in continuous or discrete\footnote{\ See Section \ref{appendix:NICA_discrete} for an explicated treatment of the latter.} time such that
\begin{equation}\label{intext:BSS_relation:num}
X = f(S) \quad \text{ for some } \ f\in C^{2,2}(D_S).
\end{equation}
In order to assess how close an estimate $\hat{S}\equiv\hat{S}(X)$ of $S$ is to the true source $S$ in \eqref{intext:BSS_relation:num}, we propose to quantify the distance between $\hat{S}$ and the orbit\footnote{\  See \eqref{intext:monomial_orbit} for notation, and recall that the elements of $\DP\cdot S$ are in a minimal distance from $S$.} $\DP\cdot S$ by way of the following intuitive\footnote{\ Recall the classical facts (e.g.\ \cite{embrechts2002}) that Kendall's (and Spearman's) rank correlation coefficient $\rho_{\mathrm{K}}$ attains its extreme values $\pm 1$ iff one of its arguments is a monotone transformation of the other, with $\rho_{\mathrm{K}}(U,V)=0$ if its arguments $U$ and $V$ are independent.} performance statistic (cf.\ Remark \ref{rem:thm_optimisation} \ref{rem:thm_optimisation:it1} for applicability).    
\begin{definition}[Monomial Discordance]\label{def:MonConc}
\hspace{-1em} Given two time series $\mathcal{X}\coloneqq(X^1_t, \cdots, X^d_t)_{t\in\mathcal{I}}$ and $\mathcal{Y}\coloneqq(Y^1_t, \cdots, Y^d_t)_{t\in\mathcal{I}}$ in $\R^d$ for $\mathcal{I}$ finite, define the \emph{concordance matrix of $(\mathcal{X}, \mathcal{Y})$} as 
\begin{equation}
\mathcal{C}(\mathcal{X}, \mathcal{Y}) \ \coloneqq \ \left(\frac{1}{|\mathcal{I}|}\sum_{t\in\mathcal{I}}|\rho_{\mathrm{K}}(X_t^i, Y_t^j)|\right)_{\!\!(i,j)\in[d]^2} \ \in \ [0,1]^{d\times d}
\end{equation} 
where $\rho_{\mathrm{K}}$ is the Kendall\footnote{\ If preferred, $\rho_{\mathrm{K}}$ might alternatively be chosen as Spearman's rank correlation coefficient.} rank correlation coefficient. Furthermore, we define 
\begin{equation}\label{def:MonConc:eq2}
\varrho(\mathcal{X}, \mathcal{Y}) \ \coloneqq \ \frac{1}{\sqrt{d(d-1)}}\min_{P\in \mathrm{P}_{\!d}}\|C(\mathcal{X}, \mathcal{Y}) - P\|_2 \, \in \, [0,1]
\end{equation}
and call this quantity the \emph{monomial discordance} of $\mathcal{X}$ and $\mathcal{Y}$. 
\end{definition}\vspace{-1em} 
\begin{restatable}{proposition}{propmondiscordance}\label{prop:monconcordance}
Let $X$ and $S$ be as in \eqref{intext:BSS_relation:num} with $S$ IC, and $h$ be $C^{1}$-invertible on some open superset of $D_X$. Then for  $\mathcal{I}\subset\mathbb{I}$ finite and $\varrho$ as in \eqref{def:MonConc:eq2}, we have that:  
\begin{equation}\label{prop:monconcordance:eq1}
\big(h(X_t)\big)_{t\in\mathcal{I}} \ \in \ \DP\cdot (S_t)_{t\in\mathcal{I}} \quad\text{ iff }\quad \varrho\big((h(X_t))_{t\in\mathcal{I}}, \, (S_t)_{t\in\mathcal{I}}\big)=0.  
\end{equation} 
\end{restatable}
\begin{proof} 
See Appendix \ref{pf:monconcordance}.
\end{proof}
Hence the smaller the monomial discordance between $S$ and a transformation $h(X)$ of its observable, the closer to optimal will be the deviation between $h(X)$ and $S$.\\[-0.5em]

\noindent
Below we provide a brief synopsis of our experiments and the results that we obtained. For brevity, the truncated approximations \eqref{sect:capping:normalised_seriescapped} of the above contrast $\bar{\kappa}_{\mathrm{IC}}$ will be denoted $\phi_{m_0}$.     
\subsection{Nonlinear Mixings With Explicitly Parametrized Inverses}\label{sect:experimentsI} First we consider three families of $C^{2}$-diffeomorphisms on the plane whose inverses are explicitly parametrized.\\[-0.5em]

\noindent
More specifically: We sample two types of source processes in $\R^2$, namely: an IC Ornstein-Uhlenbeck process $S_\mathrm{ou}=(S^1_\mathrm{ou}, S^2_\mathrm{ou})$, and an IC copula-based time-series $S_\mathrm{cy}=(S^1_\mathrm{cy}, S^2_\mathrm{cy})$ that follows the dependence model \eqref{TimeSeriesCopulaModel:eq1}.\footnote{\ With $F^{S_\mathrm{cy}^i}_t$ chosen as the cdf of $\mathcal{N}(0,1)$ and $c$ chosen as the Clayton-density (cf.\ Proposition \ref{prop:TSCopulaProp1} \ref{prop:TSCopulaProp1:item1}).} Both $S_\mathrm{ou}$ and $S_\mathrm{cy}$ are contrastive by Prop.\ \ref{cor1:GPsAreVaried} \ref{cor1:GPsAreVaried:item2} and Cor.\ \ref{cor:TSCopulaProp} \ref{prop:TSCopulaProp1:item1}, respectively.\footnote{\ Note further that the Ornstein-Uhlenbeck processes $S_{\mathrm{ou}}$, while continuous-time by nature, are processed as discrete-time observations according to their classical Euler-Maruyama approximation. The copula-based time-series $S_{\mathrm{cy}}$, on the other hand, are simulated at their observation frequency and thus showcase the applicability of our method to discrete-time signals (in accordance with Section \ref{appendix:NICA_discrete}).} These sources are first mapped to the square $[-1,1]^2$ upon centering and scaling them to unit amplitude, and then transformed by one of three mixing maps $f_j : \R^2\rightarrow\R^2$ $(j=1,2,3)$ with increasing degree of `nonlinearity', see Figure \ref{fig:planar_maps}. (For an explicit definition of the $f_j$, see \cite{githubSigNICA}.) Figure \ref{fig:planar_tableau} shows the spatial trace of a sample realisation of $S_\mathrm{ou}$ and $S_\mathrm{cy}$ (panels (a) and (b)) next to an excerpt of the time-parametrised components of these realisations, together with their nonlinear mixtures $X_\eta^{(j)}\coloneqq f_j(S_\eta)$ for $j=1,2,3$ and $\eta=\text{`$\mathrm{ou}$'}$ (panels (c), (e), (g)) and $\eta=\text{`$\mathrm{cy}$'}$ (panels (d), (f), (h)).\\[-0.5em]              

\begin{figure}
\centering
\hspace*{-1.25em}
\includegraphics[trim={20em 2em 2em 0},clip,scale=0.3]{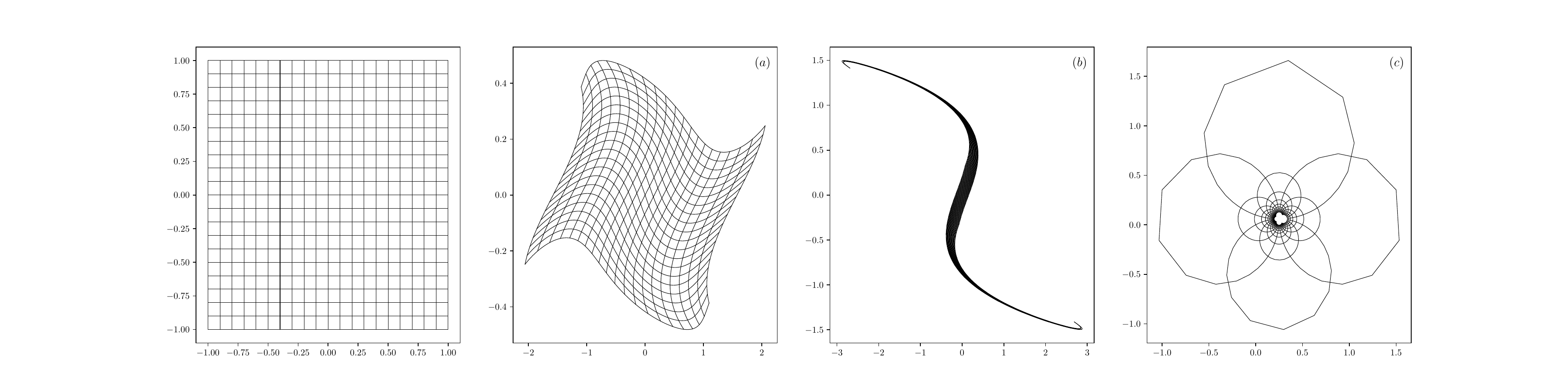}
\caption{\emph{The image of the square $[-1,1]^2$ (leftmost) under three increasingly nonlinear mixing transformations $f_1, f_2, f_3$, namely conjugates of the Hénon map ($f_1$ and $f_2$; panels (a) and (b), respectively)  and of the Möbius transformation ($f_3$; panel (c)).}}\label{fig:planar_maps}
\end{figure}

\noindent
Each of the `true' inverses $g^j\coloneqq f_j^{-1}$ $(j=1,2,3)$ are contained in an (injectively parametrized) family $\Theta_j\equiv\{g^j_\theta\in C^{2}(\R^2)\mid \theta\in\tilde{\Theta}_j\}$ of candidate de-mixing transformations $g^j_\theta$, where $\tilde{\Theta}_j\subseteq\R^2$ is some open parameter set. On these parameter sets, we consider the data-based objective functions 
\begin{equation}\label{sect:experimentsI:eq1}
\Phi^j_\eta\, : \, \tilde{\Theta}_j \rightarrow\R, \qquad \theta \,\mapsto\,\phi_{m_j}(g^j_\theta(X^{(j)}_\eta)), 
\end{equation}       
with $\phi_m\coloneqq \bar{\kappa}^{[m]}_{\mathrm{IC}}$ as in \eqref{sect:capping:normalised_seriescapped} and capped at the cumulant orders $m_1=m_2=m_3 = 6$, and compare the topography of the functions \eqref{sect:experimentsI:eq1} to that of the monotone discordances    
\begin{equation}\label{sect:experimentsI:eq2}
\delta^j_\eta\, : \, \tilde{\Theta}_j \rightarrow\R, \qquad \theta \,\mapsto\,\varrho(g^j_\theta(X^{(j)}_\eta), S_\eta) \qquad (\text{cf.\ \eqref{def:MonConc:eq2}}).
\end{equation}
Recall that the latter are `distance functions' that quantify how much a candidate source estimate $\hat{S}^\theta_\eta\coloneqq g^j_\theta(X^{(j)}_\eta)$ deviates [from the monomial orbit $\DP\cdot S_\eta$ of $S_\eta$, that is] from the true source $S_\eta$ up to order and monotone scaling of its components.\\

\noindent 
The results are displayed in the first three columns of Figure \ref{fig:explicitinvs_results}, with the `estimator's view' $\Phi^{1|2|3}_\mathrm{ou}$ of the demixing performance shown in the top-row panels and the `true view' $\delta^{1|2|3}_\mathrm{ou}$ of the demixing performance shown in the bottom-row panels.\footnote{\ For brevity, Figure \ref{fig:explicitinvs_results} shows the case $\eta=\text{$\mathrm{ou}$}$ only; the results for the case $\eta=\text{$\mathrm{cy}$}$ can be found in \cite{githubSigNICA}.} This shows clearly that within the given families $\Theta_j$ of candidate transformations, those candidate nonlinearities which map the data $X^{(j)}_\eta$ to a best-approximation of its source $S_\eta$ are precisely those that minimise the contrast \eqref{sect:experimentsI:eq1}, as asserted by Theorem \ref{thm:optimisation}.\\[-0.5em] 

\noindent 
An analogous experiment $(j=4)$ is performed for a mixing transformation $f_4 : \R^3\rightarrow \R^3$, see Figures \ref{fig:3D_transform} and \ref{fig:3D_mixings}. The results, obtained for a contrast capped at cumulant order $m_4=7$ and shown as the rightmost column of Figure \ref{fig:explicitinvs_results}, are again affirmative of Theorem \ref{thm:optimisation}.

\begin{figure} 
\centering
\hspace*{-1.25em}
\includegraphics[trim={15em 15em 5em 15em},clip,scale=0.35]{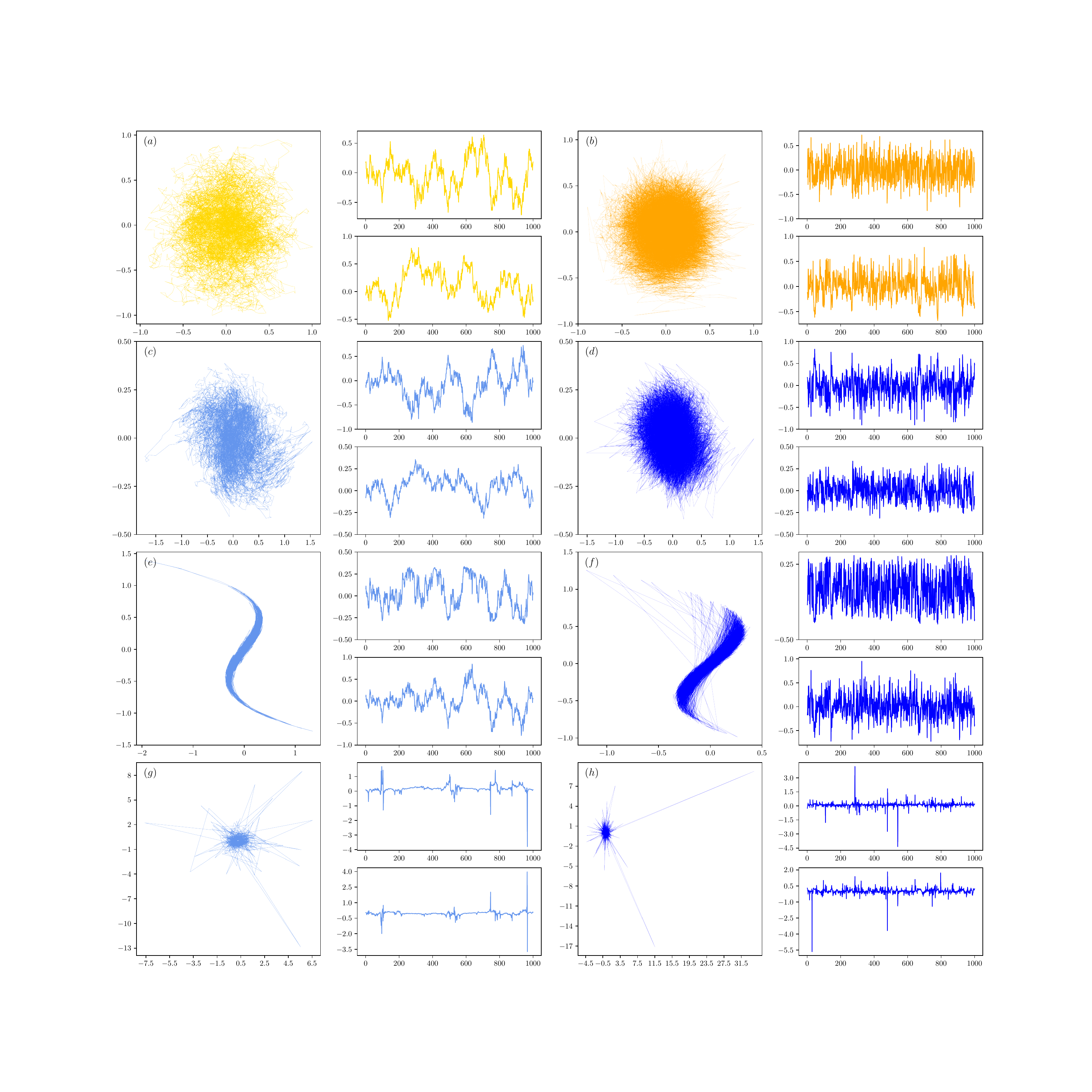}
\caption{\emph{Spatial traces and sampled components of three nonlinear mixtures of the sources $S_\mathrm{ou}$ and $S_\mathrm{cy}$ (panels (a) and (b), respectively). Depicted are the mixtures $X^{(1)}_\mathrm{ou}$ and $X^{(1)}_\mathrm{cy}$ ((c) and (d)), $X^{(2)}_\mathrm{ou}$ and $X^{(2)}_\mathrm{cy}$ ((e) and (f)), and $X^{(3)}_\mathrm{ou}$ and $X^{(3)}_\mathrm{cy}$ ((g) and (h)). The components of the mixtures, excerpted over 1000 data points each, are shown to the right of each panel.}}\label{fig:planar_tableau} 
\end{figure}

\begin{figure}
\centering
\hspace*{-2.5em}
\includegraphics[trim={15em 5em 15em 0em},clip,scale=0.3]{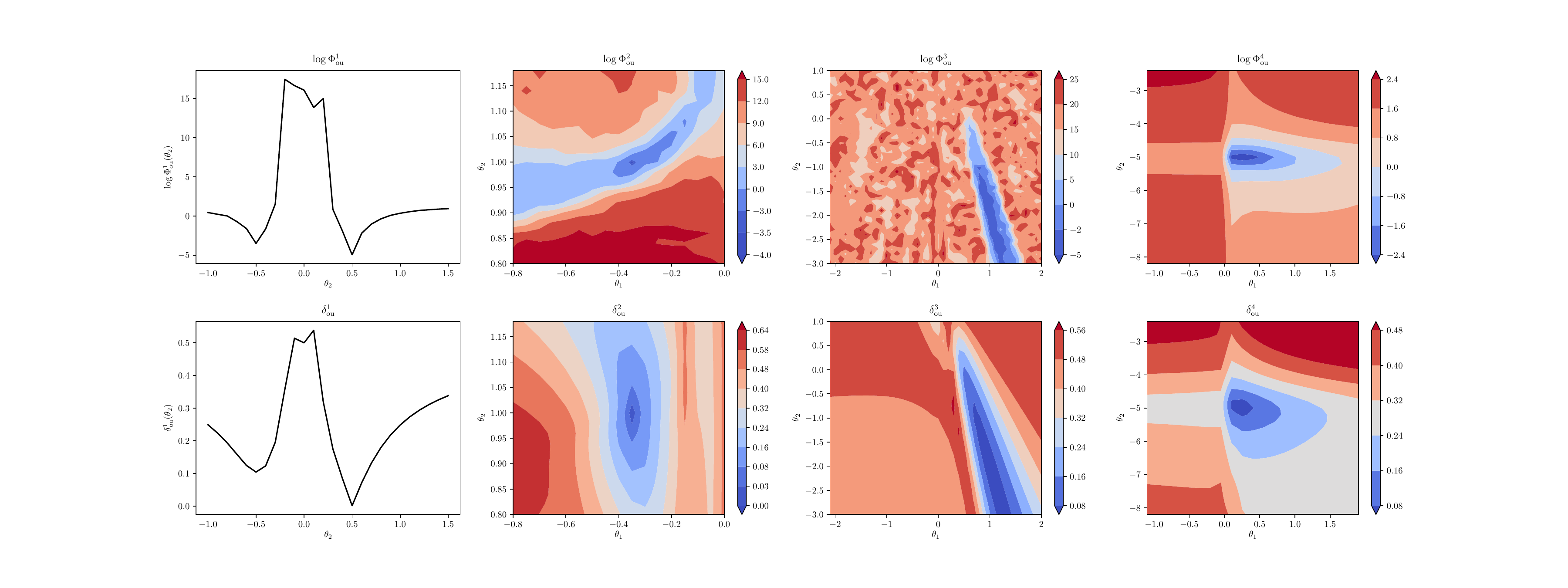}
\caption{\emph{Contour plot (leftmost column) and heatmaps of the log-transformed contrast functions \eqref{sect:experimentsI:eq1} (top row) and of the associated discordance functions \eqref{sect:experimentsI:eq2} (bottom row) for the mixings $X_\mathrm{ou}^{(j)}= f_j(S_\mathrm{ou})$, $j=1,\ldots, 4$. The parameters $\theta_\star^{(j)}\equiv(\theta_1^{(j)},\theta_2^{(j)})$ of the true inverses $f_j^{-1}\equiv g^j_{\theta_\star^{(j)}}\in\Theta_j$ are $\theta_\star^{(1)}=0.5$,\protect\footnotemark \ $\theta_\star^{(2)}=(-0.35,1)$, $\theta_\star^{(3)}=(1,-2)$, and $\theta_\star^{(4)}=(0.2,-5)$.}}\label{fig:explicitinvs_results}
\end{figure}

\begin{figure}
\centering
\subfloat{{\includegraphics[trim={10em 15em 5em 15em},clip,scale=0.225]{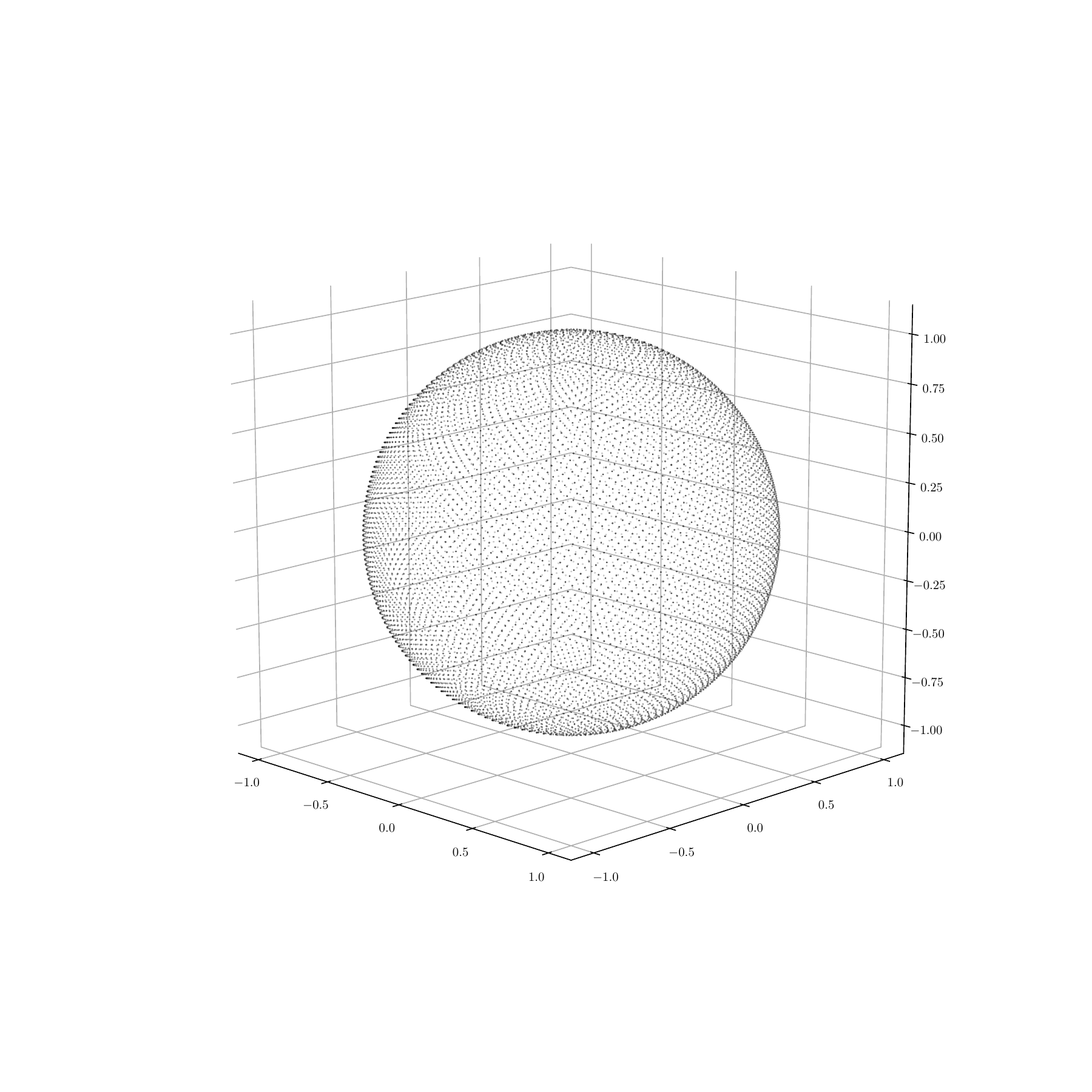} }}%
\qquad
\subfloat{{\includegraphics[trim={10em 15em 5em 15em},clip,scale=0.225]{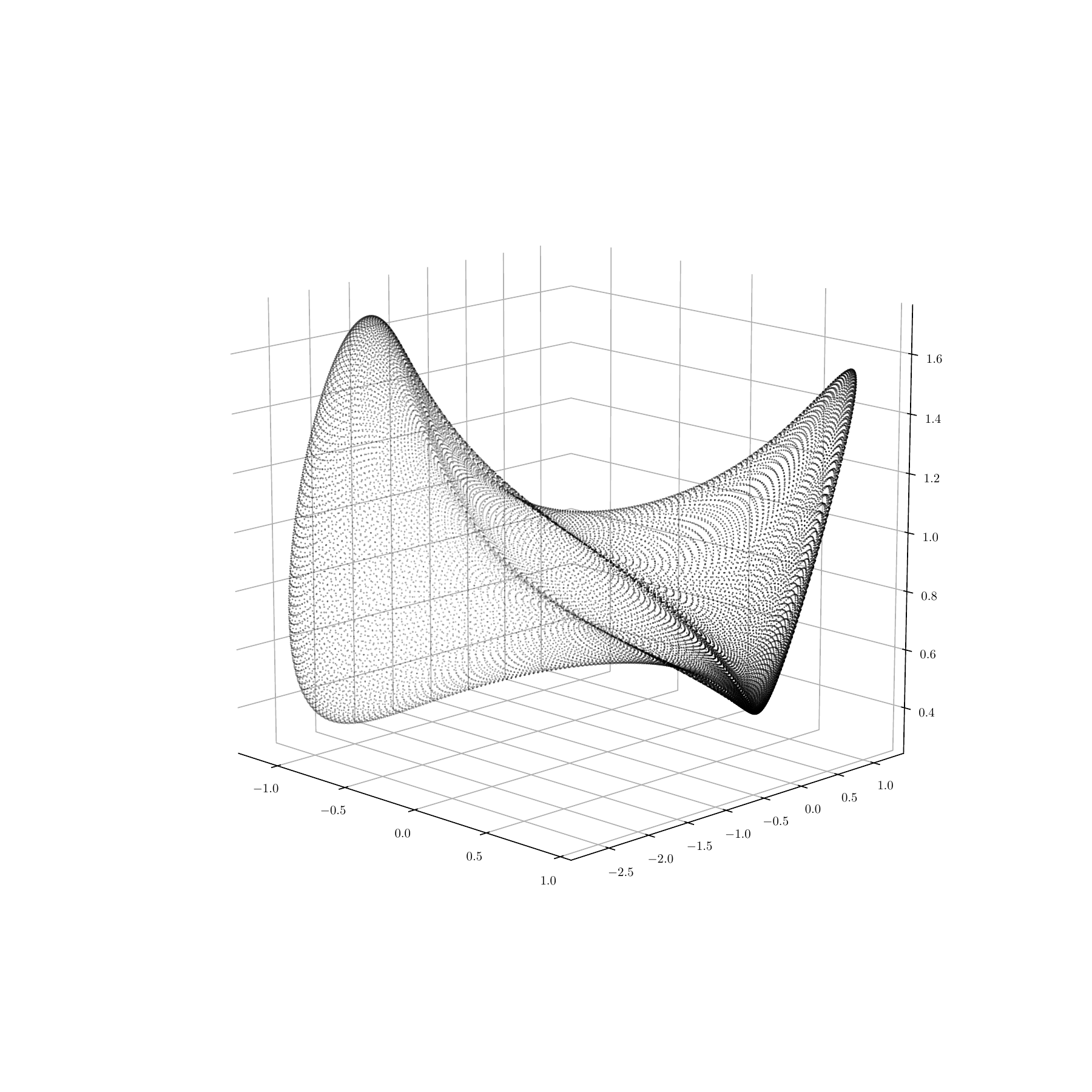} }}
\caption{\emph{Illustration of the three-dimensional mixing transform $f_4:\R^3\rightarrow\R^3$ via its action $f_4(S^2)$ (right panel) on the 2-sphere $S^2\equiv\{x\in\R^3\mid |x|=1\}$ (left panel).}}\label{fig:3D_transform}
\vspace*{\floatsep}
\hspace*{-6em}
\includegraphics[trim={15em 15em 5em 15em},clip,scale=0.225]{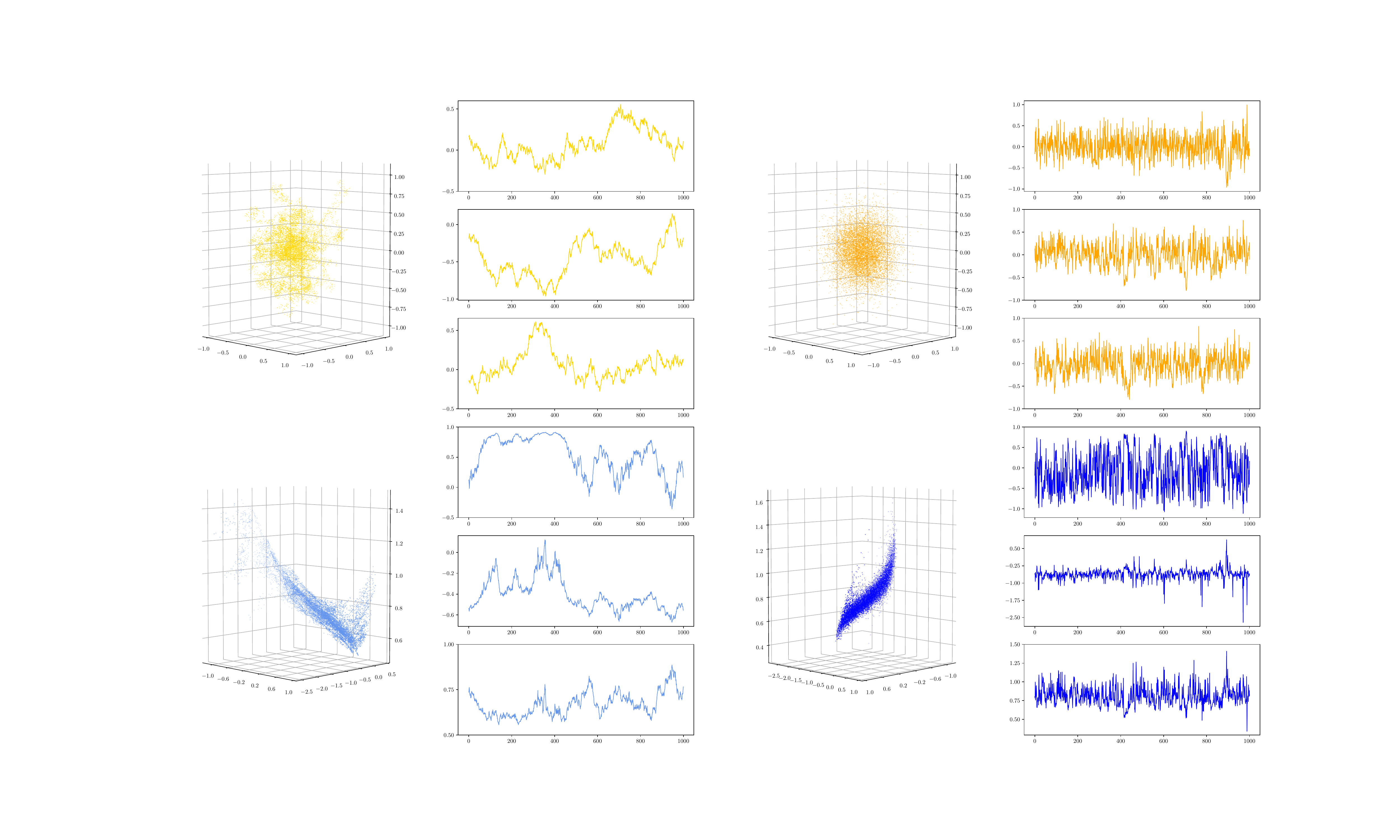}
\caption{\emph{Spatial trace and sampled components of a three-dimensional IC Ornstein-Uhlenbeck process $\tilde{S}_\mathrm{ou}$ (top left) and an IC copula-based time-series model $\tilde{S}_\mathrm{cy}$ (top right) and their respective nonlinear mixtures $f_4(\tilde{S}_\mathrm{ou})$ (bottom left) and $f_4(\tilde{S}_\mathrm{cy})$ (bottom right).}}\label{fig:3D_mixings}
\end{figure}

\noindent
\subsection{Nonlinear Mixings With Inverses Approximated By Neural Networks}\label{sect:experimentsII}The practical applicability of our ICA-method is illustrated by running the optimisation \eqref{thm:optimisation:eq1} over (approximate) demixing-transformations which are modelled by an artificial neural network.\\[-0.5em] 

\noindent
More specifically: We subject two Ornstein-Uhlenbeck sources $S^{(1)}$ and $S^{(2)}$ with two resp.\ four independent components to a two- resp.\ four-dimensional nonlinear mixing transform (see \cite{githubSigNICA} for details). The resulting mixtures $X^{(1)}$ and $X^{(2)}$ are then passed on to candidate demixing-nonlinearities $g^\nu_\theta\in\Theta_\nu$ which are given as elements of the parametrized families
\begin{equation}\label{sect:experimentsII:eq1}
\Theta_\nu \ \coloneqq \ \left\{g^{\nu}_\theta\,:\,\R^{2\nu}\rightarrow\R^{2\nu}\ \middle| \ g^\nu_\theta\text{ \ is an ANN with weights \ } \theta \in\tilde{\Theta}_\nu\right\} \qquad (\nu=1,2).
\end{equation}    
Here, the families of transformations $\Theta_\nu$ are spanned by the various configurations of some artificial neural network (ANN) instantiated over weight-vectors $\theta$ which are chosen from a given parameter set $\tilde{\Theta}_\nu$ in $\R^{m_\nu}$, where the number of weights $m_\nu$ is part of the pre-defined architecture of the ANN. Given these candidate-inverses, the optimisations \eqref{thm:optimisation:eq1} are run by 
\begin{equation}\label{sect:experimentsII:eq2}
\text{minimizing}\qquad\tilde{\Theta}_\nu \ \ni \ \theta \quad \longmapsto \quad \phi_{m_\nu}\!\big(g^\nu_\theta(X^{(\nu)})\big),
\end{equation}
i.e.\ by training each constituent ANN \eqref{sect:experimentsII:eq1} with the truncated contrast $\phi_{m_\nu}=\bar{\kappa}_{\mathrm{IC}}^{[m_\nu]}$ (cf.\ \eqref{sect:capping:normalised_seriescapped}) as its loss function, where the optimization steps are computed via backpropagation along the weights of the ANN. Technical details for the respective setups of \eqref{sect:experimentsII:eq1} and \eqref{sect:experimentsII:eq2} are reported in (\cite{githubSigNICA} and) Appendix \ref{appendix:sect:numerics_ann}.\\[-0.5em]  
 
\noindent
For the case $\nu=1$ we applied the mixing transformation depicted in Figure \ref{fig:NN2Dou} (leftmost panel), and for the case $\nu=2$ we followed the simulations of \cite{TCL,HYM} in using as a mixing transformation an invertible feedforward-neural network with four-nodal in- and output layers and two four-nodal hidden layers with $\tanh$ activation each.\footnotetext{\ Notice that: (a) by definition of $\Theta_1$, the function $\Phi^1_{\mathrm{ou}}$ depends on the one-dimensional parameter $\theta_2$ only; (b) as the concordance matrix of $\hat{S}_{-0.5}\coloneqq g^{(1)}_{-0.5}(X^{(1)}_\mathrm{ou})$ and $S_\mathrm{ou}$ is $\bigl( \begin{smallmatrix}0.053 & 0.929\\ 0.834 & 0.099\end{smallmatrix}\bigr)$ (indicating a close proximity between $\hat{S}_{-0.5}$ and $\DP\cdot S_\mathrm{ou}$, cf.\ Prop.\ \ref{prop:monconcordance}), the observation of $\Phi^{1}_{\mathrm{ou}}$ attaining a low local minimum at $-0.5$ is in accordance with Theorem \ref{thm:optimisation}.}\\[-0.5em]

\noindent
Denoting by $\theta_\nu^\ast\in\tilde{\Theta}_\nu$ the (local) optimum obtained by the minimisation of the objective \eqref{sect:experimentsII:eq2} and setting $\hat{S}^{(\nu)}\coloneqq g^\nu_{\theta_\nu^\ast}(X^{(\nu)})$ for the associated estimate of the source $S^{(\nu)}$ (cf.\ \eqref{thm:optimisation:eq1}), we as results to these experiments obtained the concordance matrices (cf.\ Definition \ref{def:MonConc}) 
\begin{align}
\mathcal{C}(\hat{S}^{(1)}, S^{(1)}) \ &\doteq \ \begin{pmatrix}
\bm{0.853} \quad & 0.065 \\
0.079 \quad & \bm{0.930}
\end{pmatrix} \qquad\text{ and } \label{sect:experimentsII:eq3.1}\\
\mathcal{C}(\hat{S}^{(2)}, S^{(2)}) \ &\doteq \ \begin{pmatrix}
\bm{0.834} \quad & 0.003 \quad & 0.037 \quad & 0.016 \\
0.148 \quad & \bm{0.725} \quad & 0.109 \quad & 0.069 \\
0.037 \quad & 0.034 \quad & \bm{0.803} \quad & 0.265 \\
0.077 \quad & 0.131 \quad & 0.072 \quad & \bm{0.787} 
\end{pmatrix}, \label{sect:experimentsII:eq3.2}
\end{align}
where we corrected for the permutation ambiguity between $\hat{S}$ and $S$ to simplify comparison.\\[-0.5em]

\noindent  
Both \eqref{sect:experimentsII:eq3.1} and \eqref{sect:experimentsII:eq3.2} indicate a good fit between $\hat{S}^{(\nu)}$ and $S^{(\nu)}$ in the sense that, to a good approximation, $\hat{S}^{(\nu)}$ and $S^{(\nu)}$ differ only up to (an inevitable permutation and) monotone scaling of their components,\footnote{\ Recall that the optimal deviation $\hat{S}^{(\nu)}\in\DP\cdot S^{(\nu)}$ between $\hat{S}^{(\nu)}$ and $S^{(\nu)}$ is achieved iff \eqref{sect:experimentsII:eq3.1} and \eqref{sect:experimentsII:eq3.2} are permutation matrices (Proposition \ref{prop:monconcordance}).} as stated by Theorem \ref{thm:optimisation}.  
A visual comparison of the original samples $S^{(1)}, S^{(2)}$ and their estimates $\hat{S}^{(1)}, \hat{S}^{(2)}$, see Figures \ref{fig:NN2Dou} and \ref{fig:NN4D}, confirms these results.\\[-0.5em]

\noindent
To reaffirm that the above results of finding good approximations to the source are not simply due to chance, we ran our experiments repeatedly with randomly chosen realisations and initial configurations for the data and the learning process \eqref{sect:experimentsII:eq1}$\,\&\,$\eqref{sect:experimentsII:eq2}, see \cite{githubSigNICA} and Figure \ref{fig:NN2Dbp}. The obtained discordances have mean $0.21$ and standard deviation $45\cdot 10^{-3}$ for the Ornstein-Uhlenbeck mixture (Fig.\ \ref{fig:NN2Dbp} (a)), and mean $0.18$ and standard deviation $55\cdot 10^{-3}$ for the copula-based time series mixture (Fig.\ \ref{fig:NN2Dbp} (b)). The associated average concordance matrix for this first (Ornstein-Uhlenbeck) type of mixtures is $\bigl( \begin{smallmatrix}0.78 & 0.10\\ 0.11 & 0.89\end{smallmatrix}\bigr)$, and for the latter (copula) type of mixtures it is $\bigl( \begin{smallmatrix}0.83 & 0.11\\ 0.11 & 0.88\end{smallmatrix}\bigr)$. The average discordance between the respective sources and their initial guesses $g_{\theta_0}(X)$ prior to applying  \eqref{sect:experimentsII:eq2} are at $0.59$ (a) and $0.56$ (b), respectively.\\[-0.5em] 

\begin{figure}
\centering
\hspace*{-1.25em}
\includegraphics[trim={0em 0 0em 0},clip,scale=0.3]{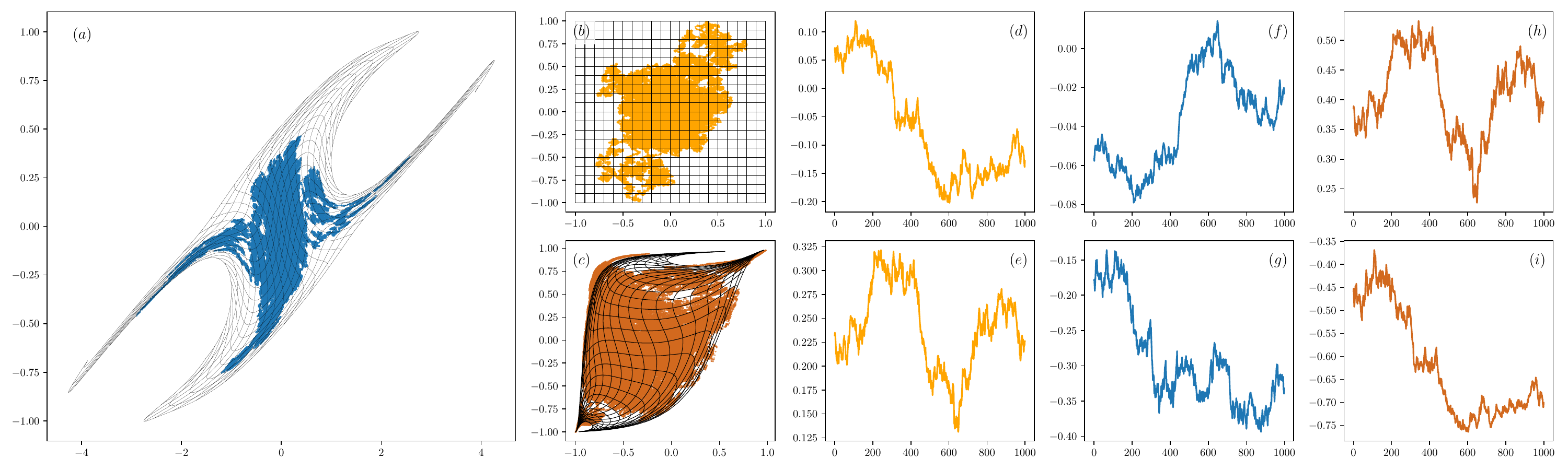}
\caption{\emph{Nonlinear mixture $X$ (sampled trace (a) and components (f), (g)) of an IC Ornstein-Uhlenbeck source $S$ ((b) and (d), (e)). Further shown is the residual $\restr{g\circ f}{[-1,1]^2}$ ((c); cf.\ \eqref{rem:method_in_practice:eq1}) for an estimate $g$ of $\restr{f^{-1}}{D_X}$. The function $g$ is found by optimising an articifial neural network $(g_\theta)$ via the loss function \eqref{sect:experimentsII:eq2}, and the resulting estimate $\hat{S}\coloneqq g(X)$ of $S$ is shown in brown ((c) and (h), (i)). To a good approximation, the source $S$ and its estimate coincide up to (a transposition and) a monotone scaling of their components, as quantified by \eqref{sect:experimentsII:eq3.1}.}}\label{fig:NN2Dou} 
\end{figure} 

\begin{figure}
\centering
\hspace*{-1.25em}
\includegraphics[trim={0em 0 0em 0},clip,scale=0.3]{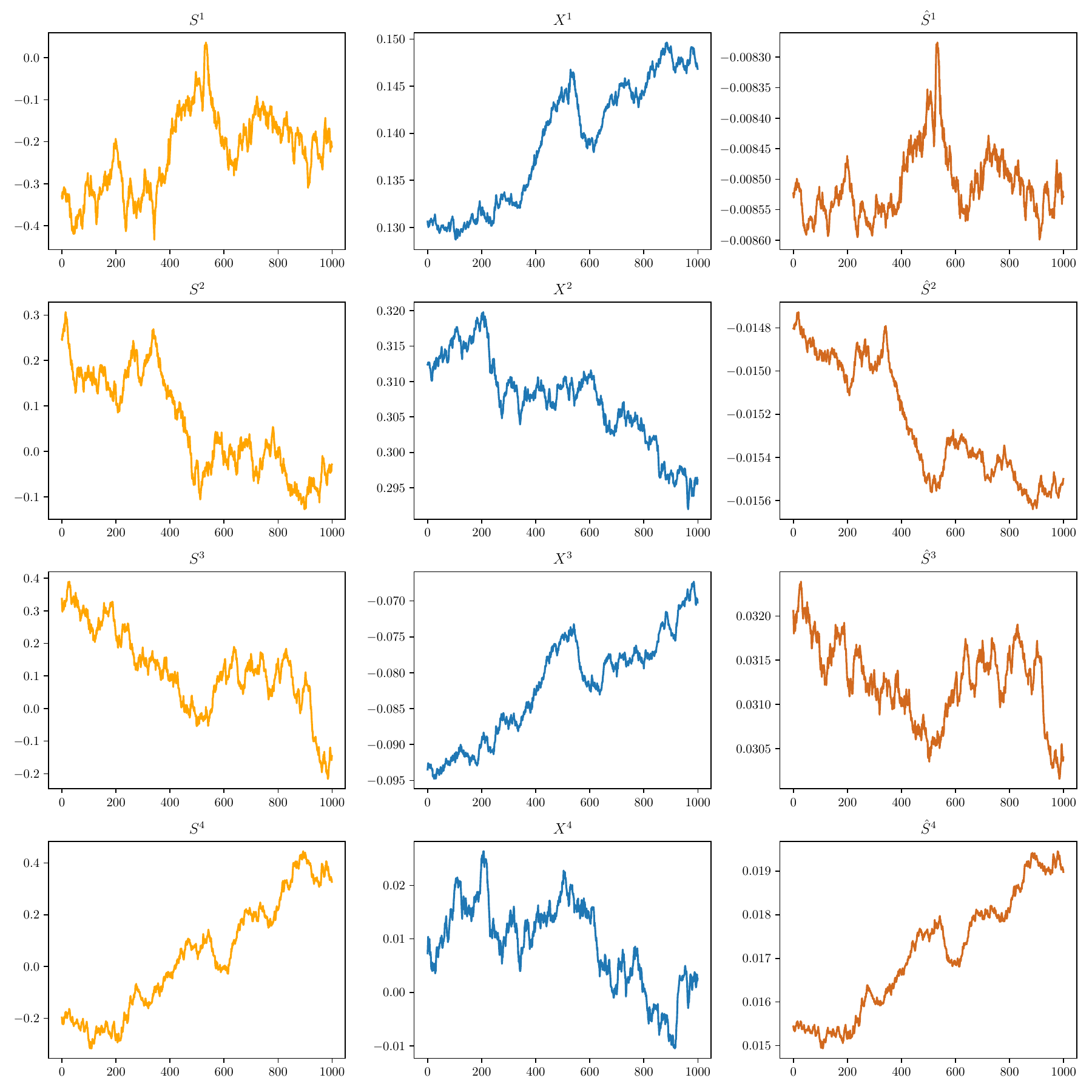}
\caption{\emph{Component processes, excerpted over 1000 data points each, of an IC Ornstein-Uhlenbeck process $S=(S^1,S^2,S^3,S^4)$ (orange), a nonlinear mixture $X=(X^1, X^2, X^3, X^4)$ (blue) of $S$, and an estimate $\hat{S}=(\hat{S}^1,\hat{S}^2,\hat{S}^3,\hat{S}^4)$ of $S$ (brown). The estimate $\hat{S}$ is obtained as $\hat{S}=g_{\theta_\star}(X)$ where $(g_\theta)$ is an ANN and $\theta_\star$ is a (local) minimum of the associated objective \eqref{sect:experimentsII:eq2}. To a good approximation, the processes $\hat{S}$ and $S$ coincide up to (a permutation and) a monotone scaling of their components. This is in accordance with Theorem \ref{thm:optimisation} and as quantified by \eqref{sect:experimentsII:eq3.2}.}}\label{fig:NN4D} 
\end{figure}

\noindent
These experiments underline the practical applicability of our proposed ICA-method.

To conclude, we note the following empirical findings.
\begin{remark}[Empirical Comments]\label{rem:method_in_practice}
\begin{enumerate}[label=(\roman*)]
\item\label{rem:method_in_practice:it1}Given an observable $X=f(S)$ together with a family $\Theta$ of candidate transformations on $\R^d$, the technical compatibility condition $\big(\restr{\DP(D_S)\cdot f^{-1}\big)}{D_X}\cap\restr{\Theta}{D_X}\neq\emptyset$ of Theorem \ref{thm:optimisation} can in practice typically not be guaranteed a priori. However, as indicated by the above findings \eqref{sect:experimentsII:eq3.1} and  \eqref{sect:experimentsII:eq3.2}, infringements of this (sufficient) technical condition might typically be innocuous, provided that at least 
\begin{equation}\label{rem:method_in_practice:eq1}
\big(\restr{\DP\cdot g\big)}{D_X}\cap\restr{\Theta}{D_X}\,\neq\, \emptyset \quad\text{for some \ $g$ \ with \ $\restr{g}{D_X}$ \! `close enough' to $\restr{f^{-1}}{D_X}$},
\end{equation}which will be satisfied if $\Theta$ is chosen large enough, say as a suitable ANN or another universal approximator. In a similar vein, our experiments indicate that the regularity condition $\Theta\subseteq C^{2}(D_X)$ may in practice be softened by merely requiring that the `approximate inverse' $g$ in \eqref{rem:method_in_practice:eq1} be `$C^{2}$-invertible on most of $D_X$' (cf.\ e.g.\ Figure \ref{fig:NN2Dou}, panel (c)) and the parametrization of $\Theta$ be `continuous' at (some) point $\tilde{g}\in\Theta$ with $\restr{\tilde{g}}{D_X}\in\DP\cdot\restr{g}{D_X}$, though this a priori reduces the optimisation \eqref{thm:optimisation:eq1} to the search for a (low) local minimum.                
\item\label{rem:method_in_practice:it2}We emphasize that the configurations of the neural networks and their backpropagation that we used in our experiments were ad hoc and not tuned for approximational optimality. Since the loss functions \eqref{sect:experimentsII:eq2} are typically non-convex with their topography crucially depending on the choice of \eqref{sect:experimentsII:eq1} (cf.\ e.g.\ Figure \ref{fig:explicitinvs_results}), we expect that the accuracy and efficiency of our estimates may be significantly improved by applying our ICA-method to ANN-based approximation schemes \eqref{sect:experimentsII:eq1}, \eqref{sect:experimentsII:eq2} which are more carefully designed.        
\end{enumerate} 
\end{remark} 

\begin{figure}
\centering
\hspace*{-1.25em}
\includegraphics[trim={0em 0 0em 0},clip,scale=0.5]{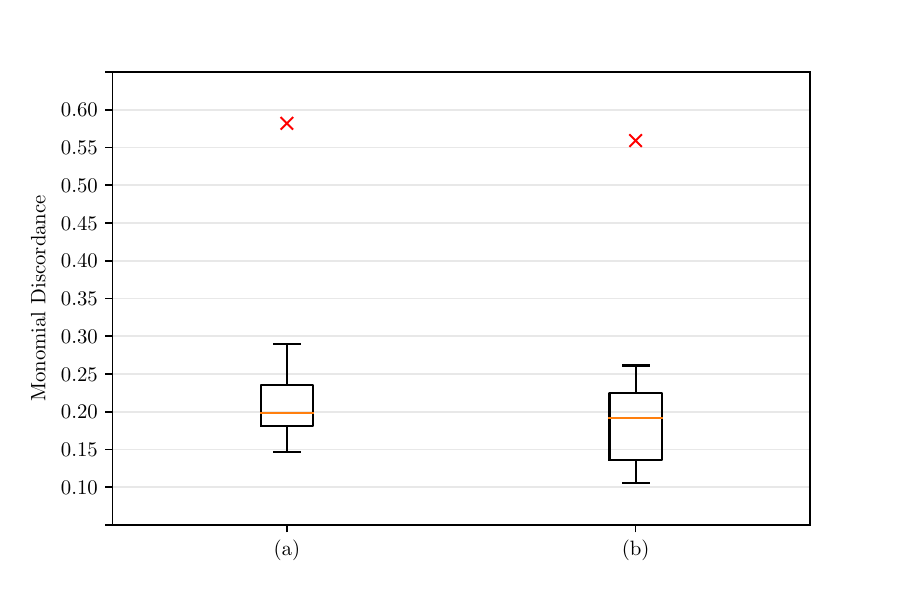}
\caption{\emph{Discrepancies (quantified by \eqref{def:MonConc:eq2}, boxplotted) between the true source $S$ and its \eqref{sect:experimentsII:eq2}-based estimates $\hat{S}$. The latter are computed from the nonlinear mixture $X$ shown in Figure \ref{fig:NN2Dou}. Displayed are the cases where: (a) the source $S$ an IC Ornstein-Uhlenbeck process, and (b) the source $S$ is a copula-based IC time-series (b), respectively. Shown in addition (red crosses) is the average discordance between $S$ and its initial estimates at the start of the optimisation \eqref{sect:experimentsII:eq2}. (The optimised discordances and their averages pre-learning are computed by applying our method to ten realisations of $S$ and ten initial configurations $g_{\theta_0}$ of the ANNs \eqref{sect:experimentsII:eq1} respectively, all drawn independently at random.)}}\label{fig:NN2Dbp} 
\end{figure}  

\newpage
\section{Conclusion}\label{sect:conclusion} 
\noindent
This paper has addressed the problem of Blind Source Separation via the classical approach of Independent Component Analysis. As our main contribution, we have formulated and proved a statistical method to recover multidimensional stochastic processes (in both continuous- as well as discrete-time) from observations of their nonlinear mixtures. Conceptually, our method assumes a source process with independent component processes and, by exploiting the temporal structure of this source, characterises its nonlinear transformations by the degree of intercomponental statistical dependence that they inflict on the source. Quantifying the latter by way of an efficiently computable contrast function derived from the signature cumulants of a stochastic process, the initial source separation problem may then be reformulated as a provably robust problem of optimisation-based function approximation which in practice can be conveniently implemented by, e.g., contemporary neural network-based learning schemes. A comprehensive consistency analysis ensures that the resulting method is usable in real-world situations (discretized time, one sample trajectory), which is further illustrated by a number of theoretical and numerical examples.                   

The mathematics of the identifiability theory established in this work appears flexible enough to allow for extensions in various further directions. For instance, by considering third-order in place of second-order finite-dimensional distributions it may be adapted to infer the identifiability of stochastic processes from their time-dependent nonlinear mixing transformations (`invertible flows'). By adapting the ideas of this paper further, it does now also seem within reach to prove the identifiability of stochastic sources from more general nonlinear relations, such as for instance in the setting of controlled differential equations where one may be interested to recover an (independent-component) stochastic control from its nonlinear response. As with most methods involving an optimisation over flexibly parametrisable nonlinearities, however, a significant practical caveat of our approach is the occurence of spurious local minima in the approximation of the demixing transformation. This leaves room for improvement that future research might explore: In addition to practical deliberations such as spanning the optimisation domain by more carefully designed learning architectures, or amplifying the contrast function by the addition of tunable hyperparameters such as weights attached to its summands, one may attempt to tame the critical optimisation task by adjusting it to (localised) polynomial approximations of the mixing nonlinearity and harvesting the additional algebraic structure that then results from the fact (\cite{COP}) that the signature transform `dualises' the action of polynomial transformations on its arguments.                                  

\section*{Acknowledgements}\noindent The authors would like to extend their gratitude to the Associate Editor, four anonymous referees, the Editor, and Aapo Hyvärinen for their very helpful comments and suggestions which helped to significantly improve the original version of this paper and its presentation. AS was financially supported by an Oxford-Cocker Graduate Scholarship and a Mathematical Institute Scholarship. HO is supported by the Hong Kong Innovation and Technology Commission (InnoHK Project CIMDA).   

\begin{appendix}
\section{Technical Lemmas, Proofs, and Remarks}\label{appendix:sect:add_proofs_and_rems}
\subsection{The Path-Space $\mathcal{C}_d$ is Cartesian}\label{appendix:sect:add_proofs_and_rems:cartesian_cd} 
Let us remark that the space $\mathcal{C}_d$ of paths from $\I$ to $\R^d$ (for $\I$ some fixed compact interval) inherits the Cartesian structure of $\R^d=\R\times\cdots\times\R$ via the canonical Banach-isometry $\psi : \mathcal{C}_1\times\cdots\times\mathcal{C}_1\rightarrow\mathcal{C}_d$ which is given by 
\begin{equation}\label{appendix:sect:add_proofs_and_rems:cartesian_cd:eq1}
\mathcal{C}_1^{\times d}\ \ni \ \big((\gamma^1_t)_{t\in\I}, \cdots, (\gamma^d_t)_{t\in\I}\big) \quad \longmapsto \quad \left[\I\,\ni\, t \ \mapsto \ \sum_{i=1}^d\gamma^i_t\cdot e_i\right] \ \in \ \mathcal{C}_d
\end{equation}for $(e_i)_{i\in[d]}$ the standard basis of $\R^d$. Indeed: It is clear that $\psi$ is a linear isometry with respect to the Banach norms $\|(\gamma^1, \cdots, \gamma^d)\|_\alpha\coloneqq \max_{i\in[d]}\|\gamma^i\|_\infty$ and $\|\gamma\|_\beta\coloneqq\max_{i\in[d]}\|\pi_i\circ\gamma\|_\infty$ on $\mathcal{C}_1^{\times d}$ and $\mathcal{C}_d$, respectively; as the norms $\|\cdot\|_\beta$ and $\|\cdot\|_\infty$ are equivalent on $\mathcal{C}_d$ (and hence induce the same Borel-structure on $\mathcal{C}_d$), the Banach isometry \eqref{appendix:sect:add_proofs_and_rems:cartesian_cd:eq1} extends to an isomorphism of Borel spaces
\begin{equation}
\big(\mathcal{C}_d, \|\cdot\|_\infty\big) \ \cong \ \big(\mathcal{C}_1\times\cdots\times\mathcal{C}_1, \|\cdot\|_\alpha\big).
\end{equation}
But since the factors $\mathcal{C}_1$ are separable metric spaces, so is $\mathcal{C}_1^{\times d}$, implying that (as $\|\cdot\|_\alpha$ induces the topology of componentwise convergence) $\mathcal{B}(\mathcal{C}_1^{\times d})=\mathcal{B}(\mathcal{C}_1)^{\otimes d}$, see e.g.\ \citep[Appendix M (p.\ 244)]{BIL}. Using now that $\mathcal{B}(\mathcal{C}_d)$ is generated by the family of cylinder sets, we (upon testing on those) obtain that the factor identity \eqref{def:stochastic_process:eq2} holds as desired.     

\subsection{Proof of Lemma \ref{lem:spat_supp}}\label{pf:lem:spat_supp}
\lemspatsupp*
\begin{proof}
\ref{lem:spat_supp:it1}\,: Recalling the support $\supp(\mu)$ of a Borel-measure $\mu : \mathcal{B}(\R^d)\rightarrow[0,1]$ to be defined as the smallest closed subset $C\subseteq\R^d$ having total mass $\mu(C)=1$, it is easy to see that $\supp(\mu\circ f^{-1}) = \overline{f(\supp(\mu))}$ for any continuous $f$ whose domain includes $C_\mu$. (Indeed: Denoting $C_\nu\coloneqq\supp(\nu)$ for brevity, we find $\mu(f^{-1}(C_{\mu\circ f^{-1}}))=1$ and hence $C_\mu\subseteq f^{-1}(C_{\mu\circ f^{-1}})$ and thus $\overline{f(C_\mu)}\subseteq C_{\mu\circ f^{-1}}$ (since both $f$ and the closure operation preserve set inclusion); conversely, for $\tilde{C}\coloneqq\overline{f(C_\mu)}$ we have $f^{-1}(\tilde{C}) \supseteq C_\mu$ and hence $\mu\circ f^{-1}(\tilde{C})=1$ and thus $C_{\mu\circ f^{-1}}\subseteq\tilde{C}$ as desired.) This implies that $D_{f(Y_t)} = \supp(\mathbb{P}_{Y_t}\circ f^{-1}) = \overline{f(D_{Y_t})}$ for each $t\in\I$, from which the assertion $D_{f(Y)} = \overline{f(D_Y)}$ is readily obtained via \eqref{def:spatial_support:eq1} and the continuity of $f$. (Indeed: On the one hand, $f(D_Y) \subseteq \overline{f(\bigcup_{t\in\I}D_{Y_t})} \subseteq \overline{\bigcup_{t\in\mathbb{I}}D_{f(Y_t)}} = D_{f(Y)}$ where the first inclusion holds by the fact that the continuous image of a closure is contained in the closure of the continuous image and the second inclusion holds by the above-established identity of fixed-time supports; on the other hand, the latter identity also implies that $D_{f(Y)} = \overline{\bigcup_{t\in\I}\overline{f(D_{Y_t})}} = \overline{\bigcup_{t\in\I}f(D_{Y_t})} = \overline{f(\bigcup_{t\in\I} D_{Y_t})} \subseteq \overline{f(D_Y)}$, where once more the last inclusion holds since the map $f$ and the closure operator both preserve set inclusions.)    

\ref{lem:spat_supp:it2}\,: By definition of $D_{Y_t}$, the event $\tilde{\Omega}_t\coloneqq Y_t^{-1}(D_{Y_t})$ has probability one for each $t\in\I$, and hence so does the countable intersection $\tilde{\Omega}\coloneqq\bigcap_{t\in\I\cap\mathbb{Q}}\tilde{\Omega}_t$. Now taking any $\tilde{\omega}\in\tilde{\Omega}$, we find by construction that the sample path $\gamma^{\tilde{\omega}}=Y(\tilde{\omega})$ satisfies $\operatorname{im}\!\big(\!\left.\gamma^{\tilde{\omega}}\right|_{\tilde{\I}}\big)\subset D_Y$ for $\tilde{\I}\coloneqq \I\cap\mathbb{Q}$. But since $\tilde{\I}$ is dense and the sample path $\gamma^{\tilde{\omega}}$ is continuous, we obtain $\mathrm{tr}(Y(\tilde{\omega}))=\operatorname{im}\!\big(\gamma^{\tilde{\omega}}\big)\subseteq D_Y$ by the fact that $D_Y$ is closed.

\ref{lem:spat_supp:it3}\,: We proceed by contradiction: If $\mathbb{P}(Y_t\in U)=0$ for each $t\in\mathbb{I}$, then $\mathbb{P}(Y_t\in U^c) = 1$ and thus (as $U^c$ is closed) $\supp(\mathbb{P}_{Y_t})\subseteq U^c$ for each $t\in\mathbb{I}$, yielding $D_Y\equiv\overline{\bigcup_{t\in\mathbb{I}}\supp(\mathbb{P}_{Y_t})}\subseteq U^c$ in contradiction to $U\subset D_Y$. 

\ref{lem:spat_supp:it4}\,: Recall that in our notation, $D_Y=\overline{\bigcup_{t\in\mathbb{I}}D_{Y_t}}^{|\cdot|}$ for $D_{Y_t}\coloneqq\supp(Y_t)$. Since $Y_t$ admits a continuous Lebesgue density $\chi_t\in C(\mathbb{R}^d)$ (vanishing identically outside of $D_{Y_t}$) by assumption, each of the sets $D_{Y_t}\stackrel{\mathrm{def}}{=}\overline{\chi_t^{-1}(\{0\}^c)}^{|\cdot|}$, $t\in\mathbb{I}$, is the closure of an open set. Denote $D_Y'\coloneqq\bigcup_{t\in\mathbb{I}}D_{Y_t}$. Then for each $u\in D_Y$ there is a sequence $(u_n)_n\subset D_Y'$ with $\lim_{n\rightarrow\infty}u_n=u$, and for each $n\in\mathbb{N}$ there is a sequence $(u_{n,m})_m \subset\mathrm{int}(D_{Y_{t_n}})\subseteq\mathrm{int}(D_Y)$ (some $t_n\in\mathbb{I}$) with $\lim_{m\rightarrow\infty}u_{n,m}=u_n$ by the fact that each element of $\langle D_{Y_t}\rangle_{t\in\mathbb{I}}$ is the closure of its interior. With this, it is easy to see that there is a subsequence $(m_n)_n\subset\mathbb{N}$ such that $\lim_{n\rightarrow\infty}u_{n,m_n} = u$, proving $u\in\overline{\mathrm{int}(D_Y)}$ as claimed. The remaining inclusion $\overline{\mathrm{int}(D_Y)}\subseteq D_Y$ is clear as $\mathrm{int}(D_Y)\subset D_Y$ and $D_Y$ is closed.

\ref{lem:spat_supp:it5}\,: Suppose that the set $\tilde{D}\coloneqq \bigcup_{(s,t)\in\Delta_2(\mathbb{I})}\dot{D}_s\cap\dot{D}_t$ is not dense in $D_S$. Then by \ref{lem:spat_supp:it4} the set $\tilde{D}$ is not dense in the interior of $D_S$, whence there exists $x_0\in D_S$ and $\varepsilon>0$ such that $B_\varepsilon(x_0)\subset D_S\setminus\tilde{D}$. Hence by  \ref{lem:spat_supp:it3}, there must then be some $t^\star\in\mathbb{I}$ such that $\mathcal{O}\coloneqq \dot{D}_{t^\star}\cap B_\varepsilon(x_0)\neq\emptyset$. 
Now since $\mathcal{O}\subset \tilde{D}^c$, we for each $x\in\mathcal{O}\subseteq\dot{D}_{t^\star}$ have that $x\in\dot{D}_{t^\star}\cap(\dot{D}_s^c)$ for all $s\neq t^\star$, the latter implying that $x\in\dot{D}_{t^\star}\setminus\big(\bigcup_{s\neq t^\star}\dot{D}_s\big)$ and hence $\upsilon_{t^\star}(x)>0$ and $\upsilon_s(x)=0$ for all $s\neq t\star$, contradicting the continuity of $s\mapsto \upsilon_s(x)$.   
\end{proof} 

\subsection{Probability Density of Projections of a Random Vector}
Given a $C^k$-distributed random vector $Z=(Z^1,\cdots, Z^n)$ in $\R^n$ with density $\varsigma$ together with some fixed subset $I\subseteq[n]$, say $I=\{i_1,\ldots, i_k\}$ with $k\coloneqq|I|$, denote by $Z_I\coloneqq\pi_I(Z)=(Z^{i_1},\cdots, Z^{i_k})$ the random vector in $\R^k$ which is given by the projection of $Z$ to its $I$-indexed subcoordinates. Then $Z_I$ is $C^k$-distributed with Lebesgue density $\varsigma_I$ given by 
\begin{equation}\label{rem:TrafoProjProbDens:eq1}
\varsigma_I \ = \ \int_{\mathbb{R}^{n-k}}\!\varsigma\ \mathrm{d}x_1\cdots\widehat{\mathrm{d}x_{i_1}}\cdots\widehat{\mathrm{d}x_{i_k}}\cdots\mathrm{d}x_n.
\end{equation}     
As an immediate consequence, we have the inclusion
\begin{equation}\label{rem:TrafoProjProbDens:eq2}
\supp(\varsigma) \ \subseteq \ \supp(\varsigma_{\,[k]})\times\supp(\varsigma_{\,[n]\setminus[k]})\qquad\text{ for each }\quad k\in [n].
\end{equation}
Indeed, setting $C_{k}\coloneqq\supp(\varsigma_{\,[k]})$ and $C'_{k}\coloneqq\supp(\varsigma_{\,[n]\setminus[k]})$, we note that $\mathbb{P}(Z\in C_k\times C'_k) = \mathbb{P}(Z_{[k]}\in C_k, \, Z_{[n]\setminus[k]}\in C'_k) \geq \mathbb{P}(Z_{[k]}\in C_k) + \mathbb{P}(Z_{[n]\setminus[k]}\in C_k') - 1 = 1$ and hence $C_k\times C'_k\supseteq\supp(\mathbb{P}_Z) = \supp(\varsigma)$, as claimed. 

\subsection{Analytical Characterisation of Separable and Pseudo-Gaussian Densities}\label{pf:lem:C2PseudoGaussian}
It will be convenient to have an analytical characterisation of the `pathological' types of densities from Definition \ref{def:PseudoGaussian}. To this end, we first declare what we mean by a symmetric set:\\[-0.5em] 

\noindent
Writing $(u,v)\equiv (u_1, \cdots, u_d, v_1, \cdots,v_d)$ for the coordinates on $\mathbb{R}^{2d}\cong \mathbb{R}^d\times\mathbb{R}^d$, a given subset $A\subseteq\mathbb{R}^{2d}$ will be called \emph{symmetric} if 
\begin{equation}
A \ = \ \tau(A) \qquad\text{ for the transposition }\quad \tau(u,v)\coloneqq (v,u).
\end{equation}A function $\varphi : G\rightarrow\mathbb{R}$, $G\subseteq\mathbb{R}^{2d}$, will be called \emph{symmetric} if $\varphi\circ\tau = \varphi$.\\[-0.5em]

(Since $\tau^2=\mathrm{id}$, it is clear that $A$ is symmetric iff $\tau(A)\subseteq A$. Also, if $A\subseteq\R^{2d}$ is symmetric then $A\subseteq \pi_{[d]}(A)\times\pi_{[d]}(A)$.) 

\begin{appendixlemma}\label{lem:C2PseudoGaussian}
Let $\zeta : G \rightarrow \mathbb{R}_{>0}$, with $G\subseteq\mathbb{R}^2$ open, be twice continuously differentiable. Then the following holds.
\begin{enumerate}[label=\upshape(\roman*)]
\item\label{lem:C2PseudoGaussian:it1}Provided that $G$ is convex, we have that:
$$\hfill\displaystyle
\partial_x\partial_y\log\zeta \,\equiv\, 0 \qquad\text{ if and only if }\qquad  \zeta \ \text{ is separable}\,;\hfill$$
\item\label{lem:C2PseudoGaussian:it2}$\zeta$ is strictly non-separable if and only if the open set
$$G'\coloneqq\{z\in G\mid \partial_x\partial_y\log\zeta(z)\neq 0\}\ \text{ is a dense subset of } G;$$
\item\label{lem:C2PseudoGaussian:it3}provided that $\mathcal{O}\subseteq G'$ is symmetric, open and convex, we have that: 
\begin{equation*}
\restr{[\partial_x\partial_y\log\zeta]}{\mathcal{O}}\ \text{ is separable and symmetric }\quad\text{ iff }\quad \restr{\zeta}{\mathcal{O}} \ \text{ is pseudo-Gaussian.}
\end{equation*}
\end{enumerate}  
\end{appendixlemma} 
\begin{proof}We use the global abbreviations $\xi\coloneqq\partial_x\partial_y\log\zeta$ and $\phi\coloneqq\log\zeta$.

\ref{lem:C2PseudoGaussian:it1}\,: \, The `if'-direction is clear, so suppose that $\partial_x\partial_y\phi =0$. Then, as $G$ is convex, $\phi\equiv\phi(x,y) = \phi_1(x) + \phi_2(y)$ and hence $\zeta = \exp(\phi) = \zeta_1(x)\cdot \zeta_2(y)$ for $\zeta_i\coloneqq \exp(\phi_i)$, as claimed.    

\ref{lem:C2PseudoGaussian:it2}\,: \, Since $\xi$ is continuous, the set $\{\xi=0\}$ is closed, whence the set $G'=G\cap\{\xi=0\}^{c}$ is open. To see that $G'$ is dense in $G$, take any $z\in G$ and note that, as $G$ is open, there is some $z$-centered open ball $B_z\subseteq G$. Since $\zeta$ is strictly non-separable, $\left.\zeta\right|_{B_z'}$ is not separable for any open $z$-centered sub-ball $B'_z\subseteq B_z$, whence by (i) there must be some $z'\in B_z'$ with $\xi(z')\neq 0$, implying $B_z'\cap G'\neq\emptyset$. 

The (contrapositive of the) converse implication in (ii) follows via (i). 

\ref{lem:C2PseudoGaussian:it3}\,: \, Let $\mathcal{O}\subseteq G'$ be symmetric, open and convex. $(\Leftarrow)$ is clear by Def.\ \ref{def:PseudoGaussian}.

$(\Rightarrow)$\,:\, Suppose that $\tilde{\xi}\coloneqq\restr{\xi}{\mathcal{O}}$ is separable and symmetric, i.e.\ assume that
\begin{equation}
\tilde{\xi}\equiv\tilde{\xi}(x,y) = f(x)\cdot g(y)\quad\text{ and }\quad \tilde{\xi}\circ\tau = \tilde{\xi}
\end{equation}
for some $f,g : \mathcal{O}_1\rightarrow\mathbb{R}$, with $\mathcal{O}_1\coloneqq\pi_1(\mathcal{O})$. Then $\tilde{\xi}\equiv\tilde{\xi}(x,y) = \mathrm{sgn}(\tilde{\xi})\cdot\eta(x)\cdot\eta(y)$ for some function $\eta:\mathcal{O}_1\rightarrow\mathbb{R}$, where $\epsilon\coloneqq\mathrm{sgn}(\tilde{\xi})$ denotes the sign of $\tilde{\xi}$ (i.e., $\mathrm{sgn}(\tilde{\xi})=\mathbbm{1}_{(0,\infty)}(\tilde{\xi}) - \mathbbm{1}_{(-\infty, 0)}(\tilde{\xi})$). Indeed, the symmetry of $\tilde{\xi}$ implies that $\tilde{\xi}^2 = \tilde{\xi}(x,y)\cdot\tilde{\xi}(y,x) = \tilde{\eta}(x)\cdot\tilde{\eta}(y)$ for the map $\tilde{\eta}(z)\coloneqq f(z)\cdot g(z)$; consequently, $\tilde{\xi} = \epsilon\cdot\sqrt{\tilde{\xi}^2} = \epsilon\cdot\eta(x)\cdot\eta(y)$ for the map $\eta\equiv\eta(z)\coloneqq\sqrt{|\tilde{\eta}(z)|}$. Now since $\mathcal{O}$ is a connected subset of $G'$, the sign of $\tilde{\xi}$ is constant, i.e.\ $\epsilon=\pm 1$. Integrating $\xi=\partial_x\partial_y\phi$ thus implies that 
\begin{equation}\label{lem:C2PseudoGaussian:aux1}
\begin{aligned}
\phi \ &= \, \int\!\partial_y\phi\,\mathrm{d}y + f_1(x) = \int\!\!\!\int\!\tilde{\xi}\,\mathrm{d}x + f_2(y)\,\mathrm{d}y + f_1(x)\\
\ &= \,\epsilon\cdot f_3(x)\cdot f_3(y) + \tilde{f}_2(y) + \tilde{f}_1(x) 
\end{aligned} 
\end{equation}for $f_3\equiv f_3(z)\coloneqq \int_{z_0}^z\!\eta(s)\,\mathrm{d}s$ (some priorly fixed $z_0\in\mathcal{O}_1$) and some additional continuous functions $f_i, \tilde{f}_i: \mathcal{O}_1\rightarrow\mathbb{R}$. Note that since $\mathcal{O}$ is convex, the integrated identity of functions \eqref{lem:C2PseudoGaussian:aux1} holds pointwise on all of $\mathcal{O}$. Exponentiating \eqref{lem:C2PseudoGaussian:aux1} now yields the claim.           
\end{proof}

\subsection{A Function Which is Strictly Non-Separable but Not Regularly Non-Separable}\label{example:strictlynonsep_not_regnonsep}
Consider the function $\varphi_0\equiv\varphi_0(x,y)$ given by
\begin{equation*}
\varphi_0(x,y) = \begin{cases}\exp\!\big(-1/(x-y)^{2}\big), \quad &x < y\\
 \hfill 0,\hfill  &x=y \\ 
 \exp\!\big(-1/(x-y)^{2}\big), \quad &y < x,\end{cases}\quad\text{ and define }\quad \varphi\coloneqq e^{\varphi_0}.
\end{equation*}   
Then clearly $\varphi\in C^2(\tilde{U}^{\times 2};\mathbb{R}_{>0})$ for $\tilde{U}\coloneqq (0,1)$, and as the mixed-log-derivatives $\partial_x\partial_y\log(\varphi) = \partial_x\partial_y\varphi_0$ vanish nowhere on the dense subset $\tilde{U}^2\setminus\Delta_{\tilde{U}}\subset\tilde{U}^{\times 2}$ the function $\varphi$ is also strictly non-separable on $\tilde{U}^2$ by Lemma \ref{lem:C2PseudoGaussian} (ii). However, since $\left.(\partial_x\partial_y\varphi_0)\right|_{\Delta_{\tilde{U}}}=0$ everywhere on $\tilde{U}$, the function $\varphi$ is clearly not regularly non-separable. 

\subsection{A Lemma on Monomial Transformations}\label{pf:lem:monomial_trafos}
The following helps us to infer the desired recovery of a source from the existence of a `well-behaved' subset of its spatial support.    
\begin{appendixlemma}\label{lem:monomial_trafos}
Let $\varrho\in C^{1}(G)$ for some $G\subseteq\R^d$ open. Then the following holds. 
\begin{enumerate}[label=\upshape(\roman*)]
\item\label{lem:monomial_trafos:it1} Let $G$ also be connected and such that $G\cap\pi_i^{-1}(\{\eta\})$ is connected for each $\eta\in\pi_i(G)$ and all $i\in[d]$. Then $\varrho\in\DP(G)$ if and only if the Jacobian $J_\varrho$ of $\varrho$ is monomial on $G$, i.e.\ such that $\{J_\varrho(u)\mid u\in G\}\subseteq\operatorname{M}_d$.
\item\label{lem:monomial_trafos:it2} Let $D$ be a convex subset of $G$ with the property that $D$ is the closure of some dense subset $\mathcal{O}$ of $D$. If $J_\varrho$ is invertible on $D$ and monomial on $\mathcal{O}$, then $\varrho$ is monomial on $D$.    
\end{enumerate} 
\end{appendixlemma}
\begin{proof}
\ref{lem:monomial_trafos:it1}\,: The `only-if' direction is clear, so let us prove that $\varrho\in\DP(G)$ if 
\begin{equation}\label{lem:monomial_trafos:aux1}
J_\varrho(u) \ = \ \big(\beta_\mu(u)\cdot\delta_{\nu,\sigma^u(\mu)}\big)_{\!\mu,\nu\in[d]}\ \in \ \mathrm{M}_d, \qquad \forall\, u\equiv(u_\mu)\in G,
\end{equation}
for some $\beta_\mu : G\rightarrow \mathbb{R}_\times$ and $\{\sigma^u\mid u\in G\}\subseteq S_d$. To this end, we first note that 
\begin{equation}\label{lem:monomial_trafos:aux2}
\sigma^u = \sigma^{u'}\quad\text{ for any }\quad u,u'\in G.
\end{equation}

Indeed: Fix any $u_0\in G$ and note that since the Jacobian $J_\varrho\equiv(J_\varrho^{ij})_{ij} : G \rightarrow\operatorname{GL}_d$ of $\varrho$ is continuous, we for each $\varepsilon_{u_0}>0$ can find a $\delta_{u_0}>0$ such that ($B_{\delta_{{u_0}}}\!(u_0)\subseteq G$ and)
\begin{equation}
\|J_{\varrho}(u_0) - J_\varrho(u)\| < \varepsilon_{u_0} \quad\text{ for all }\quad u\in B_{\delta_{{u_0}}}\!(u_0).
\end{equation}                        
Taking $\|\cdot\|$ as the infinity-norm $\|A\|\equiv\|A\|_\infty\coloneqq\max_{1\leq i\leq d}\sum_{j=1}^d|a_{ij}|$  and $\varepsilon_{u_0}\coloneqq\min_{1\leq 1\leq d}\sum_{j=1}^d|J_\varrho^{ij}(u_0)|>0$, the fact that $J_{\varrho}(u_0), J_\varrho(u)\in\mathrm{M}_d$ then readily implies that   
\begin{equation}\label{lem:monomial_trafos:aux4}
\sigma^u=\sigma^{u_0}\quad\text{ for all }\quad u\in B_{\delta_{u_0}}\!(u_0).
\end{equation} 
Let now $u,u'\in G$ be arbitrary. Since, as a connected subset of $\R^d$, the domain $G$ is also path-connected, the points $u$ and $u'$ can be joined by a continuous path $\gamma$ with $u=\gamma(0)$ and $u'=\gamma(1)$ and trace $\bar{\gamma}\coloneqq\gamma([0,1])\subseteq G$. We claim that 
\begin{equation}\label{lem:monomial_trafos:aux4.1}
\sigma^{u_1} = \sigma^{u_2} \quad\text{ for any }\quad u_1, u_2\in\bar{\gamma},
\end{equation}yielding \eqref{lem:monomial_trafos:aux2} in particular. And indeed: Since $\bar{\gamma}\subset\bigcup_{u\in\bar{\gamma}}B_{\delta_u}(u)$ for $\{\delta_u\}$ as in \eqref{lem:monomial_trafos:aux4}, the compactness of $\bar{\gamma}$ yields that, for an $m\in\N$,
\begin{equation}\label{lem:monomial_trafos:aux5}
\bar{\gamma}\ \subseteq \ \bigcup_{j\in[m]}B_{\delta_{u_j}}\!(u_j)\qquad\text{ for some } \  u_1, \ldots, u_m\in\bar{\gamma}.
\end{equation}    
But since the trace $\bar{\gamma}$ is connected, the $\{u_j\}$ from \eqref{lem:monomial_trafos:aux5} can be renumerated such that $B_{\delta_{u_i}}\!(u_i)\cap B_{\delta_{u_{i+1}}}\!(u_{i+1})\neq\emptyset$ for each $1\leq i < m$, which implies \eqref{lem:monomial_trafos:aux4.1} (and hence \eqref{lem:monomial_trafos:aux2}) by way of \eqref{lem:monomial_trafos:aux4}.    

Hence on $G$, the Jacobian \eqref{lem:monomial_trafos:aux1} of $\varrho$ is in fact of the form 
\begin{equation}\label{lem:monomial_trafos:aux6}
\restr{J_\varrho}{G} \ = \ \big(\beta_\mu\cdot\delta_{\nu,\sigma(\mu)}\big)_{\!\mu,\nu\in[d]}\quad\text{ for some }\ \beta_\mu\in C(G;\mathbb{R}_\times)\ \text{ and } \ \sigma\in S_d.
\end{equation}
The assertion that $\varrho\equiv(\varrho_i)\in\DP(G)$ now follows from \eqref{lem:monomial_trafos:aux6} and the mean value theorem (MVT): Given any $u_0=(u_0^1, \ldots, u_0^d)\in G$, the MVT implies\footnote{\ Indeed: For $u,u_0,v$ as above, define $\varphi(t)\coloneqq v\cdot\varrho(u_0 + t\cdot\eta)$ for $t\in I_\delta\equiv(-\delta, 1+\delta)$ and $\eta\coloneqq u-u_0$ and $\delta>0$ s.t.\ $\{u_0+t\cdot\eta\mid t\in I_\delta\}\subset G$ (such a $\delta$ exists as $G$ is open). Then $\varphi\in C^1(I_\delta)$, whence \eqref{lem:monomial_trafos:aux7} follows from the (classical) MVT applied to the difference $\varphi(1)-\varphi(0)$.} that for each $u=(u_1, \ldots, u_d)\in G$ which is connected to $u_0$ via the line segment $\overline{u_0,u}\equiv\{u_0 + t\cdot(u-u_0)\mid t\in[0,1]\}\subset G$ and any $v\in\R^d$, there exists a point $\xi\in\overline{u_0,u}$ such that
\begin{equation}\label{lem:monomial_trafos:aux7}
v\cdot(\varrho(u) - \varrho(u_0)) \ = \ v\cdot J_\varrho(\xi)\cdot(u-u_0).
\end{equation}Hence if for any fixed $i\in[d]$ we take $u$ with $u_{\sigma(i)} = u_0^{\sigma(i)}$ and choose $v=e_i$ (for $(e_i)_{i\in[d]}$ the standard basis of $\R^d$), then by way of \eqref{lem:monomial_trafos:aux7} and \eqref{lem:monomial_trafos:aux6} we find that  
\begin{equation}\label{lem:monomial_trafos:aux8}
\varrho_i(u) - \varrho_i(u_0) = \big[J_\varrho(\xi)\cdot(u-u_0)\big]_i = 0 \qquad (i\in[d]).
\end{equation}This implies that for any given $u_0\equiv(u^1_0, \ldots, u^d_0)\in G$ we have $\varrho_i(u) = \varrho_i(u^{\sigma(i)}_0)$ for all $u\in G_{u_0|i}\coloneqq\{u\in G\mid \text{$\exists\,$ polygonal path in $\pi_{\sigma(i)}^{-1}(\{u_0^{\sigma(i)}\})$ connecting $u$ and $u_0$}\}$. But since by assumption the slices $G_{\eta}^j\coloneqq G\cap\pi_{j}^{-1}(\{\eta\})$ are each (polygonally-)connected for any $\eta\in\R$ and $j\in[d]$, we have that $G_{u_0|i}=G_{u_0^{\sigma(i)}}^{\sigma(i)}$. As $u_0\in G$ was arbitrary, we thus find that    
\begin{equation}\label{lem:monomial_trafos:aux9}
\varrho_i(u) = \varrho_i(u_{\sigma(i)}) \quad\text{ for each } \ u\equiv(u_1,\ldots,u_d)\in G \qquad (i\in[d]), 
\end{equation}           
hence the diffeomorphism\footnote{\ Note that since by \eqref{lem:monomial_trafos:aux9} and \eqref{lem:monomial_trafos:aux6} each $\varrho_i$ is a continuously differentiable map from $\pi_{\sigma(i)}(G) (\subseteq\R)$ to $\R$ with nowhere-vanishing derivative, each $\varrho_i$ is a univariate local diffeomorphism and thus in fact a global diffeomorphism on $\pi_{\sigma(i)}(G)$ (cf.\ e.g.\ \citep[Ex.\ 1.3.3]{GUP}).} $\varrho\equiv(\varrho_1,\cdots,\varrho_d)$ is monomial on $G$ as claimed.

\ref{lem:monomial_trafos:it2}\,: This is a corollary to the above proof of \ref{lem:monomial_trafos:it1}. Indeed, let $\mathcal{O}\subseteq D$ be dense with
\begin{equation}
J_\varrho(u) \in \operatorname{M}_d \quad \text{for each } \ u\in\mathcal{O}.
\end{equation} 
Then for any fixed $z\in D$, the fact that $\mathcal{O}$ is dense in $D$ ensures that there will be a sequence $(u^{(k)})_{k\in\N}\subset \mathcal{O}$ with $\lim_{k\rightarrow\infty}u^{(k)}=z$, implying that 
\begin{equation}\label{lem:monomial_trafos:aux12}
J_\varrho(z) \ = \ \lim_{k\rightarrow\infty} J_\varrho(u^{(k)}) 
\end{equation}due to $\varrho$ being continuously differentiable on $G$. Hence and because $J_\varrho(z)\in\operatorname{GL}_d$, we obtain that in fact $J_\varrho(z)\in\operatorname{M}_d$ by \eqref{lem:monomial_trafos:aux12} and the fact that the subset $\operatorname{M}_d$ is closed in $\operatorname{GL}_d$. The claim now follows from (the proof of) \ref{lem:monomial_trafos:it1} [for $G\coloneqq D$] upon noting that, due to its convexity, the set $D$ is polygonally-connected and satisfies the slice requirements of statement \ref{lem:monomial_trafos:it1}.      
\end{proof} 
\subsection{Proof of Lemma \ref{lem:LogReg}}\label{pf:lem:LogReg}
Recall that $Y, Y^\ast$ are defined by \eqref{thm:NICA_stat:aux1} and $\rho$ is given by \eqref{thm:NICA_stat:aux3}. 
\lemLogReg*
\begin{proof}
We note first that since the support $\supp(\nu)$ of a (Borel) probability measure $\nu:\mathcal{B}(E)\rightarrow[0,1]$ is defined as the smallest closed set $C\subseteq E$ having total mass $\nu(C)=1$, it is easy to see that $\supp(\tilde{f}_\ast\nu) = \overline{\tilde{f}(\supp(\nu))}$ for any $\tilde{f}$ continuous (cf.\ the proof of Lem.\ \ref{lem:spat_supp} \ref{lem:spat_supp:it1}). For the above case, this implies $\supp(\mu) = \overline{(f\times f)(\supp(\zeta))} = \overline{(f\times f)(\bar{D})}$, whence $\partial(\supp(\mu))$ is a Lebesgue-nullset (as is $\partial \bar{D}$, by assumption, and hence also the boundary of its $C^2$-image $(f\times f)(\bar{D})$; the latter boundary in turn contains $\partial\overline{(f\times f)(\bar{D})}$ (as the boundary of the closure of a set is always contained in the boundary of that set) and hence $\mu>0$ a.e.\ on $\supp(\mu)$. Since also $\supp(\mu^\ast) = \supp(\mathbb{P}_{X_s})\times\supp(\mathbb{P}_{X_t})$, we further obtain $\supp(\mu) = \supp(\mathbb{P}_{(X_s, X_t)}) \subseteq \supp(\mu^\ast)$ by \eqref{rem:TrafoProjProbDens:eq2}, which implies that the RHS of \eqref{thm:NICA_stat:aux4} is defined a.e.\ on $\supp(\mu)$ indeed. 

Note now that since by definition the function $\rho$ equals the conditional probability of the event $\{C=1\}$ given $\bar{Y}$, we have  
\begin{equation}\label{thm:NICA_stat:aux5}
\rho\cdot\frac{\mathrm{d}\mathbb{P}_{\bar{Y}}}{\mathrm{d}y} = \mathbb{P}(C=1\,|\,\bar{Y})\cdot\frac{\mathrm{d}\mathbb{P}_{\bar{Y}}}{\mathrm{d}y} = \frac{\mathrm{d}\mathbb{P}_{\bar{Y}}(\,\cdot\,|\,C=1)}{\mathrm{d}y}\cdot\mathbb{P}(C=1)
\end{equation}
almost everywhere, where the first factor on the RHS of \eqref{thm:NICA_stat:aux5} denotes (a regular version of) the conditional density of $\bar{Y}$ given $C=1$.\footnote{\ Indeed, abbreviating $\ell\coloneqq \rho\cdot\frac{\mathrm{d}\mathbb{P}_{\bar{Y}}}{\mathrm{d}y}$ and $r\coloneqq \frac{\mathrm{d}\mathbb{P}(\,\cdot\,|\,C=1)}{\mathrm{d}y}\cdot\mathbb{P}(C=1)$, we for any $A\in\mathcal{B}(\mathbb{R}^{2d})$ find
\begin{align*}
\int_{\mathbb{R}^{2d}}\!r\cdot\mathbbm{1}_A\,\mathrm{d}y \ &= \mathbb{P}(C=1)\!\!\int_A\!\mathbb{P}_{\bar{Y}}(\mathrm{d}y\,|\,C=1) = 
\mathbb{P}((C,\bar{Y})\in\{1\}\times A)\\
&= \mathbb{P}((\bar{Y}, C)\in A\times \{1\}) = \int_A\!\mathbb{P}(C=1\,|\,\bar{Y}=y)\,\mathbb{P}_{\bar{Y}}(\mathrm{d}y) = \int_{\mathbb{R}^{2d}}\!\ell\cdot\mathbbm{1}_A\,\mathrm{d}y 
\end{align*}which implies $r=\ell$ (a.e.) by the fundamental lemma of calculus of variations. (Note that the second and the fourth of the above equations hold by definition of conditional distributions.)} Next we observe that 
\begin{equation}\label{thm:NICA_stat:aux6}
\frac{\mathrm{d}\mathbb{P}_{\bar{Y}}(\,\cdot\,|\,C=1)}{\mathrm{d}y}\cdot\mathbb{P}(C=1) \ = \ \tfrac{1}{2}\mu \qquad\text{(a.e.)}.
\end{equation} 
Indeed, denote by $\eta$ the LHS of \eqref{thm:NICA_stat:aux6} and let $A\in\mathcal{B}(\mathbb{R}^{2d})$ be arbitrary. Then, since by construction $\mathbb{P}_{(C,\bar{Y})} = \mathbb{P}_C\otimes\mathbb{P}^{\bar{Y}}_C$ and $\mathbb{P}^{\bar{Y}}_{C=1} \equiv \mathbb{P}_{\bar{Y}}(\,\cdot\,|\,C=1) = \mathbb{P}_Y$ and $\mathbb{P}(C=1)=\tfrac{1}{2}$,    
\begin{align*}
\int_{\mathbb{R}^{2d}}\!\eta\cdot\mathbbm{1}_A\,\mathrm{d}y \ &= \  \mathbb{P}(\bar{Y}\in A\,|\,C=1)\cdot\mathbb{P}(C=1) = \mathbb{P}(\bar{Y}\in A, \ C=1)\\
&= \ \mathbb{P}_{(C, \bar{Y})}(\{1\}\times A) = \mathbb{P}(C=1)\cdot\mathbb{P}_Y(A) = \int_{\mathbb{R}^{2d}}\!\tfrac{1}{2}\mu\cdot\mathbbm{1}_A\,\mathrm{d}y  
\end{align*}
from which \eqref{thm:NICA_stat:aux6} follows by the fundamental lemma of calculus of variations. Combining \eqref{thm:NICA_stat:aux5} with the fact that $\frac{\mathrm{d}\mathbb{P}_{\bar{Y}}}{\mathrm{d}y} = \tfrac{1}{2}\mu + \tfrac{1}{2}\mu^\ast$ and \eqref{thm:NICA_stat:aux6} now yields the identity $(\mu+\mu^\ast)\cdot\rho = \mu$ (a.e.), from which equation \eqref{thm:NICA_stat:aux4} follows immediately.
\end{proof}

\subsection{A Separation Lemma}The following is a core lemma for the proof of Theorem \ref{thm:NICA_stat}. 

\begin{appendixlemma}\label{lem:PermDiagForThm}
Let $U\subseteq\mathbb{R}^{d}$ be open and $\varphi_i\in C^2(U^{\times 2};\mathbb{R}_{>0})$, $i\in[d]$, with $\varphi_i(x)\equiv\varphi_i(x_i, x_{i+d})$, be a family of regularly non-separable, positive functions of which all but at most one are a.e.\ non-Gaussian. Set $\xi_i\coloneqq\partial_{x_i}\partial_{x_{i+d}}\log\varphi_i$ for each $i\in[d]$.
Then for any continuous $B : U \rightarrow \operatorname{GL}_d(\mathbb{R})$ for which the composition $\Lambda : U^{\times 2}\rightarrow\mathbb{R}^{d\times d}$ given by 
\begin{equation}\label{lem:PermDiagForThm:eq1}
\Lambda(u,v) \ \coloneqq \ B(u)^\intercal\cdot\operatorname{diag}_{i\in[d]}\!\big[\xi_i(u_i, v_i)\big]\!\cdot\!B(v)
\end{equation}$($in the coordinates $(u,v)\equiv(u_1, \ldots, u_d, v_1, \ldots, v_d)\in U^{\times 2})$ has identically-vanishing off-diagonal elements, it holds that the function $B$ is monomial on $U$, i.e.\ that 
\begin{equation}\label{lem:PermDiagForThm:eq2}
B(u) \ \in \ \operatorname{M}_d\qquad\text{ for each }\ u\in U.
\end{equation}       
\end{appendixlemma} 
\begin{proof}
Set $\check{U}\coloneqq U\times U$, and for a given $(u,v)\in\check{U}$, denote $\hat{\Lambda}_{u,v}\coloneqq\operatorname{diag}_{i\in[d]}\!\big[\xi_i(u_i, v_i)\big]$ and $\Lambda_{u,v}\coloneqq\Lambda(u,v)$ and $B_u\coloneqq B(u)$, and assume (wlog, upon re-enumeration) that $\varphi_i$ is a.e.\ non-Gaussian for each $i\in[d-1]$. 

Note that by the fact that $B$ is $\operatorname{GL}_d$-valued and continuous, the identity \eqref{lem:PermDiagForThm:eq2} holds if 
\begin{equation}\label{lem:PermDiagForThm:aux0}
\exists\, \tilde{U}\subseteq U \ \text{dense}\quad\text{ s.t.\ }\quad B_u\,\in\mathrm{M}_{d} \quad \text{for all } u\in \tilde{U}
\end{equation}
(cf.\ the argument around \eqref{lem:monomial_trafos:aux12} for details). Our proof consists of constructing a set $\tilde{U}$ for which \eqref{lem:PermDiagForThm:aux0} holds. Let to this end $i\in[d]$ be fixed, and recall that $\varphi_i$ being regularly non-separable implies that there is a closed nullset\footnote{\ Notice that if $\varphi_i$ is regularly non-separable and a.e.\ non-Gaussian, there (by Definition \ref{def:PseudoGaussian}) will be a closed nullset $\tilde{\mathcal{N}}_i\subset\pi_{(i, i+d)}(\check{U})\subseteq\mathbb{R}^2$ (s.t.\ $\tilde{\mathcal{N}}_i\cap\{(x,x)\mid x\in\mathbb{R}\}$ has Hausdorff-measure zero on the diagonal $\Delta_\mathbb{R}\coloneqq\{(x,x)\mid x\in\mathbb{R}\}$) on whose complement $\varphi_i$ is strictly non-Gaussian and non-separable and s.t.\ $\restr{\varphi_i}{\Delta_{\mathbb{R}}\setminus\tilde{\mathcal{N}}_i}$ vanishes nowhere. Hence for the (relatively) closed nullsets $\mathcal{N}_i\coloneqq \pi_{(i, i+d)}^{-1}(\tilde{\mathcal{N}}_i)\cap \check{U}\subset\mathbb{R}^{2d}$, the (relatively) closed union $\mathcal{N}\coloneqq\bigcup_{i\in[d]}\mathcal{N}_i$ works as desired.} $\mathcal{N}\subset \check{U}$ s.t.\ for the open and dense\footnote{\ Recall that for (any) $\check{U}\subseteq\R^m$ open and $\mathcal{N}\subseteq\R^m$ a Lebesgue nullset, the complement $\check{U}\setminus\mathcal{N}$ is dense in $\check{U}$. Indeed: If for $\check{U}_\circ\coloneqq \check{U}\setminus\mathcal{N}$ we had $\mathrm{clos}(\check{U}_\circ)\subsetneq \check{U}$, then there would be $u\in \check{U}$ with $B_\delta(u)\subseteq \check{U}\setminus\mathrm{clos}(\check{U}_\circ) \subseteq \mathcal{N}$ for some $\delta>0$, contradicting that the Lebesgue measure of $\mathcal{N}$ is zero.} subset $\check{U}_\circ\coloneqq\check{U}\setminus\mathcal{N}$ of $\check{U}$, each restriction $\restr{\varphi_i}{\check{U}_\circ}$ is such that 
\begin{gather} 
\restr{\varphi_i}{\check{U}_\circ} \ \text{ is strictly non-Gaussian for } i\neq d, \quad\text{ and for each } i\in[d]:\label{lem:PermDiagForThm:aux0.1.1}\\
\restr{\varphi_i}{\check{U}_\circ} \ \text{ is strictly non-separable }\quad\text{ with }\quad \restr{\xi_i}{(\Delta_U\cap\check{U}_\circ)}\neq 0 \ \text{ everywhere}.\label{lem:PermDiagForThm:aux0.1} 
\end{gather}
Given \eqref{lem:PermDiagForThm:aux0.1}, Lemma \ref{lem:C2PseudoGaussian} (ii) (together with the elementary topological facts that $(a)\,:$ a subset which lies densely inside a dense subspace is itself dense again, and $(b)\,:$ the intersection of two open dense subsets is again an open dense subset) implies that the intersection
\begin{equation}\label{lem:PermDiagForThm:aux0.2}
\check{U}_\ast\coloneqq \bigcap_{i\in[d]}\big\{z\in\check{U}_\circ \ \big| \ \xi_i(z)\neq 0 \big\}\quad\text{ is an open dense subset of } \ \check{U}. 
\end{equation}Consequently, the coordinate-projections $U'\coloneqq\pi_{[d]}(\check{U}_\ast)$ and $V'\coloneqq\pi_{(d+1, \ldots, 2d)}(\check{U}_\ast)$ are open and dense subsets of $U\,(=\pi_{[d]}(\check{U}))$. We now claim that the sets      
\begin{equation}\label{lem:PermDiagForThm:aux0.3}
\begin{gathered}
\tilde{U}_i\coloneqq\Big\{u\in U'\ \big| \ \exists\,(\emptyset\neq)\,\mathcal{V}_u\subseteq V' \text{ open}\footnotemark \ : \ \left.q_u^i\right|_{\mathcal{V}_u} \text{ is non-constant} \Big\},\\
\text{defined by the function }\qquad q^i_u(v)\coloneqq\frac{\xi_i(u,u)\cdot\xi_i(v,v)}{\xi_i(u,v)^2},   
\end{gathered}
\end{equation}\footnotetext{\ Note that for $\tilde{U}_i$ to be well-defined, we in addition require each $\mathcal{V}_u$ to be s.t.\ $\{u\}\times\mathcal{V}_u\subset\check{U}_\ast$.}are dense in $U'$ --- and hence (cf.\ fact $(a)$) are also dense in $U$ --- for each $i\in[d-1]$. 

To see that this holds, we proceed via proof by contradiction and assume that $\tilde{U}_i$ is not dense in $U'$. In this case, there exists $(\bar{u},r)\in U'\times\mathbb{R}_{>0}$ with $B_r(\bar{u})\subset U'\setminus\tilde{U}_i$ (recall that $U'$ is open). Moreover: Since we have $\Delta_U\cap\check{U}_\circ\subset \check{U}_\ast$ (by \eqref{lem:PermDiagForThm:aux0.1}) and $\check{U}_\ast$ is open, we can even find a convex open neighbourhood $\mathcal{V}_{\bar{u}}\subseteq B_r(\bar{u})$ of $\bar{u}$ such that the whole square $\mathcal{Q}_{\bar{u}}\equiv\mathcal{V}_{\bar{u}}\times\mathcal{V}_{\bar{u}}$ is contained in $\check{U}_\ast$. Now by construction, we for each slice $\{u\}\times\mathcal{V}_{\bar{u}}\subset\mathcal{Q}_{\bar{u}}$, $u\in\mathcal{V}_{\bar{u}}$, must have that $\restr{q^i_u}{\mathcal{V}_{\bar{u}}}$ is a constant function, say $\restr{q^i_u}{\mathcal{V}_{\bar{u}}}\!\equiv c_u$ for some $c_u\in\mathbb{R}$, so that in particular (for $\varrho : \mathcal{V}_{\bar{u}}\ni u\mapsto c_u\in\mathbb{R}$)       
\begin{equation}\label{lem:PermDiagForThm:aux0.4}
\xi_i(u,u)\cdot\xi_i(v,v) = \varrho(u)\cdot\xi_i(u,v)^2,\quad\forall\, (u,v)\in\mathcal{Q}_{\bar{u}}.
\end{equation}          
But since the square $\mathcal{Q}_{\bar{u}}$ contains its diagonal, i.e.: $\mathcal{Q}_{\bar{u}}\supset\{(u,u)\mid u\in\mathcal{V}_{\bar{u}}\}\eqqcolon\Delta_{\mathcal{V}_{\bar{u}}}$, we can evaluate the relation \eqref{lem:PermDiagForThm:aux0.4} for the points in $\Delta_{\mathcal{V}_{\bar{u}}}$, yielding that $\varrho\equiv 1$. Consequently,
\begin{equation}
\restr{\xi_i}{\mathcal{Q}_{\bar{u}}}\!(u,v) = \epsilon\cdot\varsigma(u)\cdot\varsigma(v),\quad\forall\, (u,v)\in\mathcal{Q}_{\bar{u}},
\end{equation}for $\varsigma:\mathcal{V}_{\bar{u}}\ni x\mapsto \varsigma(x)\coloneqq\sqrt{|\xi_i(x,x)|}$ and with $\epsilon$ denoting the sign of $\restr{\xi_i}{\mathcal{Q}_{\bar{u}}}$. Notice that since $\xi_i$ is continuous and $\mathcal{Q}_{\bar{u}}$ is connected, $\epsilon$ will be constant (i.e.\ $\epsilon\equiv\pm 1$). But since $\mathcal{Q}_{\bar{u}}\subset\check{U}_\ast$ is symmetric, open and convex, Lemma \ref{lem:C2PseudoGaussian} (iii) then implies that  
\begin{equation*}
\restr{\varphi_i}{\mathcal{Q}_{\bar{u}}} \quad \text{ is pseudo-Gaussian}, \quad\text{ in contradiction to }\eqref{lem:PermDiagForThm:aux0.1.1}.
\end{equation*}             
This proves that each of the above sets $\tilde{U}_i$, $i\in[d-1]$, must be dense in $U$. 

But since each of the dense subsets $\tilde{U}_i$ is also open by the fact that the quotients in \eqref{lem:PermDiagForThm:aux0.3} are continuous in $(u,v)$, we (once more by the above fact (b)) find that their intersection
\begin{equation}\label{lem:PermDiagForThm:aux0.5}
\tilde{U}\coloneqq\bigcap_{i\in[d-1]}\tilde{U}_i \quad\text{ is a dense subset of } \ U.
\end{equation}
We claim that the above set $\tilde{U}$ satisfies \eqref{lem:PermDiagForThm:aux0}.  

To see this, note first that \eqref{lem:PermDiagForThm:eq1} yields\footnote{\ Cf.\ the proof of Lemma \ref{lem:jacobian_system} for details.} the system of matrix equations
\begin{equation}\label{lem:PermDiagForThm:aux1}
\begin{cases}
B_u^\intercal\cdot\hat{\Lambda}_{u,u}\cdot B_{u} &= \ \Lambda_{u,u}\\
B_u^\intercal\cdot\hat{\Lambda}_{u,v}\cdot B_{v} &= \ \Lambda_{u,v} \\
B_v^\intercal\cdot\hat{\Lambda}_{v,v}\cdot B_{v} &= \ \Lambda_{v,v}
\end{cases}\qquad\text{for each}\quad (u,v)\in U^{\times 2};
\end{equation}we prove \eqref{lem:PermDiagForThm:aux0} by defining a map $\eta:\tilde{U}\rightarrow V'$ with the property that \eqref{lem:PermDiagForThm:aux1} when evaluated at $(u,v)\coloneqq(u,\eta(u))$ yields $B_u\in\mathrm{M}_d$ by necessity. Let to this end $u\in\tilde{U}$ be arbitrary. Then by the definition of $\tilde{U}$ (recalling \eqref{lem:PermDiagForThm:aux0.3}), we for each $i\in[d]$ can find some open $\mathcal{V}_u^i\subset V'$ such that the intersection $\mathcal{V}_u\coloneqq\bigcap\nolimits_{i\in[d]}\mathcal{V}_u^i$ is non-empty and the continuous map 
\begin{equation}\label{lem:PermDiagForThm:aux2}
q_u \coloneqq q_u^1\times\cdots\times q_u^d \ : \ \mathcal{V}_u \ \rightarrow \ \mathbb{R}^d \qquad (q_u^i \text{ as in \eqref{lem:PermDiagForThm:aux0.3}}) 
\end{equation} 
is non-constant in its first $(d-1)$ components. Hence\footnote{\ Recall that by their definition in \eqref{lem:PermDiagForThm:aux0.3}, each component map $q^i_u$ in \eqref{lem:PermDiagForThm:aux2} is a function of the $i$-th component of its argument vector only, i.e.\ $q^i_u : v \mapsto q^i_u(v)=q^i_u(v_i)$ for each $v\equiv(v_i)\in\R^d$. (Remember that $\xi_i\equiv\xi(x_i, x_{i+d})$ by definition, cf.\ the hypothesis of Lemma \ref{lem:PermDiagForThm}.)} by the intermediate-value theorem, the set $\pi_{[d-1]}(q_u(\mathcal{V}_u))\subseteq\mathbb{R}^{d-1}$ has non-empty interior, which implies that for 
$\nabla^\times\coloneqq\{(v_i)\in\mathbb{R}^d\mid \exists\,i, j\in[d], i\neq j\,:\, |v_i|=|v_j|\}$ (the closed nullset of all vectors in $\mathbb{R}^d$ having two components differing at most up to a sign), the preimage
\begin{equation}\label{lem:PermDiagForThm:aux3}
\tilde{\mathcal{V}}_u\coloneqq q_u^{-1}(\mathbb{R}^d\setminus\nabla^{\times}) \ \subset \ \mathcal{V}_u
\end{equation} 
will be non-empty. This observation gives rise to maps of the form     
\begin{equation}
\eta \, : \, \tilde{U}\rightarrow V', \quad u \mapsto \eta(u)\in\tilde{\mathcal{V}}_u,
\end{equation}  
and as we will now see, any such map is of the desired type that we announced above. Indeed: Taking any $(u,v)\in\bigcup_{\tilde{u}\in \tilde{U}}\{\tilde{u}\}\times\tilde{\mathcal{V}}_{\tilde{u}}\subseteq\check{U}_\ast$, we obtain from \eqref{lem:PermDiagForThm:aux1} that\footnote{\ Note that the invertibility of $\Lambda_{u,v}$ is obtained from the choice of $(u,v)$.}       
\begin{equation}\label{lem:PermDiagForThm:aux4}
B_u^{-1}\cdot\bar{\Lambda}_{u,v}\cdot B_u \ = \ \tilde{\Lambda}\qquad\text{ for }\quad \bar{\Lambda}_{u,v}\equiv\operatorname{diag}_{i\in[d]}[\lambda^i_{u,v}]\coloneqq\hat{\Lambda}_{u,u}\cdot\hat{\Lambda}_{u,v}^{-2}\cdot\hat{\Lambda}_{v,v}
\end{equation}and the diagonal matrix $\tilde{\Lambda}\coloneqq\Lambda_{u,v}^{-1}\cdot\Lambda_{v,v}\cdot\Lambda_{u,v}^{-1}\cdot\Lambda_{u,u}$. Observing now (recalling \eqref{lem:PermDiagForThm:aux0.3}) that 
\begin{equation}
(\lambda_{u,v}^1, \cdots, \lambda_{u,v}^d) \, = \, q_u(v) \quad\text{ for the function } \ q_u \ \text{ defined in \eqref{lem:PermDiagForThm:aux2}},
\end{equation}we immediately obtain by the definition \eqref{lem:PermDiagForThm:aux3} of $\tilde{\mathcal{V}}_u$ that 
\begin{equation}\label{lem:PermDiagForThm:aux5}
\text{the eigenvalues } \quad \lambda^1_{u,v}, \ldots, \lambda^d_{u,v} \ \text{ of } \ \bar{\Lambda}_{u,v} \ \quad\text{ are pairwise distinct}. 
\end{equation}Hence by the elementary fact that diagonal matrices with pairwise distinct eigenvalues are stabilised by monomial matrices only, observation \eqref{lem:PermDiagForThm:aux5} by way of the similarity equation \eqref{lem:PermDiagForThm:aux4} finally implies $B_u\in\mathrm{M}_d$ --- and hence \eqref{lem:PermDiagForThm:aux0} --- as desired.    
\end{proof} 

\subsection{A Second Separation Lemma} The next lemma underlies the proof of Theorem \ref{cor:NICA_MainCor}.  

\begin{appendixlemma}\label{lem:jacobian_system}
Suppose that in addition to \eqref{intext:BSS_relation} there are $\p_0, \p_1, \p_2\in\Delta_2(\I)$ and $(u,v)\in\R^{2d}$ such that $S$ is $C^2$-regular at $(\p_0, (u,v))$, $(\p_1, (u,u))$ and $(\p_2, (v,v))$ with density $\zeta_{\p_0}$, $\zeta_{\p_1}$ and $\zeta_{\p_2}$, respectively. Then for any map $h$ which is $C^{2}$-invertible on some open superset of $D_X$ and such that $h(X)$ has independent components, we have that
\begin{equation}\label{lem:jacobian_system:eq1}
B_\varrho(u)^{-1}\cdot\bar{\Lambda}_{\p_0,\p_1,\p_2}(u,v)\cdot B_\varrho(u) \ = \ \tilde{\Lambda}_{\p_0,\p_1,\p_2}(u,v)
\end{equation}for $B_\varrho$ the inverse Jacobian of $\varrho\coloneqq h\circ f$ and the diagonal matrices
\begin{equation}\label{lem:jacobian_system:eq2}
\begin{aligned}
\bar{\Lambda}_{\p_0,\p_1,\p_2}(u,v) \,&\coloneqq\, \Lambda_{\xi_{\mathfrak{p}_1}}(u,u)\cdot\Lambda_{\xi_{\mathfrak{p}_0}}(u,v)^{-2}\cdot\Lambda_{\xi_{\mathfrak{p}_2}}(v,v)\qquad\qquad\text{ and }\\[-0.5em]
\tilde{\Lambda}_{\p_0,\p_1,\p_2}(u,v)\,&\coloneqq\,\Lambda_{q_{\mathfrak{p}_1}}(u,u)\cdot\Lambda_{q_{\mathfrak{p}_0}}(u,v)^{-2}\cdot\Lambda_{q_{\mathfrak{p}_2}}(v,v), 
\end{aligned}
\end{equation}where $\Lambda_{\xi_{\p_\nu}}\!(x)\coloneqq\operatorname{diag}[\xi^1_{\p_\nu}\!(x), \cdots, \xi^d_{\p_\nu}\!(x)]$ for $\xi_{\mathfrak{p}_\nu}^i\!(x)\coloneqq \partial_{x_i}\partial_{x_{i+d}}\log\zeta_{\mathfrak{p}_\nu}\!(x)$ and $\Lambda_{q_{\mathfrak{p}_\nu}}\!(x)\coloneqq\mathrm{diag}[q^1_{\p_\nu}\!(x), \cdots, q^d_{\p_\nu}\!(x)]$ given by the LHS of \eqref{thm:NICA_stat:aux14.0} (in dependence of $\mathfrak{p}_\nu$).                             
\end{appendixlemma}  
\begin{proof}
Copying the argumentation that led to \eqref{thm:NICA_stat:aux15}, we obtain the congruence relations 
\begin{equation}
\Lambda_{q_{\mathfrak{p}_\nu}}(\tilde{u},\tilde{v}) \ = \ B_\varrho^\intercal(\tilde{u})\cdot\Lambda_{\xi_{\mathfrak{p}_\nu}}\!(\tilde{u},\tilde{v})\cdot B_\varrho(\tilde{v}) \quad \text{ for each } \ (\tilde{u},\tilde{v})\in\{\zeta_{\p_\nu}>0\}
\end{equation}and $\nu=0,1,2$. Evaluating these at the points $(u,v), (u,u)$ and $(v,v)$, we arrive at the system 
\begin{align}
B_\varrho(u)^\intercal\cdot\Lambda_{\xi_{\mathfrak{p}_0}}(u,v)\cdot B_\varrho(v) \ &= \ \Lambda_{q_{\mathfrak{p}_0}}(u,v)\label{cor:NICA_MainCor:aux8.1}\\
B_\varrho(u)^\intercal\cdot\Lambda_{\xi_{\mathfrak{p}_1}}(u,u)\cdot B_\varrho(u) \ &= \ \Lambda_{q_{\mathfrak{p}_1}}(u,u)\label{cor:NICA_MainCor:aux8.2}\\
B_\varrho(v)^\intercal\cdot\Lambda_{\xi_{\mathfrak{p}_2}}(v,v)\cdot B_\varrho(v) \ &= \ \Lambda_{q_{\mathfrak{p}_2}}(v,v).\label{cor:NICA_MainCor:aux8.3}
\end{align}  
From \eqref{cor:NICA_MainCor:aux8.1} we then find that
\begin{equation} 
B_\varrho(v) = \Lambda^{-1}_{\xi_{\mathfrak{p}_0}}(u,v)\cdot\big(B_\varrho(u)^\intercal\big)^{-1}\!\cdot\Lambda_{q_{\mathfrak{p}_0}}(u,v), 
\end{equation}  
which, when plugged into \eqref{cor:NICA_MainCor:aux8.3}, yields
\begin{equation}\label{cor:NICA_MainCor:aux9}
\begin{aligned}
B_\varrho(u)^{-1}\cdot\Lambda_{\xi_{\mathfrak{p}_2}}(v,v)\cdot\Lambda^{-2}_{\xi_{\mathfrak{p}_0}}(u,v)\cdot\big(B_\varrho(u)^\intercal\big)^{-1} = \Lambda_{q_{\mathfrak{p}_2}}(v,v)\cdot\Lambda_{q_{\mathfrak{p}_0}}(u,v)^{-2}
\end{aligned}
\end{equation} 
and hence, upon left-multiplying both sides of \eqref{cor:NICA_MainCor:aux8.2} by the matrix product \eqref{cor:NICA_MainCor:aux9},     
\begin{equation}
B_\varrho(u)^{-1}\cdot\Lambda_{\xi_{\mathfrak{p}_2}}\!(v,v)\Lambda^{-2}_{\xi_{\mathfrak{p}_0}}\!(u,v)\Lambda_{\xi_{\mathfrak{p}_1}}\!(u,u)\cdot B_\varrho(u) = \Lambda_{q_{\mathfrak{p}_2}}\!(u,u)\Lambda_{q_{\mathfrak{p}_2}}\!(v,v)\Lambda_{q_{\mathfrak{p}_0}}\!(u,v)^{-2}.
\end{equation}
This last equation is identical to \eqref{lem:jacobian_system:eq1}, as desired. 
\end{proof}

\subsection{Proof of Theorem \ref{cor:NICA_MainCor} for $\beta$-Contrastive Sources}\label{cor:NICA_MainCor:beta}
\begin{proof}
Suppose that $S$ is $\beta$-contrastive. In this case, we consider the set 
\begin{equation}\label{cor:NICA_MainCor:aux3.1}
\mathcal{D}_0\coloneqq \left\{u\in\mathrm{int}(D_S) \ \middle| \ \exists\,(\delta,\mathfrak{p})\in\mathbb{R}_{>0}\times\Delta_2(\mathbb{I}) \ \text{ satisfying \  \eqref{cor:NICA_MainCor:aux4} and \eqref{cor:NICA_MainCor:aux5}}\right\}
\end{equation}
with properties \eqref{cor:NICA_MainCor:aux4}, \eqref{cor:NICA_MainCor:aux5} that for a given $(\delta,\p)\in\mathbb{R}_{>0}\times\Delta_2(\mathbb{I})$ are defined as
\begin{gather}\label{cor:NICA_MainCor:aux4}
B_\delta(u) \subset D_S \quad \text{with}\quad\restr{\Lambda_{\xi_{\mathfrak{p}}}\!}{B_\delta(u)\times B_\delta(u)}\subset\operatorname{GL}_d(\mathbb{R}), \qquad \text{and \ }\\ 
\begin{gathered}\label{cor:NICA_MainCor:aux5}
\text{there is } \ \p'\in\Delta_2(\mathbb{I}) \text{ \ s.t.\ \ for \ $\mathcal{U}\coloneqq B_\delta(u)$ and each $i\in[d]$, \ the diagonal restriction}\\[-0.5em]
\xi^{i\,|\,\mathcal{U}}_{\mathfrak{p}'}\ \text{ vanishes nowhere \ and \ is such that } \ \xi_{\mathfrak{p}'}^{i\,|\,\mathcal{U}} \notin \big\langle\xi_{\mathfrak{p}}^{i\,|\,\mathcal{U}}\big\rangle_\R.
\end{gathered} 
\end{gather}  
Let us show first that $\mathcal{D}_0$ is dense in the interior $D_S^\circ\coloneqq\mathrm{int}(D_S)$. Indeed: Since $D_S^\circ$ is open, assuming that $\mathcal{D}_0$ is not dense in $D_S^\circ$ implies that there exists $(u_\ast, r)\in D_S^\circ\times\R_{>0}$ with $B_r(u_\ast)\subseteq D_S^\circ\setminus\mathcal{D}_0$. Now since $S$ is $\beta$-contrastive, there will be some $(\tilde{u}_\ast,r_1)\in B_r(u_\ast)\times(0,r)$ with $B_{r_1}(\tilde{u}_\ast)\subseteq B_r(u_\ast)$ such that both \eqref{def:log-regular:eq1} and \eqref{def:log-regular:eq2} hold everywhere on $\tilde{U}\coloneqq B_{r_1}(\tilde{u}_\ast)$ for some $\p, \tilde{\p}'\in\Delta_2(\I)$; thus also $\eqref{cor:NICA_MainCor:aux5}$ holds for $\mathcal{U}=\tilde{U}$ and $(\p,\p')\coloneqq(\tilde{\p},\tilde{\p}')$. And since the functions $\xi_{\tilde{\p}}^{i}$ are continuous at $(\tilde{u}_\ast,\tilde{u}_\ast)$, there (due to \eqref{def:log-regular:eq1}) will further be some $r_2>0$ such that $\xi^i_{\tilde{\p}}$ vanishes nowhere on $B_{r_2}(\tilde{u}_\ast)^{\times 2}\subset D_S^{\times 2}$ for each $i\in[d]$; hence also \eqref{cor:NICA_MainCor:aux4} holds for $(\delta,\p)\coloneqq(r_2,\tilde{\p})$. But this yields that both  \eqref{cor:NICA_MainCor:aux4} and \eqref{cor:NICA_MainCor:aux5} hold for $u\coloneqq\tilde{u}_\ast$ and $\delta\coloneqq\min(r_1,r_2)$ and $(\p, \p')\coloneqq(\tilde{\p}, \tilde{\p}')$, which implies that $\tilde{u}_\ast\in D_S^\circ\setminus\mathcal{D}_0$ is an element of $\mathcal{D}_0$. 

As this is obviously a contradiction, the set \eqref{cor:NICA_MainCor:aux3.1} must be dense in $D_S^\circ$.                       

Now since the interior $D_S^\circ$ is dense in $D_S$ by assumption, the theorem's assertion follows if we can show that \eqref{cor:NICA_MainCor:aux3} holds for $\mathcal{D}\coloneqq D_S^\circ$. But since in turn $\mathcal{D}_0$ is dense in $D_S^\circ$, we obtain that \eqref{cor:NICA_MainCor:aux3} holds for $\mathcal{D}\coloneqq D_S^\circ$ if we can show that \eqref{cor:NICA_MainCor:aux3} holds for $\mathcal{D}\coloneqq \mathcal{D}_0$.\footnote{\ Cf.\ \eqref{lem:PermDiagForThm:aux0} and the argument around \eqref{lem:monomial_trafos:aux12} for details.}     

Let to this end $u\in\mathcal{D}_0$ be fixed with $\p,\p'\in\Delta_2(\I)$ and $\mathcal{U}\equiv B_\delta(u)\subset D_S$ as in \eqref{cor:NICA_MainCor:aux4} and \eqref{cor:NICA_MainCor:aux5}. Then by Lemma \ref{lem:jacobian_system} we have that
\begin{equation}\label{cor:NICA_MainCor:aux9.2}
B_\varrho(\tilde{u})^{-1}\cdot\bar{\Lambda}_\nu(\tilde{u},\tilde{v})\cdot B_\varrho(\tilde{u}) \ = \ \tilde{\Lambda}_\nu(\tilde{u},\tilde{v}) \qquad\text{ for each } \ (\tilde{u},\tilde{v})\in\mathcal{U}\times\mathcal{U}
\end{equation} 
with $\nu=1,2$ and diagonal matrices $\bar{\Lambda}_1, \bar{\Lambda}_2, \tilde{\Lambda}_1, \tilde{\Lambda}_2 \in\operatorname{GL}_d(\R)$ given by
\begin{equation} 
\bar{\Lambda}_1 \coloneqq \bar{\Lambda}_{\p,\p,\p}, \quad \tilde{\Lambda}_1 \coloneqq \tilde{\Lambda}_{\p,\p,\p} \qquad \text{ and }\qquad  
\bar{\Lambda}_2 \coloneqq \bar{\Lambda}_{\p,\p,\p'}, \quad \tilde{\Lambda}_2 \coloneqq \tilde{\Lambda}_{\p,\p,\p'} 
\end{equation}    
with matrices $\bar{\Lambda}_{\p_0,\p_1,\p_2}$ and $\bar{\Lambda}_{\p_0,\p_1,\p_2}$ as defined in \eqref{lem:jacobian_system:eq2}. Combining the cases $\nu=1$ and $\nu=2$ of \eqref{cor:NICA_MainCor:aux9.2}, we for any $C\in\mathbb{R}$ obtain that
\begin{equation}\label{cor:NICA_MainCor:aux9.3}
B_\varrho(\tilde{u})^{-1}\cdot\left[\bar{\Lambda}_1(\tilde{u},\tilde{v}) + C\cdot\bar{\Lambda}_2(\tilde{u},\tilde{v})\right]\cdot B_\varrho(\tilde{u}) \ = \ \tilde{\Lambda}_1(\tilde{u},\tilde{v}) + C\cdot\tilde{\Lambda}_2(\tilde{u},\tilde{v})
\end{equation} 
for each $(\tilde{u},\tilde{v})\in\mathcal{U}\times\mathcal{U}$. Hence (and since $u\in\mathcal{D}_0$ was chosen arbitrarily), the identity \eqref{cor:NICA_MainCor:aux9.3} implies \eqref{cor:NICA_MainCor:aux3} if there is a pair $(C,\tilde{v})\in\mathbb{R}\times\mathcal{U}$ for which the diagonal entries of $[\bar{\Lambda}_1(u,\tilde{v}) + C\cdot\bar{\Lambda}_2(u,\tilde{v})]\eqqcolon\mathrm{diag}[\lambda^1_{u,\tilde{v}}, \cdots, \lambda^d_{u,\tilde{v}}]$ are pairwise distinct. We now prove this, i.e.\ we show that
\begin{gather} 
\text{there is } \ C\in\mathbb{R} \ \text{ for which we can find some } \ \tilde{v}\in \mathcal{U} \ \text{ s.t.\ the diagonal entries }\notag\\[-0.5em]
\lambda^i_{u,\tilde{v}} = \frac{\xi^i_{\mathfrak{p}}(u,u)}{\xi^i_{\mathfrak{p}}(u,\tilde{v})^2}\cdot(\xi^i_{\mathfrak{p}}(\tilde{v},\tilde{v}) + C\cdot\xi^i_{\mathfrak{p}'}(\tilde{v},\tilde{v})), \ i\in[d], \ \text{ are pw.\ distinct.}\label{cor:NICA_MainCor:aux10}
\end{gather}  
Notice that, as detailed in the proof of Lemma \ref{lem:PermDiagForThm}, the fact that by construction each of the functions $q_i : \mathcal{U}\times\mathcal{U}\ni(\tilde{u},\tilde{v})\mapsto\lambda^i_{\tilde{u},\tilde{v}}$ $(i\in[d])$ are continuous implies that \eqref{cor:NICA_MainCor:aux10} holds if 
\begin{equation}\label{cor:NICA_MainCor:aux11}
\begin{gathered}
\exists\, C\in\mathbb{R} \ \text{ such that } \ \vartheta_i \,:\, \mathcal{U}\ni\tilde{v}\,\mapsto\, q_i(u,\tilde{v}) \ \text{ is non-constant for each } \ i\in[d]. 
\end{gathered}
\end{equation}  
To prove \eqref{cor:NICA_MainCor:aux11}, notice that since for each $i\in[d]$ we have the decomposition
\begin{equation}
\vartheta_i = \theta_i + C\cdot\theta_i'\qquad\text{with } \quad \theta_i(\tilde{v})\coloneqq\frac{\xi^i_{\mathfrak{p}}(u, u)\xi^i_{\mathfrak{p}}(\tilde{v},\tilde{v})}{\xi^i_{\mathfrak{p}}(u,\tilde{v})^2}  
\end{equation} 
and $\theta_i':\mathcal{U}\times\mathcal{U}\rightarrow\mathbb{R}$ defined likewise but with the right factor in the above enumerator replaced by $\xi^i_{\mathfrak{p}'}(\tilde{v},\tilde{v})$, we find that if, for each $i\in[d]$, the functions  
\begin{equation}\label{cor:NICA_MainCor:aux13}
\theta_i \ \text{ or } \ \theta_i' \ \text{ are non-constant}, \quad\text{ then }\quad \eqref{cor:NICA_MainCor:aux11} \ \text{ holds.}
\end{equation}  
Indeed: If \emph{either} $\theta_i$ \emph{or} $\theta_i'$ is non-constant in $\tilde{v}$, then clearly their linear combination $\vartheta_i$ will be non-constant in $\tilde{v}$ for each $C\neq 0$. If $\theta_i$ or $\theta_i'$ are \emph{both} non-constant in $\tilde{v}$, then there might be some $C_i\in\mathbb{R}$ such that $\theta_i + C_i\cdot\theta_i'$ is constant in $\tilde{v}$ (define $C_i\coloneqq 1$ otherwise); in this case, setting $C\coloneqq\max_{i\in[d]}C_i +  1$ implies that $\vartheta_i = \theta_i + C\cdot\theta_i' = (\theta_i + C_i\cdot\theta_i') + (C-C_i)\cdot\theta_i'$ is non-constant in $\tilde{v}$ for each $i\in[d]$, as desired. 

To see that the premise of \eqref{cor:NICA_MainCor:aux13} holds, assume otherwise that there is $i\in[d]$ for which the function $\theta_i : \mathcal{U}\rightarrow\mathbb{R}$ is constant in $\tilde{v}$, say 
\begin{equation}
\theta_i(\tilde{v}) \eqqcolon \varsigma_i \quad\text{ for all } \ \tilde{v}\in\mathcal{U}. 
\end{equation} 
Then, as $\theta_i$ vanishes nowhere in consequence of \eqref{cor:NICA_MainCor:aux4}, we find that
\begin{equation}\label{cor:NICA_MainCor:aux15}
\Big[\xi^i_{\mathfrak{p}}(u,\cdot)\Big]^2 = \frac{\xi^i_{\mathfrak{p}}(u, u)\cdot\xi^{i\,|\,\mathcal{U}}_{\mathfrak{p}}}{\theta_i} = c_i\cdot\eta \qquad\text{on } \ \ \mathcal{U}
\end{equation}    
for the constant $c_i\coloneqq\xi^i_{\mathfrak{p}}(u,u)\cdot\varsigma_i^{-1}$ and the function $\eta : \mathcal{U}\rightarrow \mathbb{R}$ given by $\eta(\tilde{v})\coloneqq\xi^i_{\mathfrak{p}}(\tilde{v}, \tilde{v})$. Now if the function $\theta_i'$ were constant as well, say $\theta_i'\equiv\varsigma_i'\,(\neq 0)$, then we would likewise obtain that $\big[\xi^i_{\mathfrak{p}}(u,\cdot)\big]^2 = c_i'\cdot\eta'$ on $\mathcal{U}$, for the non-zero constant $c_i'\coloneqq\xi^i_{\mathfrak{p}}(u,u)\cdot(\varsigma_i')^{-1}$ and the function $\eta' : \mathcal{U}\rightarrow \mathbb{R}$ given by $\eta'(\tilde{v})\coloneqq\xi^i_{\mathfrak{p}'}(\tilde{v}, \tilde{v})$. Combined with \eqref{cor:NICA_MainCor:aux15}, we find that 
\begin{equation}   
c_i'\cdot\eta' = c_i\cdot\eta \quad \text{ and hence }\quad \xi^{i\,|\,\mathcal{U}}_{\mathfrak{p}'} = \mathrm{const.}\cdot\xi^{i\,|\,\mathcal{U}}_{\mathfrak{p}},
\end{equation} 
the latter contradicting \eqref{cor:NICA_MainCor:aux5}. This proves the premise of \eqref{cor:NICA_MainCor:aux13} and hence \eqref{cor:NICA_MainCor:aux3} for $\mathcal{D}=\mathrm{int}(D_S)$. 
\end{proof}

\subsection{Non-Convex Spatial Supports}\label{sect:nonconvexsupport}
In certain contexts of interest it may happen that the source process $S$ in \eqref{intext:BSS_relation} does not conform to Assumption \ref{sect:assumption_convexity}, i.e.\ that not every connected component of $D_S$ is convex. (As to the existence of IC processes with non-convex spatial support, think for instance of a $d$-dimensional Brownian motion (ran up to its exit time) that starts within a non-convex domain and gets killed when it hits the boundary of that domain.)\\[-0.5em]

\noindent
For such general geometries of $D_S$, we have the (not necessarily disjoint) decomposition 
\begin{equation}\label{sect:nonconvexsupport:eq1}
D_S \ = \bigcup_{\substack{\emptyset\,\neq\, C_S \,\subseteq\, D_S,\\ C_S \ \text{maximally convex}}} \hspace{-1em} C_S
\end{equation}    
of the set $D_S$ into its \emph{convex components} $C_S$ which, in analogy to the connected components of $D_S$, are defined as the maximally [wrt.\ set inclusion $\subseteq$] convex subsets of $D_S$.\\[-0.5em]

Denote by $\mathcal{C}(A)$ the set of all convex components of a set $A\subseteq\R^d$. Assumption \ref{sect:assumption_convexity} then requires that: $|\mathcal{C}(D)|=1$ for each connected component $D$ of $D_S$. Assume now that the source $S$ in \eqref{intext:BSS_relation} is a process with general support \eqref{sect:nonconvexsupport:eq1}, possibly violating Assumption \ref{sect:assumption_convexity}. A quick inspection of the proofs of Theorems \ref{thm:NICA_stat}, \ref{cor:NICA_MainCor} then shows that if $S$ is $\{\alpha, \beta, \gamma\}$-contrastive then \eqref{thm:NICA_stat:aux0} holds for each $D\in\mathcal{C}(D_S)$; in fact, from the proof of \eqref{lem:monomial_trafos:aux2} we see that $\left.J_{h\circ f}\right|_G \subseteq \{\Lambda\cdot P \mid \Lambda\in\Delta_d\}$ (inclusion holding pointwise on $G$) for any connected $G\subseteq D_S$, where $P\equiv P_G\in\mathrm{P}_d$ is some fixed $G$-dependent permutation matrix. This in combination with Lemma \ref{lem:monomial_trafos} proves the following generalisation of the characterisations \eqref{thm:NICA_stat:eq3} and \eqref{cor:NICA_MainCor:eq1}: 

\begin{appendixthm}\label{sect:nonconvexsupport:thm}
Let the process $S$ in \eqref{intext:BSS_relation} be $\alpha$-, $\beta$- or $\gamma$-contrastive with a spatial support $D_S$ for which Assumption \ref{sect:assumption_convexity} may not be satisfied. Then for any transformation $h$ which is $C^{2}$-invertible on some open superset of $D_X$, we have that $h(X)$ is IC only if: 
\begin{equation}\label{sect:nonconvexsupport:thm:eq1}
\forall\, D\subseteq D_S \ \text{ connected} \, : \, \left.(h\circ f)\right|_C \in \mathrm{DP}_d(C) \ \ \text{ for each } C\in\mathcal{C}(D).
\end{equation} 
\end{appendixthm}     

The stronger conclusions \eqref{thm:NICA_stat:eq3} and \eqref{cor:NICA_MainCor:eq1} may at times be fully rehabilitated even for non-convex geometries [of the connected components] of \eqref{sect:nonconvexsupport:eq1}, as the following example shows.    

\begin{example}[Monomial Inversion for Sources with Non-Convex Spatial Support]\label{sect:nonconvexsupport:example}Let $S=(S^1, S^2)$ be a process in $\R^2$ with a C-shaped\footnote{\ We thank one of our referees for suggesting this as a simple yet non-trivial example geometry for \eqref{sect:nonconvexsupport:thm:eq1}.} spatial support as shown in Figure \ref{fig:nonconvex-example}. If further $S$ is an $\alpha$-, $\beta$- or $\gamma$-contrastive source within the ICA-context \eqref{intext:BSS_relation}, then (although $D_S$ does not satisfy Assumption \ref{sect:assumption_convexity}) the classical conclusion \eqref{thm:NICA_stat:eq3} holds, i.e.\ with probability one:
\begin{equation}\label{sect:nonconvexsupport:example:eq1}
h(X) \ \in \ \mathrm{DP}_{\!d}\cdot S \quad\text{if and only if}\quad h(X) \ \text{ has independent components},
\end{equation}
for any transformation $h$ which is $C^{2}$-invertible on some open superset of $D_X$. 

\begin{figure}[htpb] 
\centering
\hspace*{-1.25em}
\includegraphics[trim={0em 0 0em 0},clip,scale=0.4]{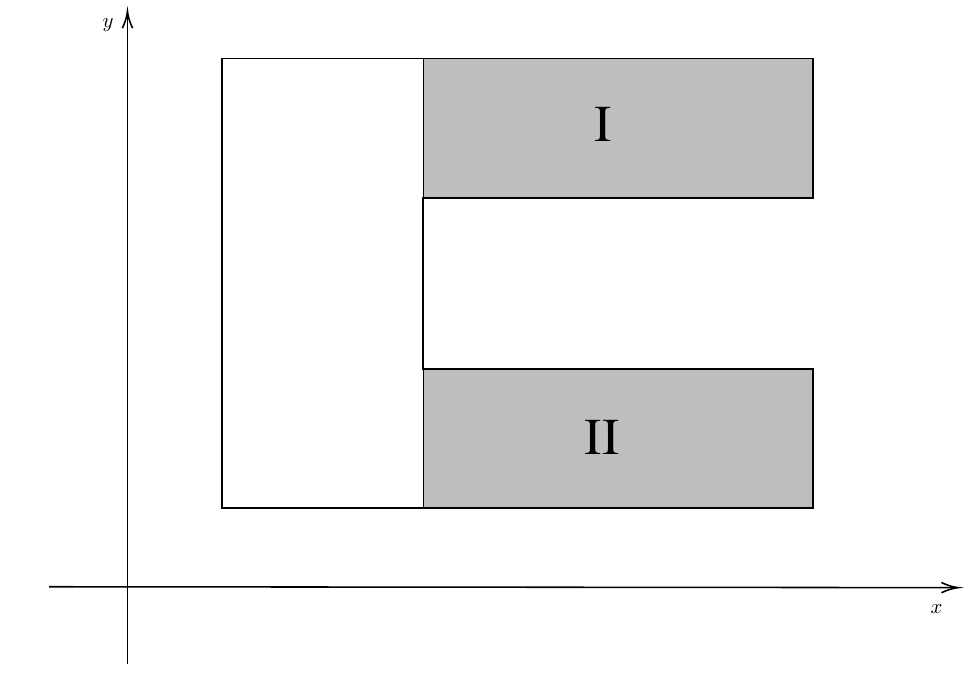}
\caption{\emph{A connected but non-convex spatial support $D_S$ (with arms $\mathrm{I}, \mathrm{II}\subset D_S$).}}\label{fig:nonconvex-example} 
\end{figure}

\emph{Proof of \eqref{sect:nonconvexsupport:example:eq1}.} Let $h : D_X \rightarrow \R^2$ be any $C^{2}$-diffeomorphism for which $\tilde{S}\coloneqq h(X)$ has independent components (as always, such an $h$ exists by the ICA-assumption \eqref{intext:BSS_relation}). By design, $D_S$ has convex subsets $C_1\coloneqq\mathrm{I}, C_2\coloneqq\mathrm{II}$ (with $C_3\coloneqq D_S\setminus(\mathrm{I}\cup\mathrm{II})$) of the form $C_1 = I_0\times I_1$ and $C_2 = I_0\times I_2$ and $C_3 = J_1\times J_2$ for some intervals $I_0, I_1, I_2, J_1, J_2\subset\R$, with $I_1$ and $I_2$ disjoint. Denoting $\varrho\coloneqq h\circ f$, the characterisation \eqref{sect:nonconvexsupport:thm:eq1} implies that $\varrho$ acts monomially on each $C_i$, that is\footnote{\ We may assume here for convenience that the [$D_S$-global, as the above $D_S$ is connected] $\varrho$-supporting permutation $P$ (cf.\ Definition \ref{def:monomial_trafos}) is simply the identity.}: 
\begin{equation}\label{sect:nonconvexsupport:example:aux1}
\varrho(x,y) = (\alpha_i(x), \beta_i(y)), \ \ \forall\, (x,y)\in C_i \qquad (i=1,2,3)  
\end{equation}   
for some $\alpha_i : I_0\rightarrow\R$, $\beta_i : I_i\rightarrow\R$ $(i=1,2)$ and $\alpha_3:J_1\rightarrow\R$, $\beta\equiv\beta_3 : J_2\rightarrow\R$, each strictly monotone; note that in fact $\beta_i = \restr{\beta}{I_i}$ for $i=1,2$. Our aim is to conclude \eqref{sect:nonconvexsupport:example:eq1} from a combination of \eqref{sect:nonconvexsupport:example:aux1} and the fact that $\tilde{S}=\varrho(S)\equiv(\tilde{S}^1_t, \tilde{S}^2_t)_{t\in\I}$ is IC. Clearly, for \eqref{sect:nonconvexsupport:example:eq1} to hold it is sufficient to show that $\alpha_1=\alpha_2$. To this end, suppose that $\alpha_1\neq \alpha_2$. Then by continuity there is some non-empty open $U\subset I_0$ with $\alpha_1(U)\cap\alpha_2(U)=\emptyset$.\footnote{\ Since by continuity of $\varrho$ we have $\alpha_1(\inf I_0) = \alpha_2(\inf I_0) = \alpha_3(\sup J_1)$, the set $U$ can be chosen such that $\alpha_1(u_0) = \alpha_2(u_0)$ for $u_0\coloneqq\inf U$.} By \eqref{sect:nonconvexsupport:example:aux1} and since for any fixed $t\in\I$ the vector $S_t\equiv(S^1_t, S^2_t)$ is IC, we have for each subset $\tilde{U}\subseteq U$ that              
\begin{equation}
a\coloneqq\mathbb{P}(\tilde{S}^1_t\in\alpha_2(\tilde{U})\mid \tilde{S}^2_t\in \beta(I_1)) = \mathbb{P}(S^1_t\in\alpha_1^{-1}(\alpha_2(\tilde{U}))\mid S^2_t\in I_1) = \mathbb{P}(S^1_t\in\alpha_1^{-1}(\alpha_2(\tilde{U}))
\end{equation}
and likewise $b\coloneqq \mathbb{P}(\tilde{S}^1_t\in\alpha_2(\tilde{U})\mid \tilde{S}^2_t\in \beta(I_2)) = \mathbb{P}(S^1_t\in\tilde{U}\mid S^2_t\in I_2) = \mathbb{P}(S^1_t\in\tilde{U})$.  Now since $\tilde{S}^1_t$ and $\tilde{S}^2_t$ are independent, we must have $a=b$ and thus (and for $\gamma\coloneqq\alpha_2^{-1}\circ\alpha_1$) find
\begin{equation}\label{sect:nonconvexsupport:example:aux2}
\mathbb{P}(S^1_t\in\tilde{U}) = \mathbb{P}(\gamma(S^1_t)\in\tilde{U}) \quad\text{for each } \tilde{U}\subseteq U.
\end{equation} 
Since the process $S$ is ($\alpha$-, $\beta$- or $\gamma$-)contrastive, we can choose $t\in\I$ and the above set $U$ such that $S^1_t$ admits a $C^1$-density $\zeta_t$ that is positive on $U$. Consequently (cf.\ Remark \ref{rem:TrafoProjProbDens}), the transformed variable   
$\gamma(S^1_t)$ admits a $C^1$-density too, say $\zeta_t^\gamma$, and \eqref{sect:nonconvexsupport:example:aux2} [via \eqref{rem:TrafoProjProbDens:eq0A}] implies that 
\begin{equation}\label{sect:nonconvexsupport:example:aux3}
\zeta_t = (\zeta_t\circ\eta)\cdot\tfrac{\mathrm{d}\eta}{\mathrm{d}x} \quad\text{on \ $U$}  
\end{equation}
for $\eta\coloneqq\gamma^{-1}$. The identity \eqref{sect:nonconvexsupport:example:aux3} is an ordinary differential equation on $U$ of the form $\dot{\eta}(x) = F(x, \eta(x))$ with initial value $\eta(u_0)=u_0$ (by choice of $U$) and right-hand side $F\equiv F(x,u)\coloneqq\zeta_t(x)/\zeta_t(u)$. Since $F$ is locally Lipschitz on $U^{\times 2}$, the solution to the initial value problem $\{\eqref{sect:nonconvexsupport:example:aux3}, \, \eta(u_0)=u_0$ is unique, which implies that $\eta = \mathrm{id}_U$. The latter, however, implies that $\alpha_1(U)=\alpha_2(U)$, in contradiction to the choice of $U$. This overrules the assumption $\alpha_1\neq\alpha_2$ and hence proves \eqref{sect:nonconvexsupport:example:eq1}, as desired. \hfill $\square$    
\end{example}

The strategy behind Example \ref{sect:nonconvexsupport:example} may be generalised to restore the stronger conclusion \eqref{thm:NICA_stat:eq3} (instead of only \eqref{sect:nonconvexsupport:thm:eq1}) for some sources for which Assumption \ref{sect:assumption_convexity} is not satisfied.

\subsection{An Error in the Proof of \citep[Theorem 1]{HYM}}\label{rem:hyvarinenmorioka2} 
The (fixable) error occurs in the proof of \citep[Lemma 2 (Supplement)]{HYM} (we use their notation for the following): Lemma 2 in \cite{HYM} requires that for each $\bar{\bm{\mathrm{u}}}^1\in\R^d$ there exists an $\bar{\bm{\mathrm{u}}}^2\in\R^d$ such that the diagonal entries $\psi_i(\bar{\mathrm{u}}^1_i,\bar{\mathrm{u}}^2_i)\equiv\psi_i(a,b)$ $(i=1,\ldots,d$; cf.\ \citep[Eq.\ (42) (Supplement)]{HYM}) of the matrix $\bm{\mathrm{D}}_{11}\bm{\mathrm{D}}_{12}^{-2}\bm{\mathrm{D}}_{22}(\bar{\bm{\mathrm{u}}}^1,\bar{\bm{\mathrm{u}}}^2)$  are pairwise distinct. For this, \citep{HYM} rely on (the continuity of the $\psi_i$ and) an indirect proof -- via contradiction to the exclusion of \citep[Def.\ 2 eq.\ (4)]{HYM} -- of the assertion that $(A):$ ``The function $\psi_i(a,\cdot)$ is non-constant for [almost] every given $a\in\R$.'' Yet instead of leading the negation of $(A)$ -- i.e., $(\neg A)$: ``The function $\psi_i(a,\cdot)$ is constant for some $a\in\R$'' -- to a contradiction, their proof by contradiction departs from the \emph{stronger} assumption that $(B):$ ``The function $\psi_i(a,\cdot)$ is constant \emph{for each} $a\in\R$'' (which implies, but is not generally implied by, $(\neg A)$). The reductio ad absurdum of $(B)$ provided in \citep{HYM} is therefore insufficient to prove $(A)$. The (negation \citep[Thm.\ 1 Assmpt.\ 3.]{HYM} of the) given factorisation property \citep[Def.\ 2 eq.\ (4)]{HYM}, on the other hand, is \emph{too weak} to contradict $(\neg A)$ as would be required for the (indirect) proof of $(A)$. To see this, it suffices to note the existence of functions $q : \R^2\rightarrow\R_\times$ which are not of the (excluded) global product form \citep[Def.\ 2 eq.\ (4)]{HYM} but factorize locally on a given `strip' $R\coloneqq I\times\R$, for some $I\subset\R$. Indeed: Given such a function\footnote{\ And assuming $q\equiv q_{x,y}$ to be of the form \citep[Def.\ 1 eq.\ (3)]{HYM} for a probability density $p_{x,y}$ on $\R^2$, which can be established by  normalisation and straightforward decay modifications.} $q$, the assumption of $(\neg A)$ (in the indirect proof of $(A)$), with $\psi_i$ and $q$ related by \citep[eq.\ (42)]{HYM} and $\psi_i(a,\cdot)$ constant iff $a\in I$ (with $I\neq\R$ in general), leads to the \emph{$R$-local} identity  
\begin{equation}\label{rem:hyvarinenmorioka2:eq1}
q(a,b) = c\alpha(a)\alpha(b) \quad \text{for } \ (a,b)\in I\times\R\equiv R \qquad (\text{cf.\ \ }\text{\citep[derivation of (45)]{HYM}}).
\end{equation} 
Now since $(\neg A)$ does not imply $R\neq \mathbb{R}^2$ in general, the identity \eqref{rem:hyvarinenmorioka2:eq1} does generally \emph{not} contradict the (global) non-factorizability assumption \citep[Thm.\ 1 Assumption 3.]{HYM}, leaving $(A)$ -- thus \citep[Lemma 2]{HYM} and hence \citep[Theorem 1]{HYM} -- unproved. Clearly however, as mentioned in Remark \ref{rem:hyvarinenmorioka}, the desired contradiction can be re-obtained by allowing for only such $q$ for which the factorisation property \eqref{rem:hyvarinenmorioka2:eq1} does not hold on any open subset $R$ in $\R^2$.\\[-0.5em]            

Suitable functions $q$ as required above can be easily constructed via cut-off functions. Indeed: Assuming for simplicity that $I$ is compact, take any (continuous) $\alpha:\R\rightarrow\R_\times$ and let $\chi : \R^2\rightarrow(0,1]$ be a smooth function with $\restr{\chi}{\tilde{R}}\equiv 1$ for $\tilde{R}$ the (closure of) some bounded non-rectangular open superset of $R$. Then $q\equiv q(x,y)\coloneqq \alpha(x)\alpha(y)\cdot\chi(x,y)$ satisfies \eqref{rem:hyvarinenmorioka2:eq1} but not \citep[Def.\ 2 eq.\ (4)]{HYM}.     
       
\subsection{Proof of Proposition \ref{prop:TSCopulaProp1}}\label{pf:prop:TSCopulaProp1}
\begin{proof}[Proof of Proposition \ref{prop:TSCopulaProp1}]
We verify that $S$ satisfies the conditions of Definition \ref{def:psG_nonsep_stochproc}. Take any $i\in[d]$ and $(s,t)\in\mathcal{P}$. Since by assumption the density $\zeta^i_{s,t}$ of $(S^i_s, S^i_t)$ exists and satisfies \eqref{TimeSeriesCopulaModel:eq1}, we find 
\begin{equation}\label{prop:TSCopulaProp1:aux2}
\xi^i_{s,t}\coloneqq\partial_x\partial_y\log\zeta^i_{s,t} = \zeta^i_s\cdot\zeta^i_t\cdot\big[(\partial_x\partial_y\log c_i)\circ\phi^i_{s,t}\big] \quad \text{on } \ \{\zeta^i_{s,t}>0\}\supseteq \tilde{D}_{s,t}^{\times 2}
\end{equation} 
for $\tilde{D}_{s,t}\coloneqq\dot{D}_s\cap\dot{D}_t$ and the map $\phi^i_{s,t}\coloneqq\mathfrak{s}^i_s\times\mathfrak{s}^i_t : \tilde{D}_{s,t}^{\times 2}\rightarrow[0,1]^{\times 2}$ with $\mathfrak{s}^i_r\coloneqq F_r^{S^i}$.\footnote{\ Recall that, by convention, we write $\zeta^i_{s,t}(x)\equiv\zeta^i_{s,t}(x_i, x_{i+d})$ for $x=(x_1, \cdots, x_{2d})\in\mathbb{R}^{2d}$.} Notice that $\phi^i_{s,t}$ is a differentiable injection since the function $\mathfrak{s}^i_r=\mathfrak{s}^i_r(x)\stackrel{\mathrm{def}}{=}\int_{-\infty}^{x_i}\!\zeta_r^i(u)\,\mathrm{d}u$ ($r\in\mathbb{I}$) has positive derivative on $\dot{D}_r$. Hence and since $\restr{(\zeta^i_s\cdot\zeta^i_t)}{\tilde{D}_{s,t}^{\times 2}} > 0$ by construction, the $\alpha$-contrastivity of $S$ follows by way of \eqref{TimeSeriesCopulaModel:eq1} and \eqref{prop:TSCopulaProp1:aux2} and assumption \eqref{prop:TSCopulaProp1:aux3}. Indeed: Setting $D_{(s,t)}\coloneqq\tilde{D}_{s,t}$ for $(s,t)\in\mathcal{P}$, we see that Definition \ref{def:psG_nonsep_stochproc} \ref{def:psG_item1} holds by the assumption on $\mathcal{P}$, while Definition \ref{def:psG_nonsep_stochproc} \ref{def:psG_item2} is immediate by \eqref{TimeSeriesCopulaModel:eq1}, \eqref{prop:TSCopulaProp1:aux2} and \eqref{prop:TSCopulaProp1:aux3} and the above-noted fact that $\phi^i_{s,t}:\tilde{D}_{s,t}^{\times 2}\rightarrow\phi^i_{s,t}(\tilde{D}_{s,t}^{\times 2})$ is a diffeomorphism (indeed: note that, by virtue of Lemma \ref{lem:C2PseudoGaussian} \ref{lem:C2PseudoGaussian:it2}, the regular non-separability of the functions $\left.{\zeta^i_{s,t}}\right|_{D_{(s,t)}^{\times 2}}$ is due to \eqref{prop:TSCopulaProp1:aux2} vanishing nowhere, while in light of \eqref{TimeSeriesCopulaModel:eq1} their a.e.-non-Gaussianity follows directly from the definition \eqref{def:PseudoGaussian:eq1} and assumption \eqref{prop:TSCopulaProp1:aux3} by which the $c^i_{s,t}$ are strictly non-Gaussian).            
\end{proof} 

\subsection{Proof of Lemma \ref{prop:GPsAreVaried}}\label{pf:prop:GPsAreVaried}
\begin{proof}[Proof of Lemma \ref{prop:GPsAreVaried}] 
Let $i\in[d]$ be fixed, and $(s,t)\in\Delta_2(\I)$ be arbitrary. Since $S^i\sim\mathcal{GP}(\mu_i, \kappa^{ii})$, we find that\footnote{\ We write $Z\sim\mathcal{N}(\mu,\Sigma)$ to say that $Z$ is normally distributed with mean $\mu\in\R^d$ and covariance $\Sigma\in\R^{d\times d}$.} $(S^i_s, S^i_t)\sim\mathcal{N}(\mu^i_{s,t}, \Sigma^i_{s,t})$ with $\mu^i_{s,t}\coloneqq(\mu_i(s), \mu_i(t))$ and $\Sigma^i_{s,t}\coloneqq(\kappa^{ii}(t_\nu, t_{\tilde{\nu}}))_{\nu,\tilde{\nu}=1,2}$ for $t_1\coloneqq s$ and $t_2\coloneqq t$, so the density $\zeta^i_{s,t}$ of $(S^i_s, S^i_t)$ exists, is smooth on $\R^2$ and reads
\begin{equation}
\begin{gathered}
\zeta^i_{s,t}(x,y) = c_i(s,t)\cdot\exp\!\big(\varphi_i(s,t,x,y)\big) \qquad\text{ with }\\
\varphi_i(s,t,x,y) = \frac{\rho_{s,t}}{(1-\rho_{s,t}^2)\sigma_{s}\sigma_{t}}\cdot xy \ + \ \eta_i(s,t,x) \ + \ \tilde{\eta}_i(s,t,y) 
\end{gathered}
\end{equation}
for $c_i(s,t) = (4\pi^2\sigma_s^2 \sigma_t^2(1-\rho_{s,t}^2))^{-1/2}$ and $-\tfrac{1}{2}((x,y)-\mu^i_{s,t})^\intercal\cdot[\Sigma^i_{s,t}]^{-1}\cdot((x,y)-\mu^i_{s,t}) - \frac{\rho_{s,t}}{(1-\rho_{s,t}^2)\sigma_{s}\sigma_{t}}\cdot xy\eqqcolon\eta_i(s,t,x) + \tilde{\eta}_i(s,t,y)$ and with the (auto-)correlations
\begin{equation}
\sigma_r\coloneqq\sqrt{\kappa^{ii}(r,r)} \qquad\text{ and }\qquad \rho_{s,t}\coloneqq \frac{\kappa^{ii}(s,t)}{\sqrt{\kappa^{ii}(s,s)\kappa^{ii}(t,t)}} = \frac{\kappa^i_{s,t}}{\sqrt{k^i_{s,t}}}.  
\end{equation}
Consequently, the mixed log-derivatives $\xi^i_{s,t}\coloneqq\partial_x\partial_y\log(\zeta^i_{s,t})$ are given as
\begin{equation}\label{prop:GPsAreVaried:eq3}
\xi^i_{s,t}(x,y) = \partial_x\partial_y\varphi_i(s,t,x,y) = \frac{\kappa^i_{s,t}}{k^i_{s,t} - (\kappa^i_{s,t})^2}. 
\end{equation} 
Since by Definition \ref{def:log-regular} the process $S$ is $\gamma$-contrastive iff there are $\p_0, \p_1, \p_2\in\Delta_2(\mathbb{I})$ with $\left(\psi(\xi^i_{\p_0}, \xi^i_{\p_1}, \xi^i_{\p_2})\right)_{i\in[d]}= \left(\frac{\xi^i_{\p_1}\xi^i_{\p_2}}{(\xi^i_{\p_0})^2}\right)_{i\in[d]}\in(\mathbb{R}^d\setminus\nabla^{\times})$, the lemma now follows from \eqref{prop:GPsAreVaried:eq3}.        
\end{proof} 

\subsection{Proof of Proposition \ref{cor1:GPsAreVaried}}\label{pf:cor1:GPsAreVaried}
\begin{proof}[Proof of Proposition \ref{cor1:GPsAreVaried}] 
We apply Lemma \ref{prop:GPsAreVaried} by showing that for each case there are $\p_0, \p_1, \p_2\in\Delta_2(\mathbb{I})$ for which \eqref{prop:GPsAreVaried:eq1} holds. Write $\Xi_i$ for the $i^{\mathrm{th}}$ component of \eqref{prop:GPsAreVaried:eq1} and let $|(s,t)|\coloneqq |t - s|$.  

\ref{cor1:GPsAreVaried:item1}\,: Fix any $\p_0, \p_1\in\Delta_2(\mathbb{I})$ with $|\p_0|\neq|\p_1|$ and take $\p_2 \coloneqq \p_0$. Then for each $i\in[d]$ we have $\Xi_i = \tilde{\Xi}_i^{(1)}\!/\,\tilde{\Xi}_i^{(0)}$ for the factors 
\begin{equation}\label{cor1:GPsAreVaried:aux1}
\tilde{\Xi_i}^{(\nu)} = \frac{\kappa^i_{\p_\nu}}{1 - (\kappa^i_{\p_\nu})^2} = \left([\kappa^i_{\p_\nu}]^{-1} - \kappa^i_{\p_\nu}\right)^{-1} = \tfrac{1}{2}\left(\sinh\!\left(\left[\frac{|\p_\nu|}{\alpha_i}\right]^{\gamma_i}\right)\right)^{-1}.
\end{equation}
Hence we have the parametrisation $\alpha_i\mapsto\Xi_i\equiv\Xi_i(\alpha_i)$ given by 
\begin{equation}
\Xi_i = \sinh\!\left(\left[\frac{|\p_0|}{\alpha_i}\right]^{\gamma_i}\right)\cdot\left[\sinh\!\left(\left[\frac{|\p_1|}{\alpha_i}\right]^{\gamma_i}\right)\right]^{-1},
\end{equation}
which for $|\p_1|\neq|\p_0|$ and $\gamma_i$ fixed is differentiable and strictly monotone in $\alpha_i>0$. Denoting by $\phi_i$ the associated (differentiable) inverse of $\alpha_i\mapsto\Xi_i(\alpha_i)$ (for $\gamma$ and $\p_0, \p_1$ fixed), we find that $(\Xi_i(\alpha_i))_{i\in[d]}\in(\mathbb{R}^d\setminus\nabla^{\times})$ for any $(\alpha_i)_{i\in[d]}$ not contained in $\mathcal{N}_{\gamma}\coloneqq (\phi_1\times\cdots\times\phi_d)(\nabla^{\times})$.
 
\ref{cor1:GPsAreVaried:item2}\,: It is well-known that the covariance function $\kappa^i$ of \eqref{cor1:GPsAreVaried:eq2} reads 
\begin{equation}\label{cor1:GPsAreVaried:aux2.1}
\kappa^i(s,t) = \gamma_i\big(e^{-\theta_i|s-t|}- e^{-\theta_i(s+t)}\big)\quad\text{for}\quad\gamma_i\coloneqq\frac{\sigma_i^2}{2\theta_i}. 
\end{equation}   
Suppose $\mathbb{I}=[0,1]$ wlog. Then for $\p_\nu\equiv(t_\nu, 2 t_\nu)\in\Delta_2(\mathbb{I})$ ($\nu=0,1$) with $t_0\neq t_1$ and $\p_2\coloneqq\p_0$, we obtain $\kappa^i_{\p_\nu}=\gamma_i(e^{-\theta_i t_\nu} - e^{-3\theta_i t_\nu})$ and $k^i_{\p_\nu}=\kappa^i_{\p_\nu}\cdot \gamma_i e^{\theta_i t_\nu}\cdot(1 - e^{-4\theta_i t_\nu})$ and, thus, $\Xi_i = \tilde{\Xi}_i^{(1)}\!/\,\tilde{\Xi}_i^{(0)}$ for each $i\in[d]$, with the factors 
\begin{equation}\label{cor1:GPsAreVaried:aux2.2}
\tilde{\Xi}_i^{(\nu)} = \frac{\kappa^i_{\p_\nu}}{k^i_{\p_\nu} - (\kappa^i_{\p_\nu})^2} = \left(\gamma_i e^{\theta_i t_\nu}\cdot(1 - e^{-4\theta_i t_\nu}) - \kappa^{i}_{\p_\nu}\right)^{-1} = \left(\frac{\sigma_i^2}{\theta_i}\sinh(\theta_i t_\nu)\right)^{-1}.
\end{equation}   
We hence have the parametrisation $\theta_i\mapsto\Xi_i\equiv\Xi_i(\theta_i) = \frac{\sinh(\theta_i t_0)}{\sinh(\theta_i t_1)}$, which due to $t_0\neq t_1$ is strictly monotone and differentiable in $\theta_i>0$. Denoting by $\tilde{\phi}_i$ the associated (differentiable) inverse of $\theta_i\mapsto\Xi_i(\theta_i)$, we obtain that $(\Xi_i(\theta_i))_{i\in[d]}\in(\mathbb{R}^d\setminus\nabla^{\times})$ provided the parameter vector $(\theta_i)_{i\in[d]}\in\mathbb{R}_{>0}^d$ is not contained in the nullset $\tilde{\mathcal{N}}\coloneqq(\tilde{\phi}_1\times\cdots\times\tilde{\phi}_d)(\nabla^{\times})$.       

\ref{cor1:GPsAreVaried:item3}\,: Choosing again $\p_\nu\equiv(t_\nu, 2 t_\nu)\in\Delta_2(\mathbb{I})$ ($\nu=0,1$) with $t_0\neq t_1$ and $\p_2\coloneqq\p_0$, we for each $i\in[d]$ find that $\Xi_i = \tilde{\Xi}_i^{(1)}\!/\,\Xi_i^{(0)}$ for the factors 
\begin{equation}\label{cor1:GPsAreVaried:aux2.3}
\tilde{\Xi}_i^{(\nu)} = \frac{4^{H_i-\tfrac{1}{2}}\cdot t_\nu^{2H_i}}{4^{H_i}\cdot t_\nu^{4H_i} - 4^{2H_i - 1}\cdot t_\nu^{4H_i}} = \left(2(1-4^{H_i-1})\cdot t_\nu^{2H_i}\right)^{-1},
\end{equation}
whence it holds that
\begin{equation}
\Xi_i = \left(\frac{t_0}{t_1}\right)^{\!2H_i}\quad\text{ for each } \ i\in[d].
\end{equation}
But since due to $t_0\neq t_1$ the assignment $h\mapsto (\frac{t_0}{t_1})^{2h}$ is clearly injective, we clearly obtain that $(\Xi_i)_{\in[d]}$ is not in $\nabla^{\times}$ whenever $(H_i)_{i\in[d]}$ is not in $\nabla^{\times}$, as claimed.

\ref{cor1:GPsAreVaried:item4}\,: As the numbers $\theta_{ij}\coloneqq\eta_i(r_0)\cdot\eta_j(r_1)$, $(i,j)\in[d]\times[d]$, are pairwise distinct and (thus) non-zero, the continuity of the functions $\vartheta_{ij} : \mathbb{I}^{\times 2}\ni (s,t)\mapsto \eta_i(s)\cdot\eta_j(t)$ allows us to find pairs $(s_0, t_0), (s_1, t_1)\in\Delta_2(\mathbb{I})$ such that for the rectangle $R\coloneqq[s_0, t_0]\times[s_1, t_1]\subseteq\mathbb{I}^{\times 2}$ the associated integrals
\begin{equation}\label{cor1:GPsAreVaried:aux3}
\int_{R}\!\vartheta_{ij}\,\mathrm{d}s\,\mathrm{d}t, \quad (i,j)\in[d]^{\times 2}, \quad \text{ are pairwise distinct and non-zero}. 
\end{equation}  
Clearly then, \eqref{cor1:GPsAreVaried:aux3} implies that for $\iota_{i}(s,t)\coloneqq\int_s^t\!\eta_i(r)\,\mathrm{d}r$, the numbers  
\begin{equation}\label{cor1:GPsAreVaried:aux4}
\iota_i(s_0, t_0)\cdot\iota_j(s_1, t_1), \ (i,j)\in[d]\times[d], \ \text{ are pairwise dinstinct}
\end{equation} 
(Note further that by \eqref{cor1:GPsAreVaried:aux3}, $s_0$ and $s_1$ may be chosen such that in addition to \eqref{cor1:GPsAreVaried:aux4} it holds $\iota_i(0, s_\nu)\neq 0$ for each $i\in[d]$.)
Now by setting $\p_2\coloneqq\p_0$ with $\p_\nu\coloneqq(s_\nu, t_\nu)$ for $\nu=0,1$ (notice that $\p_0\neq\p_1$ by \eqref{cor1:GPsAreVaried:aux4}), we once more find that $\Xi_i=\tilde{\Xi}_i^{(1)}\!/\,\Xi_i^{(0)}$ for each $i\in[d]$, this time for the factors
\begin{equation}
\tilde{\Xi}_i^{(\nu)} = \frac{\kappa^i_{\p_\nu}}{k^i_{\p_\nu} - (\kappa^i_{\p_\nu})^2} = \frac{\iota_i(0, s_\nu)}{\iota_i(0,s_\nu)\cdot\iota_i(0,t_\nu) - \iota_i(0, s_\nu)^2} = \left(\iota_i(s_\nu, t_\nu)\right)^{-1}.
\end{equation}   
Consequently, the entries of $(\Xi_i)_{i\in[d]} = \left(\frac{\iota_i(s_0, t_0)}{\iota_i(s_1,t_1)}\right)_{i\in[d]}$ are pairwise distinct. Indeed, assuming otherwise that $\Xi_i=\Xi_j$ for some $i\neq j$, we find that
\begin{equation}
\iota_i(s_0, t_0)\cdot\iota_j(s_1, t_1) = \iota_j(s_0, t_0)\cdot\iota_i(s_1, t_1), \qquad\text{contradicting \eqref{cor1:GPsAreVaried:aux4}}. 
\end{equation} 
Hence $(\Xi_i)_{i\in[d]}\in(\mathbb{R}^d\setminus\nabla^{\times})$ as desired.  
\end{proof}
\subsection{Proof of Proposition \ref{prop:SDE_GBM}}\label{pf:prop:SDE_GBM}
\propGBM*
\begin{proof}
A straightforward application of Itô's lemma yields that for any $(s,t)\in\Delta_2(\mathbb{I})$, the density $\zeta^i_{s,t}$ of $(S^i_s, S^i_t)$  is given by 
\begin{equation}
\zeta_{s,t}^i(x,y) = \rho_{s,t}^i(\log(x),\log(y))\cdot (xy)^{-1} = c_i(s,t,x,y)\cdot\exp\!\big(\varphi_i(s,t,x,y)\big)
\end{equation}     
for the functions $c_i:\Delta_2(\mathbb{I})\times\mathbb{R}_{>0}^2\rightarrow\mathbb{R}$ and $\varphi_i : \Delta_2(\mathbb{I})\times\mathbb{R}_{>0}^2\rightarrow\mathbb{R}$ defined by  
\begin{equation}\label{prop:SDE_GBM:aux6}
\begin{aligned}
c_i(s,t,x,y)&=\frac{(4\pi^2\cdot\det(\mathfrak{s}^i_{s,t}))^{-1/2}}{xy} \qquad\qquad\text{ and}\\
\varphi_i(s,t,x,y)&=-\frac{1}{2}\left([\phi^{-1}(x,y) - \mathfrak{m}^i_{s,t}]^\intercal\cdot(\mathfrak{s}_{s,t}^i)^{-1}\cdot[\phi^{-1}(x,y) - \mathfrak{m}^i_{s,t}]\right)\\
&= \beta^i_{s,t}\log(x)\log(y) + \eta_i(s,t,x) + \tilde{\eta}_i(s,t,y) 
\end{aligned} 
\end{equation} 
with $\eta_i, \tilde{\eta}_i$ given by $\eta_i(s,t,x) + \tilde{\eta}_i(s,t,y)\coloneqq \varphi_i(s,t,x,y) - \beta^i_{s,t}$ and 
\begin{equation}\label{prop:SDE_GBM:aux7}
\beta^i_{s,t}\coloneqq\frac{\kappa_i(s,t)}{\kappa_i(s,s)\kappa_i(t,t) - \kappa_i^2(s,t)} = \frac{\int_0^s\!\sigma_i^2(r)\,\mathrm{d}r}{\left(\int_0^s\!\sigma_i^2(r)\,\mathrm{d}r\right)\big(\int_0^t\!\sigma_i^2(r)\,\mathrm{d}r\big) - (\int_0^s\!\sigma_i^2(r)\,\mathrm{d}r)^2}.
\end{equation}
Consequently, the spatial support of $S$ is $D_S = \mathbb{R}_+^d = \pi_{[d]}\big(\supp\big[\mathbb{R}^{2d}\ni(u,v)\mapsto \zeta^i_{s,t}(u_i, v_i)\big]\big)$ (any $i\in[d]$), and the mixed log-derivatives of $\zeta^i_{s,t}$ read 
\begin{equation}\label{prop:SDE_GBM:aux8}
\xi^i_{s,t}\coloneqq \partial_x\partial_y\log(\zeta^i_{s,t}) = \partial_x\partial_y\varphi^i(s,t,\cdot,\cdot) = \frac{\beta^i_{s,t}}{xy}.
\end{equation} 
Hence by Definition \ref{def:log-regular}, the process $S$ is $\gamma$-contrastive iff 
\begin{equation}\label{prop:SDE_GBM:aux9}
\exists\, \p_0, \p_1, \p_2\in\Delta_2(\mathbb{I})\quad\text{ with }\quad (\Xi_i)_{i\in[d]}\coloneqq\left(\frac{\beta^i_{\p_1}\beta^i_{\p_2}}{(\beta^i_{\p_0})^2}\right)_{\!i\in[d]}\!\in\,(\mathbb{R}^d\setminus\nabla^{\times}). 
\end{equation} 
Having $\p_\nu\equiv(s_\nu, t_\nu) \in\Delta_2(\mathbb{I})$ $(\nu=0,1)$ arbitrary and $\p_2\coloneqq\p_0$ hence yields that
\begin{equation}
\Xi_i = \frac{\beta^i_{\p_1}}{\beta^i_{\p_0}} = \frac{\int_{s_0}^{t_0}\!\sigma_i^2(r)\,\mathrm{d}r}{\int_{s_1}^{t_1}\!\sigma_i^2(r)\,\mathrm{d}r} \qquad\text{ for each } \ i\in[d]. 
\end{equation}
Thus by choosing $\p_0$ and $\p_1$ as done in the proof of Proposition \ref{cor1:GPsAreVaried} (iv), we obtain $(\Xi_i)_{i\in[d]}\in(\mathbb{R}^d\setminus\nabla^{\times})$ as desired. 
\end{proof} 

\subsection{On Signature Cumulants}\label{rem:sigcumulants}This remark is to further illuminate the concept of the signature cumulant from Definition \ref{def:sigcumulant} as essentially that of a natural (`moment-like') multi-indexed $\log$-compressed coordinate vector for the law of a stochastic process.\\[-0.5em]

\noindent    
As further detailed in Section \ref{sect:expected_signature_moments}, the idea behind the classical concept of `moments' of a random variable -- both for random vectors in $\R^d$ and stochastic processes alike -- is that they provide a set of deterministic \emph{coordinates} for the law of that random variable. 

For the simplest case of a scalar random variable, i.e.\ a random vector in $\R^1$, the most established such coordinates are the sequence of its central moments, that is an ordered list of expectations over increasingly nonlinear functionals [specifically: (centered) monomials] of that random variable. Analogous coordinates for a random vector in $\R^d$, namely its multivariate (central) moments, are obtained if one considers the expectation over the multivariate (centered) monomials of that random vector.\footnote{\ The only difficulty with this multidimensional generalisation is that the resulting `coordinate vector' is no longer a linearly ordered list (isomorphic to (the coefficients of) a formal power series in a single variable) but rather a multiindexed family of numbers (isomorphic to a formal power series in several variables); cf.\ Sect.\ \ref{rem:sig_cumulants_generalise}.}

The expected signature is yet a further generalisation of this classical concept of coordinatisation, this time from random elements in $\R^d$ to random elements in $\mathcal{C}_d$, i.e.\ continuous-time stochastic processes: As before one considers a multiindexed family of numbers \eqref{def:expected_signature:eq1} given as expectations over increasingly nonlinear functionals of the process, but this time the functionals in question are no longer (multivariate) monomials on $\R^d$ but rather iterated integrals on (elements of) $\mathcal{C}_d$. Owing to the more complex (time-ordered) global structure that is innate to the realisations of a stochastic process (its sample paths), these path-space functionals are better suited [than classical multivariate monomials] to capture the (geometrical) complexity of the open sets (wrt.\ the uniform topology) of $\mathcal{C}_d$ whose numerical valuation constitutes the law of the process. The family of expectations \eqref{def:expected_signature:eq1} over these functionals of the process, also referred to as \emph{noncommutative moments} due to their non-invariance under index-permutations, does hence form a more informative global statistic of the stochastic process than can be provided by the classical moments of the process at any fixed time-point of its evolution. In fact, these expected signature coefficients do provide a high-resolution description of the stochastic process which is so fine-grained that in total they are able to characterise the law \cite{CHO} of the process. Another of their main advantages, leveraged below, is that these coefficients are interrelated by way of a rich algebraic and combinatorial global structure, which (akin to classical moments for the case of random vectors) makes core aspects of stochastic process statistics amenable to a lucid algebraic description. 

These relations reveal, however, that considered in isolation the family of coefficients \eqref{def:expected_signature:eq1}, that is the expected signature of the stochastic process, appears somewhat `bloated' in that it exhibits a certain level of internal algebraic redundancy. In fact, see Remark \ref{sect:sigmoments:logtransform} for details, it turns out that the expected signature is `close to an exponential', which allows for the information it contains to be efficiently compressed by a logarithmic `change of coordinates'. What results is the multi-indexed coordinate vector \eqref{def:sigcumulant:eq1}, an algebraically accessible reservoir of conveniently organised statistical information that characterises the law of a stochastic process.                               

\subsection{Basic Cross-Shuffle Combinatorics}\label{rem:cross_shuffles}
Using the notation of Sections \ref{sect:independence_criterion}, \ref{sect:capping} and \ref{rem:sig_cumulants_generalise} throughout this remark, consider the family of \emph{cross-shuffles} $\mathfrak{C}\equiv\bigsqcup_{k=2}^d\mathcal{W}_k = \bigsqcup_{\nu=2}^\infty\mathfrak{C}_\nu$.\\[-0.5em]

By definition \eqref{shuffle_product} of the shuffle product, each element $\bm{q}\in\mathfrak{C}_\nu$ (seen as an element of \eqref{rem:free_algebra}) is a homogeneous polynomial of degree $\nu$ whose monomial coefficients are all $1$, i.e.\ there is $c_{\bm{q}}\in\N$  such that $\bm{q} = \bm{q}_1 + \cdots + \bm{q}_{c_{\bm{q}}}$ with $\bm{q}_j\in[d]^\ast$ for each $j\in[c_{\bm{q}}]$. Partitioning
\begin{equation}
\mathfrak{C}_\nu = \bigsqcup_{k=2}^d\mathfrak{C}_{\nu|k} \quad \text{ for } \quad \mathfrak{C}_{\nu|k}\coloneqq \mathfrak{C}_\nu\cap\mathcal{W}_k,  
\end{equation}          
the definition of $\mathcal{W}_k$ yields that for any given $\bm{q}\in\mathfrak{C}_\nu$ the pair $\vartheta_{\bm{q}}\equiv(k_{\bm{q}},\mu_{\bm{q}})\in[d]_{\geq 2}\times[m-1]$, with $k_{\bm{q}}\coloneqq\max\{i\in[d]\mid i\in\bm{q}_1\}$ being the largest letter contained in $\bm{q}$ and $\mu_{\bm{q}}\coloneqq\sum_{i\in\bm{q}_1}\delta_{i,k_{\bm{q}}}$ denoting the number of times this largest letter appears in one (and hence any) of the monomials of $\bm{q}$, determines $\bm{q}$ uniquely (in $\mathfrak{C}_\nu$) up to a $\shuffle$-left-factor of word length $\nu-\mu_{\bm{q}}$. (Indeed: Given $\bm{q}\in\mathfrak{C}_\nu$, the number $k_{\bm{q}}\in\N$ is the (unique) index s.t.\ $\bm{q}\in\mathfrak{C}_{\nu|k_{\bm{q}}}\subset\mathcal{W}_{k_{\bm{q}}}$, whence $\bm{q} = \bm{w}\shuffle (k_{\bm{q}})^{\ast\mu_{\bm{q}}}$ for some $\bm{w}\in[k_{\bm{q}}-1]^\ast$ with $|\bm{w}| = \nu-\mu_{\bm{q}}$.)\\[-0.5em] 

Since the shuffle product \eqref{shuffle_product} of two words $\bm{i}, \bm{j}\in[d]^\ast$ is precisely the sum over the $c_{\bm{i}\shuffle\bm{j}}$ $\big(=\frac{(|\bm{i}|+|\bm{j}|)!}{|\bm{i}|!|\bm{j}|!}\big)$ ways of \emph{interleaving} $\bm{i}$ and $\bm{j}$, any two monomials in $\bm{i}\shuffle\bm{j}$ are composed of exactly the same letters and differ only in the order in which their letters appear. Consequently,
\begin{enumerate}[label=(\alph*)]
\item\label{rem:cross_shuffles:it1} any two $\bm{q}, \bm{q}'\in\mathfrak{C}_\nu$ have a monomial in common iff $\bm{q}=\bm{q}'$\,;
\item\label{rem:cross_shuffles:it2} given any $\bm{q}\in\mathfrak{C}_\nu$ with its unique (up to the order of summands) decomposition $\bm{q}=\bm{q}_1+\ldots+\bm{q}_{c_{\bm{q}}}$ into monic monomials $\bm{q}_1, \ldots, \bm{q}_{c_{\bm{q}}}\in[d]^\ast$, these monomials $\bm{q}_1, \ldots, \bm{q}_{c_{\bm{q}}}$ are pairwise distinct.      
\end{enumerate} 
(Note that point \ref{rem:cross_shuffles:it2} follows inductively: Let $\bm{q}\equiv\bm{w}\shuffle(k_{\bm{q}})^{\ast\mu_{\bm{q}}}\in\mathfrak{C}_\nu$ with $(k_{\bm{q}},\mu_{\bm{q}})=\vartheta_{\bm{q}}$. The assertion clearly holds if $\mu_{\bm{q}}=1$ or (by symmetry) $|\bm{w}|=1$. Fixing any $\bm{q}$ as above, assume that \ref{rem:cross_shuffles:it2} holds for any $\bm{r}\equiv\bm{w}'\shuffle(k_{\bm{q}})^{\ast\mu_{\bm{r}}}\in\mathfrak{C}$ with $\mu_{\bm{r}}=\mu_{\bm{q}}-1$ or $|\bm{w}'|=|\bm{w}|-1$. Note that 
\begin{equation}
\bm{q} = (\bm{w}'\ast\texttt{i})\shuffle(k_{\bm{q}})^{\ast(\mu_{\bm{q}}'+1)} = \bm{r}_0\ast \texttt{i} + \bm{r}_1\ast k_{\bm{q}} \qquad\quad (\bm{w}\eqqcolon\bm{w}'\ast\texttt{i}, \ \texttt{i}\in[k_{\bm{q}}-1]\setminus\{\epsilon\})
\end{equation}  
for the polynomials $\bm{r}_0\coloneqq\bm{w}'\shuffle(k_{\bm{q}})^{\ast\mu_{\bm{q}}}$ and $\bm{r}_1\coloneqq\bm{w}\shuffle(k_{\bm{q}})^{\ast(\mu_{\bm{q}}-1)}$ (by the recursive formulation of the shuffle product, e.g.\ \citep[p.\ 25\,f.]{REU}). Now since the monomials of $\bm{r}_0$ and $\bm{r}_1$ are all monic and pairwise distinct by induction hypothesis, it is clear that the same applies to $\bm{r}_0\ast\texttt{i}$ and $\bm{r}_1\ast k_{\bm{q}}$ and, hence (as $\texttt{i}\neq k_{\bm{q}}$), to $\bm{q}$. Thus by induction, assertion \ref{rem:cross_shuffles:it2} holds for $\bm{q}$ as desired.) 

\subsection{Nonlinear ICA for Discrete-Time Signals}\label{appendix:NICA_discrete}
\noindent
As detailed in this section, our approach towards the identifiability of nonlinearly mixed stochastic processes also covers the case of discrete-time signals with almost no further modifications.\\[-0.5em]

\noindent
Assume throughout that $X_\ast\equiv (X_j)_{j\in\mathbb{Z}}$ is some discrete time-series in $\R^d$ such that
\begin{equation}\label{appendix:NICA_discrete:eq1}
X_\ast \ = \ f(S_\ast) \,\equiv\, (f(S_j))_{j\in\Z} 
\end{equation}   
for some IC discrete time-series $S_\ast\equiv(S_j)_{j\in\Z}$ in $\R^d$ and $f\in C^{2,2}(D_{S_\ast};\R^d)$. Here, a time series $Y_\ast\equiv (Y_j)_{j\in\Z}$ in $\R^d$, with $Y_j\equiv(Y_j^1, \ldots, Y_j^d)$ for each $j\in\Z$, is called \emph{IC} if its componental time-series $(Y_j^1)_{j\in\Z}, \ldots, (Y_j^d)_{j\in\Z}$ are mutually independent.\\[-0.5em]

Denote further $D_{Y_\ast}\coloneqq\overline{\bigcup_{j\in\Z}\supp(Y_j)}^{|\cdot|_2}$ for the \emph{spatial support} of $Y_\ast$, and write $\Delta_2(\Z)\coloneqq\{(j_1,j_2)\in\Z^2\mid j_1 < j_2\,\}$ for the set of all strictly ordered pairs of integers.  
\begin{appendixdef}[$\bar{\alpha}, \bar{\beta}, \bar{\gamma}$-Contrastive]\label{def:alpha_bar_contrastive}
A discrete time-series $S_\ast\equiv(S_j)_{j\in\Z}$ in $\R^d$ with spatial support $D_{S_\ast}$ will be called \emph{$\bar{\alpha}$-contrastive} if $S_\ast$ is IC and there exists $\mathcal{P}\subseteq\Delta_2(\Z)$ together with a collection $(D_\p)_{\p\in\mathcal{P}}$ of open subsets in $\R^d$ such that    
\begin{enumerate}[label=(\roman*)]
\item\label{def:alpha_bar_contrastive:it1} the union $\bigcup_{\p\in\mathcal{P}}D_\p$ is dense in $D_{S_\ast}$, and 
\item\label{def:alpha_bar_contrastive:it2} for each $(i, (j_1, j_2))\in[d]\times\mathcal{P}$, the vector $(S^i_{j_1}, S^i_{j_2})$ is $C^2$-distributed with density $\zeta^i_{j_1,j_2}$ such that 
\begin{equation*}
\begin{aligned} 
&\restr{\zeta^i_{j_1, j_2}}{D_{(j_1, j_2)}^{\times 2}} \text{ is regularly non-separable for all $i\in[d]$, \quad and}\\[-0.5em]
&\restr{\zeta^i_{j_1, j_2}}{D_{(j_1, j_2)}^{\times 2}} \text{ is almost everywhere non-Gaussian for all but at most one $i\in[d]$} 
\end{aligned} 
\end{equation*}   
\end{enumerate}(cf.\ Definition \ref{def:psG_nonsep_stochproc}). The notions of \emph{$\bar{\beta}$- and $\bar{\gamma}$-contrastive} time series are defined in analogous adaptation of Definition \ref{def:log-regular}.         
\end{appendixdef}            

Analogous to before (see Assumption \ref{sect:assumption_convexity}), for the rest of Section \ref{appendix:NICA_discrete} we adopt the convenience assumption that each connected component of $D_{S_\ast}$ be convex. 
\begin{appendixthm}\label{appendix:NICA_discrete:thm}
For $X_\ast$ and $S_\ast$ as in \eqref{appendix:NICA_discrete:eq1} with spatial supports $D_{X_\ast}$ and $D_{S_\ast}$ respectively, let the time-series $S_\ast$ in  be $\bar{\alpha}$-, $\bar{\beta}$- or $\bar{\gamma}$-contrastive. Then, for any transformation $h$ which is $C^{2}$-invertible on some open superset of $D_{X_\ast}$, we have with probability one that: 
\begin{equation}\label{appendix:NICA_discrete:thm:eq1}
\restr{(h\circ f)}{\tilde{Z}} \ \in \ \mathrm{DP}_{\!d}(\tilde{Z}), \ \forall\,Z\subseteq D_{X_\ast} \text{ connected} \qquad\text{if and only if}\qquad h(X_\ast) \ \text{ is IC},
\end{equation}
where for any connected subset $Z$ of $D_{X_\ast}$ we denoted $\tilde{Z}\coloneqq f^{-1}(Z)$.     
\end{appendixthm}   
\begin{proof}
Let $S_\ast$ be $\bar{\alpha}$-contrastive, and $h\in C^{2,2}(D_{X_\ast};\R^d)$ be such that $h(X_\ast)$ is IC. Then
\begin{equation}
(h\times h)(X_{j_1}, X_{j_2}) \ \text{ is \ IC} \quad \text{for any fixed } \ (j_1, j_2)\in\mathcal{P},   
\end{equation} 
which in consequence of Definition \ref{def:alpha_bar_contrastive} \ref{def:alpha_bar_contrastive:it2} implies that 
\begin{equation}
\text{the Jacobian of } \ \ \varrho\,\coloneqq\, h\circ f \ \ \text{ is monomial \ on } D_{(j_1, j_2)},
\end{equation}
as detailed in the proof of Theorem \ref{thm:NICA_stat}. The equivalence \eqref{appendix:NICA_discrete:thm:eq1} thus follows from Def.\ \ref{def:alpha_bar_contrastive} \ref{def:alpha_bar_contrastive:it1} and Lemma \ref{lem:monomial_trafos}. The case of $S_\ast$ being $\bar{\beta}$- or $\bar{\gamma}$-contrastive follows similarly via Thm.\ \ref{cor:NICA_MainCor}.           
\end{proof} 
\noindent
Since a discrete time-series $Y_\ast$ in $\R^d$ is IC if and only if its piecewise-linear interpolation\footnote{\ $\ldots$ along any (countable) dissection of, say, $[0,1]$.} $\hat{Y}_\ast$ is IC in $\mathcal{C}_d$, the assertion of Theorem \ref{thm:optimisation} remains valid as stated\footnote{\ With the addition that in \eqref{thm:optimisation:eq1}, the monomial transformations $\alpha$ with $h(X) = \alpha(S)$ then depend on $(j,\omega)$ via the connected component of $D_{S_\ast}$ that the given realisation of $S_j = S_j(\omega)$ is contained in (details below).} if $(S, \alpha, \beta, \gamma, X)$ is replaced by $(S_\ast, \bar{\alpha}, \bar{\beta}, \bar{\gamma}, X_\ast)$ and the argument $h(X)$ in \eqref{thm:optimisation:eq1} is replaced by the piecewise-linear interpolation of $h(X_\ast)$. This shows that the identifiability theory of Sections \ref{sect:core_theory}, \ref{sect:beta_and_gamma} and \ref{sect:independence_criterion} directly applies to the discrete-time setting \eqref{appendix:NICA_discrete:eq1}, as detailed in the following remarks.  

\subsubsection{Identifiability in the Discrete-Time Case}\label{rem:discrete}
Following Lemma \ref{lem:monomial_trafos} and Theorems \ref{thm:NICA_stat} and \ref{cor:NICA_MainCor}, we have seen that the inversion (up to monomial ambiguity) of the mixing transformation $f : D_S\rightarrow D_X$ given (nothing but) $X$ is possible if there is a dense open subset $\mathcal{D}$ in $D_S$ which is `identifiability enforcing' in the sense that     
\begin{equation}\label{rem:discrete:eq:3}
J_{h\circ f}(u) \in \mathrm{M}_d \quad \text{for each } u\in\mathcal{D} 
\end{equation}                      
for any $h\in C^{2,2}(D_X)$ such that $h(X)$ is IC; see e.g.\ the proof of Theorem \ref{thm:NICA_stat}. As shown in the proofs of these theorems, such a set $\mathcal{D}$ can be induced by a subset of $C^2$-regular
\begin{equation}\label{rem:discrete:eq:2}
\text{distributions of}\quad \big\{(S_s, S_t) \ \big| \ (s,t)\in\Delta_2(\I)\big\} 
\end{equation} 
if the source process $S$ is $\alpha$-, $\beta$- or $\gamma$-contrastive in the sense of Defs.\ \ref{def:psG_nonsep_stochproc} and \ref{def:log-regular}. (The associated examples for $\mathcal{D}$ in these cases are $\bigcup_{\mathfrak{p}\in\mathcal{P}}D_{\mathfrak{p}}$ (Def.\ \ref{def:psG_nonsep_stochproc}), and $\mathrm{int}(D_S)$ or $\mathcal{U}$ (Def.\ \ref{def:log-regular}).) Now importantly, this approach towards identifiability makes no essential use of $S$ being time-continuous: Both Definitions \ref{def:psG_nonsep_stochproc}, \ref{def:log-regular} \emph{and} their consequential derivations of \eqref{rem:discrete:eq:3} only use that
\begin{equation}\label{rem:discrete:eq:2.1}
\begin{gathered}
S = (S_t)_{t\in\I} \quad \text{is a family of random variables  $S_t\equiv(S^i_t)$ in $\R^d$}\\[-0.25em]
\text{with} \quad \I \text{ \ a totally ordered subset of \,} \R \  \text{ and such that}\quad\\[-0.5em]
\text{the families} \quad (S^1_t)_{t\in\I}, \ldots, (S^d_t)_{t\in\I} \quad \text{are statistically independent},  
\end{gathered}
\end{equation} 
as a quick inspection of the proofs of Theorems \ref{thm:NICA_stat}, \ref{cor:NICA_MainCor} shows. For the special case where $\I$ is countable -- i.e.\ $S=(S_t)_{t\in\I}\equiv(S_j)_{j\in\mathbb{J}}$ $(\mathbb{J}\subseteq\mathbb{Z}$) is a discrete time-series -- we may thus `abstract \eqref{2ndfdds} from the topology on $\I$' without losing the power of our approach, as emphasised below. This means that we may simply consider the discrete `lattice' of joint laws given by 
\begin{equation}\label{rem:discrete:eq:2'}\tag{\ref*{rem:discrete:eq:2}'}
\text{the distributions of} \quad \big\{(S_{j_1}, S_{j_2}) \ \big| \ (j_1,j_2)\in \Delta_2(\Z)\big\}
\end{equation} 
in lieu of the uncountable collection \eqref{rem:discrete:eq:2}.  
In this case then still, regular dense subsets $\mathcal{D}$ of $D_S$ with the desired property \eqref{rem:discrete:eq:3} can be induced from a $C^2$-distributed subselection\footnote{\ That is, a set of (the distributions of) vectors $\{(S_{k}, S_{\ell})\mid (k,\ell)\in\mathcal{P}\}$, for some $\mathcal{P}\subseteq\Delta_2(\Z)$, such that $(S_k^i, S_\ell^i)$ is $C^2$-distributed (cf.\ Definition \ref{def:density_reg}) for each $(k,\ell)\in\mathcal{P}$ and each $i\in[d]$. (As the components of $(S_j)$ are mutually independent, $(S_k^i, S_\ell^i)$ is $C^2$-distributed \emph{for each} $i\in[d]$ iff the full vector $(S_k, S_\ell)$ is $C^2$-distributed.)} of \eqref{rem:discrete:eq:2'} \emph{by the exact same argumentation} that we have used in the proofs of Theorems \ref{thm:NICA_stat} and \ref{cor:NICA_MainCor}, as it is evident that these proofs do \emph{not} involve the topology on $\I$ but only its order. 

The associated (by direct analogy to the continuous-time case) premises for the existence of such a subselection of \eqref{rem:discrete:eq:2'} are formulated as Definition \ref{def:alpha_bar_contrastive}, which thus appears as the natural `discretization' of Definitions 6 and 7. 

There are three points in the paper where the assumption of continuous-time (specifically: the time-continuity of the samples of $S$) does make a subtle but not entirely trivial difference: 
\begin{enumerate}[label=(\roman*)]
\item\label{rem:discrete:it1}From the `pre-identifiability' property \eqref{rem:discrete:eq:3} for a (discrete- or continuous-time) stochastic process $S$ of the general form \eqref{rem:discrete:eq:2.1}, we obtain by Lemma \ref{lem:monomial_trafos} \ref{lem:monomial_trafos:it2} (recalling Definition \ref{def:monomial_trafos}) that the residual $h\circ f$ is monomial \emph{on every connected component} of $D_S$. Consequently: 
\begin{itemize}   
\item[(a)]If the source $S$ is time-continuous, then its sample path $S(\omega)\equiv(S_t(\omega))_{t\in\I}$ is a connected subset of $D_S$ with probability one (cf.\ Lemma \ref{lem:spat_supp} \ref{lem:spat_supp:it2}). The identifiability equations \eqref{thm:NICA_stat:eq3} and \eqref{cor:NICA_MainCor:eq1} of Theorems \ref{thm:NICA_stat} and \ref{cor:NICA_MainCor} then state that, almost surely, the components of (the sample paths of) the estimated source $\hat{S}(\omega)\coloneqq h(X(\omega))$ and those of the original source $S(\omega)$ coincide up to a pathwise-fixed permutation $\tau$ and some monotone scaling $(\alpha_1,\cdots,\alpha_d)$. That is, it states that with probability one there are $\tau$ and $(\alpha_i)$ such that 
\begin{equation}\label{rem:discrete:eq:6}
\hat{S}_t^i = \alpha_i(S^{\tau(i)}_t) \quad \text{ \emph{for each} $t\in\I$} \quad (i\in[d])  
\end{equation} where both $\tau$ and $(\alpha_i)$ are uniquely determined by ($X$ and) $h$ and $\omega$ via the connected component of $D_S$ that the source realisation $S(\omega)$ is contained in. 
\item[(b)]If the source $S$ is time-discrete, then its realisations $S(\omega)\equiv(S_j(\omega))_{j\in\mathbb{Z}}$ are generally not connected in $\R^d$ and might thus be spread over different connected components of $D_S$, again almost surely. Hence in this case we have with probability one that
\begin{equation}\label{rem:discrete:eq:7}
\hat{S}^i_t = \alpha_i^{(t)}(S^{\tau_t(i)}_t) \quad \text{ for each $t\in\I$} \quad (i\in[d])
\end{equation}  
where the permutations $\tau_t$ and scalings $(\alpha_i^{(t)})$ are no longer pathwise-fixed but do now depend on ($X$ and) $h$ and $\omega$ \emph{and} \emph{$t$} via the connected component of $D_S$ that each fixed-time realisation $S_t(\omega)$ is contained in;\footnote{\ That is, $\big(\tau_t, (\alpha_i^{(t)})\big) = \big(\tau_r, (\alpha_i^{(r)})\big)\,\big[\equiv \big(\tau_r(\omega), \alpha_1^{(r)}(\omega), \ldots, \alpha_d^{(r)}(\omega)\big)\big]$ if $S_t(\omega)$ and $S_r(\omega)$ are in the same connected component of $D_S$, and possibly $\big(\tau_t, (\alpha_i^{(t)})\big) \neq \big(\tau_r, (\alpha_i^{(r)})\big)$ if not.} this can be read off Theorem \ref{appendix:NICA_discrete:thm}.   
\end{itemize} 
Clearly the `minimal deviations' \eqref{rem:discrete:eq:6} and \eqref{rem:discrete:eq:7} between $\hat{S}$ and $S$ coincide if $D_S$ is connected, but if it is not connected they are generally different.\footnote{\ Notice that, as specified in Prop.\ \ref{prop:monconcordance}, the test statistic \eqref{def:MonConc:eq2} from Def.\ \ref{def:MonConc} is originally tailored to the connected case \eqref{rem:discrete:eq:6}, but upon straightforward modification it may of course also be used for the case \eqref{rem:discrete:eq:7}.} 
\end{enumerate} 
So throughout the identifiability sections of this paper (Sections \ref{sect:intro_BSS} to \ref{sect:example_sources}), the assumption of sample continuity of $S$ is merely a convenience assumption which we included because of its sufficiency for the pathwise ($t$-independent) identification \eqref{rem:discrete:eq:6}. Except for this distinction between the `connected case' \eqref{rem:discrete:eq:6} and the `disconnected case' \eqref{rem:discrete:eq:7} our identifiability results (Theorems \ref{thm:NICA_stat} and \ref{cor:NICA_MainCor}) apply without further changes, as summarised in Theorem \ref{appendix:NICA_discrete:thm}.\\[-1em] 

\noindent
Another modification for the discrete-time case, this time of a purely technical nature, concerns the optimisation in Theorem \ref{thm:optimisation}, more specifically the applicability of the contrast $\bar{\kappa}_{\mathrm{IC}}$: 

\begin{enumerate}[label=(\roman*), resume]
\item\label{rem:discrete:it2}By its definition the function $\bar{\kappa}_{\mathrm{IC}}$ can only take time-continuous processes as its arguments, but for those it is only the \emph{order} of their time-indexed values that matters (cf.\ Lemma \ref{sect:sigmoments:lem1} \ref{sect:sigmoments:lem1:it2.1}). Consequently:     
\begin{itemize}
\item[(a)]If the source $S$ is time-continuous then so is the candidate transformation $h(X)$, making $\bar{\kappa}_{\mathrm{IC}}(h(X))$ well-defined and Theorem \ref{thm:optimisation} readily applicable as stated.
\item[(b)]If the source $S$ is time-discrete, say $S = (S_j)_{j\in\mathbb{Z}}$, then we may perform the injection $S \mapsto \overline{S}$ where $\overline{S}$ is the piecewise-linear interpolation of $S$ along any (fixed) strictly ordered bounded subset $\mathcal{I}\equiv\{t_j\}\subset\R$, see Section \ref{rem:pwlinterpol}.\footnote{\ Remember that the choice of $\mathcal{I}$ is arbitrary up to order and cardinality (Lemma \ref{sect:sigmoments:lem1} \ref{sect:sigmoments:lem1:it2.1}), that is the interpolation $\overline{Y}$ only needs to preserve the time order of the data points $(Y_j)\equiv Y$; individual values $\overline{Y}_t$ for $t\notin\mathcal{I}$ are irrelevant. This also ensures \eqref{rem:discrete:eq:8} is well-defined, i.e.\ independent of the choice of interpolant of its arguments.} 
Likewise, let us for any $Y=(Y_j)_{j\in\mathbb{Z}}$ denote by $\overline{Y}$ the piecewise-linear interpolation of $Y$ along $\mathcal{I}$, and define by 
\begin{equation}\label{rem:discrete:eq:8} 
\bar{\kappa}_{\mathrm{IC}}(Y)\coloneqq\bar{\kappa}_{\mathrm{IC}}(\overline{Y})  
\end{equation}  
an extension of $\bar{\kappa}_{\mathrm{IC}}$ to discrete-time processes. Note that $\overline{Y}$ is a time-continuous process with $Y_j = \overline{Y}_{t_j}$ for each $j\in\mathbb{Z}$, and further that $Y$ is IC iff $\overline{Y}$ is IC. Thus if $X_j = f(S_j)$ on $\mathbb{Z}$ then also $\overline{X}_{t_j} = f(\overline{S}_{t_j})$ on $\mathcal{I}$, and if $S=(S_j^i)_j$ is $\bar{\alpha}$-, $\bar{\beta}$-, or $\bar{\gamma}$-contrastive then by Theorem \ref{appendix:NICA_discrete:thm}, 
\begin{equation}\label{rem:discrete:eq:9}
\boxed{h_i(X_j) = \alpha_i^{(j)}(S_j^{\tau(i)}) \ \ (\forall\, j\in\mathbb{Z})} \quad \text{iff} \quad \bar{\kappa}_{\mathrm{IC}}\big((h(X_j))_j\big) = 0
\end{equation} for any $h=(h_1,\ldots, h_d)\in C^{2,2}(D_X)$, where the boxed equation holds in the sense of \eqref{rem:discrete:eq:7}. This is the discrete-time version of Theorem \ref{thm:optimisation}.                  
\end{itemize}                        
\end{enumerate}  
In summary, we emphasize that if the source $S$ is time-discrete with its spatial support admitting a dense open subset $\mathcal{D}$ as in \eqref{rem:discrete:eq:3} --- which, for instance, will be induced by the countably indexed reservoir of joint distributions \eqref{rem:discrete:eq:2'} if $S$ is $\bar{\alpha}$-, $\bar{\beta}$- or $\bar{\gamma}$-contrastive --- then our ICA approach [Theorems \ref{thm:NICA_stat}, \ref{cor:NICA_MainCor} and \ref{thm:optimisation}] is directly applicable [in the form of Theorem \ref{appendix:NICA_discrete:thm} and \eqref{rem:discrete:eq:9}] with no additional subtleties other than points (i) and (ii) above.\footnote{\ In particular, the interpolation $Y\mapsto\overline{Y}$ (`discrete to continuous') used in \eqref{rem:discrete:eq:8} only serves to find dependence-minimising transformations $h$ via $\min_h\bar{\kappa}_{\mathrm{IC}}(h(Y))$. This interpolation is thus merely `operational' (for the use of $\bar{\kappa}_{\mathrm{IC}}$) and not related to the identifiability of $S$ itself; in particular, it does not add any geometrical or topological intricacies or complications to the latter.}\\[-1em]

\noindent
Finally, let us explicate the practically important identifiability situation where the given data is a discrete-time approximation (``observation'') of a continuous-time process $X = f(S)$, for $S$ some $\alpha$-, $\beta$- or $\gamma$-contrastive source in $\R^d$.
\begin{enumerate}[label=(\roman*), resume]
\item \label{rem:discreteobs}In this last setting, the given data is of the form $X_{\mathcal{I}}\coloneqq(X_t)_{t\in\mathcal{I}}$ for $\mathcal{I}$ discrete. For the general case that the corresponding discrete time-series $S_{\mathcal{I}}\coloneqq (S_t)_{t\in\mathcal{I}}$ is not itself $\bar{\alpha}$-, $\bar{\beta}$- or $\bar{\gamma}$-contrastive, we may not be able to exactly (i.e.\ up to minimal ambiguity) recover $S_{\mathcal{I}}$ from $X_{\mathcal{I}}$ as we did above. However, we are still guaranteed the asymptotic identification
\begin{equation}\label{rem:discreteobs:eq1}
\forall\,\varepsilon>0 \ : \ \exists\,\delta>0 \quad \text{s.t.} \quad \sup\nolimits_{t\in \mathcal{I}}\big|\hat{\theta}_{\mathcal{I}}(X_t) - \tilde{\alpha}(S_t)\big| \ \leq \varepsilon \qquad \text{if } \ \|\mathcal{I}\| \leq \delta 
\end{equation}          
for some monomial $\tilde{\alpha}\in\mathrm{DP}_{\!d}(D_S)$ depending on $\mathcal{I}$ and on the realisation of $S$ (via the connected component of $D_S$ that this realisation is contained in), cf.\ point \ref{rem:discrete:it1}, with $\hat{\theta}_{\mathcal{I}}\in\operatorname{arg\,min}_{\theta\in\Theta}\bar{\kappa}_{\mathrm{IC}}\big(\theta(X_{\mathcal{I}})\big)$ [in the sense of \eqref{rem:discrete:eq:8} and \eqref{thm:consistency:eq1}] and $\Theta$ as in Theorem \ref{thm:consistency}, and where \eqref{rem:discreteobs:eq1} holds on some ($\mathcal{I}$-dependent) $\mathbb{P}$-full set. (Notice that if $\Theta$ admits a unique minimizer of $\bar{\kappa}_{\mathrm{IC}}$, then the above $\tilde{\alpha}$ is independent of $\mathcal{I}$.) This is a special case of Theorem \ref{thm:consistency} for $(m_0,k,T)=(\infty, k, \infty)$, see also its proof in Section \ref{sect:consistlim}.  
\end{enumerate}   

\subsubsection{Consistency in the Discrete-Time Case}\label{rem:discrete:consistency}Throughout Section \ref{sect:consistency} we assumed that the data-generating signal $X$ from \eqref{sect:finsamlim:eq3} is continuous-time on $[0,1]$. If this is not the case, then the `refinement assumption' $\lim_{k\rightarrow\infty}\|\mathcal{I}_1^{(k)}\|=0$ in \eqref{sect:finsamlim:eq3} can be dropped and naturally replaced by the compensating assumption that 
\begin{equation}
\text{there is \ $k_0$ \quad s.t.\ \quad the observation} \ (X_t)_{t\in\mathcal{I}_1^{(k_0)}} \ \text{is \ $\bar{\alpha}$-, $\bar{\beta}$- or $\bar{\gamma}$-contrastive.}
\end{equation} 
This renders Section \ref{sect:interpollim} void and removes the necessity to, as in Section \ref{sect:sampling}, consider $k$-dependent protocols $\mathcal{J}_k$ with ever growing base lengths $|\mathcal{I}^{(k)}_1|$. Sections \ref{sect:capping} and \ref{sect:ergodicity}, however, stay applicable as stated and Theorem \ref{thm:consistency} remains valid -- by the same proof -- up to the following straightforward modifications. (In essence, we only need to remove the $k$-dependence from Section \ref{sect:consistency}.)\\[-0.5em]

Let $X_\ast$ and $S_\ast$ be as in \eqref{appendix:NICA_discrete:eq1} with $\tilde{X}_\ast=(\tilde{X}_j)_{j\in\mathcal{J}}$ for some $\mathcal{J}\subseteq\Z$ which admits a partition 
\begin{equation}
\mathcal{J}=\bigsqcup\nolimits_{\nu\in\N}\mathcal{I}_\nu \quad \text{with} \quad\mathcal{I}_1 < \mathcal{I}_2 < \ldots \ \text{ s.t. }\ (X_j)_{j\in\mathcal{I}_1}\eqqcolon(\tilde{X}_j)_{j\in\mathcal{I}_1} \ \text{is \ $\bar{\alpha}$-, $\bar{\beta}$- or $\bar{\gamma}$-contrastive.}  
\end{equation}  
Let further $\Theta$ be as in Theorem \ref{thm:optimisation} and such that Assumption \ref{sect:capping:assumptions} holds\footnote{\ Following the piecewise-linear interpolation of $(X_j)_{j\in\mathcal{I}_1}$ as in point \ref{rem:discrete:it2} (b) above.} for $(X_j)_{j\in\mathcal{I}_1}$. We then call the discrete process $\tilde{X}_\ast$ an \emph{ergodic observation of $X_\ast$} if $\tilde{X}_\ast$ is signature ergodic to length $|\mathcal{I}_1|\eqqcolon n$; the remaining notions of Definition \ref{def:ergodic_observation} are adopted analogously. 

Let finally $\|\cdot\|_{[n]}$ denote the uniform norm on $\R^{d\times n}$ (so that $\|(x_j)\|_{[n]} = \max_{j\in[n]}|x_j|$).    

\begin{appendixthm}\label{thm:consistency:discrete}
Let $X_\ast, S_\ast$ and $\Theta$, $\tilde{X}_\ast$ be as above, and suppose that there is $\theta_\star\in\Theta$ such that $\theta_\star(X_\ast)$ is IC. Suppose further that $\tilde{X}_\ast$ is an ergodic* \emph{[}resp.\ weakly ergodic on $\Theta$\emph{]} observation of $X_\ast$. Then for any error bound $\varepsilon>0$ there exists a capping threshold $m_0\geq 2$ such that for any fixed $m\geq m_0$ the following holds: For every sequence $(\hat{\theta}^\star_T)$ in $\Theta$ such that
\begin{equation}\label{thm:consistency:discrete:eq1}
\hat{\kappa}^{m|n}_{T}\!\big(\hat{\theta}_T^\star\big)\ \leq \  \min_{\theta\in\Theta}\,\hat{\kappa}^{m|n}_{T}(\theta) \ + \ \eta_T \quad (T\in\N), \quad \text{for \ \ $\hat{\kappa}^{m|n}_{T}$ as in $\left.\eqref{lem:ergodicity_uniformconv:eq1.2}\right|_{X_\ast\coloneqq\tilde{X}_\ast}$ \ and \ $n\coloneqq|\mathcal{I}_1|$} 
\end{equation} 
and some $(\eta_T)\subset\R_+$ with $\lim_{T\rightarrow\infty}\eta_T=0$ almost surely \emph{[}resp.\ in probability\emph{]}, it holds that       
\begin{equation}\label{thm:consistency:discrete:eq2}
\lim_{\tau\rightarrow\infty}\max\left\{\sup_{T\geq\tau}\Big[\mathrm{dist}_{\|\cdot\|_{[n]}}\!\big(\hat{\theta}^\star_T(X_\ast),\,\mathrm{DP}_d\cdot S_\ast\big)\Big], \, \varepsilon\right\} \ = \ \varepsilon    
\end{equation}  
almost surely \emph{[}resp.\ in probability\emph{]}. If $\tilde{X}_\ast$ is ergodic on $\Theta$ and the spatial support of $X_\ast$ is not necessarily compact, then \eqref{thm:consistency:eq2} holds almost surely with the above threshold $m_0$ depending on the realisation of $\tilde{X}_\ast$.              
\end{appendixthm}        

\subsection{Proof of Proposition \ref{prop:monconcordance}}\label{pf:monconcordance}
\propmondiscordance*
\begin{proof}
This is a direct consequence of the fact that Kendall's (and Spearman's) rank correlation coefficient $\rho_{\mathrm{K}}$ attains its extreme values $\pm 1$ iff one of its arguments is a monotone transformation of the other (cf.\ e.g.\ \citep[Theorem 3 (7.), (8.)]{embrechts2002}), combined with the fact that $\rho_{\mathrm{K}}(U,V)=0$ if $U$ and $V$ are independent (cf.\ e.g.\ \citep[Theorem 3 (2.)]{embrechts2002}).

Indeed, note for the `if'-direction in \eqref{prop:monconcordance:eq1} that $\varrho\big((h(X_t))_{t\in\mathcal{I}}, \, (S_t)_{t\in\mathcal{I}}\big)=0$ implies that there is $\sigma\in S_d$ with $|\rho_{\mathrm{K}}(h_i(X_t), S_t^j)| = \delta_{j,\sigma(i)}$ for each $t\in\mathcal{I}$ and $i\in[d]$, which by the above-mentioned property of $\rho_{\mathrm{K}}$ yields that $(h_i\circ f)(S_t) = \alpha_{i|t}(S^{\sigma(i)}_t)$, and hence $\restr{(h_i\circ f)}{\supp(S_t)} = \restr{\alpha_{i|t}}{\supp(S_t)}$, for some function $\alpha_{i|t} : \supp(S_t)\rightarrow\mathbb{R}$ with $\supp(S_t^{\sigma(i)})\ni x_{\sigma(i)}\mapsto\alpha_{i|t}(x_{\sigma(i)})\equiv\alpha_{i|t}(x)$ monotone. But since for each $i\in[d]$ we have $h_i\circ f\in C^1(\mathcal{O}_S)$ for some $\mathcal{O}_S\supset D_S$ open, the classical pasting lemma (see, for instance, \citep[Corollary 2.8]{leeManifolds}) guarantees that the functions $(\alpha_{i|t}\mid t\in\mathcal{I})$ can be `glued together' to an injective $C^1$-map $\alpha_{i} : \bigcup_{t\in\mathcal{I}}\supp(S^i_t)\rightarrow\R$ with $\restr{\alpha_{i}}{\supp{(S_t^i)}} = \restr{\alpha_{i|t}}{\supp{(S_t^i)}}$ for each $t\in\mathcal{I}$, implying that $(h(X_t))_{t\in\mathcal{I}} = \big(P\cdot(\alpha_{\sigma^{-1}(1)}\times\cdots\times\alpha_{\sigma^{-1}(d)})(S_t)\big)_{t\in\mathcal{I}}$ for $P=(\delta_{\sigma(i),j})_{ij}\in\operatorname{P}_{\!d}$, as claimed.     
\end{proof} 

\subsection{Implementation Details for Section \ref{sect:experimentsII}}\label{appendix:sect:numerics_ann}
The following enumeration ($\nu=1,2$) refers to the mixtures $X^{(2)} = f_\nu(S^{(2)})$ considered in Section \ref{sect:experimentsII}.\\[-0.5em]

\noindent
For the case $\nu=1$, we applied as $f_1$ the mixing transformation depicted in Figure \ref{fig:NN2Dou} (leftmost panel), and as the parametrising ANN $\Theta_1$ we chose a feedforward neural network with a two-nodal in- and output layer and two hidden layers consisting of 4 resp.\ 32 neurons with $\tanh$ activation each; the cumulant series \eqref{prop:sig_cums:eq1} was capped at maximal cumulant order $m_1=6$. For the case $\nu=2$, we followed the simulations of \cite{TCL,HYM} in using as a mixing transformation $f_2$ an invertible feedforward-neural network with four-nodal in- and output layers and two four-nodal hidden layers with $\tanh$ activation each, and as the parametrising ANN chose a feedforward network with a four-nodal in- and output layer and one hidden layer of 1024 neurons and uniformly-weighted Leaky ReLU activations; the contrast function \eqref{prop:sig_cums:eq1} was capped at the maximal cumulant order $m_2=5$. For both $\nu=1,2$, the resulting loss functions \eqref{sect:experimentsII:eq2} were optimised using stochastic gradient descent (Adam) with non-vanishing $\ell_2$-penalty.

\section{Proofs and Remarks for Section \ref{sect:consistency}}\label{appendix:sect:proof_section8}
The following subsections make tacit use of the notation introduced in Appendix \ref{sect:expected_signature_moments}.  
\subsection{Proof of Lemma \ref{sect:capping:lem1}}\label{pf:sectcappinglem}
\begin{appendixlemma}\label{sect:capping:lem1}
Let $X$ and $\Theta$ be as described in Assumption \ref{sect:capping:assumptions}. Then the following holds:
\begin{enumerate}[label=\upshape(\roman*)]
\item\label{sect:capping:lem1:it1} the functions $Q, Q_m : \Theta \rightarrow \R$ given in \eqref{sect:capping:qobjectives} are continuous\,; 
\item\label{sect:capping:lem1:it2} the capped objectives $Q_m$ approximate $Q$ uniformly as $m$ goes to infinity, in symbols:  
\begin{equation}
\lim_{m\rightarrow\infty}\|Q - Q_m\|_\Theta\ = \ 0 \qquad \text{ for } \quad \|q\|_\Theta\coloneqq\sup_{\theta\in\Theta}|q(\theta)|\,; 
\end{equation}
\item\label{sect:capping:lem1:it3} if $Q$ is uniquely minimized at $\theta_\star\in\Theta$, i.e.\ such that $Q(\theta)> Q(\theta_\star)$ if $\theta\neq\theta_\star$, then any (`minimising') sequence $(\theta_m^\star)$ in $\Theta$ such that $Q_m(\theta_m^\star)\leq\inf_{\theta\in\Theta}Q_m(\theta) + \eta_m$ for some $\eta_m\geq 0$ with $\lim_{m\rightarrow\infty}\eta_m=0$ a.s., converges to $\theta_\star$ almost surely as $m\rightarrow\infty$.      
\end{enumerate} 
\end{appendixlemma} 
\begin{proof} 
\ref{sect:capping:lem1:it1}\,:\, For $\nu\geq 2$ fixed, consider the function $q_\nu : \Theta\rightarrow V_\nu$ given by 
\begin{equation}\label{sect:capping:lem1:aux1}
q_\nu(\theta)\,\coloneqq\, \sum_{\bm{q}\in\mathfrak{C}_\nu}\kappa_{\bm{q}}\big(\theta(X)\big)\cdot\frac{\bm{q}}{\sqrt{c_{\bm{q}}}}
\end{equation}
with $c_{\bm{q}}$ the number of monomials in $\bm{q}$ (cf.\ Remark \ref{rem:cross_shuffles}). As the sets $\{c_{\bm{q}}^{-1/2}\cdot\bm{q}\mid\bm{q}\in\mathfrak{C}_\nu\}$ are each finite and orthonormal wrt.\ the Euclidean structure on $V_\nu$ (cf.\ \eqref{rem:freealg_identify_tensor}), we have that $Q_m = \sum_{\nu=2}^m\|q_\nu\|_\nu^2$ for each $m\geq 2$ and thus obtain the continuity of $Q_m$ from the continuity of \eqref{sect:capping:lem1:aux1}. To convince ourselves of the latter, fix any index $\bm{q}\in\mathfrak{C}_\nu$ and let $(\theta_j)_{j\in\N}$ be an arbitrary convergent sequence in $\Theta$, with limit $\lim_{j\rightarrow\infty}\theta_j\eqqcolon\tilde{\theta}$. By a classical interpolation inequality (see e.g.\ \citep[Proposition 5.5. (i)]{FVI}) we for any $2>p'>p$ have that    
\begin{equation}\label{sect:capping:lem1:aux2}
\|\tilde{\theta}(X) - \theta_j(X)\|_{p'\text{-$\mathrm{var}$}} \ \leq \ C\cdot \|\tilde{\theta}(X) - \theta_j(X)\|_\infty^{1-p/p'} \ \longrightarrow \ 0 \quad\text{ a.s.}\quad \ (\text{as } \ j\rightarrow\infty) 
\end{equation} 
for the a.s.\ finite (by \eqref{sect:capping:thetax_reg1}) random variable $C\coloneqq 2^{1-p/p'}\sup_{j\geq 1}\!\big[\|\tilde{\theta}(X) - \theta_j(X)\|_{p\text{-$\mathrm{var}$}}\big]^{p/p'}$, where the convergence in \eqref{sect:capping:lem1:aux2} then follows by the compact convergence $\theta_j\rightarrow\tilde{\theta}$ and the fact that almost every realisation of $X$ has a compact trace in $D_X$ (Lemma \ref{lem:spat_supp} \ref{lem:spat_supp:it2}).

Hence by the $p'$-variation continuity of $\sig$ (Lemma \ref{sect:sigmoments:lem1} \ref{sect:sigmoments:lem1:it2}) followed by dominated convergence (cf.\ \eqref{sect:capping:thetax_reg2}) and the fact that $\log_{[\nu]}\equiv\pi_{[\nu]}\circ\log = \log_{[\nu]}\circ\pi_{[\nu]}$ is continuous (Lemma \ref{sect:sigmoments:lem1} \ref{sect:sigmoments:lem1:it3}), we see that the convergence \eqref{sect:capping:lem1:aux2} implies  that, for any $\bm{q}\in\mathfrak{C}_\nu$, 
\begin{equation}\label{sect:capping:lem1:aux3}
\begin{aligned}
\kappa_{\bm{q}}\!\big(\theta_j(X)\big) &\stackrel{\mathrm{def}}{=} \big\langle\log\!\big[\E[\sig(\theta_j(X))]\big],\,\bm{q}\big\rangle \ = \ \big\langle\log_{[\nu]}\!\big[\E\!\big[(\pi_{[\nu]}\circ\sig)\!\big(\theta_j(X)\big)\big]\big],\,\bm{q}\big\rangle\\
&\longrightarrow \ \big\langle\log_{[\nu]}\!\big[\E\!\big[(\pi_{[\nu]}\circ\sig)\!\big(\tilde{\theta}(X)\big)\big]\big],\,\bm{q}\big\rangle \, = \,\kappa_{\bm{q}}\!\big(\tilde{\theta}(X)\big) \quad \text{ as \ } j\rightarrow\infty.
\end{aligned}
\end{equation} 
Since our topology on $\Theta$ is metrizable, the sequential convergence \eqref{sect:capping:lem1:aux3} characterizes the continuity of $\Theta\ni\theta\mapsto\kappa_{\bm{q}}(\theta(X))$, which yields that \eqref{sect:capping:lem1:aux1} (hence $Q_m$) is continuous as desired. The continuity of $Q$ thus follows from assertion (ii) of this lemma, i.e.\ from the claim 
\begin{equation}\label{sect:capping:lem1:aux4}
Q_m \rightarrow Q \quad\text{uniformly \ on \ $\Theta$} \qquad \text{as \ $m\rightarrow\infty$}.
\end{equation}
\ref{sect:capping:lem1:it2}\,:\, To see that \eqref{sect:capping:lem1:aux4} holds, observe that since $\sup_{\theta\in\Theta}\vertiii{\mathfrak{S}(\theta(X))-1}_\lambda\leq 1$ for some $\lambda>2$ by assumption, Lemma \ref{sect:sigmoments:lem1} \ref{sect:sigmoments:lem1:it3} yields that the set $\mathcal{L}\coloneqq\log\!\big(\{\mathfrak{S}(\theta(X))\mid\theta\in\Theta\}\big)\equiv\{\ell(\theta)\mid\theta\in\Theta\}$ of signature cumulants is $\vertiii{\cdot}_\rho$-bounded for some $\rho>1$. Hence by Lemma \ref{sect:sigmoments:lem1} \ref{sect:sigmoments:lem1:it4},          
\begin{equation}\label{sect:capping:lem1:aux5}
\varsigma_m \, \coloneqq \, \sup_{\bm{\ell}\in\mathcal{L}}\sum_{\nu>m}\|\pi_\nu(\bm{\ell})\|_\nu \ \longrightarrow \ 0 \qquad \text{ as } \ m\rightarrow\infty. 
\end{equation} 
Writing now $\bm{q}=\bm{w}_1(\bm{q}) + \ldots + \bm{w}_{c_{\bm{q}}}(\bm{q})$ for the decomposition of a polynomial $\bm{q}\in\mathfrak{C}_\nu$ into its (monic) monomials $\bm{w}_j(\bm{q})\in[d]^\star_\nu$ (cf.\ Remark \ref{rem:cross_shuffles} \ref{rem:cross_shuffles:it2}), we have that for each $\bm{q}\in\mathfrak{C}_\nu$ the monomials $\bm{w}_1(\bm{q}), \ldots, \bm{w}_{c_{\bm{q}}}(\bm{q})$ are pairwise distinct (Rem.\ \ref{rem:cross_shuffles} \ref{rem:cross_shuffles:it2}), and further that the union $\bigcup_{\bm{q}\in\mathfrak{C}_\nu}\{\bm{w}_1(\bm{q}), \ldots, \bm{w}_{c_{\bm{q}}}(\bm{q})\} \subset [d]^\ast_\nu$ is disjoint (Rem.\ \ref{rem:cross_shuffles} \ref{rem:cross_shuffles:it1}). Hence and since we have\footnote{\ To ease notation, we in \eqref{sect:capping:lem1:aux6} drop the argument of the cumulants, i.e.\ denote $\kappa_{q}\equiv\kappa_{q}(\theta\cdot X)$.}  
\begin{equation}\label{sect:capping:lem1:aux6}
\kappa_{\bm{q}}^2 \,\stackrel{\mathrm{def}}{=}\, \left(\kappa_{\bm{w}_1(\bm{q})} + \ldots + \kappa_{\bm{w}_{c_{\bm{q}}}(\bm{q})}\right)^{\!2} \ \leq \ c_{\bm{q}}\cdot\sum_{j=1}^{c_{\bm{q}}}\kappa_{\bm{w}_j(\bm{q})}^2      
\end{equation}  
by the Cauchy-Schwarz inequality, we for each $m\geq 2$ obtain the estimate 
\begin{equation}\label{sect:capping:lem1:aux7}  
\begin{aligned}
\|Q - Q_m\|_\Theta \ &\leq \ \sup_{\theta\in\Theta}\sum_{\nu>m}\sum_{\bm{q}\in\mathfrak{C}_\nu}c_{\bm{q}}^{-1}\cdot\kappa_{\bm{q}}\!\big(\theta(X)\big)^2 \\
&\leq \ \sup_{\theta\in\Theta}\sum_{\nu>m}\sum_{{\bm{w}\in[d]^\ast_\nu}}\kappa_{\bm{w}}\!\big(\theta(X)\big)^2 \ = \ \sup_{\bm{\ell}\in\mathcal{L}}\sum_{\nu>m}\|\pi_\nu(\bm{\ell})\|_\nu^2.
\end{aligned}       
\end{equation} 
Hence, and since $\lim_{\nu\rightarrow\infty}\|\pi_\nu(\bm{\ell})\|_\nu=0$ uniformly on $\mathcal{L}$ by \eqref{sect:capping:lem1:aux5}, there will be an $m_0\geq 2$ such that $\sup_{\bm{\ell}\in\mathcal{L},\,\nu\geq m_0}\|\pi_\nu(\bm{\ell})\|_\nu < 1$ and therefore, by \eqref{sect:capping:lem1:aux7}, $\|Q - Q_m\|_\Theta\leq \varsigma_m$ for all $m\geq m_0$, implying \eqref{sect:capping:lem1:aux4} as claimed. 

\ref{sect:capping:lem1:it3}\,:\, Let $\theta_\star\in\Theta$ be as above, and $\varepsilon>0$ be arbitrary. Since $\Theta$ is compact so is its closed subset\footnote{\ The topology (of compact convergence) on $\Theta$ is metrizable (cf.\ Appendix \ref{subsubsect:theta_metrizable}), and $B_\varepsilon(\theta_\star)$ denotes the open ball of radius $\varepsilon$ defined wrt.\ any applicable metric on $\Theta$.} $C_\varepsilon\coloneqq\Theta\setminus B_\varepsilon(\theta_\star)$, and for $\zeta_\varepsilon\coloneqq\inf_{\theta\in C_\varepsilon}Q(\theta) = Q(\theta_\epsilon) > Q(\theta_\star)$ (for some $\theta_\varepsilon\in C_\varepsilon$; recall that $Q$ is continuous) and any $\theta\in\Theta$ we have the obvious implication that:
\begin{equation}\label{sect:capping:lem1:aux8}
\text{if}\quad Q(\theta)\,<\,\zeta_\varepsilon, \qquad \text{then}\quad \theta\,\in\, B_\varepsilon(\theta_\star). 
\end{equation} 
Let now $(\theta_m^\star)\subset\Theta$ be a minimising sequence of the required kind. As then $Q(\theta_\star)\leq Q(\theta_m^\star)$ and $Q_m(\theta_m^\star)\leq Q_m(\theta_\star) + \eta_m$ for each $m\geq 2$, we find that $Q(\theta_\star)\leq Q_m(\theta_m^\star) + \big(Q(\theta_m^\star) - Q_m(\theta_m^\star)\big) \leq Q_m(\theta_\star) + \big(Q(\theta_m^\star) - Q_m(\theta_m^\star) + \eta_m\big) = Q(\theta_\star) + r_m$ for $r_m\coloneqq \big(Q_m(\theta_\star) - Q(\theta_\star) + Q(\theta_m^\star) - Q_m(\theta_m^\star) + \eta_m\big).$ Hence $Q(\theta_\star)\leq Q(\theta_m^\star) \leq Q(\theta_\star) + r_m$ and therefore
\begin{equation}\label{sect:capping:lem1:aux9}
\lim_{m\rightarrow\infty}|Q(\theta_\star) - Q(\theta_m^\star)| \ \leq\ \lim_{m\rightarrow\infty} r_m \ = \ 0 \qquad (\text{a.s.}), 
\end{equation}
where the last identity is due to the uniform convergence \ref{sect:capping:lem1:it2} (and our assumption on $(\eta_m)$). Consequently $Q(\theta^\star_m)<\zeta_\varepsilon$ for almost all $m$, which in light of \eqref{sect:capping:lem1:aux8} implies that $\theta_m^\star\in B_\varepsilon(\theta_\star)$ for almost all $m$ (a.s.), as desired.             
\end{proof}

\subsection{Linear Interpolation of Discrete-Time Data}\label{rem:pwlinterpol}
Let $\I$ be a compact interval; say $\I=[0,1]$ wlog. Any dissection $\mathcal{I}\equiv\{t_0, \ldots, t_{n-1}\mid t_0 < \ldots < t_{n-1}\}$ of $\I$ can be uniquely assigned the family of \emph{$\mathcal{I}$-centered hat functions} $\tau_0, \ldots, \tau_{n-1}\in C(\I;\R)$ characterised by:
\begin{equation}\label{rem:pwlinterpol:eq2}
\tau_j \ \text{ is \ $\mathcal{I}$-piecewise affine } \qquad \text{and}\qquad \tau_j(t_\nu)\,=\,\delta_{j\nu} \ \text{ for each $\nu\in[n-1]_0$}
\end{equation} 
for all $j\in[n-1]_0$. A path in $\mathcal{C}\equiv C(\I;\R^d)$ will be called \emph{$\mathcal{I}$-piecewise linear} if it lies in 
\begin{equation}\label{rem:pwlinterpol:eq3}
\mathcal{C}_{\mathcal{I}} \ \coloneqq \ \big\{v_0\cdot\tau_0 + \ldots + v_{n-1}\cdot\tau_{n-1} \ \big| \ v_0, \ldots, v_{n-1}\in\R^d\big\}
\end{equation}  
(the `vectorial span' of \eqref{rem:pwlinterpol:eq2}). Clearly, the set \eqref{rem:pwlinterpol:eq3} is a closed linear subspace of $(\mathcal{C}, \|\cdot\|_\infty)$, in fact of $(\mathcal{BV}, \vertiii{\cdot}_{p\text{-$\mathrm{var}$}})$ (cf.\ below), and each element $\hat{x}=(\hat{x}_t)\in\mathcal{C}_{\mathcal{I}}$ is of the form 
\begin{equation}\label{rem:pwlinterpol:eq4}
\hat{x}_t \ = \ \hat{x}_{t_{j-1}} + \ \frac{t - t_{j-1}}{t_j - t_{j-1}}\cdot(\hat{x}_{t_j} - \hat{x}_{t_{j-1}}) \quad \text{ for } \quad t\in[t_{j-1}, t_j] \qquad (j\in[n-1]). 
\end{equation}
The space $\mathcal{C}_{\mathcal{I}}$ is the co-domain of two natural operators, namely the linear projection 
\begin{equation}\label{rem:pwlinterpol:proj}
\hat{\pi}_{\mathcal{I}} \ : \ \mathcal{C} \, \twoheadrightarrow \, \mathcal{C}_{\mathcal{I}}, \quad \hat{\pi}_{\mathcal{I}}(x)\coloneqq x_{t_0}\cdot\tau_0 + \ldots + x_{t_{n-1}}\cdot\tau_{n-1}\equiv \hat{x}_{\mathcal{I}}    
\end{equation} 
as well as the (continuous wrt.\ both $\|\cdot\|_\infty$ and $\vertiii{\cdot}_p$; set $Z\coloneqq\R^d$) linear injection 
\begin{equation}\label{rem:pwlinterpol:inj}
\hat{\iota}_{\mathcal{I}} \ : \ Z^{\times n} \,\hookrightarrow\, \mathcal{C}_{\mathcal{I}}, \quad \hat{\iota}_{\mathcal{I}}(v_0, \ldots, v_{n-1})\coloneqq v_0\cdot\tau_0 + \ldots + v_{n-1}\cdot\tau_{n-1}. 
\end{equation}
It is clear that the linear operator $\hat{\pi}_{\mathcal{I}}$ is bounded on $\mathcal{C}$ with operator norm $\|\hat{\pi}_{\mathcal{I}}\|= 1$. 
\begin{remark}\label{rem:pwlinterpol:rem1}
\begin{enumerate}[label=(\roman*)]
\item  
As any two points in $\R^d$ uniquely determine the affine path-segment that joins them, the $\mathcal{I}$-piecewise linear projection $\hat{x}_{\mathcal{I}}$ of a path $x$ can be seen as the `unbiased continuous-time approximation' of $x$ given the observations $(x_t \mid t\in\mathcal{I})$.\footnote{\ Likewise, the injection \eqref{rem:pwlinterpol:inj} can be seen as the `unbiased $\mathcal{I}$-centered \emph{continuous-time localisation}' of a sequence $(v_1, \ldots, v_n)\in Z^{\times n}$.} 
\item\label{rem:pwlinterpol:rem1:it2} For any $(z_j)_{j\in[n]}\in Z^{\times n}$ and any $\I$-dissection of cardinality $|\mathcal{I}|=n$, \begin{equation}
\big\|\hat{\iota}_{\mathcal{I}}(z_1,\ldots, z_n)\big\|_{1\text{-$\mathrm{var}$}} = \sum_{j=1}^{n-1}|z_{j+1}-z_j| \ \leq \ 2\|(z_1,\ldots,z_n)\|_1.
\end{equation}  
\end{enumerate}    
\end{remark} 
\noindent
Denote by $\mathcal{BV}_p\coloneqq\{x\in\mathcal{C}\mid \eqref{rem:p-varseminorm} \text{ is finite}\}$ the space of all continuous paths of bounded $p$-variation $(p\geq 1)$, and remark that each $\mathcal{BV}_p$ is a Banach space wrt.\ the norm $\vertiii{x}_{p\text{-$\mathrm{var}$}}\coloneqq \|x\|_{p\text{-$\mathrm{var}$}} + |x(0)|$ (e.g.\  \citep[Thm.\ 5.25 (i)]{FVI}). 
\begin{appendixlemma}\label{rem:pwlinterpol:lem}
For $(\mathcal{I}_n)_{n\in\N}$ a refined sequence of dissections of a compact interval $\I$, 
\begin{equation}\label{rem:pwlinterpol:lem:eq1}
\lim_{n\rightarrow\infty}\hat{\pi}_{\mathcal{I}_n} \ = \ \mathrm{id}_{\mathcal{C}} \quad \text{ pointwise on } \ \mathcal{C}(\I;\R^d) 
\end{equation}
where for each argument the above convergence is understood to take place in $(\mathcal{C}, \|\cdot\|_\infty)$. In addition, the family of operators $(\hat{\pi}_{\mathcal{I}_n}\mid n\in\N)$ is equicontinuous, whence in particular the convergence \eqref{rem:pwlinterpol:lem:eq1} is uniform on compact subsets of $\mathcal{C}$. The family of operators  $(\hat{\pi}_{\mathcal{I}_n}\mid n\in\N)$ remains equicontinuous if $(\mathcal{C}, \|\cdot\|_\infty)$ is replaced by $(\mathcal{BV}_p, \vertiii{\cdot}_{p\text{-$\mathrm{var}$}})$ for any $p\geq 1$.     
\end{appendixlemma}  
\begin{proof} 
The pointwise convergence \eqref{rem:pwlinterpol:lem:eq1} is an easy consequence of definition \eqref{rem:pwlinterpol:proj} and the fact that every element of $\mathcal{C}$ is uniformly continuous on $\I$. The equicontinuity of the family of linear operators $\big(\hat{\pi}_{\mathcal{I}_n}\mid n\in\N\big)$ is immediate by the fact that this family is uniformly bounded (by $1$) in the operator norm. That a pointwise convergent sequence of equicontinuous functions on a metric space (with values in a complete metric space) converges uniformly on compact subsets of its domain is a well-known fact from real analysis (e.g.\ \citep[Exercise 7.16]{RUD}). As shown in \citep[Prop.\ 5.20]{FVI}, the operator family $(\hat{\pi}_{\mathcal{I}} : \mathcal{BV}_p\rightarrow\mathcal{BV}_p \mid n\in\N)$ remains uniformly bounded in the operator norm (and hence is equicontinuous) if the latter is defined wrt.\ the $p$-variation norm $\vertiii{\cdot}_{p\text{-$\mathrm{var}$}}$ on the Banach space $\mathcal{BV}_p$.     
\end{proof}

This subsection concludes with a proof that for protocol-indexed discrete time-series \eqref{sect:finsamlim:eq2}, the ergodicity notions of Remark \ref{def:sigergodicity:rem} \ref{def:sigergodicity:rem:it2} and Definition \ref{def:sigergodicity} coincide.\\[-0.5em]  

\noindent
Note to this end that for $n\leq \hat{n}$, we may embed any $\mathcal{I}\equiv\{t_0<\cdots <t_{n-1}\}$ monotonically into $\mathcal{E}_{\hat{n}}$ via  
\begin{equation}\label{def:sigergodicity:rem:eq3}
t_j \, \mapsto \, \hat{t}_j\coloneqq q_j/(\hat{n}-1)\in\mathcal{E}_{\hat{n}} \quad \text{for}\quad q_j\coloneqq\Big\lceil \tfrac{t_j-t_0}{t_{n-1}-t_0}\cdot (\hat{n}-1)\Big\rceil,
\end{equation} 
where $\lceil\cdot\rceil$ is the ceiling function; lifting \eqref{def:sigergodicity:rem:eq3} to a map $\varphi_{\mathcal{I}}^{\hat{n}}\,:\, [t_0, t_{n-1}]\rightarrow [0,1]$ via piecewise-linear extension defines a strictly monotonous continuous injection of intervals. The embedding \eqref{def:sigergodicity:rem:eq3} of $\mathcal{I}$ will be denoted $\mathcal{I}_{\mathcal{E}_{\hat{n}}}\,\big(\!=\{\hat{t}_j\mid j\in[n-1]_0\} = \varphi^{\hat{n}}_{\mathcal{I}}(\mathcal{I})\big)$.\\[-0,5em] 

\noindent
Given $(X, \mathcal{J}_k)$ as in \eqref{sect:finsamlim:eq2} with maximal observation length $\hat{n}_k$, define the \emph{equidistant augmentation of $(X, \mathcal{J}_k)$} as the time-series $\bar{X}^\ast_{\mathcal{J}_k}\coloneqq (\hat{X}_t^\ast)_{t\in\overline{\mathcal{J}}_k}$ given by (for some fixed $q>1$)  
\begin{equation}
\begin{gathered}
\overline{\mathcal{J}}_k \coloneqq \bigsqcup_{\nu\in\N} \overline{\mathcal{I}}_\nu^{(k)} \quad\text{with}\quad \overline{\mathcal{I}}_\nu^{(k)}\coloneqq q(\nu-1) + \mathcal{E}_{\hat{n}_k} \equiv\Big\{\bar{t}^{\,(k|\nu)}_0 < \cdots < \bar{t}^{\,(k|\nu)}_{\hat{n}_k-1}\Big\}, \\
\text{and} \qquad \hat{X}^\ast \coloneqq \sum_{\nu=1}^\infty\hat{\iota}_{q(\nu-1) + \big[\mathcal{I}^{(k)}_\nu\big]_{\mathcal{E}_{{\hat{n}_k}}}}\!\!\!\!\!\!\big(X_s\mid s\in\mathcal{I}_\nu^{(k)}\big)\cdot\mathbbm{1}_{\big[\bar{t}^{\,(k|\nu)}_0\!,\,\bar{t}^{\,(k|\nu)}_{\hat{n}_k - 1}\big]}   
\end{gathered}     
\end{equation}     
(where $\hat{\iota}_{q(\nu-1) + \big[\mathcal{I}^{(k)}_\nu\big]_{\mathcal{E}_{{\hat{n}_k}}}}\!\!\!\!\!\!(X_s\mid s\in\mathcal{I}_\nu^{(k)})$ is the piecewise-linear interpolation \eqref{rem:pwlinterpol:inj} of the observation $(X_t)_{t\in\mathcal{I}_\nu^{(k)}}$ along the $\overline{\mathcal{I}}_\nu^{(k)}$-embedded (via \eqref{def:sigergodicity:rem:eq3}) equidistant dissection $q(\nu-1) + \big[\mathcal{I}^{(k)}_\nu\big]_{\mathcal{E}_{{\hat{n}_k}}}$ of $\big[\bar{t}^{\,(k|\nu)}_0\!,\,\bar{t}^{\,(k|\nu)}_{\hat{n}_k - 1}\big]$). Then the following holds.
\begin{appendixlemma}\label{def:sigergodicity:rem:lem_unnecessary}
A time-series $(X_t)_{t\in\mathcal{J}_k}$ for $\mathcal{J}_k$ as in \eqref{sect:finsamlim:eq2}, is [weakly] $m^{\mathrm{th}}$-order signature ergodic in the sense of \eqref{def:sigergodicity:rem:eq2.1} iff its equidistant augmentation $\bar{X}^\ast_{\mathcal{J}_k}$ is [weakly] $m^{\mathrm{th}}$-order signature ergodic to length $\hat{n}_k$ in the sense of Definition \ref{def:sigergodicity}.   
\end{appendixlemma}   
\begin{proof}
This follows from Lemma \ref{sect:sigmoments:lem1} \ref{sect:sigmoments:lem1:it2.1}. Indeed: Fix any $\nu$. Denoting $\bar{\mathcal{J}}\coloneqq q(\nu-1) + \big[\mathcal{I}^{(k)}_\nu\big]_{\mathcal{E}_{{\hat{n}_k}}}\, \big(=\varphi(\mathcal{I}^{(k)}_\nu)$ for $\varphi\equiv q(\nu-1) + \varphi^{\hat{n}_k}_{\mathcal{I}_\nu^{(k)}}$, with $\varphi^{\hat{n}_k}_{\mathcal{I}_\nu^{(k)}}$ defined as in \eqref{def:sigergodicity:rem:eq3} above$\big)$, note that $\bar{\mathcal{J}}\subseteq\overline{\mathcal{I}}^{(k)}_\nu\eqqcolon\bar{\mathcal{I}}$ and hence $\hat{\pi}_{\bar{\mathcal{J}}} = \hat{\pi}_{\bar{\mathcal{I}}}\circ\hat{\pi}_{\bar{\mathcal{J}}}$ (directly by \eqref{rem:pwlinterpol:proj}). Consequently,    
\begin{equation}\label{def:sigergodicity:rem:lem_unnecessary:aux1}
\hat{X}^\ast_{(\nu)}\coloneqq\restr{\hat{X}^\ast}{\big[\bar{t}^{\,(k|\nu)}_0\!,\,\bar{t}^{\,(k|\nu)}_{\hat{n}_k - 1}\big]} = \hat{\iota}_{\bar{\mathcal{J}}}(X_s\mid s\in\mathcal{I}^{(k)}_\nu) = \hat{\pi}_{\bar{\mathcal{J}}}\big(X_{\varphi^{-1}}\big) = \hat{\pi}_{\bar{\mathcal{I}}}\big(\hat{X}^\ast_{(\nu)}\big)       
\end{equation} 
and therefore, for $\phi$ and $\tilde{\phi}$ as in \eqref{def:sigergodicity:eq2} and \eqref{def:sigergodicity:rem:eq2.1} respectively, and $Y\coloneqq \bar{X}^\ast_{\mathcal{J}_k}$ and $\hat{n}\coloneqq\hat{n}_k$, 
\begin{equation}\label{def:sigergodicity:rem:lem_unnecessary:aux2}
\phi(Y_{(\hat{n}(\nu-1),\hat{n}\nu]}) = \sig_{[m]}\big(\hat{\pi}_{\bar{\mathcal{I}}}\big(\hat{X}^\ast_{(\nu)}\big)\big) = \tilde{\phi}\big(\hat{X}^\ast_{(\nu)}\big).
\end{equation}
Now since $\hat{X}^\ast_{(\nu)} \stackrel{\eqref{def:sigergodicity:rem:lem_unnecessary:aux1}}{=} \hat{\pi}_{\bar{\mathcal{J}}}(X_{\varphi^{-1}}) \stackrel{\eqref{rem:pwlinterpol:proj}}{=} \sum_{t\in\varphi(\mathcal{I})}X_{\varphi^{-1}(t)}\cdot \tau^{[\bar{\mathcal{J}}]}_t = \sum_{s\in\mathcal{I}}X_{\varphi^{-1}(\varphi(s))}\cdot\tau^{[\bar{\mathcal{J}}]}_{\varphi(s)} = \sum_{s\in\mathcal{I}}X_s\cdot\big(\tau_s^{[\mathcal{I}]}\circ\varphi^{-1}\big) = \hat{\pi}_{\mathcal{I}}(X)\circ\varphi^{-1} = \hat{\pi}_{\mathcal{I}}(\hat{X}_{\mathcal{I}})\circ\tilde{\varphi}$ for $\mathcal{I}\coloneqq\mathcal{I}^{(k)}_\nu$ and $\tilde{\varphi}\coloneqq \varphi^{-1}$, we have
\begin{equation}\label{def:sigergodicity:rem:lem_unnecessary:aux3}
\tilde{\phi}\big(\hat{X}^\ast_{(\nu)}\big) = \tilde{\phi}\big(\hat{\pi}_{\mathcal{I}}(\hat{X}_{\mathcal{I}})\big) = \tilde{\phi}\big(\hat{X}_{\mathcal{I}_\nu^{(k)}}\big) 
\end{equation}
where the first of these equalities is due to Lemma \ref{sect:sigmoments:lem1} \ref{sect:sigmoments:lem1:it2.1}. As the above the choice of $\nu$ was arbitrary, we by combination of \eqref{def:sigergodicity:rem:lem_unnecessary:aux2} and \eqref{def:sigergodicity:rem:lem_unnecessary:aux3} obtain that the identities \eqref{def:sigergodicity:eq2} and \eqref{def:sigergodicity:rem:eq2.1} are equivalent. This concludes the proof.                 
\end{proof}   

\subsection{Proof of Lemma \ref{lem:interpolim}} 
\label{pf:interpolim}
\leminterpolim*
\begin{proof}
Recalling that $\mathrm{tr}(X)\subseteq D_X$ with probability one (Lemma \ref{lem:spat_supp} \ref{lem:spat_supp:it2}), notice that
\begin{equation}\label{lem:interpolim:aux1}
\lim_{n\rightarrow\infty}\,\sup_{\theta\in\Theta}\Big\|\hat{X}^\theta_{\mathcal{I}_n} - \theta(X)\Big\|_{\tilde{p}\text{-$\mathrm{var}$}} \ = \ 0 \qquad \text{almost surely} 
\end{equation}
for any $\tilde{p}>p$ with $p\geq 1$ as in \eqref{sect:capping:thetax_reg1}. Indeed: Denoting by $x\coloneqq (X_t(\omega))_{t\in\I}\subseteq D_X$ (inclusion with probab.\ one) a given realisation of $X$, we remark first that (cf.\ Section \ref{rem:pwlinterpol} for notation)
\begin{equation}\label{lem:interpolim:aux2}
\Theta_x \ \coloneqq \ \big\{x^\theta\equiv\big(\theta(x_t)\big)_{\!t\in\I}\ \big| \ \theta\in\Theta\big\} \text{ \ \ is a compact subset of \ \ }(\mathcal{BV}_{\tilde{p}}, \vertiii{\cdot}_{\tilde{p}\text{-$\mathrm{var}$}})
\end{equation}for any $\tilde{p}> p$. To see that \eqref{lem:interpolim:aux2} holds, observe that \citep[Lemma 5.27 (i)]{FVI} (together with \eqref{sect:capping:thetax_reg1}) implies that, for any $\tilde{p}>p$, each of the functions
\begin{equation}\label{lem:interpolim:aux3}
\alpha_n, \, \alpha \ : \ \Theta \rightarrow \mathcal{BV}_{\tilde{p}}, \quad \alpha_n(\theta)\coloneqq\hat{\pi}_{\mathcal{I}_n}\!\big(\theta(x)\big) \ \text{ and } \ \alpha(\theta)\coloneqq \theta(x) \qquad (n\in\N)
\end{equation}    
are continuous. In particular, $\Theta_x=\alpha(\Theta)$ is compact (as continuous image of a compact set). 

In addition, \citep[Lemma 5.27 (i)]{FVI} (by virtue of Lemma \ref{rem:pwlinterpol:lem} \eqref{rem:pwlinterpol:lem:eq1} and \citep[Proposition 5.20 (5.13)]{FVI}) implies that $\lim_{n\rightarrow\infty}\alpha_n=\alpha$ pointwise on $\Theta$. Hence by \eqref{lem:interpolim:aux2} and the last assertion of Lemma \ref{rem:pwlinterpol:lem} (which implies that $(\hat{\pi}_{\mathcal{I}_n})$ converges uniformly on $\Theta_x$; see the proof of Lemma \ref{rem:pwlinterpol:lem} for details), we obtain that $\lim_{n\rightarrow\infty}\alpha_n=\alpha$ uniformly on $\Theta$, which in turn yields \eqref{lem:interpolim:aux1} by the fact that $\hat{X}^\theta_{\mathcal{I}_n}=\hat{\pi}_{\mathcal{I}_n}(\theta(X))$ for each $\theta\in\Theta$.         

Given \eqref{lem:interpolim:aux1} for any fixed $\tilde{p}\in(p,2)$, the $\tilde{p}$-variation continuity of $\mathfrak{sig}$ (Lemma \ref{sect:sigmoments:lem1} \ref{sect:sigmoments:lem1:it2}) together with the equicontinuity of $(\hat{\pi}_{\mathcal{I}_n}:\mathcal{BV}_{\tilde{p}}\rightarrow\mathcal{BV}_{\tilde{p}}\mid n\in\N)$ (Lemma \ref{rem:pwlinterpol:lem}) yields that 
\begin{equation}\label{lem:interpolim:aux4}
\lim_{n\rightarrow\infty}\sup_{\theta\in\Theta}\left\|\sig_m(\hat{X}^\theta_{\mathcal{I}_n}) - \sig_m\!\big(\theta(X)\big)\right\|_m \ = \ 0 \quad \text{almost surely} \qquad (m\in\N).
\end{equation}  
Indeed, the above holds path-wise, with probability one, by Lemma \ref{appendixlem:equicont_uniformconv} (applied to $\Theta$ as above, $B=\mathcal{BV}_{\tilde{p}}$, $V= V_{[m]}$, $\Psi=\sig_m$, $\alpha$ and $\alpha_n$ as in \eqref{lem:interpolim:aux3} and $\tau=\alpha$). 

Thus for $\mathfrak{S}_{m|n}(\theta)\coloneqq\mathbb{E}\big[\sig_m(\hat{X}^\theta_{\mathcal{I}_n})\big]$ and $\mathfrak{S}_{m}(\theta)\coloneqq\mathbb{E}\big[\sig_m\!\big(\theta(X)\big)\big]$ we have that  
\begin{equation}\label{lem:interpolim:aux5}
\lim_{n\rightarrow\infty}\mathfrak{S}_{m|n}(\theta)\ = \ \mathfrak{S}_m(\theta) \quad \ \text{ uniformly on } \ \Theta
\end{equation} 
due to \citep[Theorem 22 (p.~241)]{graves1946} (note that the hypothesis in loc.cit.\ of $(\mathfrak{S}_{m|n})_n$ to be ``absolutely continuous uniformly'' is met in light of \citep[Thm.\ 11 (p.~192)]{graves1946} and assumption \eqref{sect:capping:thetax_reg2}).

Finally, the fact that $\log_{[m]}\equiv\pi_{[m]}\circ\log$ is continuous (Lemma \ref{sect:sigmoments:lem1} \ref{sect:sigmoments:lem1:it3}) together with the uniform convergence \eqref{lem:interpolim:aux5} of $\mathfrak{S}_{[m]|n}\coloneqq\sum_{\nu=0}^m\mathfrak{S}_{\nu|n}$ towards $\mathfrak{S}_{[m]}\coloneqq\sum_{\nu=0}^m\mathfrak{S}_\nu$, yields that 
\begin{equation}\label{lem:interpolim:aux6}
\kappa^{[m]}_n\coloneqq \log_{[m]}\circ\,\mathfrak{S}_{[m]|n} \ \ \stackrel{n\rightarrow\infty}{\longrightarrow} \ \ \log_{[m]}\circ\,\mathfrak{S}_{[m]}\eqqcolon\kappa^{[m]} \quad \ \text{ uniformly on } \ \Theta.  
\end{equation}
In particular, $\kappa_{\bm{q}}(\hat{X}^\theta_{\mathcal{I}_n}) = \langle\kappa^{[m]}_n(\theta),\, \bm{q}\rangle \rightarrow \langle\kappa^{[m]}(\theta),\, \bm{q}\rangle = \kappa_{\bm{q}}(\theta(X))$ uniformly on $\Theta$ for each $\bm{q}\in V_{[m]}$, which by definitions \eqref{sect:interpollim:data:eq1} (of $\widehat{Q}_m$) and \eqref{sect:capping:qobjectives} (of $Q_m$) yields \eqref{lem:interpolim:eq1} as desired.               
\end{proof} 
\begin{appendixlemma}\label{appendixlem:equicont_uniformconv}
Let $\Theta$ be a compact metric space, $B$ and $V$ be Banach spaces, $\Psi : B \rightarrow V$ be a continuous map, and $\alpha,\,\alpha_n, \, \tau : \Theta\rightarrow B$, $n\in\N$, be continuous functions such that
\begin{equation}
\alpha_n = p_n\circ\tau, \ \ \ n\in\N, \quad \text{ with } \quad  \big(p_n : \tau(\Theta) \rightarrow B \ \big| \ n\in\N\big) \ \text{ equicontinuous}  
\end{equation} 
and $\lim_{n\rightarrow\infty}\alpha_n=\alpha$ pointwise on $\Theta$. Then $\lim_{n\rightarrow\infty}\Psi\circ\alpha_n = \Psi\circ\alpha$ uniformly on $\Theta$. 
\end{appendixlemma} 
\begin{proof}
Let $(\theta_n)_{n\in\N}$ be any convergent sequence in $\Theta$, say $\theta_n\rightarrow \theta$ for some $\theta\in\Theta$. Then 
\begin{equation}\label{appendixlem:equicont_uniformconv:aux1}
\lim_{n\rightarrow\infty}\Phi_n(\theta_n) \ = \ \Phi(\theta) \qquad \text{for} \qquad \Phi_n\coloneqq\Psi\circ\alpha_n \ \text{ and } \ \Phi\coloneqq\Psi\circ\alpha,
\end{equation}
since $\|\Phi(\theta) - \Phi_n(\theta_n)\|_V \leq \|\Phi(\theta) - \Phi_n(\theta)\|_V + \|\Phi_n(\theta) - \Phi_n(\theta_n)\|_V $ with $\lim_{n\rightarrow\infty}\|\Phi(\theta) - \Phi_n(\theta)\|_V=0$ and $\lim_{n\rightarrow\infty}\|\Phi_n(\theta) - \Phi_n(\theta_n)\|_V=0$. For the latter convergence, take any $\varepsilon>0$ and let $\delta_1>0$ be such that $\sup_{b\in B_{\delta_1}(\alpha(\theta))}\|\Psi(\alpha(\theta)) - \Psi(b)\|_V\leq\varepsilon$, and $\delta_2>0$ be such that $\rho_n\coloneqq\sup_{b\in B_{\delta_2}(\tau(\theta))\cap\tau(\Theta)}\|p_n(\tau(\theta)) - p_n(b)\|_B \leq \delta_1$ for all $n\in\N$. Taking $n_0\geq 1$ such that $\sup_{n\geq n_0}\|\tau(\theta) - \tau(\theta_n)\|_B\leq\delta_2$ then implies that 
\begin{equation}\label{appendixlem:equicont_uniformconv:aux2} 
\sup_{n\geq n_0}\|\alpha_n(\theta) - \alpha_n(\theta_n)\|_B = \sup_{n\geq n_0}\|p_n(\tau(\theta)) - p_n(\tau(\theta_n))\|_B \ \leq \ \sup_{n\geq n_0}\rho_n \leq \delta_1   
\end{equation} 
and therefore $\sup_{n\geq n_0}\|\Phi_n(\theta) - \Phi_n(\theta_n)\|_V\leq\varepsilon$, as required.   

Conclude by observing that \eqref{appendixlem:equicont_uniformconv:aux1} implies $\Phi_n\rightarrow\Phi$ uniformly on $\Theta$, as desired.

Indeed, assume otherwise that $\Phi_n\nrightarrow \Phi$ uniformly, i.e.\ that there is $\tilde{\varepsilon}>0$ such that 
\begin{equation}\label{appendixlem:equicont_uniformconv:aux3}
\forall\, k\in\N \ : \ \exists\, n_k\in\N_{\geq n} \ \text{ with } \ \sup_{\theta\in\Theta}\|\Phi(\theta) - \Phi_{n_k}(\theta)\|_V > \tilde{\varepsilon}.    
\end{equation} 
Then \eqref{appendixlem:equicont_uniformconv:aux3} informs the choice of a subsequence $(\theta_{n_k})_k\subseteq\Theta$, with $(n_k)_k\subseteq\N$ increasing, s.t.\ 
\begin{equation}\label{appendixlem:equicont_uniformconv:aux4}
\|\Phi(\theta_{n_k}) - \Phi_{n_k}(\theta_{n_k})\|_V \ > \ \tilde{\varepsilon}\quad \text{ for each } \ k\in\N.
\end{equation}As $\Theta$ is compact, we may assume, by passing to a further subsequence if necessary, that this subsequence converges, say to $\tilde{\theta}\in\Theta$. The continuity of $\Phi$ then implies $\lim_{k\rightarrow\infty}\Phi(\theta_{n_k})=\Phi(\tilde{\theta})$, while property \eqref{appendixlem:equicont_uniformconv:aux1} combined with a doubling argument (as in \citep[Sect.\ 3.5*: remark on p.\ 98]{remmert1998}) yields $\lim_{k\rightarrow\infty}\Phi_{n_k}(\theta_{n_k}) = \Phi(\tilde{\theta})$. Hence $\lim_{k\rightarrow\infty}\|\Phi(\theta_{n_k}) - \Phi_{n_k}(\theta_{n_k})\|_V=0$, in contradiction to \eqref{appendixlem:equicont_uniformconv:aux4}.       
\end{proof} 

\subsection{Proof of Proposition \ref{lem:ergodicity_theta}}\label{pf:ergodicity_theta} 
\propergodicitytheta* 
\begin{proof}
For $\tilde{m}\geq 1$ and $\theta\in\Theta$ and $w\in V_{\tilde{m}}$ all arbitrary but fixed, let $\phi=\phi_{\tilde{m}}$ be as in \eqref{def:sigergodicity:eq2} and set $\xi\coloneqq \langle \phi\circ\theta^{\times n}, \,w\rangle$. Set further $\hat{\xi}_T(z)\coloneqq\frac{1}{T}\sum_{j=1}^T\xi(z_{n(j-1)+1}, \ldots, z_{nj})$ for any given sequence $z=(z_\nu)_{\nu\in\N}$ in $\R^d$. The lemma then asserts that, under the given integrability and ergodicity conditions, 
\begin{equation}\label{lem:ergodicity_theta:aux1}
\E\big[\xi(X_1, \ldots, X_n)\big] \, = \, \lim_{T\rightarrow\infty} \hat{\xi}_T(X_\ast) \quad \text{a.s.\ \quad [resp.\footnotemark \ in probab.]}. 
\end{equation}\footnotetext{\ For simplicity of exposition, we present the case of almost sure convergence first and give the changes necessary for the case of convergence in probability at the end of this proof.}
To see that \eqref{lem:ergodicity_theta:aux1} holds, note first that for $\hat{X}_1\coloneqq\hat{\iota}_{\mathcal{E}_n}(X_1, \ldots, X_n)$ (cf.\ Def.\ \ref{def:sigergodicity} and \eqref{rem:pwlinterpol:inj}),
\begin{equation}\label{lem:ergodicity_theta:aux2}
\xi(X_{[n]}) \, = \, \varphi(\hat{X}_1) \qquad\text{for}\qquad \varphi(x)\coloneqq\left\langle\sig_{[\tilde{m}]}\!\big(\hat{\pi}_{\mathcal{E}_n}\!\big(\tilde{\theta}(x)\big)\big), \, w\right\rangle 
\end{equation}
where $\tilde{\theta}$ is any fixed continuous extension of $\theta$ to $\widehat{D}\coloneqq\mathrm{conv}(D_{X_\ast})$, the convex hull of $D_{X_\ast}$. (Recall that such a $\tilde{\theta}$ exists by Tietze's extension theorem.) Since the function $\varphi : \widehat{\mathcal{BV}}_n\rightarrow \R$ defined by \eqref{lem:ergodicity_theta:aux2} on the compact\footnote{\ By \citep[Prop. 1.7]{FLO} and the facts that: (a) the convex hull operator on $\R^d$ preserves compactness, and (b) the Cartesian product of compact sets is compact (noting that $\widehat{\mathcal{BV}}_n\cong\widehat{D}^{\times n}$).} subset  
\begin{equation}
\widehat{\mathcal{BV}}_n\,\coloneqq\,\Big\{x\in\mathcal{C}_{\mathcal{E}_n} \ \Big| \ x_t\in\widehat{D} \ \text{ for each } \ t\in\mathcal{E}_n\Big\} \ \subset \ \mathcal{BV} \qquad \text{(cf.\ \eqref{rem:pwlinterpol:eq3})} 
\end{equation} 
is continuous (by \citep[Prop.\ 5.20]{FVI} and Lemma \ref{sect:sigmoments:lem1} \ref{sect:sigmoments:lem1:it2}), the universality property of the signature (e.g.\ Lemma \ref{sect:sigmoments:lem1} \ref{sect:sigmoments:lem1:it6}) implies that
there is a sequence $(\bm{\ell}_j)_{j\in\N}$ in $V^\circ$ such that
\begin{equation}\label{lem:ergodicity_theta:aux4}
\varphi \ = \ \lim_{j\rightarrow\infty}\langle\sig(\cdot),\,\bm{\ell}_j\rangle  \qquad \text{in } \ \big(C(\widehat{\mathcal{BV}}_n),\,\|\cdot\|_\infty\big).
\end{equation}   
For $\big(\hat{\E}_T^{(m)}(X_\ast)\big)_{T\in\N}\coloneqq\big(\frac{1}{T}\sum_{\nu=1}^T\sig_{[m]}(\hat{X}_\nu)\big)_{T\in\N}$ with $\hat{X}_\nu\coloneqq\hat{\iota}_{\mathcal{E}_n}(X_{n(\nu-1)+1}, \ldots, X_{n\nu})$, our assumption on $X_\ast$ gives that, for each $m\in\N$,
\begin{equation}\label{lem:ergodicity_theta:aux5}
\E\big[\sig_{[m]}(\hat{X}_1)\big] = \lim_{T\rightarrow\infty}\hat{\E}^{(m)}_T(X_\ast) \quad \text{a.s.\ \quad in \ $\operatorname{conv}(\sig_{[m]}(\widehat{\mathcal{BV}}_n))$}.
\end{equation}    
Hence upon combining \eqref{lem:ergodicity_theta:aux2} and \eqref{lem:ergodicity_theta:aux4}, and using that dominated convergence applies as both sides of \eqref{lem:ergodicity_theta:aux4} are bounded (cf.\ Lemma \ref{sect:sigmoments:lem1} \ref{sect:sigmoments:lem1:it2}), we find that with probability one, 
\begin{equation}\label{lem:ergodicity_theta:aux7}
\begin{aligned}
\E\big[\xi(X_1, \ldots, X_n)\big] &= \lim_{j\rightarrow\infty}\big\langle\E\big[\sig(\hat{X}_1)\big],\,\bm{\ell}_j\big\rangle = \lim_{j\rightarrow\infty}\lim_{T\rightarrow\infty}\big\langle\hat{\E}_T^{(d_{\ell_j})}\!(X_\ast),\,\bm{\ell}_j\big\rangle \\
&=\ \lim_{T\rightarrow\infty}\lim_{j\rightarrow\infty}\frac{1}{T}\,\sum_{\nu=1}^{T}\big\langle\sig_{[d_{\bm{\ell}_j}]}(\hat{X}_\nu), \, \bm{\ell}_j\big\rangle \\
&=\ \lim_{T\rightarrow\infty}\,\frac{1}{T}\sum_{\nu=1}^{T}\lim_{j\rightarrow\infty}\big\langle\sig(\hat{X}_\nu),\,\bm{\ell}_j\big\rangle \, \stackrel{\eqref{lem:ergodicity_theta:aux4}}{=} \, \lim_{T\rightarrow\infty}\,\frac{1}{T}\sum_{\nu=1}^{T}\xi(\hat{X}_\nu),
\end{aligned}   
\end{equation}      
where we denoted $d_{\bm{\ell}}$ for the degree of the index-polynomial $\bm{\ell}\in V^\circ$. Notice that the interchange of limits in the second line of \eqref{lem:ergodicity_theta:aux7} is permissible as the convergence in \eqref{lem:ergodicity_theta:aux4} is uniform (see, e.g., \citep[Theorem 7.11]{rudin1976}). This shows the almost-sure case of \eqref{lem:ergodicity_theta:aux1}.  

To prove that \eqref{lem:ergodicity_theta:aux1} holds in probability if \eqref{lem:ergodicity_theta:aux5} holds in probability for each $m\in\N$ (which is true by assumption if $X_\ast$ is weakly signature ergodic), we resort to a subsequence argument, recalling that (as the topology of weak convergence is metrizable) a sequence converges in probability iff each of its subsequences admits yet another subsequence that converges almost surely. To this end, abbreviate $\mu_{m,T}\coloneqq \hat{\E}_T^{(m)}(X_\ast)$ and assume that
\begin{equation}\label{lem:ergodicity_theta:aux8}
\mu_m\coloneqq\E[\mathfrak{sig}_{[m]}(\hat{X}_1)] \, = \, \lim_{T\rightarrow\infty}\mu_{m,T} \quad \text{in probability} \qquad \text{for each } \ m\in\N.
\end{equation}
Then for any fixed subsequence $(T_k)_{k\in\N}\subset\N$, there is a subsequence $T^{(1)}_k\! < T^{(1)}_{k+1}$ of $(T_k)$ such that $\lim_{k\rightarrow\infty}\mu_{1, T^{(1)}_k}=\mu_1$ almost surely. But since, by \eqref{lem:ergodicity_theta:aux8}, $\lim_{k\rightarrow\infty}\mu_{2,T^{(1)}_k}=\mu_2$ in probability, there will be a subsequence $T^{(2)}_k\! < T^{(2)}_{k+1}$ of $(T^{(1)}_k)$ such that $\lim_{k\rightarrow\infty}\mu_{2, T^{(2)}_k}=\mu_2$ almost surely (thus $\lim_{k\rightarrow\infty}\mu_{1, T^{(2)}_k} = \mu_1$ a.s.\ in particular). Iterating this procedure, Cantor's diagonal trick (e.g.\ \citep[Proof of Theorem I.24]{reedsimon1972}) thus allows for the choice of a subsequence $T_k^{(\infty)}\!\! < T_{k+1}^{(\infty)}$ of $(T_k)$ such that $\lim_{k\rightarrow\infty}\mu_{m,T^{(\infty)}_k}=\mu_m$ almost surely for each $m\in\N$.     

Repeating the above calculation \eqref{lem:ergodicity_theta:aux7} then shows that the subsequence $(\hat{\xi}_{T^{(\infty)}_k}(X_\ast))_{k\in\N}$ of $(\hat{\xi}_{T_k}(X_\ast))_{k\in\N}$ converges almost surely to $\E[\xi(X_{[n]})]$, as desired.    
\end{proof}

\subsection{Proof of Lemma \ref{lem:ergodicity_uniformconv}}\label{pf:ergodicity_uniformconv}
\lemergodicityuniformconv* 
\begin{proof}
Let $Z\coloneqq\R^d$ and $\mathcal{E}_n$ be as in Def.\ \ref{def:sigergodicity}, and for every $\theta\in\Theta$ denote
\begin{equation}\label{lem:ergodicity_uniformconv:aux1}
\xi_\theta \,\coloneqq\,\sig_{[m]}\circ\hat{\iota}_{\mathcal{E}_n}\circ\theta^{\times n} \ : \ Z^{\times n} \, \longrightarrow \, V_{[m]}\cap V_{(1)}.
\end{equation}   
The parametrisation-invariance of $\sig$ (Lemma \ref{sect:sigmoments:lem1} \ref{sect:sigmoments:lem1:it2.1}) gives that 
\begin{equation}\label{lem:ergodicity_uniformconv:aux2}
\hat{\mathfrak{S}}_T^{m|n}\!(\theta) = \frac{1}{T}\sum_{j=1}^T\xi_\theta(\bar{X}_j) \ \eqqcolon \ \hat{\E}_T[\xi_\theta(X_\ast)] \quad\text{ for }\quad \bar{X}_j\coloneqq\big(X_{n(j-1)+1}, \ldots, X_{nj}\big),  
\end{equation}  
and the continuity of $\log_{[m]}$ (Lemma \ref{sect:sigmoments:lem1} \ref{sect:sigmoments:lem1:it3}) yields that \eqref{lem:ergodicity_uniformconv:eq2} follows from the convergence  
\begin{equation}\label{lem:ergodicity_uniformconv:aux3}
\lim_{T\rightarrow\infty}\,\sup_{\theta\in\Theta}\big\|\E[\xi_\theta(\bar{X}_1)] - \hat{\E}_T[\xi_\theta(X_\ast)]\big\|_{[m]} = \ 0 \qquad \text{a.s.\ \ \ [in probab.]}   
\end{equation}
for the norm $\|\cdot\|_{[m]}\coloneqq\sum_{\nu=1}^m\|\cdot\|_\nu$, followed by an application of the continuous mapping theorem (e.g.\ \citep[Theorem 2.3]{vdv1998}). As \eqref{lem:ergodicity_uniformconv:aux3} is equivalent to the coordinatewise convergences  
\begin{equation}\label{lem:ergodicity_uniformconv:aux4}
\lim_{T\rightarrow\infty}\,\sup_{w\in {[d]}^\ast_k}\sup_{\theta\in\Theta}\left|\E[\langle\xi_\theta(\bar{X}_1), w\rangle] - \langle\hat{\E}_T[\xi_\theta(X_\ast)], w\rangle\right| = \ 0 \qquad \text{for } \ k\in[m]   
\end{equation} 
almost surely (resp.\ in prob.), we can see that \eqref{lem:ergodicity_uniformconv:aux4} holds by fixing any $w\in {[d]}^\ast_k$ and showing
\begin{equation}\label{lem:ergodicity_uniformconv:aux5}
\lim_{T\rightarrow\infty}\,\sup_{\theta\in\Theta}\left|\E\big[\tilde{\xi}_\theta(\bar{X}_1)\big] - \hat{\E}_T\!\big[\tilde{\xi}_\theta(X_\ast)\big]\right| = \ 0 \quad [\text{a.s./in prob.}] \quad\text{for} \quad \tilde{\xi}_\theta\coloneqq\langle\xi_\theta,w\rangle
\end{equation}              
and $\hat{\E}_T[\tilde{\xi}_\theta(X_\ast)]\coloneqq \langle\hat{\E}_T[\xi_\theta(X_\ast)], w\rangle = T^{-1}\sum_{j=1}^T\tilde{\xi}_\theta(\bar{X}_j)$. To this end, note that the function 
\begin{equation}
\tilde{\xi} \ : \ Z^{\times n}\times\Theta \longrightarrow \R, \quad (z,\theta) \,\mapsto \, \tilde{\xi}_\theta(z),
\end{equation} 
is continuous in $\theta$ for every $z\in Z^{\times n}$, as is seen directly from \eqref{lem:ergodicity_uniformconv:aux1} (recalling the continuity of $\hat{\iota}_{\mathcal{E}_n}:Z^{\times n}\rightarrow\mathcal{BV}$ (Rem.\ \ref{rem:pwlinterpol:rem1} \ref{rem:pwlinterpol:rem1:it2}) and Lemma \ref{sect:sigmoments:lem1} \ref{sect:sigmoments:lem1:it2}). Also, by assumption, $\Theta$ is compact with $\E[\sup_{\theta\in\Theta}|\tilde{\xi}_\theta(\bar{X}_1)|]<\infty$ and $\hat{\E}_T[\xi_\theta(X_\ast)]\rightarrow\E[\xi_\theta(\bar{X}_1)]$ pointwise, which altogether implies that $\mathcal{F}\coloneqq\{\tilde{\xi}(\cdot,\theta)\mid\theta\in\Theta\}$ is Glivenko-Cantelli via \citep[Lem.\ 6.1, Thm.\ 6.1]{WEL}, i.e.\ that
\begin{equation}\label{lem:ergodicity_uniformconv:aux7}
\lim_{T\rightarrow\infty}\,\sup_{\varphi\in\mathcal{F}}\Bigg|\E[\varphi(\bar{X}_1)] - \frac{1}{T}\sum_{j=1}^T\varphi(\bar{X}_j)\Bigg| \ = \ 0, 
\end{equation}
where the mode of the convergence in \eqref{lem:ergodicity_uniformconv:aux7} (almost surely or in probability) coincides with the mode of the pointwise convergence $\hat{\E}_T[\xi_\theta(X_\ast)]\rightarrow\E[\xi_\theta(\bar{X}_1)]$ on $\Theta$ (cf.\ \citep[(Proof of) Theorem 6.1]{WEL}). As \eqref{lem:ergodicity_uniformconv:aux7} is identical to \eqref{lem:ergodicity_uniformconv:aux5}, the proof is finished.               
\end{proof} 

\subsection{Some Sufficent Conditions for Signature-Ergodicity}\label{pf:sigergodicity1}
Let $X_\ast\equiv(X_j)_{j\in\N}$ be a sequence of $\R^d$-valued random variables.  
\begin{appendixdef}\label{appendixdef:mixing}
The sequence $X_\ast$ is called \emph{$\alpha$-mixing} if for the sub-$\sigma$-algebras $\mathcal{X}_k^\ell\coloneqq\sigma\big(X_\nu\mid k\leq\nu\leq\ell\big)$ it holds that $\lim_{\nu\rightarrow\infty}\alpha_\nu(X_\ast)=0$ for the sequence
\begin{equation}\label{eq:alphamixing}
\alpha_\nu(X_\ast)\ \coloneqq \sup_{A\in\mathscr{X}_1^j,\, B\in \mathscr{X}_{j+\nu}^\infty,\,j\in\mathbb{N}}\big|\mathbb{P}(A\cap B) - \mathbb{P}(A)\mathbb{P}(B)\big|\,,
\end{equation}
and $X_\ast$ is called \emph{$\phi$-mixing} if it holds that $\lim_{\nu\rightarrow\infty}\phi_\nu(X_\ast)=0$ for the sequence
\begin{equation}\label{eq:phimixing}
\phi_\nu(X_\ast)\ \coloneqq \sup_{A\in\mathscr{X}_1^j,\, B\in\mathscr{X}_{j+\nu}^\infty,\, \mathbb{P}(A)>0,\,j\in\mathbb{N}}\big|\mathbb{P}(B\,|\,A) - \mathbb{P}(B)\big|.
\end{equation} 
\end{appendixdef}Note that $\phi$-mixing implies $\alpha$-mixing, and see e.g.\ \cite{BR2} for further information. 
 
\begin{appendixdef}\label{appendixdef:seasonal_increments}
The sequence $X_\ast$ will be said to have \emph{$n$-seasonal increments}, $n\in\N$, if the sequence $\Delta(X_\ast)\coloneqq(X_{j+1}-X_j)_{j\in\N}$ of its increments is suff.\ integrable and such that
\begin{equation}\label{appendixdef:seasonal_increments:eq1}
\Delta(X_\ast)_{[n]} \ \stackrel{\mathrm{d}}{=} \ \Delta(X_\ast)_{(n(j-1):\,nj]} \quad \text{ for each } j\in\N.
\end{equation}
We further say that a time series $(X_j)_{j\in\N}$ has \emph{$(m,n)$-stationary sigmoments} if the $[m]^{\mathrm{th}}$-signature moments of the batches $(X_1, \ldots, X_n), (X_{n+1}, \ldots, X_{2n}), \ldots$ exist and are equal, i.e.\ if for $\phi=\phi_m$ as in \eqref{def:sigergodicity:eq2} we have: $\E[\phi_m(X_{[n]})] = \E[\phi_m(X_{(n(j-1):nj]})]$ for each $j\in\N$.  
\end{appendixdef} 
{\renewcommand\footnote[1]{}}      
\begin{appendixlemma}\label{lem:sigergodicity1}
For $X_\ast\equiv (X_j)_{j\in\N}$ uniformly integrable and $n\in\N$, the following holds.
\begin{enumerate}[label=\upshape(\roman*)]
\item\label{lem:sigergodicity1:it1} If $X_\ast$ is $\alpha$-mixing and has $(m,n)$-stationary sigmoments $(m\in\N)$, then ${X_\ast}$ is $m^{\mathrm{th}}$-order weakly signature-ergodic to length $n$; 
\item\label{lem:sigergodicity1:it2} if $X_\ast$ is $\phi$-mixing with $\sum_{\nu=1}^\infty\phi_{1 + (\nu-1)n}^{1/2}(X_\ast)\tfrac{\log \nu}{\nu}<\infty$ and has $n$-seasonal increments, then $X_\ast$ is signature-ergodic to length $n$.       
\end{enumerate}
The assertions \ref{lem:sigergodicity1:it1} and \ref{lem:sigergodicity1:it2} persist if $X_\ast$ is replaced by $\theta(X_\ast)=(\theta(X_j))_{j\in\N}$ for any measurable $\theta:D_{X_\ast}\rightarrow\R^d$. 
\end{appendixlemma} 
\begin{proof}
Starting from definition \eqref{sect:sigmoments:signature}, a direct calculation yields that for any $\ell_1<\ell_2$,
\begin{equation}\label{lem:sigergodicity1:aux1}
\sig_m(\hat{\iota}_{\mathcal{E}}(X_{\ell_1}, \ldots, X_{\ell_2})) \ = \ \sum_{(i_1,\ldots,i_m)\in (\ell_1:\,\ell_2]^{\times m}}c_{i_1\cdots i_m}\cdot\Delta_{i_1}\otimes\cdots\otimes\Delta_{i_m}   
\end{equation}
for certain $c_{i_1\cdots i_m}\in\R$ and increments $\Delta_j\coloneqq X_j - X_{j-1}$, where $\mathcal{E}\equiv\mathcal{E}_{\ell_1,\ell_2}$ is the equidistant (or any other) $\I$-dissection of cardinality $\ell_2-\ell_1+1$. Let now $n\in\N$ be fixed. If we introduce the shift-map $\vartheta(i)\coloneqq i+n$ (with $\vartheta^0\coloneqq \mathrm{id}$ and $\vartheta^j\coloneqq\vartheta\circ\vartheta^{j-1}$) for convenience and denote 
\begin{equation}
Y_j \ \coloneqq \ \sig_m(\hat{\iota}_{\mathcal{E}}(\theta\cdot X_{n(j-1)+1}, \ldots, \theta\cdot X_{nj})) \qquad (j\in\N) 
\end{equation} 
for brevity, then the above shows that each $Y_j$ is a measurable function of the arguments $X_{\vartheta^{j-1}(1)}, \ldots, X_{\vartheta^{j-1}(n)}$. This in turn implies the inclusion of $\sigma$-algebras 
\begin{equation}
\mathscr{Y}_p^q\coloneqq\sigma(Y_p,\ldots,Y_q) \ \ \subseteq \ \ \sigma(X_{\vartheta^{p-1}(1)}, \ldots, X_{\vartheta^{p-1}(n)}, \ldots, X_{\vartheta^{q-1}(1)}, \ldots, X_{\vartheta^{q-1}(n)})   
\end{equation}
for any  $p\leq q$, whence in particular $\mathscr{Y}_1^j\subseteq\mathscr{X}_1^{\vartheta^{j-1}(n)}$ and $\mathscr{Y}_{j+\nu}^\infty\subseteq\mathscr{X}_{\vartheta^{j+\nu-1}(1)}^\infty$ for all $\nu\in\N$. Since $\vartheta^{j-1}(n)=jn$ and $\vartheta^{j+\nu-1}(1) = jn + \vartheta^{\nu-1}(1)$, we can use Definition \ref{appendixdef:mixing} to for $Y_{\ast}\coloneqq (Y_j)_{j\in\N}$  and $\gamma\in\{\alpha, \phi\}$ conclude that 
\begin{equation}
\gamma_\nu(Y_\ast) \ \leq \ \gamma_{\vartheta^{\nu-1}(1)}(X_\ast) =\gamma_{1 + (\nu-1)n}(X_\ast) \quad\text{ for each $\nu\in\N$},  
\end{equation}
which shows that if $X_\ast$ is $\alpha$-mixing ($\phi$-mixing) then so is $Y_\ast$. 

The proof of statement \ref{lem:sigergodicity1:it1} is finished by a coordinatewise application of the weak law of large numbers for non-stationary $\alpha$-mixing time series given in \citep[Theorem 7.15]{VDV}. 

As to \ref{lem:sigergodicity1:it2}, we note similarly that if $X_\ast$ has $n$-seasonal increments and is $\phi$-mixing at the assumed rate, then $Y_\ast$ is stationary (by \eqref{lem:sigergodicity1:aux1}) and $\phi$-mixing with 
\begin{equation}
\sum_{\nu=1}^\infty\phi^{1/2}_\nu(Y_\ast)\frac{\log\nu}{\nu} \ \leq \ \sum_{\nu=1}^\infty\phi^{1/2}_{1 + (\nu-1)n}(X_\ast)\frac{\log\nu}{\nu} \ < \ \infty\,,
\end{equation}
whence assertion \ref{lem:sigergodicity1:it2} follows from a coordinatewise application of \citep[Corollary 1]{KUC}. 

This proof of the statements \ref{lem:sigergodicity1:it1} and \ref{lem:sigergodicity1:it2} goes through without changes if the sequence $(X_j)_{j\in\N}$ is replaced by $(\theta\cdot X_j)_{j\in\N}$ for any (Borel-)measurable map $\theta:D_{X_\ast}\rightarrow\R^d$.      
\end{proof} 

\subsection{Complementary Remarks and Proofs for Theorem \ref{thm:consistency}}\label{pf:thm:consistency:add}
Throughout this subsection, the setting and notation from the proof of Theorem \ref{thm:consistency} (pp.\ \pageref{thm:consistency:aux1}) applies. 

\subsubsection{The Compact-Open Topology on $\Theta$ is Metrizable}\label{subsubsect:theta_metrizable} Since $D_X$ is a closed subset of $\R^d$, there are $\{K_\nu\}\subseteq D_X$ compact with $K_\nu\subseteq K_{\nu+1}$ and $D_X = \bigcup_{\nu\in\N_0}K_\nu$, and the topology of compact convergence on $C(D_X;\R^d)$ coincides with the compact-open topology on $C(D_X; \R^d)$, e.g.\ \citep[Theorem 46.8]{munkres2000}. Defining $\|\theta\|_K\coloneqq\sup_{u\in K}|\theta(u)|$, this topology is induced by the metric (see, e.g., \citep[Proposition VII.1.6]{conway1978complexvariable})
\begin{equation}\label{pf:thm:consistency:add:eq1}
\tilde{d}(\theta, \tilde{\theta})\,\coloneqq\, \sum_{\nu=0}^\infty 2^{-\nu}\rho_\nu(\theta,\tilde{\theta}) \quad \text{with} \quad \rho_\nu(\theta,\tilde{\theta})\coloneqq \frac{\|\theta-\tilde{\theta}\|_{K_\nu}}{1 + \|\theta - \tilde{\theta}\|_{K_\nu}}\,;   
\end{equation}    
we choose $K_\nu\coloneqq \overline{B_{r_\nu}(0)}$ for any $r_\nu\uparrow\infty$ monotonously with $r_0\coloneqq 0$ for convenience.\\[-0.5em]

Note that the metrics (on $C(K_\nu;\R^d)$) $\rho_\nu$ and $d_\nu(\theta,\tilde{\theta})\coloneqq\|\theta-\tilde{\theta}\|_{K_\nu}$ are equivalent for all $\nu\in\N_0$. Specifically, for each $\nu\in\N_0$ we have $d_\nu(\theta, \tilde{\theta}) \leq d_{\nu+1}(\theta, \tilde{\theta})$ for any $\theta, \tilde{\theta}\in\Theta$, and
\begin{equation}\label{pf:thm:consistency:add:eq1.1}
d_\nu(\theta, \tilde{\theta}) \, \leq \, 2\rho_\nu(\theta, \tilde{\theta}) \qquad \text{if}\quad \rho_\nu(\theta,\tilde{\theta})\leq\tfrac{1}{2}. 
\end{equation} 
For $\eta=d_\nu, \rho_\nu, \tilde{d}$, denote $B^\eta_r(\theta_\ast)\coloneqq\{\theta\in C(D_X;\R^d)\mid \eta(\theta, \theta_\ast) < r\}$ and $\tilde{B}_r(\theta_\ast)\coloneqq B_r^{\tilde{d}}(\theta_\ast)$.\\      

Below are the proofs of Theorem \ref{thm:consistency} for the cases $(X, \mathcal{J})$ ergodic resp.\ weakly ergodic.         
\subsubsection{Proof of Theorem \ref{thm:consistency} for Ergodic Observations}
Let $(\tilde{X}, \mathcal{J})$ be an ergodic observation of $X$, where now $D_X$ is not necessarily compact. In this case, Theorem \ref{thm:consistency} asserts that each $\varepsilon>0$ comes with a $\mathbb{P}$-full set $\tilde{\Omega}_\varepsilon\in\mathscr{F}$ such that for each $\omega\in\tilde{\Omega}_\varepsilon$ the following holds:     
\begin{equation}\label{pf:thm:consistency:add:eq2.1}
\begin{gathered}
\exists\, m_0\equiv m_0(\omega)\geq 2 \ : \ \forall\, m\geq m_0 \ \text{ there is } \ k_0\equiv k_0(m)\in\N \ \text{ s.t.\ } \ \forall\, k\geq k_0:\\
\lim_{\tau\rightarrow\infty}\max\left\{\sup_{T\geq\tau}\Big[\mathrm{dist}_{\|\cdot\|_\infty}\!\big(\hat{\theta}^\star_T( X(\omega)),\,\mathrm{DP}_d\cdot S(\omega)\big)\Big], \, \varepsilon\right\} \ = \ \varepsilon\\
\text{for any } \ (\theta^\star_T)_{T\in\N}\equiv(\theta^\star_T(m, k, \omega))_{T\in\N}\equiv(\theta^{m|k}_T(\omega))_{T\in\N}\subset\Theta \ \text{ as in \eqref{thm:consistency:eq1}.} 
\end{gathered}
\end{equation} 
The above proof of \eqref{thm:consistency:aux1}, which did not involve any compactness assumption on $D_X$, remains valid without any changes, so that \eqref{pf:thm:consistency:add:eq2.1} holds if it can be derived from \eqref{thm:consistency:aux1}.
 
To this end, let $\Omega''\in\mathscr{F}$ be the $\mathbb{P}$-full set on which the traces of $X$ are all contained in $D_X$ and \eqref{thm:optimisation:eq1} holds, and for each $n\in\N$ denote by $\Omega_n$ the $\mathbb{P}$-full set on which \eqref{thm:consistency:aux1} holds for $\tepsilon=\tfrac{1}{n}$. Set $\tilde{\Omega}\coloneqq\Omega''\cap\bigcap_{n\in\N}\Omega_n$ (another $\mathbb{P}$-full set) and let $\varepsilon>0$ be arbitrary. Take any $\omega\in\tilde{\Omega}$. Then $\mathrm{tr}(X(\omega))\subset K_{\nu_0}$ for some $\nu_0\in\N$, whence for any $n_0\in\N$ with $n_0^{-1}\leq 2^{-\nu_0}\varepsilon/4$ we have
\begin{equation}
\begin{gathered}
\exists\, m_0\equiv m_0(n_0)\geq 2 \ : \ \forall\, m\geq m_0 \ \text{ there is } \ k_0\equiv k_0(m)\in\N \ \text{ s.t.\ } \ \forall\, k\geq k_0: \\
\alpha^{m|k}(\omega) \, \leq \, n_0^{-1} \quad \text{and hence}\quad \sup\nolimits_{T\geq\tau_0}\big\|\theta^{m|k}_T(\omega)(X(\omega)) - \theta_T(X(\omega))\big\|_\infty \,\leq\, \varepsilon   
\end{gathered}
\end{equation}               
(for some $\tau_0\,(\equiv\tau_0(\omega))\in\N$ and some $(\theta_T)_{T\in\N}\,(\equiv(\theta_T(\omega))_{T\in\N})\subset\Theta_\star$) by the exact same argumentation that led us to \eqref{thm:consistency:aux2.2}. \hfill $\square$   

\subsubsection{Proof of Theorem \ref{thm:consistency} for Weakly Ergodic Observations} Let $(\tilde{X}, \mathcal{J})$ be a weakly ergodic observation of $X$. Adopting the setting and notation from pp.~\pageref{thm:consistency} f., suppose now that  
\begin{equation}\label{pf:thm:consistency:add:eq4}
\begin{gathered}
\forall\,\tepsilon>0\,:\, \exists\, m_0\geq 2\,:\, \text{for each } m\geq m_0 \text{ there is } k_0\equiv k_0(m) \text{ such that\,:}\\
\lim_{\tau\rightarrow\infty}\alpha_\tau^{m|k} \vee \tepsilon = \tepsilon \quad \text{in probability,\footnotemark \ \ \ for each } k\geq k_0.    
\end{gathered} 
\end{equation}\footnotetext{\ Remark that the (usual) notion of convergence in probability is well-defined for $\Theta$-valued random variables since the topology of compact convergence on $\Theta$ is metrizable, second-countable (e.g.\ \cite{mccoy1978}) and, hence, separable.}Spelled out, \eqref{pf:thm:consistency:add:eq4} implies that for any given $(\varepsilon', \delta')\in(0,\infty)^2$ and $(m,k)\,(\equiv(m,k)_{\tepsilon}$ as in \eqref{pf:thm:consistency:add:eq4} for $\tepsilon\coloneqq\varepsilon'/2$) 
\begin{equation}\label{pf:thm:consistency:add:eq4.1}
\text{there is } \ \ \tau_*\equiv\tau_*(\varepsilon',\delta')\in\N \quad \text{such that} \quad \sup\nolimits_{\tau\geq\tau_*}\mathbb{P}(\alpha_\tau^{m|k}\geq\varepsilon')\,\leq\,\delta'.
\end{equation} 
(Indeed: for any $m,k$ as in \eqref{pf:thm:consistency:add:eq4} with $\tepsilon\coloneqq\varepsilon'/2$, it holds $\mathbb{P}(\alpha^{m|k}_\tau\geq\varepsilon') \leq \mathbb{P}((\alpha^{m|k}_\tau\vee\tfrac{\varepsilon'}{2}) \geq \varepsilon') = \mathbb{P}(|(\alpha^{m|k}_\tau\vee\tfrac{\varepsilon'}{2}) - \tfrac{\varepsilon'}{2}|\geq\tfrac{\varepsilon'}{2})\rightarrow 0$ as $\tau\rightarrow\infty$.)             
In particular, for any given $\p\equiv(\varepsilon, \delta)\in(0,\infty)^2$ there will be $m_{\p}\in\N$ such that for every $m'\geq m_{\p}$ there is $k_\p\equiv k_\p(m')$ with the property that: for any $k'\geq k_\p$ there is $\tau_\p'\equiv\tau_\p'(k')\in\N$ with  
\begin{equation}\label{pf:thm:consistency:add:eq4.2}
\varrho_{\p} \, \coloneqq \, \sup\nolimits_{\tau\geq\tau_\p'}\mathbb{P}\big(\sup\nolimits_{T\geq\tau}\mathrm{dist}_{\|\cdot\|_\infty}\!\big(\theta^{m'|k'}_T(X), \Theta_\star\cdot X\big)\geq\varepsilon\big) \ \leq \ \delta,    
\end{equation}
which due to $\sup_{\tau\geq\tau_{\p'}}\mathbb{P}\big(\sup_{T\geq\tau}\mathrm{dist}_{\|\cdot\|_\infty}\!(\theta^{m'|k'}_T\!(X), \mathrm{DP}_d\cdot S)\geq\varepsilon\big)\leq\varrho_{\p}$ implies that the asserted convergence \eqref{thm:consistency:eq2} holds in probability. To see that \eqref{pf:thm:consistency:add:eq4.2} holds, fix $\varepsilon, \delta>0$ and note
\begin{equation}\label{pf:thm:consistency:add:eq5}
\Big\{\sup\nolimits_{T\geq\tau}\mathrm{dist}_{\|\cdot\|_\infty}\!\big(\theta_T(X), \Theta_\star\cdot X\big)\geq\varepsilon\Big\}\cap\Omega' \ \ \subseteq \ \ \bigcup\nolimits_{\nu\in\N}A_\nu^{\hat{\theta}_\tau}\cap B_\nu 
\end{equation}   
for any given sequence $\hat{\theta}\equiv(\theta_T)$ of $\Theta$-valued random variables, $\tau\in\N$, and for the events\footnote{\ As the functions $\varphi_\theta : \tilde{\theta}\mapsto d_\nu(\tilde{\theta},\theta)$ are continuous, their infimum $\varphi(\tilde{\theta})\coloneqq\inf_{\theta\in\Theta_\star}\varphi_\theta(\tilde{\theta}) = d_\nu(\tilde{\theta},\Theta_\ast)$ is upper semicontinuous and hence Borel-measurable, whence the sets $A_\nu^{\hat{\theta}_\tau} = \{\sup_{T\geq\tau}\varphi(\theta_T)\geq\varepsilon\}$ are measurable. As $X$ has continuous realisations, we have $\sup_{t\in\I}|X_t| = \sup_{t\in\I\cap\mathbb{Q}}|X_t|$ so that $B_\nu^\theta$ is measurable.} $A^{\hat{\theta}_\tau}_\nu\coloneqq\big\{\sup_{T\geq\tau}d_\nu(\theta_T, \Theta_\star)\geq\varepsilon\big\}$ and $B_\nu\coloneqq\{\sup_{t\in\I}|X_t|\geq r_{\nu-1}\}$, where $r_\nu$ denotes the radius of the $0$-centered closed ball $K_\nu$. Noting that $A^{\hat{\theta}_\tau}_\nu \subseteq A^{\hat{\theta}_\tau}_{\nu+1}$ and $B_{\nu+1}\subseteq B_\nu$ for all $\nu\in\N$, we from  \eqref{pf:thm:consistency:add:eq5} obtain that        
\begin{equation}\label{pf:thm:consistency:add:eq6}
\mathbb{P}\big(\sup\nolimits_{T\geq\tau}\mathrm{dist}_{\|\cdot\|_\infty}\!\big(\theta_T(X), \Theta_\star\cdot X\big)\geq\varepsilon\big) \ \leq \ \mathbb{P}(A_{\nu_0}^{\hat{\theta}_\tau}) \ + \ \mathbb{P}(B_{\nu_0+1}) 
\end{equation}  
for any fixed $\nu_0\in\N$. Denoting $\mu_X\coloneqq\E[\sup_{t\in\I}|X_t|]$, Markov's inequality implies that 
\begin{equation}\label{pf:thm:consistency:add:eq7}
\mathbb{P}(B_{\nu_0+1}) \ \leq \ \frac{\mu_X}{r_{\nu_0}} \quad \longrightarrow \quad 0 \qquad (\nu_0\rightarrow\infty), 
\end{equation}  
while \eqref{pf:thm:consistency:add:eq1.1} implies $B^{d_\nu}_\varepsilon(\theta) \supseteq B^{\rho_\nu}_{\varepsilon/2}(\theta)$ for each $\theta\in\Theta$ (if $\varepsilon<1$, assumable wlog) and hence
\begin{equation}\label{pf:thm:consistency:add:eq8}
\mathbb{P}(A_{\nu_0}^{\hat{\theta}_\tau}) \ \leq \ \mathbb{P}(\sup\nolimits_{T\geq\tau}\rho_{\nu_0}(\theta_T, \Theta_\star) \geq \varepsilon/2) \ \leq \ \mathbb{P}(\sup\nolimits_{T\geq\tau}\tilde{d}(\theta_T, \Theta_\ast)\geq 2^{-\nu_0}\varepsilon/2).
\end{equation}     
Given \eqref{pf:thm:consistency:add:eq7} and \eqref{pf:thm:consistency:add:eq8}, we may now fix an $\nu_0\in\N$ large enough such that $\mathbb{P}(B_{\nu_0+1})\leq \delta/2$, and for this choice of $\nu_0$ obtain an $m_\p\in\N$, as guaranteed by \eqref{pf:thm:consistency:add:eq4} for $\tepsilon=\tepsilon_\star$ with $\tepsilon_\star\coloneqq 2^{-\nu_0}\varepsilon/4$, such that for every $m\geq m_\p$ there is $k_\p\,(\equiv k_\p(m))$ with the property that: for any $k\geq k_\p$ there is $\tau_\ast\equiv\tau_\ast(2\tepsilon_\star, \delta/2)\in\N$, as guaranteed by \eqref{pf:thm:consistency:add:eq4.1}, such that 
\begin{equation}
\sup\nolimits_{\tau\geq\tau_\ast}\!\mathbb{P}(A^{\hat{\theta}_\tau}_{\nu_0}) \stackrel{\eqref{pf:thm:consistency:add:eq8}}{\leq} \sup\nolimits_{\tau\geq\tau_\ast}\!\mathbb{P}(\alpha_\tau^{m|k}\geq \tepsilon_\star) \ \leq \ \delta/2 \qquad\text{for } \ \hat{\theta}=(\theta^{m|k}_T).    
\end{equation} 
Taken altogether, the estimate \eqref{pf:thm:consistency:add:eq6} then allows us to conclude that  
\begin{equation}
\sup\nolimits_{\tau\geq\tau_\ast}\mathbb{P}(\sup\nolimits_{T\geq\tau}\mathrm{dist}_{\|\cdot\|_\infty}\!\big(\theta_T(X), \Theta_\star\cdot X\big)\geq\varepsilon) \ \leq \ \delta/2 + \delta/2 \ \leq \ \delta,
\end{equation}
which (via \eqref{thm:consistency:aux_new2} and Thm.\ \ref{thm:optimisation}) yields the desired conclusion \eqref{thm:consistency:eq2} for the weakly ergodic case. 

It hence remains to prove \eqref{pf:thm:consistency:add:eq4}, for which we may follow the previous lines of pp.~\pageref{thm:consistency:aux3} with only slight adaptations. Indeed: Since in the weakly ergodic case the $\Theta$-uniform estimator convergence \eqref{lem:ergodicity_uniformconv:eq2} holds in probability, we obtain -- by way of the very same argumentation as for \eqref{thm:consistency:aux7} -- that
\begin{equation}\label{pf:thm:consistency:add:eq11}
\lim_{T\rightarrow\infty}\bar{\kappa}_{m,k}(\theta^\star_T) \, = \, 0 \quad \text{in probability}, \quad \text{with} \ \ (\theta^\star_T)\equiv(\theta^{m|k}_T)
\end{equation}
as in \eqref{thm:consistency:eq1} for $m,k$ as in \eqref{thm:consistency:aux3} for some (arbitrary but) fixed $\tepsilon>0$. From this we obtain that in the present context, the convergence \eqref{thm:consistency:aux6} holds in probability. Indeed: Assuming otherwise implies the existence of $\varepsilon_0, \delta_0>0$ such that
\begin{equation}\label{pf:thm:consistency:add:eq12}
\mathbb{P}(\mathrm{dist}(\theta^\star_{T_j}, \mathcal{M})\geq \varepsilon_0) \,\geq\, \delta_0 \quad \text{ for each $j\in\N$}, 
\end{equation}
for some sequence $(T_j)_{j\in\N}\subset\N$. As a subsequence of $(\bar{\kappa}_{m,k}(\theta^\star_T))_{T\in\N}$, we by way of \eqref{pf:thm:consistency:add:eq11} find that $(\bar{\kappa}_{m,k}(\theta^\star_{T_j}))_{j\in\N}$ converges to $0$ in probability, whence there is yet another subsequence $(T_{j_\ell})_{\ell\in\N}$ of $(T_j)_{j\in\N}$ such that $\lim_{\ell\rightarrow\infty}\bar{\kappa}_{m,k}(\theta^\star_{T_{j_\ell}})=0$ almost surely. Applying the (essentially) same argument which brought `$\eqref{thm:consistency:aux7} \Rightarrow\eqref{thm:consistency:aux6}$' now yields that $\lim_{\ell\rightarrow\infty}\mathrm{dist}(\theta^\star_{T_{j_\ell}}, \mathcal{M})=0$ almost surely and hence in probability, contradicting \eqref{pf:thm:consistency:add:eq12}.              

\noindent
As this proves $\lim_{\tau\rightarrow\infty}\sup\nolimits_{T\geq\tau}\mathrm{dist}(\theta^{m|k}_T\!,\, \mathcal{M})=0$ in probability, we for any $\epsilon>0$ obtain 
\begin{equation}
\begin{aligned}
\mathbb{P}\big(|\alpha^{m|k}_\tau\vee\tepsilon - \tepsilon|\geq\epsilon\big) \ \leq \ \mathbb{P}\big(\alpha^{m|k}_\tau\geq\tepsilon\big) \ \leq \ \mathbb{P}\big(\sup\nolimits_{T\geq\tau}\mathrm{dist}(\theta^{m|k}_T\!,\, \mathcal{M})\geq \tepsilon/2\big) \ \rightarrow \ 0  
\end{aligned}  
\end{equation}    
as $\tau\rightarrow\infty$, where the last inequality is due to \eqref{thm:consistency:aux5}. This shows \eqref{pf:thm:consistency:add:eq4} as required.
\hfill $\square$

\section{A `Moment-Like' Coordinate Description for the Laws of Stochastic Processes}\label{sect:expected_signature_moments}
\subsection{The Expected Signature: A Coordinate Vector for Stochastic Processes}\label{sect:subsect:sigmoments}Many results in statistics, including Corollary \ref{cor:Comon} via \eqref{intext:Comon:CF}, are based on the well-known fact that the distribution of a random vector $Z=(Z^1, \cdots, Z^d)$ in $\R^d$ can be characterised by a set of coordinates with respect to a basis of nonlinear functionals on $\R^d$. More specifically, any such vector $Z$ can be assigned its `moment coordinates' $\left(\mathfrak{m}_{\bm{i}}(Z)\right)_{\bm{i}\in[d]^\star}\subset\bar{\R}$ defined by 
\begin{align}\label{sect:classicmoments:eq2}
\mathfrak{m}_{i_1\cdots i_m}(Z) \ &\coloneqq \ \mathbb{E}\left[Z^{i_1}\cdots Z^{i_m}\right] \ = \ \int_{\R^d}\!x_{i_1}\cdots x_{i_m}\,\mathbb{P}_Z(\mathrm{d}x). 
\end{align}
As the linear span of the monomials $\{x_{\bm{i}}\equiv x_{i_1}\cdots x_{i_m}\mid \bm{i}\equiv(i_1,\ldots,i_m)\in[d]^\star\}$ is uniformly dense in the spaces of continuous functions over compact subsets of $\R^d$, the coordinatisation  \eqref{sect:classicmoments:eq2} is faithful in the sense that, under certain conditions \cite{KLS}, the (coefficients of) the moment vector $\mathfrak{m}(Z)\equiv(m_{\bm{i}}(Z))_{\bm{i}\in[d]^\star}$ determine the distribution of $Z$ uniquely.\\     

\noindent
Now, if instead of a random vector in $\R^d$ one seeks to find a convenient coordinatisation for the distribution of a stochastic process $Y$ in $\R^d$, i.e.\ a random path in $\mathcal{C}_d$, then one can -- perhaps surprisingly -- resort to a natural generalisation of \eqref{sect:classicmoments:eq2}, which is known as the expected signature of $Y$: Analogous to how the monomials $\{x_{\bm{i}}\mid \bm{i}\in[d]^\star\}$ are a basis\footnote{\ Cf.\ the trivial fact the monomials $x_1=\langle\cdot\,, e_1\rangle, \ldots, x_d=\langle \cdot\,,e_d\rangle$ determine each vector in $\R^d$ uniquely.} of nonlinear functionals on $\R^d$ that provides coordinates $(\mathfrak{m}_{i_1\cdots i_m}\mid\eqref{sect:classicmoments:eq2})\subset\bar{\R}$ for a random vector in $\R^d$, there is a basis $\{\chi_{i_1\cdots i_m} \mid (i_1,\ldots, i_m)\in[d]^\star\}$ of nonlinear functionals on (regular enough subspaces of) $\mathcal{C}_d$ which provides coordinates $(\sigma_{\bm{i}})_{\bm{i}\in[d]^\star}\subset\bar{\R}$ for (the law of) a random path $Y$ in $\mathcal{C}^d$. This path-space basis is defined as follows:\\[-0.5em] 

\noindent
Given a path $x=(x^1_t, \cdots, x^d_t)_{t\in\I}\in\mathcal{C}_d$ of bounded variation in $\R^d$ (assuming $\I=[0,1]$ wlog), consider the noncommutative moments of $x$, that is the iterated Stieltjes-integrals
\begin{equation}\label{sect:sigmoments:eq1}
\chi_{i_1\cdots i_m}(x) \ \coloneqq \ \int_{0\leq t_1 \leq t_2 \leq \cdots \leq t_m \leq 1}\mathrm{d}x^{i_1}_{t_1}\mathrm{d}x^{i_2}_{t_2}\cdots\mathrm{d}x^{i_m}_{t_m}, \qquad (i_1,\ldots, i_m)\in[d]^\star, 
\end{equation}  
of $(x^{i_1}, \cdots, x^{i_m})\in\mathcal{C}_m$ over the standard $m$-simplex $\{(t_1,\ldots,t_m)\in\I^m\mid t_1\leq\cdots\leq t_m\}$ (for $\epsilon$ the empty index in $[d]^\star$, we set $\chi_\epsilon\equiv 1$).\footnote{\ For $x$ defined on a general compact interval $\I\subset\R$, set $\chi_{i_1\cdots i_m}(x)\coloneqq |\I|^{-m}\int_{\Delta_m(\I)}\mathrm{d}x^{i_1}_{t_1}\mathrm{d}x^{i_2}_{t_2}\cdots\mathrm{d}x^{i_m}_{t_m}$, with the $m$-simplex $\Delta_m(\I)$ over $\I$ defined as above.} Then, the nonlinear functionals $x \mapsto \chi_{i_1\cdots i_m}(x)$ define a dual basis for the vector $x\in\mathcal{C}_d$, in the sense that the coefficients $(\chi_{i_1\cdots i_m}(x)\mid\eqref{sect:sigmoments:eq1})$ determine the path $x$ uniquely; see \cite{chen1977iterated,fliess1981, HLY}.\footnote{\ Up to a negligible indeterminacy known as `tree-like equivalence', see \cite{HLY}.} The resulting family of coordinates
\begin{equation}\label{sect:sigmoments:signature}
\mathfrak{sig}(x) \ \coloneqq \ \big(\chi_{i_1\cdots i_m}(x) \ \big| \ (i_1,\ldots,i_d)\in[d]^\star\big)
\end{equation}
is known as the \emph{signature} of the path $x$.   

\begin{remark}\label{rem:sig_basis}
Similarly still to the monomial dual basis $\{x_{i_1}\cdots x_{i_m}\}$ on $\R^d$, the linear span of the above functionals $\{\chi_{i_1\cdots i_m}\mid\eqref{sect:sigmoments:eq1}\}$ is closed under pointwise multiplication and hence forms an algebra over the space of applicable paths in $\mathcal{C}_d$, from which one obtains that their linear span is uniformly dense in the space of continuous functions over (certain) compact subsets of $\mathcal{C}_d$ (Stone-Weierstrass), see e.g.\ \citep[Thm.\ 2.15]{FLO} (and Lemma \ref{sect:sigmoments:lem1} \ref{sect:sigmoments:lem1:it6}).
\end{remark} 
\noindent
As a consequence of Remark \ref{rem:sig_basis},\footnote{\ Recall that by Riesz representation theorem, a (signed) Borel measure on a compact metric space $K$ acts as a continuous linear functional over the space $C(K)$ of continuous functions on $K$ and is hence uniquely determined by its (dual) functional action on a dense subset of $C(K)$.} one can infer in analogy to \eqref{sect:classicmoments:eq2} that the dual coefficients 
\begin{equation} 
\begin{aligned}\label{sect:sigmoments:eq2}
\sigma_{i_1\cdots i_m}(Y) \ &\coloneqq \ \int_{\mathcal{C}_d}\!\chi_{i_1\cdots i_m}(x)\,\mathbb{P}_Y(\mathrm{d}x) \qquad \qquad (i_1,\ldots, i_m\in[d], \ m\geq 0) \\
 &\phantom{:}= \ \mathbb{E}\!\left[\int_{{0\leq t_1 \leq t_2 \leq \cdots \leq t_m \leq 1}}\!\mathrm{d}Y^{i_1}_{t_1}\mathrm{d}Y^{i_2}_{t_2}\cdots\mathrm{d}Y^{i_m}_{t_m}\right]
\end{aligned}
\end{equation} 
define a complete set of coordinates for the distribution of a stochastic process $Y=(Y^1_t, \cdots, Y^d_t)_{t\in\I}$ in $\R^d$ that has compact support (and sample paths of bounded variation).\\[-0.5em]

\noindent
The signature-based coordinatisation \eqref{sect:sigmoments:eq2} of a random path in $\mathcal{C}_d$ can thus be regarded as a natural generalisation of the moment-based coordinatisation \eqref{sect:classicmoments:eq2} of a random vector in $\R^d$.\\[-0.5em]

\noindent
The assumption of compact support is of course much too restrictive on a non-locally compact space like $\mathcal{C}_d$, but under additional decay conditions \cite{CHL} or by using a normalization \citep[Theorem 5.6]{CHO} it can be shown that the coefficients $(\sigma_{i_1\cdots i_k}(Y)\mid \eqref{sect:sigmoments:eq2})$ indeed characterize the distribution of $Y$ uniquely even if the compactness assumption is dropped. Extending the definition of \eqref{sect:sigmoments:signature} to paths less regular (`rougher') than of bounded variation is at the centre of the Theory of Rough Paths (\citep{FVI, FLO, LYQ}).\\  

The first application of the coordinates \eqref{sect:sigmoments:eq2} in statistics was given in \cite{PVA} for SDE parameter estimation, with more recent applications including the development of non-commutative cumulants \cite{bonnier2019signature} and Hurst parameter estimation \cite{diehl2016pathwise}. 

\subsection{A Coordinate Space for the Laws of Stochastic Processes}\label{rem:sig_cumulants_generalise} In order to make the information provided by \eqref{sect:sigmoments:eq1} and \eqref{sect:sigmoments:eq2} amenable to mathematical analysis, it will be convenient to regard $\mathfrak{sig}(x)$ and $(\sigma(Y)_{\bm{i}}\mid \bm{i}\in[d]^\star)$ as elements of a suitable topological space.\\[-0.5em]

\noindent
To this end, we denote by $[d]^\ast$ the free monoid on the alphabet $[d]=\{1,\ldots, d\}$, and identify each multiindex $(i_1,\cdots,i_m)\in[d]^{\star}$ in \eqref{def:expected_signature:eq1} with the \emph{word} $\texttt{i}_1\cdots\texttt{i}_m\in[d]^\ast$ it defines.\footnote{\ In light of this, the set $[d]^\star_+$ from Notation \ref{notation:index_sum} is an additive subgroup of the free algebra over $[d]^\ast$.}\\[-0.5em] 

From this view, both $\mathfrak{sig}(x)$ and $\mathfrak{S}(Y)\equiv(\sigma_{\bm{i}})_{\bm{i}\in[d]^\star}$ can then be treated as \emph{formal power series in the variables $\{1,\ldots, d\}$}, i.e.\ as elements of the \emph{free algebra} 
\begin{equation}\label{rem:free_algebra}
\R[d]^\ast\coloneqq\{\bm{t} : [d]^\ast\rightarrow\R\mid \bm{t}\text{ is a map}\}\equiv\left\{\sum\nolimits_{w\in[d]^\ast}\bm{t}(w)\cdot w\ \middle| \ \mathfrak{t}\in\R[d]^\ast\right\};
\end{equation}indeed: $\mathfrak{S}(Y) \cong \sum_{w\in[d]^\star}\bm{t}_\sigma(w)\cdot w \in\R[d]^\ast$ with $\bm{t}_\sigma(\texttt{i}_1\cdots\texttt{i}_m)\coloneqq \sigma_{i_1\cdots i_m}$. For convenience, we may henceforth write $\bm{t}(w)\eqqcolon\langle\bm{t}, w\rangle$ $(w\in[d]^\star)$ for a given $\bm{t}\in\R[d]^\star$.\\[-0.5em]

\noindent 
The space $\R[d]^\ast$ thus serves as a graded coordinate space for (the laws of applicable) continuous stochastic processes in $\R^d$.

\begin{remark}\label{rem:hopf}
The coordinate space $\R[d]^\ast$ is not just an $\R$-vector space but a twofold \emph{bialgebra} (in fact: a \emph{Hopf algebra}), namely wrt.\ the two multiplications given by (a) the concatenation product $\ast$ (the $\R$-bilinear extension of the word-concatenation on $[d]^\ast$), and (b) the shuffle product $\shuffle$ from \eqref{shuffle_product}, see \citep[pp.\ 29 and 31]{REU} for details. The bi-algebra structures associated to these two products are an algebraic reflection of the duality between \eqref{sect:classicmoments:eq2} and \eqref{sect:sigmoments:eq1}, cf.\ also \eqref{intext:coords_indep_class}.
\end{remark}

\subsubsection{The Log Transform}\label{sect:sigmoments:logtransform}Accordingly, the expected signature $\mathfrak{S}(Y)$ of $Y$ can be seen as a coordinate vector of $Y$ wrt.\ the monomial standard basis $\mathfrak{B}\coloneqq\{\texttt{i}_1\cdots\texttt{i}_k\mid i_1, \ldots, i_k\in[d], \ k\geq 0\}$ of $\R[d]^\ast$.
The vector $\mathfrak{S}(Y)$ itself, however, is contained in a nonlinear subspace of $\R[d]^\star$; more specifically, $\mathfrak{S}(Y)$ is `close to an exponential'.\footnote{\ Algebraically, $\mathfrak{S}(Y)$ lies in the convex hull of the Lie-group $\{\sum_{w\in[d]^\star}\chi_w(x)\cdot w\mid x\in\mathcal{BV}\}\subset\R[d]^\ast$, where $\mathcal{BV}\coloneqq\{x \in\mathcal{C}_d\mid \|x\|_{1\mathrm{\text{-}var}} < \infty\}$ and $\chi_w : \mathcal{BV}\rightarrow\R$ are the functionals in \eqref{sect:sigmoments:eq1} (e.g.\ \citep[Cor.\ 3.5]{REU}).}\\[-0.5em]

\noindent
It is hence reasonable to expect a more parsimonious coordinatisation of $Y$ wrt.\ $\mathfrak{B}$ to be achieved by, instead of the vector $\mathfrak{S}(Y)$, considering the $\mathfrak{B}$-coordinates of the faithful linearisation $\Phi(\mathfrak{S}(Y))$ that is effected by the log-transform $\Phi(\bm{t})\equiv\log(\bm{t})$ defined by  
\begin{equation}\label{rem:log_transform} 
\log(\bm{t})\coloneqq\sum_{m\geq 1}\frac{(-1)^{m-1}}{m}(\bm{t}-\epsilon)^{\ast m}
\end{equation}for $\bm{t}\in\R[d]^\ast$ with $\langle \bm{t}, \epsilon\rangle =1$; this linearised coordinate description \mbox{is accounted for by Def.\ \ref{def:sigcumulant}.}
It is sometimes convenient or required to instead of the infinite series \eqref{rem:log_transform} consider only one of its (finite) partial sums; a canonical way to achieve this is via truncation, see Remark \ref{rem:coordspace_quotient}.      
 
\begin{remark}[Signature Cumulants Generalise Classical Cumulants]\label{rem:classic_cumulants}In the same way that the expected signature generalises the classical concept of moments, the signature cumulant generalises the classical concept of cumulants from vector-valued to path-valued random variables, cf.\ \cite{bonnier2019signature}: The (classical) cumulants of a random vector $Z$ in $\R^d$ read\footnote{\ For $\pi_{\mathrm{Sym}}(\texttt{i}_1\cdots\texttt{i}_m) \coloneqq \sum_{\tau\in S_m}\texttt{i}_{\tau(1)}\cdots\texttt{i}_{\tau(m)}$ the projection onto the ($[d]$-adic closure of) the subspace spanned by all symmetric polynomials in $\R[d]^\ast$.} 
\begin{equation}\label{rem:classic_cumulants:eq1}
\kappa_{i_1\cdots i_m}^Z \, = \, \langle\pi_{\mathrm{Sym}}(\log[\mathfrak{m}(Z)]), \, \texttt{i}_1\cdots\texttt{i}_m\rangle
\end{equation}
and hence are identical to the signature cumulant of the linear process $Y\coloneqq (Z\cdot t)_{t\in[0,1]}$. 

\noindent
Given this relation between \eqref{def:sigcumulant:eq1} and \eqref{rem:classic_cumulants:eq1}, Proposition \ref{prop:sig_cums} appears as a natural generalisation of the well-known fact that the (classical) cumulant relations \eqref{intext:coords_indep_class}, that is 
\begin{equation}\label{intext:coords_indep_class2}
\kappa_{\tilde{w}}^Z \ = \ 0  \quad \text{ for all } \quad \tilde{w}\in\bigsqcup_{k=2}^d\big\{\tilde{u}\ast\tilde{v} \, \mid \, \tilde{u}\in [k-1]^\ast\setminus\{\epsilon\}, \ \tilde{v}\in\{k\}^\ast\setminus\{\epsilon\}\big\},  
\end{equation}are characteristic of a random vector $Z$ in $\R^d$ to have mutually independent components.  
\end{remark}     

\subsubsection{The Coordinate Space and Its Topology}\label{rem:sig_cumulants_generalise:subsect:coordspace}As of yet, the coordinate space \eqref{rem:free_algebra} provides only a `purely algebraic container' for the coordinate tuples \eqref{sect:sigmoments:signature} and \eqref{sect:sigmoments:eq2}. Statistical analysis, however, typically concerns convergence and thus requires a topology. 

A convenient such topology on $\R[d]^\ast$ can be defined by 
\begin{equation}\label{rem:freealg_identify_tensor}
\text{identifying } \quad \R[d]^\ast \quad \text{with the tensor algebra} \quad V^\infty\coloneqq\prod_{m=0}^\infty V_m   
\end{equation}    
where $V_0\coloneqq\R$ and $V_m\coloneqq V_1^{\otimes m}$ for $V_1\equiv(\R^d, |\cdot|_2)$, via $[d]^\ast\ni \texttt{i}_1\cdots\texttt{i}_m \, \leftrightarrow \, e_{i_1}\otimes\cdots\otimes e_{i_m}\in V_m$ and $\epsilon \,\leftrightarrow \, 1\in\R$ (with $(e_i)_{i\in[d]}$ the standard basis in $V_1$); in other words, we identify the free algebra \eqref{rem:free_algebra} with the Cartesian product $V^\infty$ which we then endow with its natural tensor algebra structure with $1$ (cf.\ e.g.\ \citep[Sect.\ 2.2.1, Rem.\ 1.24 f.]{FLO} for details). Denote by $\|\cdot\|_m$ the Euclidean (i.e., $|\cdot|_2$-induced) tensor norm on $V_m$, and write $\pi_m : V^\infty\rightarrow V_m \ (\hookrightarrow V)$, $\pi_m((v_j)_{j\geq 0}) = v_m$, for the canonical projection of $V^\infty$ onto its $m^{\mathrm{th}}$ factor. Let further $V_{[m]}\coloneqq\prod_{\nu=0}^m V_\nu$ be the truncated tensor algebra, and $\pi_{[m]}\coloneqq\sum_{\nu=0}^m\pi_\nu$ the truncation map.

\begin{remark}[Truncation]\label{rem:coordspace_quotient}
Notice that $V_{[m]}$ comes equipped with a natural algebra structure, namely the one realised as the quotient of $V^\infty$ by the ideal $\prod_{\nu>m} V_\nu$; the map $\pi_{[m]}$ is then the canonical quotient epimorphism. Note in particular  that
\begin{equation}\label{rem:coordspace_quotient:eq1}
\pi_{[m]}\big(\!\log(\bm{t})\big) = \sum_{k=1}^m\frac{(-1)^{k-1}}{k}\big[\pi_{[m]}(\bm{t}-\epsilon)\big]^{\underline{\ast} k}\eqqcolon \log_{[m]}(\bm{t})\,,
\end{equation} 
defining a (bijective) polynomial map $\log_{[m]} : V_{(1)}\rightarrow V_{[m]}$ $(\hookrightarrow V^\infty$; the space $V_{[m]}$ is embedded as a (closed) linear subspace of $V^\infty$ but not as a subalgebra). In the above, $\underline{\ast}$ denotes the multiplication in the algebra $V_{[m]}$, i.e.: $\displaystyle\pi_{[m]}(\bm{t}_1)\,\underline{\ast}\,\pi_{[m]}(\bm{t}_2)\stackrel{\mathrm{def}}{=}\pi_{[m]}(\bm{t}_1\ast\bm{t}_2)$, $\forall\, \bm{t}_1, \bm{t}_2\in V^\infty$.   
\end{remark}    

\noindent
Our topological coordinate space (for (random) paths and their laws) is 
\begin{equation}\label{rem:topcoordspace}
V \ \coloneqq \ \left\{\bm{t}\in V^\infty \ \middle| \ \vertiii{\bm{t}}_\lambda\coloneqq{\textstyle\sum\nolimits_{m\geq 0}}\|\pi_m(\bm{t})\|_m\cdot\lambda^m \, < \, \infty, \ \forall\,\lambda>0 \right\}
\end{equation}
equipped with the locally convex topology induced by the (fundamental) family of norms $(\vertiii{\cdot}_\lambda\mid\lambda>0)$; cf.\ \citep[Section 2]{CHL}, where the locally $m$-convex algebra \eqref{rem:topcoordspace} was first introduced to the analysis of signatures and their expectation. (Note that the subspace topology on $V_m\subset V$ coincides with the (Euclidean) topology on $(V_m, \|\cdot\|_m)$.) The factorial decay 
\begin{equation}\label{rem:coordspace:factodecay}
\big|\chi_{i_1\cdots i_m}(x)\big| \ \lesssim \ \|x\|^m_{1\text{-var}}/m!
\end{equation}
of the functionals \eqref{sect:sigmoments:eq1} implies that $\mathfrak{sig}(x), \mathfrak{S}(Y)\in V$, cf.\ also Lemma \ref{sect:sigmoments:lem1} below.\\[-0.5em]  

\noindent
For convenience, we also introduce the dilation maps
\begin{equation}
\delta_\lambda \ : \ V \rightarrow V, \quad (v_m)_{m\geq 0} \, \mapsto \, (\lambda^m\cdot v_m)_{m\geq 0}, \qquad (\lambda>0) 
\end{equation}
as well as the subspaces $V_{(c)}\coloneqq\{\bm{t}\in V \mid \pi_0(\bm{t})=c\}$ and $V^\infty_{(c)}$ (defined analogously), and recall that the space $\mathcal{BV}\coloneqq\mathcal{C}_d\cap\mathrm{BV}$ of continuous $\R^d$-valued paths of bounded variation can be endowed with the \emph{$p$-variation topology} (any $p\geq 1$) defined via the $p$-variation seminorm 
\begin{equation}\label{rem:p-varseminorm}
\|x\|_{p\text{-$\mathrm{var}$}}\ \coloneqq \ \Bigg[\sup_{\mathcal{D}}\sum_{(t_k)\in\mathcal{D}}\big|x_{t_k} - x_{t_{k-1}}\big|^p\Bigg]^{\!1/p} 
\end{equation}where the sup is taken over the set $\mathcal{D}$ of all dissections of $[0,1]$; e.g.\ \citep[Sect.\ 1.2]{FLO} for details. Let further $\mathcal{T}$ denote the subset of tree-like paths in $\R^d$, see e.g.\ .\\[-0.5em] 

The next lemma collects basic facts on $V$, \eqref{sect:sigmoments:signature} and \eqref{rem:log_transform} that are useful for Section \ref{sect:consistency}.   

\begin{appendixlemma}\label{sect:sigmoments:lem1}
Let $V$ and $\mathfrak{sig}$ and $\log$ be as above, and $\rho>1$. Then the following holds: 
\begin{enumerate}[label=\upshape(\roman*)]
\item\label{sect:sigmoments:lem1:it1} the space $V$ is a separable and metrizable Hausdorff space;
\item\label{sect:sigmoments:lem1:it2} the signature transform $x \mapsto \mathfrak{sig}(x)$ defines a map $\mathfrak{sig}:\mathcal{BV}\rightarrow V$ which for any $1\leq p < 2$ is continuous wrt.\ the $p$-variation topology on $\mathcal{BV}$;
\item\label{sect:sigmoments:lem1:it2.1} the signature is invariant under order-preserving time-domain reparametrisations of its arguments, i.e.\ $\sig(x) = \sig(x_\varphi)$ for $x_\varphi\equiv(x_{\varphi(t)})_{t\in\mathbb{J}}$ with $\varphi\in C(\mathbb{J};\I)$ strictly monotone; 
\item\label{sect:sigmoments:lem1:it6} for each $\varphi\in C(\mathcal{K})$ with $\mathcal{K}/\mathcal{T}\subset\mathcal{BV}$ compact, there is a sequence of index-polynomials $(\bm{\ell}_j)_{j\in\N}$ in $V^\circ\coloneqq\bigoplus_{m\geq 0}V_m\subset V^\infty$ such that $\varphi = \lim_{j\rightarrow\infty}\langle\sig(\cdot),\,\bm{\ell}_j\rangle$ wrt.\ $\|\cdot\|_\infty$; 
\item\label{sect:sigmoments:lem1:it3} the capped logarithm $\log_{[m]}: V_{(1)}\rightarrow V$ from \eqref{rem:coordspace_quotient:eq1} satisfies $\log_{[m]} = \log_{[m]}\circ\,\pi_{[m]}$ and is continuous for each $m\geq 0$, and $\log$ from \eqref{rem:log_transform} maps subsets of $\{\bm{t}\in V_{(1)} \mid \vertiii{\bm{t}-1}_\lambda\leq 1\}$ to subsets of $\{\bm{\ell}\in V^\infty\mid \vertiii{\bm{\ell}}_\rho \leq \sum_{m\geq 0}(2\rho/\lambda)^m\}$ for any $\lambda>2\rho$;
\item\label{sect:sigmoments:lem1:it4} for each $m\geq 0$, it holds that on $\vertiii{\cdot}_\rho$-bounded subsets the projections $\pi_{[m]} : V^\infty \rightarrow V^\infty$ converge uniformly wrt.\ $\vertiii{\cdot}_1$ to the identity operator on $V^\infty$;
\item\label{sect:sigmoments:lem1:it5} for each $\lambda>0$, we have that $\delta_\lambda\circ\log = \log\circ\,\delta_\lambda$ and $\delta_\lambda[\mathfrak{sig}(x)] = \mathfrak{sig}(\lambda\cdot x)$, any $x\in\mathcal{BV}$.             
\end{enumerate} 
\end{appendixlemma}  
\begin{proof}
Statements \ref{sect:sigmoments:lem1:it1}, \ref{sect:sigmoments:lem1:it2} and \ref{sect:sigmoments:lem1:it2.1} are well-known, see e.g.\ \citep[Cor.\ 2.4 and Cor.\ 5.5]{CHL} and \citep[Thm.\ 4]{HLY} (or simply apply the change of variables theorem for direct verification). The approximation property \ref{sect:sigmoments:lem1:it6}, which is sometimes referred to as the \emph{universality} of the signature, is an immediate consequence of the Stone-Weierstrass theorem and the fact that the set $\{\langle\sig(\cdot),\bm{\ell}\rangle\mid \bm{\ell}\in V^\circ\}$ is a subalgebra of $C(\mathcal{K})$ which contains the constants and separates points (the latter due to \citep[Thm.\ 4]{HLY} which implies that $\sig:\mathcal{BV}/\mathcal{T}\rightarrow V$ is injective), see e.g.\ \citep[Thm.\ 5.6. (2)]{CHO}. 

\ref{sect:sigmoments:lem1:it3}\,:\, As is immediate from \eqref{rem:coordspace_quotient:eq1}, the map $\log_{[m]}$ is a polynomial and hence $V$-valued and continuous, the latter by the fact that (both $\pi_{[m]}$ and) the multiplication $\ast$ ($\cong$ tensor multiplication $\otimes$; \eqref{rem:freealg_identify_tensor}) on $V$ is continuous (e.g.\ \citep[Section 3]{CHL}). The commutativity of $\log_{[m]}$ and $\pi_{[m]}$ is clear again from \eqref{rem:coordspace_quotient:eq1}. As to the boundedness assertion, let $\lambda>2\rho$ and denote $B_\lambda\coloneqq\{\bm{t}\in V_{(1)}\mid \vertiii{\bm{t}-1}_\lambda\leq 1\}$. Then in particular $\sup_{\bm{t}\in B_\lambda}\|\pi_m(\bm{t})\|_m \, \leq \, \lambda^{-m}$ for every $m\geq 0$, whence for each $\bm{t}\in B_\lambda$ and $\bm{\ell}\coloneqq\log(\bm{t})$ we have that, for all $m\geq k \geq 1$,  
\begin{equation}
\left\|\pi_m\big[(\bm{t}-1)^{\ast k}\big]\right\|_m \ \leq \ \sum_{\substack{m_1+\ldots+m_k=m\\m_\nu\geq 1}}\left\|\pi_{m_1}(\bm{t})\ast\cdots\ast\pi_{m_k}(\bm{t})\right\|_m \ \leq \ \binom{m-1}{k-1}\cdot\lambda^{-m}  
\end{equation}(as the tensor norms $\|\cdot\|_m$ are each submultiplicative), and hence find from \eqref{rem:log_transform} that 
\begin{equation}
\|\pi_m(\bm{\ell})\|_m \ \leq \ \sum_{k=1}^m\frac{1}{k}\binom{m-1}{k-1}\lambda^{-m} \ \leq \ 2^m\lambda^{-m} \qquad \text{for each \ } m\geq 1, 
\end{equation}implying that $\sup_{\ell\in\log(B_\lambda)}\vertiii{\bm{\ell}}_\rho \leq \sum_{m\geq 1}(2\rho/\lambda)^m < \infty$, as desired. 

\ref{sect:sigmoments:lem1:it4}\,:\, Let $B\subset V^\infty$ be bounded wrt.\ $\vertiii{\cdot}_\rho$, i.e.\ suppose that $\beta_\rho\coloneqq\sup_{\bm{t}\in B}\vertiii{\bm{t}}_\rho < \infty$. Then there will be some $0<q<1$ together with an index $m_0\geq 1$ such that 
\begin{equation}\label{sect:sigmoments:lem1:aux1}
\sup_{\bm{t}\in B}\|\pi_m(\bm{t})\|_m \ \leq \ q^m \qquad \text{for each } \ \ m\geq m_0. 
\end{equation}Indeed: Assuming otherwise that the above does not hold, we for any given $q\in(0,1)$ obtain the existence of a sequence $(\bm{t}^{(n)})_{n\in\N}\subset B$ with the property that   
\begin{equation}
\|\pi_{m_n}(\bm{t}^{(n)})\|_{m_n} \ > \ q^{m_n} \qquad \text{for each } \ \ n\in\N
\end{equation} 
for some strictly increasing sequence $(m_n)_{n\in\N}\subset\N$. But choosing $q>\rho^{-1}$ then implies that 
\begin{equation}
\beta_\rho \ \, \geq \, \ \sup_{n\in\N}\bvertiii{\bm{t}^{(n)}}_\rho \ \geq \ \sup_{n\in\N} (\rho\cdot q)^{m_n} \ = \ \infty   
\end{equation} 
in contradiction to the $\vertiii{\cdot}_\rho$-boundedness of $B$. Thus \eqref{sect:sigmoments:lem1:aux1} holds, and with it (by convergence of the geometric series) the claimed uniform $\vertiii{\cdot}_1$-convergence of $(\pi_{[m]})_{m\in\N}$.

\ref{sect:sigmoments:lem1:it5}\,:\, This is clear by inspection of \eqref{rem:log_transform} and \eqref{sect:sigmoments:eq1}, respectively. 
\end{proof}

%
%
\end{appendix}

%
%
 

\nocite{KGL}

\bibliographystyle{plain}
\bibliography{References}

\begin{thebibliography}{10}

\bibitem{ALM}
L.~B. Almeida.
\newblock {MISEP -- Linear and Nonlinear ICA Based on Mutual Information}.
\newblock {\em J.\ Mach.\ Learn.\ Res.}, \textbf{4}:1297--1318, 2003.

\bibitem{ardizzone2018applications}
L.~Ardizzone, J.~Kruse, S.~Wirkert, D.~Rahner, E.~W. Pellegrini, R.~S. Klessen,
  L.~Maier-Hein, C.~Rother, and U.~K{\"o}the.
\newblock {Analyzing Inverse Problems With Invertible Neural Networks}.
\newblock {\em Published as a conference paper at ICLR 2019, preprint available
  at \url{arXiv:1808.04730}}, 2018.

\bibitem{BJO}
F.~R. Bach and M.~I. Jordan.
\newblock {Kernel Independent Component Analysis}.
\newblock {\em J.\ Mach.\ Learn.\ Res.}, \textbf{3}:1--48, 2002.

\bibitem{BES}
A.~J. Bell and T.~Sejnowski.
\newblock {An information maximisation approach to blind separation and blind
  deconvolution}.
\newblock {\em Neural Computation}, \textbf{7}.6:1129--1159, 1995.

\bibitem{BCM}
A.~Belouchrani, K.~A. Meraim, J.~F. Cardoso, and E.~Moulines.
\newblock {A Blind Source Separation Technique Using Second-Order Statistics}.
\newblock {\em IEEE Trans.\ on Signal Processing}, \textbf{45}.2:434--444,
  1997.

\bibitem{berner2019}
J.~Berner, D.~Elbrächter, and P.~Grohs.
\newblock {How degenerate is the parametrization of neural networks with the
  ReLU activation function?}
\newblock {\em \url{arXiv:1905.09803}}, 2019.

\bibitem{BIL}
P.~Billingsley.
\newblock {\em Convergence of Probability Measures}.
\newblock Second Edition, John Wiley \& Sons, 1999.

\bibitem{bonnier2019signature}
Patric Bonnier and Harald Oberhauser.
\newblock Signature cumulants, ordered partitions, and independence of
  stochastic processes.
\newblock {\em Bernoulli}, \textbf{26}.4:2727--2757, 2020.

\bibitem{boucheron2013}
S.~Boucheron, G.~Lugosi, and P.~Massart.
\newblock {\em Concentration Inequalities: A Nonasymptotic Theory of
  Independence}.
\newblock Oxford University Press, 2013.

\bibitem{BR2}
R.C. Bradley.
\newblock {Basic Properties of Strong Mixing Conditions. A Survey and Some Open
  Questions}.
\newblock {\em Probability Surveys}, \textbf{2}:107--144, 2005.

\bibitem{BRB}
P.~Brakel and Y.~Bengio.
\newblock {Learning independent features with adversarial nets for non-linear
  ICA}.
\newblock {\em Preprint}, page \url{arXiv:1710.05050 [stat.ML]}, 2017.

\bibitem{CD3}
J.~F. Cardoso.
\newblock {High-order Contrasts for Independent Component Analysis}.
\newblock {\em Neural Computation}, \textbf{11}:157--192, 1999.

\bibitem{CD2}
J.~F. Cardoso and A.~Souloumiac.
\newblock {Blind beamforming for non Gaussian signals}.
\newblock {\em IEE Proceedings-F}, \textbf{140}.6:362–370, 1993.

\bibitem{chen1977iterated}
K.-T. Chen.
\newblock Iterated path integrals.
\newblock {\em Bulletin of the American Mathematical Society}, 83(5):831--879,
  1977.

\bibitem{chenfanCopula}
X.~Chen and Y.~Fan.
\newblock {Estimation of Copula-Based Semiparametric Time Series Models}.
\newblock {\em Journal of Econometrics}, \textbf{130}.2:307--335, 2006.

\bibitem{CHL}
I.~Chevyrev and T.~Lyons.
\newblock {Characteristic functions of measures on geometric rough paths}.
\newblock {\em Ann.\ Probab.}, \textbf{44}.6:4049--4082, 2016.

\bibitem{CHO}
I.~Chevyrev and H.~Oberhauser.
\newblock {Signature Moments to Characterize Laws of Stochastic Processes}.
\newblock {\em Preprint}, page \url{arXiv:1810.1097}, 2018.

\bibitem{COP}
L.~Colmenarejo and R.~Preiß.
\newblock {Signatures of paths transformed by polynomial maps}.
\newblock {\em Preprint}, page \url{arXiv:1812.05962v1}, 2018.

\bibitem{COM}
P.~Comon.
\newblock {Independent Component Analysis, a new concept?}
\newblock {\em Signal Processing}, \textbf{36}.3:287--314, 1994.

\bibitem{HBS}
P.~Comon and C.~Jutten, editors.
\newblock {\em Handbook of Blind Source Separation: Independent Component
  Analysis and Applications}.
\newblock Academic press, 2010.

\bibitem{conway1978complexvariable}
J.~B. Conway.
\newblock {\em Functions of One Complex Variable}.
\newblock Second Edition. Graduate Texts in Mathematics \textbf{11}, Springer,
  1978.

\bibitem{cranmer2020applications}
K.~Cranmer, J.~Brehmer, and G.~Louppe.
\newblock {The Frontier of Simulation-Based Inference}.
\newblock {\em Proceedings of the National Academy of Sciences},
  \textbf{117}.48:30055--30062, 2020.

\bibitem{DAR}
G.~Darmois.
\newblock {Analyse générale des liaisons stochastiques}.
\newblock {\em Rev.\ Inst.\ Intern.\ Stat.}, \textbf{21}:2--8, 1953.

\bibitem{darsow1992}
W.~F. Darsow, B.~Nguyen, and T.~Olsen.
\newblock {Copulas and Markov Processes}.
\newblock {\em Illinois J.\ Math.}, \textbf{36}.4:600--642, 1992.

\bibitem{denker1989}
M.~Denker.
\newblock {The central limit theorem for dynamical systems}.
\newblock {\em Dynamical Systems and Ergodic Theory, (K. Krzyzewski, ed.)},
  pages 33--62, Banach Center Publications, Polish Scientific Publishers,
  Warsaw, 1989.

\bibitem{diehl2016pathwise}
Joscha Diehl, Peter Friz, Hilmar Mai, et~al.
\newblock Pathwise stability of likelihood estimators for diffusions via rough
  paths.
\newblock {\em The Annals of Applied Probability}, 26(4):2169--2192, 2016.

\bibitem{ding2019miningmachine}
Hua Ding, Yiliang Wang, Zhaojian Yang, and Olivia Pfeiffer.
\newblock {Nonlinear blind source separation and fault feature extraction
  method for mining machine diagnosis}.
\newblock {\em Applied Sciences}, 9(9):1852, 2019.

\bibitem{embrechts2002}
P.~Embrechts, A.~McNeil, and D.~Straumann.
\newblock {Correlation and Dependence in Risk Management: Properties and
  Pitfalls.}
\newblock {\em Risk Management: Value at Risk and Beyond}, \textbf{1}:176--223,
  2002.

\bibitem{emura2017}
T.~Emura, T.-H. Long, and L.-H. Sun.
\newblock {R routines for performing estimation and statistical process control
  under copula-based time series models.}
\newblock {\em Communications in Statistics - Simulation and Computation},
  \textbf{46}.4:3067--3087, 2017.

\bibitem{ERK}
J.~Eriksson and V.~Koivunen.
\newblock { Identifiability, separability and uniqueness of linear ICA models.}
\newblock {\em IEEE Signal Processing Letters}, \textbf{11}:601--604, 2004.

\bibitem{fan2005overview}
J.~Fan.
\newblock {A Selective Overview of Nonparametric Methods in Financial
  Econometrics}.
\newblock {\em Statistical Science}, \textbf{20}.4:317--337, 2005.

\bibitem{fliess1981}
Michel Fliess.
\newblock Fonctionnelles causales non lin{\'e}aires et ind{\'e}termin{\'e}es
  non commutatives.
\newblock {\em Bulletin de la soci{\'e}t{\'e} math{\'e}matique de France},
  109:3--40, 1981.

\bibitem{FVI}
P.~K. Friz and N.~B. Victoir.
\newblock {\em Multidimensional Stochastic Processes as Rough Paths: Theory and
  Applications}.
\newblock Cambridge Studies in Advanced Mathematics \textbf{120}, Cambridge
  University Press, 2010.

\bibitem{fryzlewicz2011mixing}
P.~Fryzlewicz and S.~S. Rao.
\newblock {Mixing properties of ARCH and time-varying ARCH processes}.
\newblock {\em Bernoulli}, \textbf{17}.1:320--346, 2011.

\bibitem{graves1946}
L.~M. Graves.
\newblock {\em Theory of functions of real variables}.
\newblock McGraw-Hill, 1946.

\bibitem{GRE}
A.~Gretton, R.~Herbrich, A.~Smola, O.~Bousquet, and B.~Schölkopf.
\newblock {Kernel methods for measuring independence}.
\newblock {\em J.\ Mach.\ Learn.\ Res.}, \textbf{6}:2075–2129, 2005.

\bibitem{GUP}
V.~Guillemin and A.~Pollack.
\newblock {\em Differential Topology}.
\newblock AMS Chelsea Publishing, 1974.

\bibitem{halva2021disentangling}
H.~H{\"a}lv{\"a}, S.~Le~Corff, L.~Leh{\'e}ricy, Y.~So, J.and~Zhu, E.~Gassiat,
  and A.~Hyvarinen.
\newblock {Disentangling Identifiable Features from Noisy Data with Structured
  Nonlinear ICA}.
\newblock {\em Preprint}, page \url{arXiv:2106.09620}, 2021.

\bibitem{HLY}
B.~Hambly and T.~Lyons.
\newblock {Uniqueness for the signature of a path of bounded variation and the
  reduced path group}.
\newblock {\em Ann.\ of Math.}, \textbf{171}.1:109--167, 2010.

\bibitem{HAR}
S.~Harmeling, A.~Ziehe, M.~Kawanabe, and K.R. Müller.
\newblock {Kernel-based nonlinear blind source separation}.
\newblock {\em Neural Computation}, \textbf{15}.5:1089--1124, 2003.

\bibitem{HAT}
T.~Hastie and R.~Tibshirani.
\newblock {Independent component analysis through product density estimation}.
\newblock {\em Advances in Neural Information Processing Systems},
  \textbf{15}:649–656, 2003.

\bibitem{he2018applications}
Q~Peter He and Jin Wang.
\newblock {Statistical Process Monitoring as a Big Data Analytics Tool for
  Smart Manufacturing}.
\newblock {\em Journal of Process Control}, 67:35--43, 2018.

\bibitem{HJE}
R.~D. Hjelm et~al.
\newblock {Learning deep representations by mutual information estimation and
  maximization}.
\newblock {\em Preprint}, page \url{arXiv:1808.06670v5 [stat.ML]}, 2018.

\bibitem{HYR}
A.~Hyvärinen.
\newblock {\em Independent Component Analysis by Minimization of Mutual
  Information}.
\newblock Technical Report (Report A46), Helsiniki University of Technology,
  Department of Computer Science and Engineering, September 1997.

\bibitem{HYE}
A.~Hyvärinen.
\newblock {New Approximations of Differential Entropy for Independent Component
  Analysis and Projection Pursuit}.
\newblock {\em Adv.\ Neural Inf.\ Process Syst.}, pages 273--279, 1998.

\bibitem{HYF}
A.~Hyvärinen.
\newblock {Fast and robust fixed-point algorithms for independent component
  analysis}.
\newblock {\em IEEE Trans.\ Neural Netw.}, \textbf{10}.3:626–634, 1999.

\bibitem{HRS}
A.~Hyvärinen.
\newblock {Independent component analysis: recent advances}.
\newblock {\em Phil.\ Trans.\ R.\ Soc.\ A}, \textbf{371}.1984:20110534, 2013.

\bibitem{HKO}
A.~Hyvärinen, J.~Karhunen, and E.~Oja.
\newblock {\em Independent Component Analysis}.
\newblock John Wiley \& Sons, 2001.

\bibitem{TCL}
A.~Hyvärinen and H.~Morioka.
\newblock {Unsupervised feature extraction by time-contrastive learning and
  nonlinear ICA}.
\newblock {\em NeurIPS2016}, pages 3765--3773, 2016.

\bibitem{HYM}
A.~Hyvärinen and H.~Morioka.
\newblock {Nonlinear ICA of Temporally Dependent Stationary Sources}.
\newblock {\em PMLR}, \textbf{54}:460--469. Supplementary Material at
  \url{http://proceedings.mlr.press/v54/hyvarinen17a/hyvarinen17a--supp.pdf},
  2017.

\bibitem{HYP}
A.~Hyvärinen and P.~Pajunen.
\newblock {Nonlinear Independent Component Analysis: Existence and Uniqueness
  Results}.
\newblock {\em Neural Networks}, \textbf{12}.3:429--439, 1999.

\bibitem{AUX}
A.~Hyvärinen, H.~Sasaki, and R.~Turner.
\newblock {Nonlinear ICA Using Auxiliary Variables and Generalized Contrastive
  Learning}.
\newblock {\em AISTATS}, 2019.

\bibitem{ican2017applications}
Ö. Ican and T.~B. Celik.
\newblock {Stock Market Prediction Performance of Neural Networks: A Literature
  Review}.
\newblock {\em International Journal of Economics and Finance},
  \textbf{9}.11:100--108, 2017.

\bibitem{khemakhem20iVAE}
I.~Khemakhem, D.~P. Kingma, R.~P. Monti, and A.~Hyvärinen.
\newblock Variational autoencoders and nonlinear {ICA}: A unifying framework.
\newblock {\em Proc.\ Artificial Intelligence and Statistics (AISTATS2020)},
  2020.

\bibitem{khoshnevis2019applications}
S.~A. Khoshnevis and R.~Sankar.
\newblock {Applications of Higher Order Statistics in Electroencephalography
  Signal Processing: A Comprehensive Survey}.
\newblock {\em IEEE Reviews in Biomedical Engineering}, \textbf{13}:169--183,
  2019.

\bibitem{KGL}
P.~Kidger and T.~Lyons.
\newblock {Signatory: differentiable computations of the signature and
  logsignature transforms, on both CPU and GPU}.
\newblock {\em Published at ICLR 2021, available at
  \url{https://github.com/patrick-kidger/signatory}}, 2020.

\bibitem{KLS}
C.~Kleiber and J.~Stoyanov.
\newblock {Multivariate distributions and the moment problem}.
\newblock {\em J.\ Multivariate Anal.}, \textbf{113}:7--18, 2013.

\bibitem{KUC}
A.~Kuczmaszewska.
\newblock {On the Strong Law of Large Numbers for $\phi$-Mixing and
  $\rho$-Mixing Random Variables}.
\newblock {\em Acta Math.\ Hungar.}, \textbf{132}.1-2:174–189, 2011.

\bibitem{leeManifolds}
J.~M. Lee.
\newblock {\em Introduction to Smooth Manifolds}.
\newblock Graduate Texts in Mathematics \textbf{218}, Springer, 2013.

\bibitem{liyanwangpeng2016}
Z.~Li, X.~Yan, X.~Wang, and Z.~Peng.
\newblock {Detection of gear cracks in a complex gearbox of wind turbines using
  supervised bounded component analysis of vibration signals collected from
  multi-channel sensors}.
\newblock {\em Journal of Sound and Vibration}, \textbf{371}:406–433, 2016.

\bibitem{FLO}
T.~J. Lyons, M.~Caruana, and T.~Lévy.
\newblock {\em Differential Equations Driven by Rough Paths}.
\newblock Springer, 2007.

\bibitem{LYQ}
T.~J. Lyons and Z.~Qian.
\newblock {\em System Control and Rough Paths}.
\newblock Oxford Mathematical Monographs, Oxford University Press, 2002.

\bibitem{marcus1972sample}
M.~B. Marcus and L.~A. Shepp.
\newblock Sample behavior of gaussian processes.
\newblock {\em Proceedings of the Sixth Berkeley Symposium on Mathematical
  Statistics and Probability, Volume 2: Probability Theory}, pages 423--441,
  1972.

\bibitem{mokkadem1988}
H.~Masuda.
\newblock {Mixing properties of ARMA processes}.
\newblock {\em Stochastic Process.\ Appl.}, \textbf{19}:297–303, 1988.

\bibitem{masuda2007ergodicity}
H.~Masuda.
\newblock {Ergodicity and exponential $\beta$-mixing bounds for
  multidimensional diffusions with jumps}.
\newblock {\em Stochastic Process.\ Appl.}, \textbf{117}(1):35--56, 2007.

\bibitem{mccoy1978}
R.~A. McCoy.
\newblock {Second countable and separable function spaces}.
\newblock {\em The American Mathematical Monthly}, 85.6:487--489, 1978.

\bibitem{MNT}
J.~Miettinen, K.~Nordhausen, and S.~Taskinen.
\newblock {Blind source separation based on joint diagonalization in R: The
  packages JADE and BSSasymp}.
\newblock {\em Journal of Statistical Software}, \textbf{76}.2, 2017.

\bibitem{MOU}
E.~Moulines, J.~F. Cardoso, and E.~Gassiat.
\newblock {Maximum likelihood for blind separation and deconvolution of noisy
  signals using mixture models}.
\newblock {\em Proc.\ IEEE Int.\ Conf.\ on Acoustics, Speech and Signal
  Processing (ICASSP’97)}, page 3617–3620, 1997.

\bibitem{munkres2000}
J.~R. Munkres.
\newblock {\em Topology}.
\newblock Second Edition. Prentice Hall, 2000.

\bibitem{nelsen2006}
R.~B. Nelsen.
\newblock {\em An Introduction to Copulas.}
\newblock Springer Series in Statistics. Second Edition. Springer, 2006.

\bibitem{noe2020applications}
F.~No{\'e}, A.~Tkatchenko, K.-R. M{\"u}ller, and C.~Clementi.
\newblock {Machine Learning for Molecular Simulation}.
\newblock {\em Annual Review of Physical Chemistry}, \textbf{71}:361--390,
  2020.

\bibitem{PVA}
C.~Papavasiliou, A;~Ladroue.
\newblock {Parameter Estimation for Rough Differential Equations}.
\newblock {\em Ann.\ Stat.}, \textbf{39}.4:2047–2073, 2011.

\bibitem{petersen2020}
P.~Petersen, M.~Raslan, and F.~Voigtlaender.
\newblock {Topological Properties of the Set of Functions Generated by Neural
  Networks of Fixed Size}.
\newblock {\em Found. Comput. Math}, \textbf{21}:375--444, 2021.

\bibitem{PHG}
D.~T. Pham and P.~Garrat.
\newblock {Blind Separation of Mixture of Independent Sources Through a
  Quasi-Maximum Likelihood Approach}.
\newblock {\em IEEE Trans.\ Signal Process.}, \textbf{45}.7:1712--1725, 1997.

\bibitem{pham1985}
T.~D. Pham and L.~T. Tran.
\newblock {Some mixing properties of time series models}.
\newblock {\em Stochastic Process.\ Appl.}, \textbf{19}:297–303, 1985.

\bibitem{rasmuwillGP2006}
C.~E. Rasmussen and C-K.~I. Williams.
\newblock {\em Gaussian Processes for Machine Learning}.
\newblock The MIT Press, 2006.

\bibitem{reedsimon1972}
M.~Reed and B.~Simon.
\newblock {\em Methods of Modern Mathematical Physics: Vol.\ I, Functional
  Analysis}.
\newblock Academic Press, 1972.

\bibitem{reiersol1950}
O.~Reiers{\o}l.
\newblock {Identifiability of a linear relation between variables which are
  subject to error}.
\newblock {\em Econometrica: Journal of the Econometric Society}, pages
  375--389, 1950.

\bibitem{remmert1998}
R.~Remmert.
\newblock {\em Theory of Complex Functions}.
\newblock Fourth corrected printing. Springer, 1998.

\bibitem{REU}
C.~Reutenauer.
\newblock {\em Free Lie Algebras}.
\newblock London Mathematical Society Monographs, New Series. Oxford Science
  Publications \textbf{7}, The Clarendon Press, 1993.

\bibitem{rogerswilliams2000}
L.~C.~G. Rogers and D.~Williams.
\newblock {\em Diffusions, Markov Processes, and Martingales}.
\newblock Volume 1, Cambridge University Press, 2000.

\bibitem{rosenblatt1971}
M.~Rosenblatt.
\newblock {\em {Markov Processes, Structure and Asymptotic Behavior}}.
\newblock Springer-Verlag, New York, 1971.

\bibitem{rudin1976}
W.~Rudin.
\newblock {\em Principles of Mathematical Analysis}.
\newblock (3rd ed.) McGraw-Hill, 1976.

\bibitem{RUD}
W.~Rudin.
\newblock {\em Real and Complex Analysis}.
\newblock (3rd ed.) McGraw-Hill, 1987.

\bibitem{SAM}
R.J. Samworth and M.~Yuan.
\newblock {Independent component analysis via nonparametric maximum likelihood
  estimation}.
\newblock {\em Ann.\ Statist.}, \textbf{40}.6:2973--3002, 2012.

\bibitem{githubSigNICA}
A.~Schell.
\newblock {\em SigNICA}.
\newblock GitHub repository,
  \url{https://github.com/alexander-schell/SigNICA.git} Code for Section 9
  (Jupyter Notebooks and Python files)., 2021.

\bibitem{schlemm2012}
E.~Schlemm and R.~Stelzer.
\newblock {Multivariate CARMA processes, continuous-time state space models and
  complete regularity of the innovations of the sampled processes}.
\newblock {\em Bernoulli}, \textbf{18}.1:46--63, 2012.

\bibitem{SKI}
V.~P. Skitovich.
\newblock {On a Property of a Normal Distribution}.
\newblock {\em Doklady Akad.\ Nauk.\ SSSR}, \textbf{89}:217--219 (in Russian),
  1953.

\bibitem{TWZ}
Y.~Tan, J.~Wang, and J.~M. Zurada.
\newblock {Nonlinear blind source separation using a radial basis function
  network}.
\newblock {\em IEEE Trans.\ Neural Netw.}, \textbf{12}.1:124--134, 2001.

\bibitem{teshima2020}
T~et~al. Teshima.
\newblock {Coupling-based invertible neural networks are universal
  diffeomorphism approximators}.
\newblock {\em \url{arXiv:2006.11469}, accepted at NeurIPS 2020}, 2020.

\bibitem{vdv1998}
A.~W. van~der Vaart.
\newblock {\em Asymptotic Statistics}.
\newblock Cambridge Series in Statistical and Probabilistic Mathematics, 1998.

\bibitem{VDV}
A.~W. van~der Vaart.
\newblock {\em Time Series}.
\newblock Version 242013, Lecture Notes, Universiteit Leiden, 2013.

\bibitem{WEL}
J.~A. Wellner.
\newblock {\em Empirical Processes: Theory and Applications}.
\newblock Special Topics Course Spring 2005, Delft Technical University, June
  2005. (Available at
  \url{https://www.stat.washington.edu/people/jaw/RESEARCH/TALKS/Delft/emp-proc-delft-big.pdf}).

\bibitem{xiaohong2010}
C.~Xiaohong, L.~P. Hansen, and M.~Carrasco.
\newblock {Nonlinearity and temporal dependence}.
\newblock {\em Econometrics}, \textbf{155}.2:155--169, 2010.

\bibitem{yu2015sampling}
Q.~Zhou and J.~Yu.
\newblock {Asymptotic theory for linear diffusions under alternative sampling
  schemes}.
\newblock {\em Economics Letters}, \textbf{128}:1--5, 2015.

\end{thebibliography}
{\small
\vspace{1.5em}
\vfill
\begin{center}
\address{Mathematical Institute\\University of Oxford\\AWB, ROQ, Woodstock Road\\Oxford, OX2 6GG, UK} 
\end{center}
\vfill
}
\end{document}